\documentclass{article}

\usepackage{amsmath}
\usepackage{verbatim}

\usepackage[font=small,labelfont=bf]{caption} % for jmlr, this should precede jmlr2e
\usepackage[open,openlevel=1]{bookmark}

\usepackage{dirtytalk}
\usepackage{mathtools} %\DeclarePairedDelimiter
\usepackage{dsfont} %indicator
\usepackage{booktabs} % for fancy tables
\usepackage{makecell} % linebreak inside a cell
\usepackage{algorithm} % for algorithm box
\usepackage{algpseudocode} % for algorithm box
\usepackage{tikz}%for figures
\usepackage{subcaption}
\usepackage{graphics}%for figures
\usepackage{natbib} 
\usepackage{amssymb}
\usepackage{latexsym}
\usepackage{epsfig}
\usepackage{amscd}
\usepackage{amsthm}
\usepackage{multirow} %table
\usepackage{adjustbox} % table size
\usepackage{float}
\usepackage{footmisc}
\usepackage{color}
\usepackage{enumerate}
\usepackage{bm}
\usepackage{setspace}
\usepackage{lipsum}
\usepackage{blindtext}

\usepackage[titletoc,title]{appendix}
\usepackage[T1]{fontenc}

\usepackage{url}

\usepackage{hyperref}[] % hyperlink
\hypersetup{
	colorlinks = true,
	linkcolor = blue, %Colour of internal links
	filecolor = magenta,
	urlcolor = blue, %Colour for external hyperlinks
	citecolor = blue, %Colour of citations
	pdfpagemode=UseOutlines
}
\usepackage{scalerel,stackengine} %really wide hats
\stackMath
\newcommand\reallywidehat[1]{%
	\savestack{\tmpbox}{\stretchto{%
			\scaleto{%
				\scalerel*[\widthof{\ensuremath{#1}}]{\kern.1pt\mathchar"0362\kern.1pt}%
				{\rule{0ex}{\textheight}}%WIDTH-LIMITED CIRCUMFLEX
			}{\textheight}% 
		}{2.4ex}}%
	\stackon[-6.9pt]{#1}{\tmpbox}%
}
\usepackage{geometry}
\allowdisplaybreaks[1]
\allowbreak

\newcommand\numberthis{\addtocounter{equation}{1}\tag{\theequation}}
\newcommand{\vecOne}[1]{\mathbf{1}_{#1}}
\newcommand{\vecZero}[1]{\mathbf{0}_{#1}}
\newcommand{\matOne}[1]{\mathbf{J}_{#1}}

\newcommand{\test}{\Delta}

\newcommand{\distClassGeneric}{\mathcal{P}}
\newcommand{\distclassMultinomial}{\distClassGeneric_{\mathrm{multi}}^{(\alphabetSize)}}
\newcommand{\minSep}{\rho} %separation to define alternative hypothesis in minimax testing

\newcommand{\minimaxTestingRateTwosampleLDP}{\minSep^\ast_{\sampleSize_1, \sampleSize_2, \privacyParameter}}
\newcommand{\rhoTwosample}{\rho_{\sampleSize_1, \sampleSize_2}}

\newcommand{\rvLDPnoise}{W}
\newcommand{\numbersetRealizedLap}{w}

\newcommand{\numbersetNonnegInt}{\mathbb{N}_0}
\newcommand{\numbersetReal}{\mathbb{R}} %set of real numbers

\newcommand{\floor}[1]{\lfloor #1 \rfloor} %floor function
\newcommand{\indicator}[1]{\mathds{1}\left( #1 \right) }%indicator funcion

\newcommand{\Ell}{\mathbb{L}}
\newcommand{\EllTwo}{\Ell_2} %L2 norm
\newcommand{\EllInfty}{\Ell_{\infty}} %inf norm
\newcommand{\normEllp}[3]{
	\vertiii{#1}_{\Ell_{#2}}^{#3}
}
\newcommand{\vertiii}[1]{
	{\left\vert\kern-0.25ex\left\vert\kern-0.25ex\left\vert #1 
		\right\vert\kern-0.25ex\right\vert\kern-0.25ex\right\vert}
}%triple norm notation

\newcommand{\rvOne}{X}
\newcommand{\rvTwo}{Y}
\newcommand{\rvThree}{Z}

\newcommand{\rvTwoObs}{y}

\newcommand{\sampleSizeOne}{\sampleSize_1}

\newcommand{\vectorize}[1]{\mathbf{#1}}

\newcommand{\private}[1]{\tilde{#1}}
\newcommand{\rvCodomain}[1]{\mathcal{#1}}

\newcommand{\sigmafield}{\mathcal{F}}
\newcommand{\sigmafieldPriv}{\tilde{\mathcal{F}}}

\newcommand{\rvX}{X} % random variable X
\newcommand{\rvXRealized}{x}
\newcommand{\rvXCodomain}{\mathcal{\rvX}} % support of the distribution of X
\newcommand{\rvXPriv}{\tilde{\rvX}} % 
\newcommand{\rvXPrivCodomain}{\tilde{\rvXCodomain}}

\newcommand{\rVecX}{\vectorize{\rvX}}
\newcommand{\rVecXRealized}{\vectorize{\rvXRealized}}
\newcommand{\rVecXPriv}{\tilde{\rVecX}} % 
\newcommand{\rVecXPrivRealized}{\tilde{\rVecXRealized}} % 

\newcommand{\rvY}{Y}

\newcommand{\rvYPriv}{\tilde{\rvY}} % 

\newcommand{\rVecY}{\vectorize{\rvY}}

\newcommand{\vecRandomPrivTwoSampleYNumber}[1]{\tilde{\mathbf{Y}}_{#1}}

\newcommand{\rvZ}{Z}

\newcommand{\rvZPriv}{\tilde{\rvZ}} % 

\newcommand{\rVecZ}{\vectorize{\rvZ}}

\newcommand{\rVecZPriv}{\tilde{\rVecZ}} % 

\newcommand{\mE}{\mathbb{E}} %expectation
\newcommand{\mP}{\mathbb{P}} %probability
\newcommand{\mVPQ}{\mathrm{Var}_{P,Q}} % variance under privacy
\newcommand{\sampleIndexOne}{i}
\newcommand{\sampleIndexTwo}{j}
\newcommand{\sampleSets}[3]{\{{#1}_{#2}\}_{#2 \in [#3]}}

\newcommand{\adaptiveBinNumIndex}{t}

\newcommand{\normSqMultinomMax}{b}
\newcommand{\dimDensity}{d} % d := dimension of multivariate continuous distribution
\newcommand{\alphabetSize}{k} %k := number of categories for multinomial distribution (alphabet size)
\newcommand{\vectorIndex}{m}
\newcommand{\sampleSize}{n}

\newcommand{\dimensionIndex}{p}
\newcommand{\probVecElement}[2]{p_{{#1}{#2}}}
\newcommand{\probVec}{\mathbf{p}} % p := parameter of a multinomial distribution
\newcommand{\privacyDensity}{q}
\newcommand{\smoothness}{s}
\newcommand{\sizeMonteCarlo}{B}

\newcommand{\binner}{D} 	  % D := discretizing function from [0,1]^d to kappa^d
\newcommand{\kernelOperator}{L}
\newcommand{\kernelMoment}{M} % M := kernel moment for two moment method by Ilmun Kim
\newcommand{\dataGenDist}{P}  % P := data generating distribution
\newcommand{\privacyMechanism}{Q}
\newcommand{\ballRadius}{R}

\newcommand{\ONset}{\mathbb{B}}
\newcommand{\nTest}{\mathcal{N}} %caligraphic N: number of tests in Bonferroni type adaptive test
\newcommand{\privacyMechanismClass}{\mathcal{Q}}
\newcommand{\privacyParameter}{\alpha} % LDP parameter
\newcommand{\privacyParameterrappor}{\privacyParameter_{\mathrm{bf}}}
\newcommand{\maxErrorTypeTwo}{\beta} % maximum type II error
\newcommand{\maxErrorTypeOne}{\gamma} %significance level i.e., maximum type I error
\newcommand{\smallNumber}{\delta}
\newcommand{\smallNumberrappor}{\smallNumber_{\mathrm{bf}}}
\newcommand{\distparamDiscLap}{\zeta}% parameter of discrete Laplace distribution
\newcommand{\binNum}{\kappa}           % kappa: number of bins when discretizing the continuous distribution into multinomilas
\newcommand{\separation}{\rho} %separation to define alternative hypothesis in minimax testing
\newcommand{\wavFatherFunc}{\phi}
\newcommand{\wavMotherFunc}{\psi} % mother wavelet
\newcommand{\coef}{\theta}
\newcommand{\permutation}{\pi}
\newcommand{\WavFatherBasisSet}{\overline{\Phi}}
\newcommand{\WavMotherBasisSet}{\overline{\uppercase{\Psi}}}

\newcommand{\pValueMonteCarlo}{\hat{p}_{\sizeMonteCarlo}}

\newcommand{\LapUParam}{\sigma_{\privacyParameter}}
\newcommand{\discLapUParam}{\distparamDiscLap_{\privacyParameter}}

\newcommand{\momentTwosampleVarCondexpY}{\kernelMoment_{\rvTwo,1}(\dataGenDist, \privacyMechanism)}
\newcommand{\momentTwosampleVarCondexpZ}{\kernelMoment_{\rvThree,1}(\dataGenDist, \privacyMechanism)}
\newcommand{\momentTwosampleVarCondexpBarY}{\bar{\kernelMoment}_{\rvTwo, 1}(\dataGenDist, \privacyMechanism)}
\newcommand{\momentTwosampleVarCondexpBarZ}{\bar{\kernelMoment}_{\rvThree, 1}(\dataGenDist, \privacyMechanism)}
\newcommand{\momentTwosampleExpSquare}{\kernelMoment_{\rvTwo \rvThree,2}(\dataGenDist, \privacyMechanism)}
\newcommand{\momentTwosampleExpSquareBar}{\bar{\kernelMoment}_{\rvTwo \rvThree,2}(\dataGenDist, \privacyMechanism)}

\newcommand{\besovParamMicroscope}{q}
\newcommand{\resLev}{j}
\newcommand{\primResLev}{J}

\newcommand{\wavFatherUnivIndex}{k}
\newcommand{\wavFatherIndex}{\boldsymbol{\wavFatherUnivIndex}}
\newcommand{\wavMotherUnivIndex}{\ell}
\newcommand{\wavMotherIndex}{\boldsymbol{\wavMotherUnivIndex}}
\newcommand{\wavMotherBooleanUnivIndex}{\epsilon}
\newcommand{\wavMotherBooleanIndex}{\mathbf{\wavMotherBooleanUnivIndex}}

\newcommand{\wavCoef}{\theta}

\newcommand{\multivInhomoWavFatherBasis}{\overline{\Phi}_{\primResLev}}

\newcommand{\wavGenericFatherCoef}{\wavCoef_{\wavFatherFunc}}

\newcommand{\wavGenericMotherCoef}{\wavCoef_{\wavMotherFunc}}

\newcommand{\adaptiveSingleTest}[1]{\Delta^{#1}}

\newcommand{\domainTs}{
	[0,1]^{{\dimDensity}}
}

\newcommand{\LtwoSpace}{\mathbb{L}_{2}(\domainTs)}

\newcommand{\ballDistn}{\mathcal{B}}
\newcommand{\besovBall}[2]{\ballDistn_{\dimDensity,\smoothness, #2}^{\mathrm{B}}(\ballRadius)}
\newcommand{\holderBall}{\ballDistn_{\dimDensity, \smoothness}^{\mathrm{H}}(\ballRadius)} %new

\newcommand{\pBesovTs}{ %Besov distributions
	\distClassGeneric_{%subscript
		\dimDensity, \smoothness, \besovParamMicroscope
	}^{%superscript
		\mathrm{B,2}
	}
	(\ballRadius)
}

\newcommand{\pBesovGof}{ %Besov distributions
	\distClassGeneric_{%subscript
		\dimDensity, \smoothness, \besovParamMicroscope
	}^{%superscript
		\mathrm{B,1}
	}
	(\ballRadius)
}

\newcommand{\pHolderTs}{ %Besov distributions
	\distClassGeneric_{%subscript
		\dimDensity, \smoothness
	}^{%superscript
		\mathrm{H,2}
	}
	(\ballRadius)
}

\newcommand{\pHolderGof}{ %Besov distributions
	\distClassGeneric_{%subscript
		\dimDensity, \smoothness
	}^{%superscript
		\mathrm{H,1}
	}
	(\ballRadius)
}

\newcommand{\bumpHolder}{\varphi}
\newcommand{\vecRandomPrivTwoSampleZNumber}[1]{\tilde{\mathbf{Z}}_{#1}}

\newcommand{\kernelTwoSample}{h_{ts}}
\newcommand{\kernelTwoSampleSym}{\check{h}_{ts}}

\newcommand{\hypoNull}{\distClassGeneric_0}
\newcommand{\pAlterTwosample}{\distClassGeneric_1(\rhoTwosample)}

\newtheorem{definition}{Definition}[section]
\newtheorem{theorem}{Theorem}[section]
\newtheorem{lemma}{Lemma}[section]

\newtheorem{remark}{Remark}[section]

\begin{document}

 \begin{center}
 	\textbf{\LARGE Minimax Optimal Two-Sample Testing under Local Differential Privacy}	
 
 	\vspace*{.2in}
 	
 	\begin{author}
 		A
 		Jongmin Mun$^{\dagger}$  \quad Seungwoo Kwak $^{\ddagger}$ \quad Ilmun Kim$^{\S}$ 	\end{author}
 
 	\vspace*{.1in}
 	
 	\begin{tabular}{c}
 		$^{\dagger}$Data Sciences and Operations Department, Marshall School of Business, University of Southern California\\
 		$^{\ddagger}$Division of Future Convergence, Sungkonghoe University\\
 		$^{\S}$Department of Statistics and Data Science, Yonsei University 	\end{tabular}

 	\vspace*{.1in}
 	
 	\today
 	
 	\vspace*{.1in}
 \end{center}

\begin{abstract}%   <- trailing '%' for backward compatibility of .sty file
We explore the trade-off between privacy and statistical utility in private two-sample testing under local differential privacy (LDP) for both multinomial and continuous data. We begin by addressing the multinomial case, where we introduce private permutation tests using practical privacy mechanisms such as Laplace, discrete Laplace, and Google's \texttt{RAPPOR}. We then extend our multinomial approach to continuous data via binning and study its uniform separation rates under LDP over H\"{o}lder and Besov smoothness classes. The proposed tests for both discrete and continuous cases rigorously control the type I error for any finite sample size, strictly adhere to LDP constraints, and achieve minimax separation rates under LDP. The attained minimax rates reveal inherent privacy-utility trade-offs that are unavoidable in private testing. To address scenarios with unknown smoothness parameters in density testing, we propose an adaptive test based on a Bonferroni-type approach that ensures robust performance without prior knowledge of the smoothness parameters. We validate our theoretical findings with extensive numerical experiments and demonstrate the practical relevance and effectiveness of our proposed methods.
\end{abstract}

%\begin{keyword}
%\kwd{local differential privacy}
%\kwd{two-sample testing}
%\kwd{independence testing}
%\kwd{minimax separation rates}
%\end{keyword}

%\end{frontmatter} #uncomment this to use the Bernoulli template
% \end{comment}
%%%%%%%%% uncomment the codes above to use Bernoulli template (Part II)%%%%%%
%%%%%%%%%%%%%%%%%%%%%%%%%%%%%%%%%%%%%%%%%%%%%%%%%%%%%%%%%%%%%%%%%%%%%%%%%%%%%

\setcounter{tocdepth}{2}

\section{Introduction}\label{section_intro}
%new intro start
Large-scale internet services such as Netflix and Amazon collect sensitive data from massive user bases, allowing companies to conduct cost-effective randomized experiments by assigning users to two different user interfaces or marketing campaigns. By testing whether the resulting two independent sets of samples originate from the same distribution{---}a procedure known as A/B testing or two-sample testing{---}companies can statistically assess the impact of new interfaces or campaigns on various user behaviors. However, the sensitivity of detailed personal data raises substantial privacy concerns in data analysis. Since privacy protection inherently conceals some of the information in the data and in turn compromises statistical utility, it is crucial to characterize and balance the trade-off between statistical utility and data privacy. Differential privacy \citep[DP;][]{CynthiaDwork2006CalibratingAnalysis} provides a rigorous framework for this trade-off, defining data privacy as a mathematical concept that supports such balancing.

We briefly review two notions of DP: central DP and local DP \citep[LDP;][]{Kasiviswanathan2008WhatPrivately}.
The central DP constraint, illustrated in Figure \hyperref[dp_graphical_central]{1(a)}, assumes that a trusted data curator (or distributor) has access to the entire original data set and calculates a noisy statistical result. This centralized approach requires that the probability of any event remains essentially the same when a single data entry is arbitrarily perturbed.
Under this constraint, one cannot reliably extract any individual-level information from the noisy statistical result. In contrast, under the local DP constraint, illustrated in Figure \hyperref[dp_graphical_local]{1(b)}, data owners do not place trust in the curator. Instead, each data owner independently reports a noisy version of their data. This more stringent local constraint inevitably impairs statistical utility more than the central constraint. However, it provides a higher level of privacy and separates the data curator from the responsibility for disclosure risk.

These benefits have driven widespread adoption of the LDP framework in the internet-scale deployment of data analysis under privacy, such as in
Apple \citep{Apple2017Privacy},
Google \citep{erlingsson_rappor_2014},
Microsoft \citep{Ding2017CollectingPrivately},
and Uber \citep{Near2017Uber}. With the massive user bases, these companies require stringent privacy protections, while also having the capacity to obtain large samples which allow statistically meaningful analyses within strong privacy constraints. Various implementations of LDP in industrial applications has naturally spurred a substantial body of work~\citep{Acharya2021DifferentiallyCam, Duchi2018MinimaxEstimation, Lam-Weil2021MinimaxConstraint, Lalanne2023on, Tony2021DP}, examining the intrinsic trade-off between privacy and statistical utility in various settings. Our paper contributes to the existing body of work by exploring this trade-off in the context of the two-sample testing problem.

\begin{figure}[t!]
	\centering
	\begin{subfigure}{0.41\textwidth}
		\includegraphics[width=\textwidth]{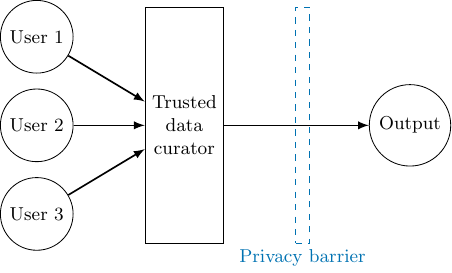}
		\caption{Central differential privacy}
		\label{dp_graphical_central}
	\end{subfigure}
	\hskip 1.1cm
	\begin{subfigure}{0.5\textwidth}
		\includegraphics[width=\textwidth]{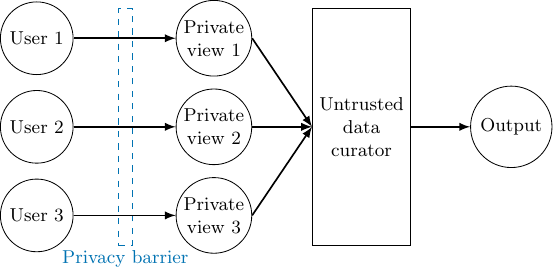}
		\caption{Local differential privacy}
		\label{dp_graphical_local}
	\end{subfigure}	
	\caption{Graphical illustration of central differential privacy and local differential privacy.}\label{dp_graphical}
\end{figure}
The two-sample testing problem, which originates from the classical two-sample $t$-test~\citep{Student1908Twosample}, has gained renewed interest in recent years due to the emergence of high-dimensional and complex data.  
Notably, several novel methodological approaches have been developed including kernel-based tests~\citep{Gretton2009FastKernel, Gretton2012ATest}, distance-based tests~\citep{Szekely2004Energy, szekely_new_2005} and regression/classification-based tests~\citep{kim2019,kim2021}, which have demonstrated promising capabilities in dealing with modern data sets.
On the theoretical front, researchers have explored the fundamental limit of this problem through the lens of minimax analysis both in the statistics~\citep[for example,][]{Arias-Castro2018RememberDimension, kim_minimax_2022, Schrab2021MMDTest} and computer science~\citep[for example,][]{Batu2000NonprivateTwosample, chan2014optimal, Diakonikolas2016Nonprivate, Goldreich2000Nonprivate} literature.
In addition to the methodological and theoretical advancements, the topic has found contemporary applications such as in education research \citep{Rabin2019ModelingTests}, network traffic analysis \citep{Kohout2018NetworkTest}, and audio segmentation \citep{Harchaoui2009ASegmentation}.
Despite its long history and fundamental roles in practice, most of the existing work on two-sample testing has focused on non-private settings with a few exceptions. Exceptions include private versions of
multinomial tests~\citep{Acharya2018dPGofTwosample, Aliakbarpour2019PrivatePermutations,Aliakbarpour2018DPgof},
traditional non-parametric tests \citep[for example,~Mann-Whitney and Wilcoxon signed-rank tests;][]{Couch2019NonparamTwosample,Task2016Wilcoxon},
partitioning-based test of univariate continuous distributions \citep{sheffet_differentially_2024}
and kernel tests \citep{Raj2020ATest,kim2023dp} under central DP. The literature on two-sample testing under LDP is even more scarce. Notable contributions in this area include \cite{Ding2018TwosampleMean} and \cite{waudby-smith_nonparametric_2023}, both of which are mainly concerned with detecting differences in location. In contrast, our primary goal is to develop two-sample tests for general alternatives under LDP, focusing on both multinomial and multivariate continuous data. Additionally, we shed new light on the fundamental limits of the two-sample problem  under LDP.

In the following subsection, we begin with a review of related work and then discuss our techniques and contributions. 

\subsection{Related Prior Work}\label{section:related_works}
Private hypothesis testing has been extensively studied in the statistics and computer science literature.
Among various studies on this topic, we briefly review those closely related to our work. Initially motivated by the privacy attack on genome-wide association study~(GWAS)~\citep{homer2008resolving}, the early work on private testing mainly concentrates on private versions of chi-square tests and explores their asymptotic properties~\citep{Gaboardi2016DPChisq, Gaboardi2018LDPChisq, Johnson2013PrivacyStudies, Kifer2017DPChisq, Uhler2013PrivacyStudies, Vu2009DifferentialEvaluations, Wang2015DPChisq, Yu2014Chisq}. In contrast, a recent line of work in computer science is concerned with non-asymptotic properties of private tests designed for multinomial data sets, and studies optimal sample complexities of testing problems from a minimax perspective. This line of work has been initiated by \citet{Cai2017DPGofPrivit} for central DP and \citet{Sheffet2018LDP} for LDP, and continued by \citet{Acharya2018dPGofTwosample}, \citet{Aliakbarpour2019PrivatePermutations, Aliakbarpour2018DPgof}, and \citet{sheffet_differentially_2024} for central DP and \citet{Acharya2020Lowerbound, Acharya2021Lowerbound} for LDP, respectively. The optimal sample complexity is usually achieved through a systematic analysis of both its upper and lower bounds. For the upper bound analysis, the prior works extend non-private multinomial tests, such as in \citet{Acharya2015Nonprivate}, \citet{chan2014optimal}, \citet{Diakonikolas2018Nonprivate}, \citet{Diakonikolas2016Nonprivate}, \citet{Goldreich2000Nonprivate}, and \citet{Valiant2017Nonprivate}, to corresponding private counterparts by incorporating randomization mechanisms. On the other hand, the lower bound analysis relies on information-theoretic techniques, such as Le Cam's method, while treating the DP requirement as information constraints \citep[see, for example,][for detailed discussions]{Acharya2020Lowerbound, Duchi2018MinimaxEstimation}.

There are also a few recent papers from the statistics community that explore univariate goodness-of-fit testing under LDP. Specifically, \citet{Dubois2022} propose minimax optimal goodness-of-fit tests for H\"{o}lder densities under LDP in both non-interactive and interactive scenarios. \citet{Lam-Weil2021MinimaxConstraint} also consider goodness-of-fit testing under LDP, and develop minimax optimal tests for multinomials and for continuous densities over Besov balls. Our work builds on their framework and extends the focus from goodness-of-fit testing to two-sample testing for both (i) multinomials and (ii) multivariate H\"{o}lder and Besov densities. It is worth highlighting that previous works \citep{Lam-Weil2021MinimaxConstraint, Dubois2022} rely  solely on the Laplace mechanism \citep{Dwork2014Book} to establish their theoretical results. In contrast, we explore various LDP mechanisms that achieve similar optimality properties and empirically demonstrate that the Laplace mechanism can underperform in practical scenarios. Specifically, we delve into the Google's \texttt{RAPPOR} \citep{erlingsson_rappor_2014}, generalized randomized response~\citep{Gaboardi2018LDPChisq}, and (discrete) Laplace mechanisms~\citep{ghosh2009}, and illustrate their theoretical and empirical performance.
\subsection{Techniques and Results}\label{subsection:techniques}
Previous work on hypothesis testing under local differential privacy has primarily focused on goodness-of-fit testing~\citep{Dubois2022, Lam-Weil2021MinimaxConstraint}. We instead target a broader and arguably more challenging settings of two-sample testing of multinomials and multivariate densities.  In particular, we provide testing methods that are both practically reliable and theoretically optimal.
Our practical reliability stems from both the privatization mechanism and the testing procedure. For privatization, one of our methods leverages Google's \texttt{RAPPOR}, a widely adopted open-source privacy mechanism  which has demonstrated effectiveness through years of large-scale deployment in Chrome browser. 
Although prior works~\citep{duchi2013local, acharya_test_2019, acharya_estimating_2021} adopt \texttt{RAPPOR} and analyze its statistical performance under the minimax framework, their focuses are limited to $\ell_1$ separation and multinomial data. Our work expands upon them by establishing minimax optimality for both multivariate continuous and multinomial data under  $\ell_2$ separation, marking the first result of its kind.
For the testing procedure, our methods rigorously control the type I error in all scenarios. At the heart of achieving the blend of practicality and theory lies the permutation test.
Achieving practical reliability poses a significant challenge, especially in calibrating the critical value within the non-asymptotic regime while accounting for the randomization effects introduced by local differential privacy.
Under a composite null hypothesis of two-sample testing, critical values cannot be determined through Monte-Carlo-approximated population quantile of test statistics,
unlike in goodness-of-fit testing.
The critical values obtained through concentration inequalities, on the other hand, usually depend on unspecified constants and thus are not reliable in practice.
By employing a permutation procedure, our testing methods guarantee type I error control at any sample size and with sufficiently large number of permutations (which does not depend on the sample size). The permutation procedure also facilitates theoretical analysis of power, leading to minimax upper bound analyses. In particular, it enables us to leverage the technique of \citet{kim_minimax_2022}, namely the two moments method therein.
This technique allows us to avoid directly analyzing the permutation distribution under LDP, and provides a sufficient condition for type II error control based solely on the first two moments of the test statistic. Equipped with this tool, we analyze our test statistics, which are U-statistics derived from perturbed data with a carefully selected perturbation level. A bulk of our technical effort is dedicated to bounding the moments related to the U-statistic in the presence of this data perturbation.

%%%%%%%%%%%%Lower bound%%%%%%%%%
By obtaining matching information-theoretic lower bounds, we establish the optimality of our methods and gain insight into the fundamental trade-off between statistical power and privacy.
For the lower bound analysis, we leverage a recently developed technique by~\citet{Lam-Weil2021MinimaxConstraint}. This technique builds upon Ingster's method~\citep{Ingster1993AsymptoticallyAlternatives}, a classical approach for deriving minimax separation rates in testing problems, and adapts it to incorporate the LDP constraint. At the heart of Ingster's method is bounding the chi-square divergence between a simple null distribution and a mixture of alternative distributions. The key idea behind obtaining a tight lower bound under LDP is to construct a mixture distribution in Ingster's method using the singular values and singular vectors of the privacy mechanism. Such construction naturally imposes extra restrictions caused by the LDP constraint, enabling us to achieve a tight lower bound under LDP. Our technical effort lies in extending the univariate result of \cite{Lam-Weil2021MinimaxConstraint} to more general settings, including the multivariate H\"{o}lder ball and Besov ball.

\vskip 1em

\noindent \textit{Summary of our contributions.} We highlight our contributions and contrast them with prior work as follows. We also refer readers to Table~\ref{table:rates}, which summarizes the non-private and private minimax separation rates for two-sample testing, derived from both prior work and our findings.

\begin{table*}
	\small
	\begin{center}
		\begin{tabular}{lcc}
			\toprule
			&Non-private rate & Private rate under LDP\ \\
			\midrule
			\rule{0pt}{15pt}
			\makecell{Testing for multinomials \\ in $\ell_2$ separation}
			&
			%non-private rate
			\makecell{$\sampleSize_1^{-1/2}$
				\\
				\citep{chan2014optimal,kim_minimax_2022}
			}
			&
			%private rate
			\makecell{$
				\dfrac{\alphabetSize^{1/4}}{{(\sampleSize_1 \privacyParameter^2)^{1/2}}}
				\vee
				\sampleSize_1^{-1/2}
				$
				\\
				(Theorem~\ref{theorem:twosample_multinomial_rates})
			}
			\\
			\rule{0pt}{25pt}
			\makecell{
				Testing for H\"{o}lder and Besov \\densities in $\mathbb{L}_2$ separation}
			&%non-private rate
			\makecell{ $\sampleSize_1^{\frac{-2s}{4\smoothness+\dimDensity}}$\\ \citep{Arias-Castro2018RememberDimension}
			}
			&%private rate
			\makecell{$(\sampleSize_1\privacyParameter^2)^{\frac{-2s}{4\smoothness+3\dimDensity}} 
				\vee
				\sampleSize_1^{\frac{-2s}{4\smoothness+\dimDensity}}$
				\\(Theorem~\ref{theorem:twosample_conti_rate})
			}
			\\
			\bottomrule
		\end{tabular}
	\end{center}
				\caption{Non-private and private minimax rates for two-sample multinomial and density testing in $\ell_2$ and $\mathbb{L}_2$ separation where $\sampleSize_1$ denotes the minimum sample size and $\privacyParameter$ denotes the privacy level. For multinomial testing, $\alphabetSize$ stands for the number of categories, and for density testing, $\smoothness$ represents the smoothness parameter.
		The rates exhibit elbow effects{---}phase transitions at specific levels of privacy.
		See Section~\ref{section:twosample_disc} and~\ref{section:twosample_conti} for details.
	}\label{table:rates}
\end{table*}
\begin{itemize}
	\item \textbf{Optimal multinomial testing under LDP~(Theorem~\ref{theorem:twosample_multinomial_rates}):} We start by developing a private two-sample test for multinomials, and present minimax separation rates in terms of the $\ell_2$ distance under LDP. The prior work \citep{Acharya2018dPGofTwosample, Aliakbarpour2019PrivatePermutations,Aliakbarpour2018DPgof} for private two-sample testing generally focuses on central DP and imposes conditions such as equal sample sizes and Poisson sampling that may not be practically relevant. In contrast, our approach does not rely on such unnecessary conditions, and obtain optimality under more practical settings. We also highlight that our upper bound result is established using three distinct LDP mechanisms{---}namely Laplace, discrete Laplace and \texttt{RAPPOR} mechanisms{---}which diversifies the toolkit in practice. As mentioned earlier, this is in contrast to the prior work \citep{Lam-Weil2021MinimaxConstraint,Dubois2022}, which mainly focuses on the Laplace mechanism. Moreover, we show that the use of generalized randomized response mechanism~\citep{Gaboardi2018LDPChisq} can lead to suboptimal power in Appendix~\ref{genrr_suboptimal_theory}. 
	\item \textbf{Optimal density testing under LDP~(Theorem~\ref{theorem:twosample_conti_rate}):} We next consider the two-sample problem for continuous data and derive optimal $\mathbb{L}_2$ separation rates under LDP, by leveraging the prior work \citep{Lam-Weil2021MinimaxConstraint, kim_minimax_2022}. In particular, we examine both H\"{o}lder and Besov smoothness classes, and show that the proposed private test is optimal for both classes with the finite-sample validity. This approach differs from the prior work on a similar topic. For instance, unlike \citet{sheffet_differentially_2024} that consider central DP with Poissonization, we focus on the more stringent setting of LDP and consider the standard sampling with fixed sample sizes. Moreover, in contrast to the prior work under LDP~\citep{Ding2018TwosampleMean, waudby-smith_nonparametric_2023}, primarily focused on location alternatives, our private test is sensitive against a broad range of nonparametric alternatives. Lastly, we highlight that our method controls the type I error in any finite sample sizes, and exactly satisfies the LDP condition, distinguishing it from the prior work of \cite{Raj2020ATest}.
	\item \textbf{Adaptive density testing under LDP (Theorem~\ref{theorem:twosample_adaptive_upper}):} Similar to other nonparametric methods for density testing, the optimality of the proposed density test relies on the knowledge of the underlying smoothness parameter, which is typically unknown. To tackle this issue, we introduce a Bonferroni-type approach that adapts to the unknown smoothness parameter at the expense of extra logarithmic factors in the separation rate. The proof of the adaptation result leverages the exponential inequality of the permuted U-statistic~\citep{kim_minimax_2022}. This technique leads to an improvement of the adaptive technique used in \citet{Lam-Weil2021MinimaxConstraint}, which resorts to a simple upper bound for the variance of the U-statistic along with Chebyshev's inequality.
	
	\item \textbf{Numerical validation (Section~\ref{section:simulation}):}
	Lastly, we assess the empirical performance of the proposed tests under various scenarios and showcase the trade-off between privacy and utility through numerical simulations. It is important to emphasize that previous research on private testing has primarily centered on theoretical optimality, often lacking empirical validation of their findings. We address this gap by complementing theoretical justifications with empirical evaluation, thereby enhancing practical relevance. Since no previous methods exist for two-sample testing for multinomials or densities under LDP, we create baseline methods by extending one-sample LDP $\chi^2$-tests~\citep{Gaboardi2018LDPChisq} to the two-sample problem (Appendix~\ref{appendix:baseline}), and compare their empirical performance with our main proposals. To facilitate the use of our method, we provide a Python package \texttt{privateAB} that implements all proposed and baseline methods, available at \url{https://pypi.org/project/privateAB/0.0.2/}.

\end{itemize}

\subsection{Notation}\label{subsection:notation}
Throughout this paper, real numbers are represented by lowercase, non-bold letters, such as \(a\), while vectors in \( \mathbb{R}^d \) for \(d \geq 2\) are written in boldface lowercase, such as \( \mathbf{a} \). Constant vectors are denoted using bold numerals, such as \( \mathbf{1} \) and \( \mathbf{0} \). The $i$th element of \( \mathbf{a} \) is denoted by \( a_i \). For indexed vector such as $\vectorize{a}_j$, its $i$th element is denoted as \( a_{ji} \). Unless otherwise specified, random variables are written in uppercase non-bold (for example, \( X \)), while random vectors use bold uppercase (for example, \( \mathbf{X} \)). The $i$th element of $\vectorize{X}$ and $\vectorize{X}_j$ are denoted as $X_i$ and $X_{ji}$, respectively.
The set of non-negative integers is denoted by \( \mathbb{N}_0 := \{0,1,2,\ldots\} \). For positive integers \( u \) and \( v \), \( [u] \) represents \( \{1, \ldots, u\} \), and \( [u]^v \) denotes its Cartesian product taken \( v \) times. A set of \( u \) elements indexed by \( i \) is written as \( \{a_i\}_{i \in [u]} := \{a_1, \ldots, a_u\} \).
For any real \( s > 0 \), \( \lfloor s \rfloor \) denotes the largest integer strictly smaller than \( s \). For \( a, b \in \mathbb{R} \), we define \( a \vee b := \max\{a, b\} \) and \( a \wedge b := \min \{a, b\} \). Given \( \mathbf{w} \in \mathbb{R}^d \) and \( p \geq 1 \), its \( \ell_p \)-norm is defined by \( \|\mathbf{w}\|_p := \bigl(\sum_{i=1}^d |w_i|^p\bigr)^{1/p} \). Similarly, the \( \mathbb{L}_p \)-norm of a function supported on \( [0,1]^d \) is defined as follows:
$$
\vertiii{f}_{\mathbb{L}_p}
:=
\biggl(
\int_{\domainTs}
|f(\vectorize{x})|^p d\vectorize{x}
\biggr)^{1/p},
\quad
1 \leq p < \infty.
$$
For any  $\vectorize{u}, \vectorize{v} \in \mathbb{R}^\dimDensity$,
with $u_j \leq v_j$ for  $j=1, \ldots, d$, a hyperrectangle $[\vectorize{u},\vectorize{v}]$ is defined as follows:
$$
[\vectorize{u},\vectorize{v}] := \prod_{j=1}^\dimDensity [u_j, v_j].
$$
Given the privacy parameter $\alpha > 0$, we write $z_\alpha := e^{2\alpha}- e^{-2\alpha} = 2 \sinh(2\alpha)$. Throughout, a constant that only depends on parameters $\theta_1,\theta_2,\ldots$ is denoted as $C(\theta_1,\theta_2,\ldots)$. The indicator function $\indicator{\mathcal{A}} $ takes value 1 if the event $\mathcal{A}$ is true and 0 otherwise.

\subsection{Outline of the Paper}
The remainder of the paper is organized as follows:
Section~\ref{section:background} introduces the necessary background on LDP, the minimax framework, and the permutation procedure.
Section~\ref{section:twosample_disc} illustrates the minimax analysis for multinomial testing under LDP and the optimal permutation testing procedure.
Building on this result, Section~\ref{section:twosample_conti} presents the minimax analysis for multivariate two-sample density testing under LDP and an optimal permutation testing procedure.
Finally, Section~\ref{section:simulation} presents numerical validation of the procedures proposed in Sections \ref{section:twosample_disc} and \ref{section:twosample_conti}. All proofs and additional simulation results are deferred to the appendix.
\section{Background}\label{section:background}
This section introduces the notion of local differential privacy and two-sample testing under LDP. We then explain the minimax framework for two-sample testing under LDP, along with the permutation test procedure considered throughout this paper. 
%
%%%%%%%%%%%%%%%%%%%%%%%%%%%%%%%%%%%%%%%%%%%%%%%%%%%%%%%%%%%%%%%%
\subsection{Two-Sample Testing under Local Differential Privacy Constraint}\label{subsection:LDPtwosampleFormulation}
Let $\distClassGeneric$ denote the space of pairs of distributions of interest{---}multinomial distributions taking values in $[k]$ or continuous distributions taking values in $[0,1]^{\dimDensity}$. Given a pair of distributions $(P_\rvY, P_\rvZ) \in \distClassGeneric$, for each $\sampleIndexOne \in [\sampleSize_1]$, the $\sampleIndexOne$th  data owner draws $\rvY_\sampleIndexOne$ from $P_\rvY$ independently from the others. Similarly, for each $\sampleIndexTwo \in [\sampleSize_2]$,  the $\sampleIndexTwo$th data owner independently draws $\rvZ_\sampleIndexTwo$ from $P_\rvZ$.
We allow the sample sizes $\sampleSize_1$ and $\sampleSize_2$ to differ, and assume $\sampleSize_1 \leq \sampleSize_2$, without loss of generality, throughout this paper.  We denote the pooled sample size as $\sampleSize := \sampleSize_1 + \sampleSize_2$. Under the LDP constraint, each owner releases only a randomized transformation of their raw sample as follows:
%%%%%%%%%%%%%%%%% Local differential privacy defintion %%%%%%%%%%%%%%%%%%%%
\begin{definition}[Local differential privacy]\label{def:LDP}
	Given a privacy level $\privacyParameter > 0$, let $X_i$ and $\tilde{X}_i$ be random elements  mapped to measurable spaces $(\rvXCodomain, \sigmafield)$ and $(\rvXPrivCodomain_i, \sigmafieldPriv_\sampleIndexOne)$, respectively, for each $i \in [n]$. Then $\tilde{X}_i$ is an $\privacyParameter$-local differentially private ($\privacyParameter$-LDP) view of $X_i$ if there exists a bivariate function $Q_\sampleIndexOne(\cdot \,|\, \cdot)$ on $\sigmafieldPriv_\sampleIndexOne \times \rvXCodomain$  such that:
	\begin{enumerate}
		\item For any $x \in \rvXCodomain$, $Q_\sampleIndexOne ( \cdot \,|\, x)$ is a conditional distribution of $\tilde{X}_i$ given $X_i=x$,
		\item For any $A \in \sigmafieldPriv_\sampleIndexOne$, $x \mapsto Q_\sampleIndexOne (A\,|\,x) $ is a measurable function on $\rvXCodomain$,~\text{and}
		\item For any $x, x' \in \rvXCodomain \times \rvXCodomain$ and $A \in \sigmafieldPriv_i$, the inequality
		$Q_\sampleIndexOne
		(
		A\,|\,x
		)  \leq  e^{\privacyParameter}	Q_\sampleIndexOne
		(
		A\,|\,
		x'
		)$ holds.
	\end{enumerate}
	Let $\privacyMechanismClass_\privacyParameter$
	be the set of joint distributions whose marginals satisfy the above properties.
	Then $\privacyMechanism \in \privacyMechanismClass_\privacyParameter$ is called an $\privacyParameter$-LDP mechanism (or channel) associated with $\{\rvOne_\sampleIndexOne\}_{\sampleIndexOne=1}^\sampleSize$.
\end{definition}
\noindent
The curator, aware of the privacy level $\alpha$ and the  mechanism $Q$, only receives the $\privacyParameter$-LDP views $\{\tilde{X}_i\}_{i \in [n]}$, consisting of $\sampleSets{\rvYPriv}{\sampleIndexOne}{\sampleSize_1}$ and $\sampleSets{\rvZPriv}{\sampleIndexTwo}{\sampleSize_2}$, and uses them to decide whether $H_0: P_\rvY = P_\rvZ$  or $H_1: P_\rvY \neq P_\rvZ$.
The definition of LDP above is non-interactive in a sense that the $\sampleIndexOne$th conditional distribution $Q_\sampleIndexOne$ is assumed to be independent of other private views $\tilde{X}_1,\ldots,\tilde{X}_{i-1},\tilde{X}_{i+1},\ldots,\tilde{X}_n$.
It has been pointed out that allowing $Q_i$ to be interactive with private views $\{ \tilde{X}_{i'} \}_{i' \in [\sampleSize]\setminus{i}}$ can yield more efficient statistical procedures \citep[for example,][]{Acharya2022InteractiveConstraints,Berrett2020Interactive, Kasiviswanathan2008WhatPrivately}.
The non-interactive approach, however, requires less communication between data owners and a curator and thus is more suitable for large-scale statistical inference than the interactive counterpart~\citep{Berret2021Classification, Joseph2019Interactivity}. Thereby, we focus on the non-interactive privacy mechanisms throughout this paper.

Unless otherwise specified, all statements regarding expectations and variances in this paper are with respect to the distributions of the $\privacyParameter$-LDP views defined as:
\begin{equation*}
	P_{ \tilde{Y}_{i}}(A) := \int_{\rvXCodomain_i}
	Q_\sampleIndexOne
	(
	A\,|\,
	y
	)
	P_Y(dy)~\text{and}~
	P_{ \tilde{Z}_{j}}(B) := \int_{\rvXCodomain_j}
	Q_\sampleIndexOne
	(
	B\,|\,
	z
	)
	P_Z(dz),
\end{equation*}
for $i \in [n_1]$, $j \in [n_2]$, $A \in \sigmafieldPriv_i$, and $B \in \sigmafieldPriv_j$,  where the raw samples are marginalized out.

\subsection{Non-Private and LDP Minimax Framework for Two-Sample Testing}\label{subsection:testing}
Let \( \distClassGeneric_0 \) and \( \distClassGeneric_1 \) denote the collections of null and alternative distributions, respectively, corresponding to \( H_0 \) and \( H_1 \) introduced in Section \ref{subsection:LDPtwosampleFormulation}.   In the minimax framework, we focus on a subset of $\distClassGeneric_1$, denoted as $\distClassGeneric_1(\minSep_{\sampleSize_1, \sampleSize_2})$, where $\minSep_{\sampleSize_1, \sampleSize_2}$ indicates a minimum separation between $\distClassGeneric_1(\minSep_{\sampleSize_1, \sampleSize_2})$ and \( \distClassGeneric_0 \). 

For a given privacy mechanism $Q$, the curator evaluates a test function
$\test_Q: \prod_{i=1}^{n} \rvXPrivCodomain_i
\mapsto
\{0,1\}$ on the $\privacyParameter$-LDP views
and rejects $H_0$ if $\test_\privacyMechanism = 1$ and accepts $H_0$ otherwise. Our objective is to design the private test $\test_Q$ that controls the type I and II errors uniformly over distributions in $\distClassGeneric_0$ and $\distClassGeneric_1(\minSep_{\sampleSize_1, \sampleSize_2})$, respectively. In particular, for fixed $\maxErrorTypeOne, \maxErrorTypeTwo \in (0,1)$, we aim for the private test $\test_Q$ to satisfy the following conditions: 
\begin{equation}
	\begin{aligned} \label{eq:uniform_control}
		& \text{Type I error:} \, \sup_{(P_Y,P_Z) \in \distClassGeneric_0}
		\mathbb{E}
		[ \Delta_Q ] \leq \gamma,
		\;\; \text{and} \\%\;\;
		& \text{Type II error:} \, \sup_{
			(P_Y,P_Z)
			\in
			\mathcal{P}_1(\rho_{n_1, n_2})
		}
		\mathbb{E}
		[1-\Delta_Q] \leq \beta.
	\end{aligned}
\end{equation}
Let $\Phi^{\privacyParameter}_{\maxErrorTypeOne}$ be the set of $\privacyParameter$-LDP level $\maxErrorTypeOne$ tests, which take $\privacyParameter$-LDP views as inputs
and control the type I error as in~\eqref{eq:uniform_control}.
The quality of a test $\test_{Q, \gamma} \in \Phi^{\privacyParameter}_\maxErrorTypeOne$ is assessed by its uniform separation rate, which quantifies the proximity between two hypotheses that can still be successfully distinguished by the test.
In more technical terms, the uniform separation rate
$\tilde{\minSep}_{\sampleSize_1, \sampleSize_2}( \test_{Q, \gamma} )$ is the smallest separation $\minSep_{\sampleSize_1, \sampleSize_2}$ which accomplishes the type II error control as in~\eqref{eq:uniform_control}, namely,
\begin{align*}
	\tilde{\rho}_{\sampleSize_1, \sampleSize_2}
	( \Delta_{Q, \gamma} )
	:=
	\inf
	\Bigl\{
	\minSep_{\sampleSize_1, \sampleSize_2} > 0
	:
	\sup_{(P_Y,P_Z) \in \distClassGeneric_1(\minSep_{\sampleSize_1, \sampleSize_2})}
	\mathbb{E}
	[
	1-\Delta_{Q, \gamma}
	]\leq \beta
	\Bigr\}.
	\numberthis \label{uniform_separation_rate_ldp}
\end{align*}
Since different $\privacyParameter$-LDP mechanism $\privacyMechanism$ may assume different $\rvXPrivCodomain_i$'s, a test $\test_{Q, \gamma}$ depends on a particular LDP mechanism $\privacyMechanism$.
Therefore, an optimal private level $\maxErrorTypeOne$ test can be described as a pair of an $\privacyParameter$-LDP mechanism and a test function, designed to achieve the minimal uniform separation rate~\eqref{uniform_separation_rate_ldp}. We define this minimal uniform separation rate as the optimal testing rate under LDP.

\begin{definition}[$\privacyParameter$-LDP non-asymptotic minimax testing rate]
	For a fixed privacy level $\privacyParameter > 0$,
	the $\privacyParameter$-LDP non-asymptotic minimax rate of testing is defined as
	\begin{align*}
		\minSep_{\sampleSize_1, \sampleSize_2,\privacyParameter}^\ast
		:=
		\inf_{Q \in \mathcal{Q}_\privacyParameter}
		\inf_{\Delta_{Q,\gamma} \in \Phi^{\privacyParameter}_{\gamma}}
		\tilde{\rho}_{\sampleSize_1, \sampleSize_2}
		( \Delta_{Q, \gamma} ).
		\numberthis \label{def:minimax_rate_LDP}
	\end{align*}
\end{definition}
\noindent Our main interest is to figure out the price to pay for privacy by comparing the private minimax testing rate~\eqref{def:minimax_rate_LDP} with the non-private~(unconstrained) rate
$\minSep_{\sampleSize_1, \sampleSize_2}^\ast
:=
\inf_{\Delta_{\gamma} \in \Phi_\gamma}
\tilde{\minSep}_{\sampleSize_1, \sampleSize_2}
( \Delta_{\gamma})$,
where $\Phi_\gamma$ denotes the set of all level $\gamma$ tests without privacy constraints. 
Since it is mostly infeasible to obtain a test that achieves the exact minimax risk in a nonparametric setting, we follow the convention~\citep{Ingster1994Bin,Ingster2000AdaptiveTests,Baraud2002minimax} and focus on minimax \emph{rate} optimality. In particular, we say that a test is minimax rate optimal if its uniform separation rate is upper bounded by the minimax testing rate up to some constant. 

\subsection{Permutation Testing Procedure} \label{subsection:permutation}
The permutation test is a simple yet powerful method to calibrate a test statistic, yielding a valid $\gamma$-level test under exchangeability of the samples, meaning that when the distribution of the permuted samples is the same as the distribution of the original samples.
Note that some $\privacyParameter$-LDP mechanisms might violate exchangeability even when the raw samples are i.i.d.
In order to guarantee exchangeability, throughout this paper, we only consider $\privacyParameter$-LDP mechanisms with identical marginals. Below, we briefly explain the permutation test in the two-sample setting.

Let $\mathbb{X}_n$ denote the pooled sample
$\{\xi_1,\ldots,\xi_n\} := \{Y_1,\ldots,Y_{\sampleSize_1},Z_{1},\ldots,Z_{\sampleSize_2}\}$.
Let $\boldsymbol{\pi}_n$ be the set of all possible permutations of $[\sampleSize]$,
and denote its cardinality as $|\boldsymbol{\pi}_\sampleSize|$.
Given a permutation $\permutation := (\permutation_1, \ldots, \permutation_n)$ sampled from a uniform distribution over $\boldsymbol{\pi}_n$, the permuted version of $\mathbb{X}_n$ is denoted  as $\mathbb{X}_n^\permutation := \{ \xi_{\permutation_1},\ldots,\xi_{\permutation_n} \}$.
For a two-sample statistic $T_\sampleSize(\mathbb{X}_n)$,
its permutation distribution function conditional on $\mathbb{X}_n$ is $F_{T}(t):=
\sum_{\permutation \in \boldsymbol{\pi}_n}
\mathds{1}\bigl(
T(\mathbb{X}_n^\permutation) \leq t
\bigr) / |\boldsymbol{\pi}_\sampleSize|$.
The permutation testing procedure rejects the null hypothesis if $T(\mathbb{X}_n) >  \inf \{t:F_{T}(t) \geq 1 - \gamma\}$.
If the exchangeability assumption of $
\mathbb{X}_n \stackrel{d}{=} \mathbb{X}_n^\permutation~\text{for any}~\permutation \in \boldsymbol{\pi}_\sampleSize$ is satisfied under $H_0$,
the resulting test controls the type I error non-asymptotically \citep[see, for example,][]{ramdas2023permutation}.

In practice, however, it is computationally infeasible to consider all $|\boldsymbol{\pi}_\sampleSize|$ permutations. Therefore, it is now a standard practice to consider a Monte Carlo-based~(MC-based) permutation test which uses much smaller number of permutations. To explain, for a given $\sizeMonteCarlo >0$, let $\permutation_1, \ldots, \permutation_\sizeMonteCarlo$ be $\sizeMonteCarlo$ independent random permutations sampled from uniform distribution over $\boldsymbol{\pi}_n$.
These permutations are used to calculate the following MC-based permutation p-value:
\begin{equation}\label{eq:monte_carlo_perm_p_value}
	\pValueMonteCarlo:=\frac{1}{\sizeMonteCarlo+1}
	\biggl[
	1 + \sum_{b=1}^{\sizeMonteCarlo}
	\mathds{1} \bigl\{ T_{\sampleSize}(\mathbb{X}_{\sampleSize}^{\permutation_b}) \geq T_{\sampleSize}(\mathbb{X}_{\sampleSize})
	\bigr\}
	\biggr],
\end{equation}
which controls the type I error non-asymptotically under exchangeability for any $B$ and any test statistic \citep[see, for example,][for details]{hemerik_exact_2018,ramdas2023permutation}.
Regarding type II error, the MC-based permutation test using the U-statistic in~\eqref{def:statistic_elltwo} achieves power comparable to the full permutation test if $\sizeMonteCarlo$ is sufficiently large—independent of sample size and data dimension, and much smaller than $|\boldsymbol{\pi}_\sampleSize|$ \citep[see Proposition I.1 of][for details]{kim_minimax_2022}.

\section{Two-Sample Testing for Multinomials under LDP Constraint}\label{section:twosample_disc} 
Having introduced the background, we now proceed to present the main results of this paper. In Section~\ref{subsection:twosample_multinomial_rates}, we consider the problem of comparing two multinomial distributions, and establish the corresponding minimax rate under the LDP constraint. The upper bound for the minimax rate is attained by the LDP permutation tests that we propose in Section~\ref{subsection:twosample_disc_upperbound}. These tests play a pivotal role in establishing the minimax rate for multivariate continuous data, as discussed in Section~\ref{section:twosample_conti}.

\subsection{Private Minimax Rates for Two-Sample Multinomial Testing}\label{subsection:twosample_multinomial_rates}
The problem of interest is formulated as follows. Let \( \distclassMultinomial \) denote the set  of pairs of probability vectors  with \( \alphabetSize \) categories. Suppose the raw sample sets \( \sampleSets{\rvY}{\sampleIndexOne}{\sampleSize_1} \) and \( \sampleSets{\rvZ}{\sampleIndexTwo}{\sampleSize_2} \) are drawn from multinomial distributions with probability vectors \( (\probVec_\rvY, \probVec_\rvZ) \in \distclassMultinomial \). The curator receives two sets of \(\privacyParameter\)-LDP views, \( \sampleSets{\rvYPriv}{\sampleIndexOne}{\sampleSize_1} \) and \( \sampleSets{\rvZPriv}{\sampleIndexTwo}{\sampleSize_2} \), and determines whether \( (\probVec_\rvY, \probVec_\rvZ) \) belongs to  \( \distClassGeneric_{0,\mathrm{multi}} := \{ (\probVec_\rvY, \probVec_\rvZ) \in \distclassMultinomial : \probVec_\rvY = \probVec_\rvZ \} \) or to the alternative hypothesis set defined as:
\begin{equation}\label{eq:settingTwosampleDisc}
	\mathcal{P}_{1,\mathrm{multi}}(\rho_{n_1,n_2})
	:=
	\bigl\{
	(\probVec_\rvY , \probVec_\rvZ) \in \distclassMultinomial:
	\| \probVec_\rvTwo - \probVec_\rvThree \|_2 \geq \rhoTwosample
	\bigr\}.   
\end{equation}
Let us fix the type I error $\maxErrorTypeOne$ and the type II error $\maxErrorTypeTwo$ such that $2\maxErrorTypeOne + \maxErrorTypeTwo < 1$. This constraint arises from Ingster's minimax lower bounding method, as similarly considered in \cite{Lam-Weil2021MinimaxConstraint}. The main result of this section, stated below, establishes a lower bound as well as an upper bound for the minimax separation rate for this multinomial problem under LDP in terms of the $\ell_2$ distance.

\begin{theorem}[Minimax rate for two-sample multinomial testing under LDP]	\label{theorem:twosample_multinomial_rates}There exist  positive constants 
	$C_\ell(\maxErrorTypeOne, \maxErrorTypeTwo)$ and $C_u(\maxErrorTypeOne, \maxErrorTypeTwo)$
	such that the $\privacyParameter$-LDP minimax testing rate $\minimaxTestingRateTwosampleLDP$ over the class of alternatives $\mathcal{P}_{1,\mathrm{multi}}$ in \eqref{eq:settingTwosampleDisc} is bounded as
	\begin{equation}\label{rate:twosample_disc}
		C_\ell(\maxErrorTypeOne, \maxErrorTypeTwo)
		\Bigg[
		\biggl(
		\frac
		{\alphabetSize^{1/4}}
		{(\sampleSize_1 z_{\alpha}^2)^{1/2}}
		\wedge
		(\alphabetSize \log \alphabetSize)^{-1/2}
		\biggr)
		\vee
		\frac
		{1}
		{{\sampleSize_1}^{1/2}}
		\Biggr]
		%%%%%%%%%%%%%%%
		\leq
		%%%%%%%%%%%%%%%
		\minimaxTestingRateTwosampleLDP
		%%%%%%%%%%%%%%
		\leq
		%%%%%%%%%%%%%%
		C_u(\maxErrorTypeOne, \maxErrorTypeTwo)
		\biggl[
		\frac
		{\alphabetSize^{1/4}}
		{(\sampleSize_1 \alpha^2)^{1/2}} 
		\vee
		\frac
		{1}
		{{\sampleSize_1}^{1/2}}
		\biggr],
	\end{equation}
	where we recall $z_\alpha = 2 \mathrm{sinh}(2\alpha)$.
\end{theorem}
Theorem~\ref{theorem:twosample_multinomial_rates} states that the private separation rate for two-sample multinomials is notably different from its non-private counterpart $\sampleSize_1^{-1/2}$ with respect to the $\ell_2$ distance~\citep{chan2014optimal,kim_minimax_2022}. In particular, in the high privacy regime, we observe additional dependence on $\alphabetSize$ and $\privacyParameter$. The result also indicates that the privacy guarantee can be obtained at no additional cost in the low privacy regime where $\sampleSize_1^{-1/2}$ dominates the other term. 

We point out that while the lower bound and the upper bound do not exactly match, the gap is notably small. For instance, in the regime where the $(\alphabetSize \log \alphabetSize)^{-1/2}$ term is negligible, the only different terms are $\alpha$ and $z_\alpha$. We note that these two terms are the same, up to a constant factor, as long as $\alpha$ is bounded. Hence, in most practical scenarios where a small value of $\alpha$ is of interest, the upper bound matches the lower bound. 

We now discuss the proof of Theorem~\ref{theorem:twosample_multinomial_rates}, and a detailed analysis can be found in Appendix~\ref{proof:twosample_multinomial_rates}.
\begin{itemize}
	\item \textbf{Lower bound.} We can obtain the lower bound almost for free by observing that the two-sample problem is more difficult than the one-sample problem. In particular, one can always turn the one-sample problem into the two-sample problem by drawing additional samples from the target distribution. Therefore, a minimax lower bound for the one-sample problem does not exceed that of the two-sample problem \citep[see Lemma 1 of][for a formal argument]{Arias-Castro2018RememberDimension}. Given this insight, the lower bound follows by combining lower bound results of Theorem 3.2 in \citet{Lam-Weil2021MinimaxConstraint} for the one-sample problem under LDP, as well as \citet{chan2014optimal} and \citet{kim_minimax_2022} for the two-sample problem without privacy constraints. We point out that the lower bound in \citet{chan2014optimal} and \citet{kim_minimax_2022} depends on $\max\{\|\vectorize{p}_Y\|_2, \|\vectorize{p}_Z\|_2\}$. This quantity becomes a constant at the worst case scenario and it can be thereby disregarded in the lower bound for the global minimax rate.
	\item \textbf{Upper bound.} To prove the upper bound, we leverage the $\ell_2$ permutation test in \cite{kim_minimax_2022}, which achieves the non-private minimax rate. The considered test statistic is essentially a U-statistic of $\| \probVec_\rvTwo - \probVec_\rvThree \|_2^2$ based on private views of indicator functions generated by privacy mechanisms. Given this privatized U-statistic detailed in Section~\ref{subsection:twosample_disc_upperbound}, we derive the upper bound by using the two moments method of \citet{kim_minimax_2022}, which provides a sufficient condition for significant power of a permutation test based on the first two moments of the test statistic.
	The following subsection outlines the privacy mechanisms involved,
	introduces a specific form of the test statistic,
	and presents the testing procedure that achieves the upper bound. 
\end{itemize}

\subsection{Privacy Mechanisms and Testing Procedure for the Upper Bound}\label{subsection:twosample_disc_upperbound}
This subsection is dedicated to explaining the private permutation testing procedure which achieves the upper bound stated in~\eqref{rate:twosample_disc}. As mentioned earlier, this procedure builds on the U-statistic~\citep{kim_minimax_2022} estimating the squared $\ell_2$ distance $\| \probVec_\rvTwo - \probVec_\rvThree \|_2^2$.
In a similar way, we deal with the U-statistic of $\alphabetSize \| \probVec_\rvTwo - \probVec_\rvThree \|_2^2$, but a notable difference is that ours is based on $\privacyParameter$-LDP views generated by privacy mechanisms. Before formally introducing the test statistic, let us explain the considered LDP mechanisms: Laplace, discrete Laplace, and \texttt{RAPPOR} mechanisms.

\vskip 1em

\noindent \textit{Privacy mechanisms.}  The first one is the standard Laplace mechanism, also considered in \cite{Lam-Weil2021MinimaxConstraint} and \cite{Berrett2020Interactive}, which adds an independent Laplace noise to the original data represented as indicator variables.

\begin{definition}[Laplace mechanism for multinomial data: $\texttt{LapU}$]\label{def:LapU_formal}
	Consider a pooled raw multinomial sample $\{X_i\}_{i \in [n]}$ with $\alphabetSize$ categories. Fix the privacy level $\privacyParameter>0$. Each data owner adds noise to their data point, with the noise variance parameterized by
	$$
	\LapUParam
	:=
	\frac{2 \sqrt{2\alphabetSize}}{\privacyParameter}.
	$$
	The resulting locally privatized sample $\{\tilde{\rVecX}_{i}\}_{i \in [n]} = \{\tilde{\rVecY}_i\}_{i\in [n_1]} \cup \{\tilde{\rVecZ}_i\}_{i \in [n_2]}$ is a set of 
	$\alphabetSize$-dimensional random vectors whose $\vectorIndex$th element $\tilde{\rvOne}_{im}$ is defined as follows:
	\begin{align*}
		\tilde{\rvOne}_{im}
		:=
		\sqrt{\alphabetSize} \;
		\mathds{1}
		\bigl(
		\rvX_{\sampleIndexOne} = \vectorIndex
		\bigr)
		+
		\LapUParam
		W_{i\vectorIndex}.
		\numberthis \label{def:LapU}
	\end{align*}
	Here, $\{W_{i \vectorIndex}\}_{i \in [n], m \in [k]} \stackrel{i.i.d.}{\sim} \mathrm{Lap}(1/\sqrt{2})$ is independent of $\{X_i\}_{i \in [n]}$, and $\mathrm{Lap}(1/\sqrt{2})$ denotes the centered Laplace distribution with variance one. 
\end{definition}

The second mechanism is based on discrete Laplace noise. We say that a random variable $W$ follows a discrete Laplace distribution with parameter $\distparamDiscLap \in (0,1)$, denoted as $W \sim \mathrm{DL}(\distparamDiscLap)$, if its probability mass function satisfies
\begin{align*}\numberthis \label{def:discrete_laplace_distribution}
	\mP(\rvLDPnoise = \numbersetRealizedLap)
	=
	\frac{
		1-\distparamDiscLap
	}{
		1+\distparamDiscLap
	}
	\distparamDiscLap^{|\numbersetRealizedLap|}, 
	\quad 
	\text{for all $\numbersetRealizedLap \in \mathbb{Z}.$}
\end{align*}
The second mechanism is similar to the first one but it replaces continuous Laplace noise with discrete Laplace noise. 

\begin{definition}[Discrete Laplace mechanism for multinomial data: \texttt{DiscLapU}]\label{def:DiscLapU_formal}
	Consider a 
    pooled raw multinomial sample $\{X_i\}_{i \in [n]}$ with $\alphabetSize$ categories. Fix the privacy level $\privacyParameter>0$. Each data owner adds noise to their data point, with the noise distribution parametrized by
	$$\discLapUParam := e^{
		-
		\frac{\privacyParameter}{2\sqrt{\alphabetSize}}
	},$$
	where $\discLapUParam \in (0,1)$  for any value of $\privacyParameter>0$ and $\alphabetSize \geq 2$. The resulting locally privatized sample $\{\tilde{\rVecX}_{i}\}_{i \in [n]} = \{\tilde{\rVecY}_i\}_{i\in [n_1]} \cup \{\tilde{\rVecZ}_i\}_{i \in [n_2]}$ is a set of 
	$\alphabetSize$-dimensional random vectors whose $\vectorIndex$th element $\tilde{\rvOne}_{im}$ is defined as follows:
	$$
		\tilde{\rvOne}_{im}
		:=
		\sqrt{\alphabetSize} \;
		\mathds{1}
		\bigl(
		\rvX_{\sampleIndexOne} = \vectorIndex
		\bigr)
		+
		\LapUParam
		W_{i\vectorIndex}.
	$$
	Here, $\{W_{i \vectorIndex}\}_{i \in [n], m \in [k]} \stackrel{i.i.d.}{\sim} DL(\discLapUParam)$.
\end{definition}

The third mechanism that we consider privatizes multinomial data vectors by randomly flipping individual components, instead of injecting additive random noise.
This mechanism, proposed by \citet{duchi2013local}, is equivalent to Google's basic one-time RAPPOR \citep[randomized aggregatable privacy-preserving ordinal response;][]{erlingsson_rappor_2014}.
\begin{definition}[Basic one-time RAPPOR mechanism for multinomial data: \texttt{RAPPOR}] \label{def:rappor_formal}
	Given a pooled raw multinomial sample $\{X_i\}_{i \in [n]}$ with $\alphabetSize$ categories, fix the privacy level $\privacyParameter>0$.  Each data owner perturbs their data point, resulting in the privatized sample $\{\tilde{\rVecX}_{i}\}_{i \in [n]} = \{\tilde{\rVecY}_i\}_{i\in [n_1]} \cup \{\tilde{\rVecZ}_i\}_{i \in [n_2]}$; a set of 
	$\alphabetSize$-dimensional random vectors whose $\vectorIndex$th element $\tilde{\rvOne}_{im}$ is defined as follows:
	\begin{align*}
		\tilde{\rvOne}_{\sampleIndexOne \vectorIndex}
		:=
		\begin{cases}
			\mathds{1}
			\bigl(
			\rvX_{\sampleIndexOne} = \vectorIndex
			\bigr)
			\quad & \text{with probability} \quad \dfrac{
				e^{\alpha/2}
			}{
				e^{\alpha/2}+1
			},
			\\
			\mathds{1}
			\bigl(
			\rvX_{\sampleIndexOne} \neq \vectorIndex
			\bigr)
			\quad & \text{with probability} \quad \dfrac{
				1
			}{
				e^{\alpha/2}+1
			}.
		\end{cases}
	\end{align*}
\end{definition}
The next lemma proves the $\privacyParameter$-LDP guarantee for Laplace, discrete Laplace, and \texttt{RAPPOR} mechanisms. The proof for   discrete Laplace mechanism can be found in Appendix~\ref{proof:LapUdiscLapULDP}, whereas the proof for 
Laplace and \texttt{RAPPOR} mechanisms are provided in 
Lemma 4.2 of \citet{Lam-Weil2021MinimaxConstraint} and
Section 3.2 of \citet{duchi2013local}, respectively.

\begin{lemma}[LDP guarantee] \label{lemma:LapUdiscLapULDP}
	The random vectors $\{\rVecXPriv_{i}\}_{i \in [n]}$ generated by any of \emph{\texttt{LapU}} mechanism (Definition~\ref{def:LapU_formal}), \emph{\texttt{DiscLapU}} mechanism (Definition~\ref{def:DiscLapU_formal}) and \emph{\texttt{RAPPOR}} (Definition~\ref{def:rappor_formal}) are $\privacyParameter$-LDP views of $\{X_i\}_{i \in [n]}$.
\end{lemma}

As defined below, our test statistic for multinomial testing builds on the $\privacyParameter$-LDP views $\{\rVecXPriv_{i}\}_{i \in [n]}$ from one of the $\{\texttt{LapU}, \texttt{DiscLapU}, \texttt{RAPPOR}\}$ mechanisms. All of these mechanisms rigorously maintain local differential privacy and return a test that achieves the same separation rate as in Theorem~\ref{theorem:twosample_multinomial_rates}. 

\vskip 1em

\noindent \textit{Testing procedure.}
 Given $\privacyParameter$-LDP views from one of the $\{\texttt{LapU}, \texttt{DiscLapU}, \texttt{RAPPOR}\}$ mechanisms, we use the following U-statistic:
\begin{equation}\label{def:statistic_elltwo}
	U_{\sampleSize_1,\sampleSize_2}
	:=
	\frac
	{1}
	{\sampleSize_1(\sampleSize_1-1)}
	\sum_{1 \leq i_1 \neq  i_2 \leq n_1}
\hskip	-3mm
	\tilde{\rVecY}_{i_1}^\top \tilde{\rVecY}_{i_2}
	+
	\frac{1}{\sampleSize_2(\sampleSize_2-1)}
	\sum_{1 \leq j_1 \neq  j_2 \leq n_2}
		\hskip -3mm
	\tilde{\rVecZ}_{j_1}^\top \tilde{\rVecZ}_{j_2}
	-
	\frac{2}{\sampleSize_1\sampleSize_2}
	\sum_{i=1}^{\sampleSize_1} \sum_{j=1}^{\sampleSize_2}
	\tilde{\rVecY}_i^\top \tilde{\rVecZ}_j.
\end{equation}
To carry out a test, the test statistic is calibrated by the permutation procedure described in Section~\ref{subsection:permutation} with the pooled $\privacyParameter$-LDP views $\{\tilde{\rVecX}_i\}_{i \in [n]}$. Specifically, we reject the null when the $p$-value based on the test statistic~\eqref{def:statistic_elltwo} is smaller than or equal to significance level $\gamma$. The type I error of the resulting test is controlled at $\gamma$ as $\{\tilde{\rVecX}_i\}_{i \in [n]}$ are i.i.d.~random vectors under the null hypothesis. Our technical effort lies in studying the type II error guarantee of the proposed test, and in turn proving the upper bound in Theorem~\ref{theorem:twosample_multinomial_rates}. We refer to Appendix~\ref{proof:twosample_multinomial_rates} for details.

The statistic~\eqref{def:statistic_elltwo} based on either Laplace or discrete Laplace mechanism is an unbiased estimator of a scaled and squared $\ell_2$ distance between probability vectors:
\begin{equation}\label{equation:lapu_elltwo_unbiased}
	\mE[U_{\sampleSize_1,\sampleSize_2}
	]= \alphabetSize \| \probVec_\rvTwo - \probVec_\rvThree \|_2^2.
\end{equation}
On the other hand, the test statistic based on the \texttt{RAPPOR} mechanism does not maintain the unbiasedness property, with its expectation
shrinking to zero as $\alpha$ decreases (see Lemma~\ref{appendix:rappor_moment}). It therefore requires a more careful and, indeed, more challenging analysis compared to the other two mechanisms.

Despite the fact that all three mechanisms ensure the minimax separation rate, their finite-sample power performance may differ in various scenarios as illustrated in Section~\ref{section:simulation}. In particular, our numerical studies demonstrate that tests based on Laplace or discrete Laplace mechanisms tend to underperform compared to those based on the \texttt{RAPPOR} mechanism. This underperformance is partly because the private views from Laplace and discrete Laplace mechanism can take extreme values due to their unbounded support, whereas those from \texttt{RAPPOR} mechanism are always bounded. Accordingly, we advocate for using the \texttt{RAPPOR} mechanism over the Laplace and discrete Laplace mechanisms, even though they present the same theoretical guarantee in terms of the separation rate.

Another LDP mechanism for multinomial testing suggested by \citet{Gaboardi2016DPChisq} is the generalized randomized response mechanism (\texttt{GenRR}; see Appendix \ref{appendix:baseline} for details). In contrast to \texttt{RAPPOR} mechanism, which destroys the structure of one-hot vectors, the generalized randomized response mechanism maintains the one-hot vector format while randomly altering the position of the non-zero component. It turns out, however, that the test based on the generalized randomized response mechanism is notoriously suboptimal in terms of the separation rate, as we show in Appendix~\ref{genrr_suboptimal_theory}. Demonstrating this negative result requires a careful asymptotic analysis of the U-statistic, which we want to highlight as our technical contribution.
\section{Two-Sample Testing for H\"{o}lder and Besov Densities under LDP}\label{section:twosample_conti}
In this section, we switch gears to testing for equality between two multivariate densities under the LDP constraint. To this end, we consider two classes of smooth densities, namely the H\"{o}lder ball and Besov ball, and establish the minimax rate in terms of the $\mathbb{L}_2$ distance for each class. Especially, we derive the upper bound for the minimax rate by building on the multinomial permutation test introduced in Section~\ref{subsection:twosample_disc_upperbound} with a careful discretization scheme. We also introduce an aggregated test, which is adaptive to the unknown smoothness parameter. 

\subsection{H\"{o}lder and Besov Smootheness Classes}\label{subsection:smoothness_class}
We start by formally defining the H\"{o}lder ball and Besov ball.
The H\"{o}lder ball generalizes  Lipschitz continuity and can be thought of as functions with bounded fractional derivatives.
The following definition of the H\"{o}lder ball rephrases the one stated in Section 2.1 of~\citet{Arias-Castro2018RememberDimension}.
%%%%%%%%%%%%%%%%%%%%%%%%%%%%%%%%%%
\begin{definition}[H\"{o}lder ball] \label{def:holder_ball}
	The H\"{o}lder ball with smoothness parameter \( \smoothness > 0 \) and radius \( R > 0 \), denoted as $\holderBall$, is the class of functions \( f : [0,1]^d \mapsto \mathbb{R} \) satisfying the following conditions:
	\begin{enumerate}
		\item  
		$
		\left| f^{(\lfloor \smoothness \rfloor)}(\vectorize{x}) - f^{(\lfloor \smoothness \rfloor)}(\vectorize{x}') \right| \leq R \, \| \vectorize{x} - \vectorize{x}' \|^{\smoothness - \lfloor \smoothness \rfloor}, \quad \text{for all } \vectorize{x}, \vectorize{x}' \in [0,1]^\dimDensity.
		$
		%%%%%%
		\item  
		$\vert\kern-0.25ex\vert\kern-0.25ex\vert f^{(s')} 
		\vert\kern-0.25ex \vert\kern-0.25ex \vert_\infty \leq R,
		\quad$for each $s' \in \bigl[ \lfloor \smoothness \rfloor \bigr]$,
	\end{enumerate}
	where $f^{(\lfloor \smoothness \rfloor)}$ denotes the $\lfloor \smoothness \rfloor$-order derivative of $f$.
\end{definition}
The Besov ball, on the other hand, measures the smoothness of a function by capturing its abrupt oscillations through wavelets.
In this respect, it can address spatially inhomogeneous functions whose smoothness can vary substantially across their domain.
To elaborate, we consider an orthonormal wavelet basis of $\EllTwo(\domainTs)$ at a fixed prime resolution level $\primResLev \in \numbersetNonnegInt$.
We denote this basis as $\multivInhomoWavFatherBasis
\cup (\cup_{\resLev \geq \primResLev}
\WavMotherBasisSet_{\resLev})$, which are classified into two distinct categories.
For $\wavFatherFunc \in \multivInhomoWavFatherBasis$, the scaling coefficient 
$\wavGenericFatherCoef(f)
:=
\int_{[0,1]^d}
f(\vectorize{x})
\wavFatherFunc(\vectorize{x})
\;
d\vectorize{x}$
detects an overall trend of $f$.
%, and includes more subtle changes as $\primResLev$ increases.
On the other hand, for 
$\wavMotherFunc \in \WavMotherBasisSet_{\resLev}$ with its resolution level $\resLev \geq \primResLev$, the wavelet coefficient
$\wavGenericMotherCoef(f)
:=
\int_{[0,1]^d}
f(\vectorize{x})
\wavMotherFunc(\vectorize{x})d\vectorize{x}$
captures abrupt oscillations from the general trend.
As $\primResLev$ and $\resLev$ increase, the coefficients capture more detailed behaviors.
Among many existing types of wavelet basis, this paper focuses on Haar multivariate wavelet basis.
It is useful because projecting densities onto a subset of this basis is equivalent to applying equal-sized binning, as used in our test proposed in Section \ref{subsection:twosample_conti_upperbound}. Consequently, the discretization error inherent in the test can be characterized by the corresponding wavelet coefficients.
This basis is constructed by taking tensor products of many rescaled and shifted versions of two basic univariate functions~(see Appendix~\ref{appendix:basis} for the details).

The Besov ball can be defined through the magnitudes of wavelet coefficients. The following definition paraphrases the one presented in Section 3 of \citet{tang2023Besov}.

\begin{definition}[Besov seminorm and Besov ball]\label{def:besov_norm}
	Fix a smoothness parameter $\smoothness>0$, a microscopic parameter  $1 \leq \besovParamMicroscope \leq \infty$, and a wavelet basis $
	\WavFatherBasisSet_{0}
	\cup
	(
	\bigcup_{\resLev \geq 0}
	\WavMotherBasisSet_{\resLev} 
	)
	$.
	The Besov seminorm of $f \in \mathbb{L}_2(\domainTs)$ is defined using the sequences of its wavelet coefficients
	$\bigl(\wavGenericMotherCoef(f)\bigr)_{\wavMotherFunc \in \WavMotherBasisSet_{\resLev}}$ as follows:
	\begin{align*}
		\|f\|_{\smoothness,2, \besovParamMicroscope}
		:=
		\begin{cases}
			\biggl[
			\sum_{\resLev=0}^\infty
			2^{\resLev \smoothness \besovParamMicroscope}
			\bigl(
			\sum_{\wavMotherFunc \in \WavMotherBasisSet_{\resLev}}
			|
			\wavGenericMotherCoef(f)
			|^2
			\bigr)^{\frac{\besovParamMicroscope}{2}}
			\biggr]^{\frac{1}{\besovParamMicroscope}},
			\quad
			&1 \leq \besovParamMicroscope < \infty,
			\\
			%
			% if q = infty
			\sup_{j \in \mathbb{N}_0}
			2^{ \resLev \smoothness}
			\bigl(
			\sum_{\wavMotherFunc \in \WavMotherBasisSet_{\resLev}}
			|
			\wavGenericMotherCoef(f)
			|^2
			\bigr)^{\frac{1}{2}},
			\quad
			&\besovParamMicroscope = \infty.
		\end{cases}
	\end{align*}
	For a radius $\ballRadius > 0$, we define the Besov ball  $\besovBall{2}{\besovParamMicroscope} $ as
	\begin{equation*}
		\besovBall{2}{\besovParamMicroscope} :=
		\{f \in \mathbb{L}_2(\domainTs) : \|f\|_{\smoothness,2, \besovParamMicroscope} \leq \ballRadius\}.
	\end{equation*}
\end{definition}
Neither Definition~\ref{def:holder_ball} nor Definition~\ref{def:besov_norm} is restricted to density functions. Instead, we construct our models by defining classes of distribution pairs where the differences in their density functions  lie within these smooth function classes.
\begin{definition}[Smooth distribution pair classes]\label{def:smooth_distribution_class}
	Let $\pHolderTs$ denote the set of pairs of distributions $(P_{\vectorize{Y}},P_{\vectorize{Z}})$ that satisfy the following conditions:
	\begin{enumerate}
		\item Both $P_{\vectorize{Y}}$ and $P_{\vectorize{Z}}$ have densities $f_{\vectorize{Y}}$ and $f_{\vectorize{Z}}$, respectively, with their $\EllInfty$ norms bounded by $\ballRadius$.
		\item The difference of these densities $(f_{\vectorize{Y}}- f_{\vectorize{Z}})$ lies in $\holderBall$. 
	\end{enumerate}
	Similarly, let $\pBesovTs$ denote the set of pairs of distributions that satisfy the two conditions above, with $\holderBall$ replaced by $\besovBall{2}{q}$.  
\end{definition}
The superscript 2 in  $\pHolderTs$  and $\pBesovTs$ indicates that these sets consist of pairs of distributions, distinguishing them from the sets of single distributions used in Appendix~\ref{proof:twosample_conti_lower_bound}.
For $\pBesovTs$, we extend the analysis of \citet{Lam-Weil2021MinimaxConstraint} into a multivariate setting, focusing on the Besov ball defined using a multivariate Haar wavelet basis and $\smoothness < 1$.
Details of the basis functions are provided in Appendix~\ref{appendix:basis}.
Notably, there is no restriction on the microscopic parameter $q$.
\subsection{Private Minimax Testing Rates for Two-Sample Density Testing}\label{subsection:twosample_conti_rates}
Building on the smooth distribution classes defined in Definition~\ref{def:smooth_distribution_class}, we formally define the two density testing problems of interest and present the minimax testing rate applicable to both. In this subsection, we assume that the smoothness parameter $s$ for H\"{o}lder ball and Besov ball is known, addressing  the case of unknown $s$ in Section~\ref{subsection:adaptive}. %

Assume that the data-generating distributions $(P_{\vectorize{Y}}, P_{\vectorize{Z}})$ is contained in a class $\distClassGeneric$.
The curator uses two sets of $\privacyParameter$-LDP views $\sampleSets{\tilde{\vectorize{\rvTwo}}}{\sampleIndexOne}{\sampleSize_1}$ and $\sampleSets{\rVecZPriv}{\sampleIndexTwo}{\sampleSize_2}$, privatized as described in Section~\ref{subsection:LDPtwosampleFormulation},  to decide whether $(P_{\vectorize{Y}}, P_{\vectorize{Z}})$ came from
\begin{equation}\label{hypothesis:twosample_Holder_alt} 
	\distClassGeneric_0
	:=
	\{
	(P_{\vectorize{Y}}, P_{\vectorize{Z}})
	\in
	\distClassGeneric
	:
	f_{\vectorize{Y}}
	= f_{\vectorize{Z}} \}~\text{ or }
	~
	\distClassGeneric_1
	(
	\rho_{n_1,n_2}
	)
	:=
	\bigl\{
	(P_{\vectorize{Y}}, P_{\vectorize{Z}}) \in \distClassGeneric:
	\normEllp{f_{\vectorize{Y}} - f_{\vectorize{Z}}}{2}{} \geq \rhoTwosample
	\bigr\},
\end{equation}
where $f_{\vectorize{Y}}$ and $f_{\vectorize{Z}}$ are densities of $P_{\vectorize{Y}}$ and $P_{\vectorize{Z}}$, respectively.
We consider the  problems of $\distClassGeneric=\pHolderTs$ and $\distClassGeneric=\pBesovTs$.
Notably, while we focus on the class of either $\distClassGeneric$ being H\"{o}lder or Besov smooth  distributions, the permutation-based test presented in Section~\ref{subsection:twosample_conti_upperbound} guarantees type I error control over a broader class of null distributions beyond those defined over H\"{o}lder or Besov balls.

Let us fix the type I error $\maxErrorTypeOne \in (0,1)$ and the type II error $\maxErrorTypeTwo \in (0,1)$ such that $2\maxErrorTypeOne + \maxErrorTypeTwo < 1$ as in Section~\ref{subsection:twosample_multinomial_rates},
and
assume further that $\sampleSize_1 \privacyParameter^2 \geq 1$,
similarly considered in \cite{Lam-Weil2021MinimaxConstraint}.
The main result of this section, stated below, establishes a lower bound as well as an upper bound for the minimax separation rate for two-sample density testing under the LDP constraint.

\begin{theorem}[Minimax rates for two-sample density testing under LDP] \label{theorem:twosample_conti_rate}
	Assume $\sampleSize_1 \privacyParameter^2 \geq 1$.
	For the testing problem of distinguishing between the null and alternatives as in~\eqref{hypothesis:twosample_Holder_alt}, where the distribution class $\distClassGeneric$ is $\pHolderTs$,
	there exist positive constants
	$C_\ell(  \gamma, \beta, \ballRadius, \smoothness) \equiv C_\ell$
	and
	$C_u(  \gamma, \beta, \ballRadius, \smoothness, \dimDensity) \equiv C_u$
	such that the $\privacyParameter$-LDP minimax testing rate 
	$\rho_{\sampleSize_1, \sampleSize_2, \privacyParameter}^\ast$ is bounded as
	\begin{equation}\label{rates:twosample_conti}
		C_\ell \left[
		\frac{(\sampleSize_1z_\privacyParameter^2)^{\frac{-2s}{4\smoothness+3\dimDensity}}}{\sqrt{\log(\sampleSize_1z_\privacyParameter^2)}}
		\vee
		\sampleSize_1^{\frac{-2s}{4\smoothness+\dimDensity}}
		\right]
		\leq
		\rho_{\sampleSize_1, \sampleSize_2, \privacyParameter}^\ast
		\leq
		%
		%upper bound
		C_u
		\bigl[
		(\sampleSize_1\privacyParameter^2)^{\frac{-2s}{4\smoothness+3\dimDensity}} 
		\vee
		\sampleSize_1^{\frac{-2s}{4\smoothness+\dimDensity}}
		\bigr],
	\end{equation}
	where we recall $z_\alpha = 2 \mathrm{sinh}(2\alpha)$.
	Similarly, for the testing problem of distinguishing between the null and alternatives as in~\eqref{hypothesis:twosample_Holder_alt}, where the distribution class $\distClassGeneric$ is $\pBesovTs$,
	the minimax testing is also bounded as~\eqref{rates:twosample_conti}.
	
\end{theorem}
\noindent
Theorem~\ref{theorem:twosample_conti_rate} indicates that the private minimax separation rate for two-sample multivariate H\"{o}lder and Besov densities is noticeably different from its non-private counterpart $\sampleSize_1^{-2s/(4\smoothness+\dimDensity)}$ with respect to the $\EllTwo$ distance~\citep{Arias-Castro2018RememberDimension, kim_minimax_2022}.
We point out that in the high privacy regime, a polynomial degradation on the minimax rate is observed,
and this degradation becomes worse as the data dimension $\dimDensity$ increases.
The result also implies the privacy guarantee can be secured at no additional charge in the low privacy regime where $\sampleSize_1^{-2s/(4\smoothness+\dimDensity)}$ dominates the other term.

The bounds~\eqref{rates:twosample_conti} are tight up to a logarithmic factor in the denominator of the lower bound,
which can be omitted when $n_1 z_\privacyParameter^2 \geq 1$.
As already noted in Theorem~\ref{theorem:twosample_multinomial_rates},
$\privacyParameter$ and $z_\privacyParameter$ are the same, up to a constant factor, as long as $\alpha$ is bounded. Hence, in most practical scenarios where a small value of $\alpha$ is of interest, the upper bound matches the lower bound.

We now discuss the proof of Theorem~\ref{theorem:twosample_conti_rate}, and a detailed analysis can be found in Appendix~\ref{proof:twosample_conti_rates}.
\begin{itemize}
	\item \textbf{Lower bound.} 
	To obtain the lower bound, we once again use the observation that two-sample testing is harder than goodness-of-fit testing. Based on this observation, we employ the lower bound result for goodness-of-fit testing, mirroring the approach employed in Theorem~\ref{theorem:twosample_multinomial_rates} for multinomial testing. We then extend the strategy presented in \citet{Lam-Weil2021MinimaxConstraint} from the univariate case to the multivariate case. The same proof strategy also applies to the multivariate H\"{o}lder ball with only minor modifications, and details can be found in Appendix~\ref{proof:twosample_conti_lower_bound}.
	\item \textbf{Upper bound.} To prove the upper bound, we leverage the private multinomial permutation test in Section~\ref{subsection:twosample_disc_upperbound}, which achieves the private minimax rate.
	For that purpose, we divide each side of the support $[0,1]^\dimDensity$ into $\binNum$ equally-sized subintervals, effectively transforming the initially continuous observations into multinomial observations. The detailed procedure is outlined in Section~\ref{subsection:twosample_conti_upperbound}. We then apply the same privacy mechanism and permutation testing procedure as outlined in Section~\ref{subsection:twosample_disc_upperbound}.
	We exploit the H\"{o}lder and Besov smoothness conditions to find the optimal number of bins $\binNum$ that effectively controls the discretization error and thus leads to a tight upper bound.
\end{itemize}

\subsection{Minimax Optimal Privacy Mechanism and Permutation Testing Procedure}\label{subsection:twosample_conti_upperbound}
%last revised: october 2024
Our proposed density test applies our proposed private multinomial test to data discretized by binning the density support into a certain sample-size dependent number of bins, defined as follows: 
\begin{align*}
	\binNum 
	:=
	\begin{cases}
		\lfloor
		(\sampleSize_1^{2/(4\smoothness+\dimDensity)})
		\wedge
		(\sampleSize_1 \privacyParameter^2)^{2/(4\smoothness+3\dimDensity)}
		\rfloor
		& \text{if}~(P_{\vectorize{\rvTwo}}, P_{\vectorize{\rvThree}}) \in \pHolderTs,
		\\
		\sup 
		\bigl\{
		2^\primResLev
		: \primResLev \in \mathbb{N}_0
		\quad \text{and} \quad
		2^\primResLev
		\leq
		(\sampleSize_1^{2/(4\smoothness+\dimDensity)})
		\wedge
		(\sampleSize_1 \privacyParameter^2)^{2/(4\smoothness+3\dimDensity)}
		\bigr\}
		& \text{if}~(P_{\vectorize{\rvTwo}}, P_{\vectorize{\rvThree}}) \in \pBesovTs.
	\end{cases}
	\numberthis \label{eq:kappaValueTwosample}
\end{align*}
Let $\{B_1,...,B_{ \binNum^{\dimDensity}} \}$ be an enumeration of $\dimDensity$-dimensional hypercubes whose length is set to $1/\binNum$.
Each data owner bins their raw sample using the equal-sized binning function $\binner_{\binNum} : \domainTs \mapsto [\binNum^\dimDensity]$, such that $\binner_{\binNum}(\vectorize{x}) = \vectorIndex$ if and only if $\vectorize{x} \in B_\vectorIndex$.
Then, based on the discretized data $\{\binner_{\binNum}(\mathbf{Y}_1), \dots, \binner_{\binNum}(\mathbf{Y}_{n_1})\}$ and $\{\binner_{\binNum}(\mathbf{Z}_1), \dots, \binner_{\binNum}(\mathbf{Z}_{n_2})\}$, we carry out our proposed private multinomial test in Section~\ref{subsection:twosample_disc_upperbound}; the permutation test with test statistic~\eqref{def:statistic_elltwo}. The resulting test achieves the tight upper bound in \eqref{rates:twosample_conti}.

We now outline how this test maintains the privacy and testing error guarantees.
Since the discretized data remain i.i.d.~under the null, the permutation procedure maintains the type I error at $\gamma$.
As for the privacy guarantee, the reduced distinguishability between samples due to discretization, combined with any of $\{\texttt{LapU},\texttt{DiscLapU},\texttt{RAPPOR}\}$, guarantees $\privacyParameter$-LDP.
Regarding the type II error, first note that the test compares two multinomial distributions with $\binNum^\dimDensity$ categories. The corresponding probability vectors, $\probVec_{\mathbf{\rvTwo}}$ and $\probVec_{\mathbf{\rvThree}}$, are defined as 
$\probVec_{\mathbf{\rvTwo}}(\vectorIndex) :=
\int_{B_\vectorIndex}
f_{\vectorize{Y}}(\vectorize{t})
\;
d\vectorize{t}$
and
$\probVec_{\mathbf{\rvThree}}(\vectorIndex)
:=
\int_{B_\vectorIndex} 
f_{\mathbf{\rvThree}}
(\vectorize{t})
\;
d\vectorize{t}$
for $\vectorIndex \in [\binNum^\dimDensity]$.
The boldface subscripts in $\probVec_{\mathbf{\rvTwo}}$ and $\probVec_{\mathbf{\rvThree}}$ indicate that these vectors correspond to multivariate data, distinguishing them from $\mathbf{p}_Y$ and $\mathbf{p}_Z$ in Section~\ref{section:twosample_disc}. Thus the test statistic~\eqref{def:statistic_elltwo} estimates
$
\binNum^{\dimDensity}
\|
\probVec_{\vectorize{\rvTwo}} - \probVec_{\vectorize{\rvThree}}
\|_2^2$ instead of $\normEllp{ f_{\vectorize{Y}} - f_{\vectorize{Z}} }{2}{2}$, introducing an approximation error that depends on the number of bins $\binNum$.  This error is controlled using smoothness conditions on densities and techniques from \citet{Arias-Castro2018RememberDimension} and \citet{Lam-Weil2021MinimaxConstraint}~(see Appendix~\ref{appendix:disc_error} for details). We show that when the number of bins $\binNum$ is chosen as in \eqref{eq:kappaValueTwosample} and the $\EllTwo$ distance between $f_{\vectorize{Y}}$ and $f_{\vectorize{Z}}$ exceeds the threshold in \eqref{rates:twosample_conti}, the multinomial test applied to the binned data controls the type II error to be at most $\beta$. A detailed proof is provided in Appendix~\ref{proof_theorem:twosample_conti_rate}. 

We note that our choice of $\binNum$ in~\eqref{eq:kappaValueTwosample} depends on the smoothness parameter $\smoothness$, typically unknown. In the next subsection, we introduce an adaptive procedure that accommodates this unknown $\smoothness$ without significant loss of power.

\subsection{Adaptive Procedure for Private Two-Sample Density Testing }\label{subsection:adaptive}%last revised: october 2024
To achieve the tight upper bound in Theorem~\ref{theorem:twosample_conti_rate} via our proposed test in Section~\ref{subsection:twosample_conti_upperbound}, the number of bins must be determined based on the unknown smoothness parameter~$\smoothness$. To circumvent this requirement, this section introduces a multiscale permutation testing procedure that  adapts to the unknown $\smoothness$. 
Following \citet{Ingster2000AdaptiveTests}, it aggregates test results from different bin numbers using a Bonferroni-corrected significance level. This procedure does not significantly sacrifice power compared to the one relying on the true value of $s$.

Fix the privacy parameter $\privacyParameter>0$. Denote the number of the test for the Bonferroni-type procedure as:
\begin{equation}\label{def:n_test_adaptive}
	\mathcal{N}
	:=
	\left\lceil \bigg\{
	\dfrac{2}{\dimDensity}
	\log_2{
		\left(
		\dfrac
		{\sampleSize_1}
		{\log \log \sampleSize_1}
		\right) \bigg\}
	}
	\wedge
	\bigg\{\dfrac{2}{3\dimDensity}
	\log_2{
		\left(
		\dfrac
		{\sampleSize_1 \privacyParameter^2}
		{(\log \sampleSize_1)^2\log \log \sampleSize_1}
		\right) \bigg\}
	}	
	\right\rceil.
\end{equation}
For each ${\adaptiveBinNumIndex} \in [\nTest],$
let $\adaptiveSingleTest{\adaptiveBinNumIndex}_{\maxErrorTypeOne/\nTest}$ denote the test function of our proposed method in Section~\ref{subsection:twosample_conti_upperbound},
using $2^{\adaptiveBinNumIndex}$  bins, significance level $\gamma/\nTest$,
and $( \privacyParameter/\nTest)$-LDP guarantee.
The $\privacyParameter$-LDP adaptive test is formally defined as follows:
\begin{align*}
	\Delta^{\text{adapt}}_{\maxErrorTypeOne}
	:=
	\max_{{\adaptiveBinNumIndex} \in [\nTest]} \adaptiveSingleTest{\adaptiveBinNumIndex}_{\maxErrorTypeOne/\nTest}.
	\numberthis \label{adaptive_test_def}
\end{align*}
\noindent%\textbf{Privacy guarantee.}
This adaptive procedure queries $\nTest$ number of $( \privacyParameter/\nTest)$-LDP views per observation, each using a different number of bins for discretization.
By the composition theorem of differential privacy~\citep{mcsherry_mechanism_2007},  releasing $\nTest$ number of $(\privacyParameter/\nTest)$-LDP views satisfies the LDP constraint with a privacy level of $\nTest \times \privacyParameter/\nTest  = \privacyParameter$. The type I error is at most $\gamma$ by the union bound, and for the type II error, Theorem~\ref{theorem:twosample_adaptive_upper} states that the adaptive procedure achieves the same testing rate up to logarithmic factors.
\begin{theorem}[Minimax upper bound for the adaptive private testing procedure]\label{theorem:twosample_adaptive_upper}
	For the problems and conditions stated in Theorem~\ref{theorem:twosample_conti_rate},
	further assume that $\sampleSize_1 \asymp \sampleSize_2$,
	$\maxErrorTypeOne \leq e^{-1},
	\sampleSize_1 > e^e$ and
	$\privacyParameter \leq n_1 \leq \sampleSize_2$.
	Then there exists a positive constant
	$C_u \equiv C(\smoothness, \dimDensity, \ballRadius, \gamma, \beta)$ such that the condition
	\begin{align*}
		\rhoTwosample
		\geq
		C_u
		\left[
		% non-private term
		\left(
		\dfrac
		{n}
		{\log \log \sampleSize_1 }
		\right)
		^{
			\frac{-2s}{4\smoothness+\dimDensity}
		}
		\vee
		\left(
		\dfrac
		{\sampleSize_1 \privacyParameter^2}
		{
			(\log^2{\sampleSize_1})
			\log \log \sampleSize_1 
		}
		\right)^{
			\frac{-2s}{4\smoothness+3\dimDensity}
		}
		\right]
	\end{align*} 
	implies that the testing errors of the adaptive test $\Delta^{\mathrm{adapt}}_{\maxErrorTypeOne}$ in~\eqref{adaptive_test_def} are uniformly bounded as in~\eqref{eq:uniform_control}.
\end{theorem}
\noindent
Comparing Theorems~\ref{theorem:twosample_adaptive_upper} and~\ref{theorem:twosample_conti_rate},
in the high privacy regime,
we find that the adaptive procedure incurs an additional cost of $\bigl( (\log^2{\sampleSize_1})  \log{\log{\sampleSize_1}} \bigr) ^{2s/(4s + 3\dimDensity)}$ . In the low privacy regime, the additional cost $(\log\log \sampleSize_1)^{2s/(4s + \dimDensity)}$ matches the adaptivity cost for non-private testing rates found in \citet{Fromont2006AdaptiveModel} and \citet{kim_minimax_2022}. Whether these additional logarithmic factors are necessary or can be improved upon remains an open question for future work.
We briefly discuss the proof of Theorem~\ref{theorem:twosample_adaptive_upper}; a detailed analysis is provided in Appendix~\ref{proof:twosample_adaptive}.
\begin{itemize}
	\item \textbf{Type I error control.} 
	The overall type I error is at most $\gamma$ by the union bound, though this approach usually is conservative in practice.   Introducing an additional layer of calibration, as in \citet{Schrab2021MMDTest}, could mitigate this issue but would increase the noise level of privacy mechanisms. Developing an adaptive test with precise type I error control and robust power guarantee is an interesting avenue for future research.
	
	\item \textbf{Type II error control.} 
	Since  the significance level now depends on $\dimDensity$ and $\sampleSize_1$ via $\mathcal{N}$, we use a refined version of the two moments method presented in \citet{kim_minimax_2022}. 
	It improves the dependence on the significance level $\gamma$
	from $\sqrt{1/{\gamma}}$ of Theorem~\ref{theorem:twosampleTwoMomentsMethod} to $\log(1/\gamma)$,
	at the cost of an additional requirement of $\sampleSize_1 \asymp \sampleSize_2$ on the sample sizes.
	At the heart of this refinement is an exponential concentration inequality for permuted U-statistics. This technique allows us to improve the adaptivity result of \citet{Lam-Weil2021MinimaxConstraint} in their Theorem 5.2, replacing logarithmic factors with iterated logarithmic factors. 
\end{itemize}

%Indeed, in the setting of , they end up using a simple upper bound on the variance of the U-statistic and bears the cost of  $\log^{5/2}n$ factor in the high privacy condition and $\sqrt{\log n}$ factor in the low privacy condition.

%%%%%%%%%%%%%%%%%%%%%%%%%%%%%%%%%%%%%%%%%%%%%%%%%%%%%%%%%%%%%%%%%%
%%%%%%%%%%%%%%%%%%%%%%%%%%%%%%%%%%%%%%%%%%%%%%%%%%%%%%%%%%%%%%%%%%

\section{Numerical Results}\label{section:simulation}
In this section, we conduct a series of simulation studies to illustrate the finite sample performance of our proposed tests. Specifically, we investigate the privacy-utility trade-offs by varying the privacy level parameter $\privacyParameter$. These simulation studies aim to confirm our theoretical results regarding these trade-offs, and also to demonstrate the rate at which the power diminishes as the privacy parameter $\privacyParameter$ decreases in practical scenarios.

It is worth pointing out that there is no baseline method available in the literature for the problem we tackle.
Therefore, we create the baseline methods by extending the LDP goodness-of-fit tests~\citep{Gaboardi2018LDPChisq} into a two-sample setting.
The first method combines the generalized randomized response  \citep[\texttt{GenRR};][]{Gaboardi2018LDPChisq} privacy mechanism and the classical chi-square statistic~(\texttt{Chi}).
The second method combines \texttt{RAPPOR} privacy mechanism and projected chi-square statistic \citep[\texttt{ProjChi};][]{Gaboardi2018LDPChisq}.
The third method combines the generalized randomized response mechanism with the $\ell_2$-type U-statistic in~\eqref{def:statistic_elltwo}.
A formal description and asymptotic properties of these extensions can be found in Appendix~\ref{appendix:baseline}.
From now on, we refer to the LDP two-sample testing methods as ``privacy mechanism+test statistic'', for example, \texttt{RAPPOR}+\texttt{ProjChi} or \texttt{GenRR}+\texttt{Chi} for ease of reference.

Recall that the proposed method for density testing, defined as a multinomial test applied to equal-sized binned data, requires the original data to lie within the unit hypercube $[0,1]^\dimDensity$. In order to apply the multinomial test to continuous data with larger and potentially unbounded domains, we transform the data through a map \(T : \mathbb{R}^d \mapsto [0,1]^d\), which is applied on a component-wise basis. A specific transformation that we focus on in this simulation is given as:
\begin{equation}\label{def:transform_copula}
	T(\mathbf{x}) = \bigl(\Phi(x_1), \Phi(x_2), \ldots, \Phi(x_d)\bigr),   
\end{equation}
where \(\Phi(x) := (2\pi)^{-1/2} \int_{-\infty}^{x} e^{-t^2/2} \, dt\) is the standard normal cumulative distribution function. We then apply our procedure in Section~\ref{subsection:twosample_conti_upperbound} to these transformed observations through the map in~\eqref{def:transform_copula}. 
Recall that the minimax-optimal value of \( \binNum \) in~\eqref{eq:kappaValueTwosample} depends on the unknown smoothness parameter \( \smoothness \). Since this value is not directly usable, practitioners must select \( \binNum \) to balance two competing effects: a large $\binNum$ reduces discretization error, but potentially weakens the signal and becomes vulnerable to the impact of added noise due to the privacy mechanism. Our simulation results indicate that $\binNum=4$ is a reasonable choice in most of the scenarios considered in this section, and thus we stick with the equal-sized binning scheme with $\binNum=4$ for density testing. 

In all simulation scenarios, we consider equal sample sizes, denoted as \(\sampleSize := \sampleSize_1 = \sampleSize_2\), and fix the significance level at \(\maxErrorTypeOne = 0.05\). We estimate the power by independently repeating the test 2000 times, and calculating the rejection ratio of the null hypothesis. The permutation procedure employs the Monte Carlo $p$-value given in \eqref{eq:monte_carlo_perm_p_value} with \(\sizeMonteCarlo = 999\). The code for replicating the numerical results is available at \url{https://github.com/Jong-Min-Moon/optimal-local-dp-two-sample.git}.

\begin{figure}[t!]
	\centering
	\includegraphics[width=0.85\linewidth]{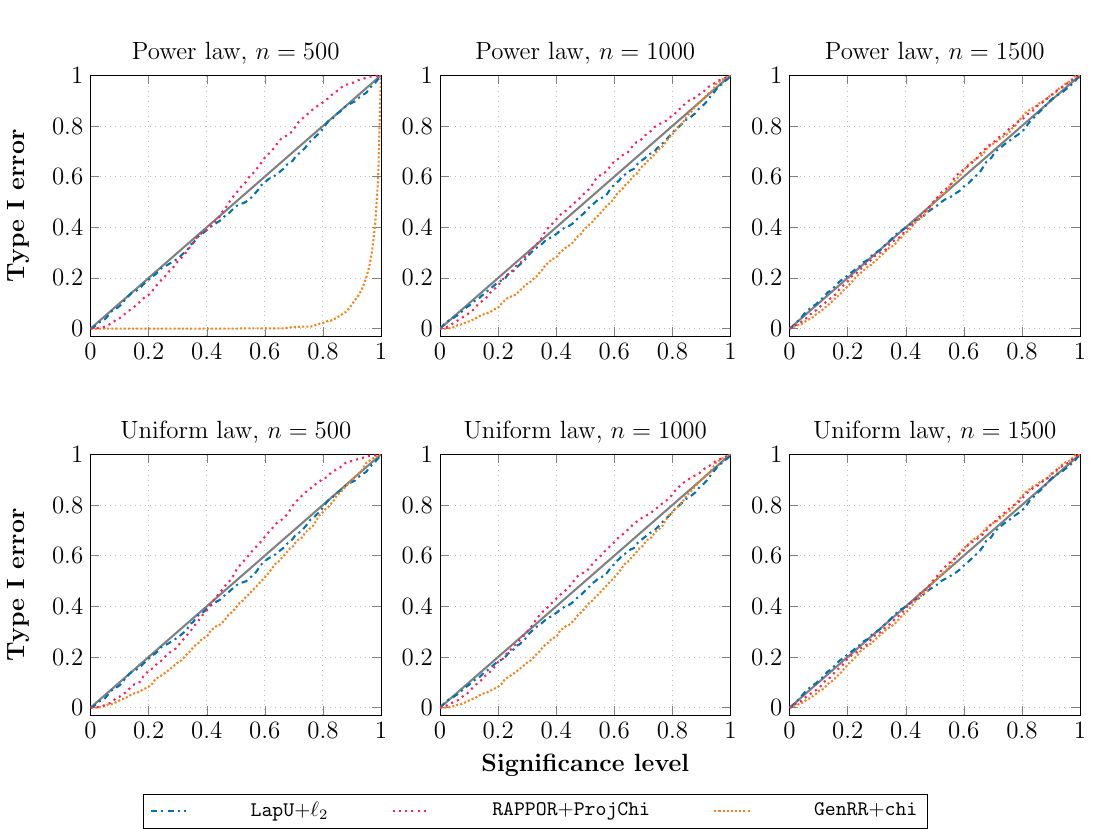}
	\caption{Comparison of type I error control between a permutation-calibrated method~(\texttt{LapU}+$\ell_2$) and methods calibrated through asymptotic chi-square null distribution (\texttt{RAPPOR}+\texttt{ProjChi} and \texttt{GenRR}+\texttt{Chi}),
		across configurations of uniform (top row) and power-law (bottom row) null distributions as described in~\eqref{nulldist}. The gray solid lines represent the $y=x$ line, indicating perfect type I error control.
	}
	\label{fig:typeone}
\end{figure}

\vskip 1em

\noindent \textit{Type I error control.}
First, we compare the type I error rates of three methods for multinomial testing: \texttt{LapU}+\texttt{$\ell_2$}, \texttt{GenRR}+\texttt{Chi}, and \texttt{RAPPOR}+\texttt{ProjChi}, to highlight the advantages of the permutation approach. We consider two null distributions, where for \( m \in [\alphabetSize] \), the \( m \)th elements of \( \vectorize{p}_Y \) and \( \vectorize{p}_Z \) are defined as:
\begin{equation}\label{nulldist}
	\text{(a) power law: }\probVecElement{\rvTwo}{\vectorIndex} = \probVecElement{\rvThree}{\vectorIndex} \propto 1/\vectorIndex \quad \text{and} \quad \text{(b) uniform law: }\probVecElement{\rvTwo}{\vectorIndex} = \probVecElement{\rvThree}{\vectorIndex}=1/\alphabetSize.
\end{equation}
We set \(\alphabetSize = 500\) and \(\privacyParameter = 0.1\), and investigate how the type I error rate varies with sample sizes $n \in \{500, 1000, 1500\}$. The results, displayed in Figure~\ref{fig:typeone}, indicate that the type I error of the permutation test, \texttt{LapU}+\texttt{$\ell_2$}, is well controlled at any sample size and significance level (up to a small numerical error) as expected. In contrast, we see that the asymptotic tests, namely \texttt{RAPPOR}+\texttt{ProjChi} and \texttt{GenRR}+\texttt{Chi}, have the size significantly deviated from the straight baseline, especially when the sample size is small or moderate. This indicates that the resulting asymptotic test can be either conservative or anti-conservative depending on the significance level.
\begin{figure}[t!]
	\centering
	\includegraphics[width=0.95\linewidth]{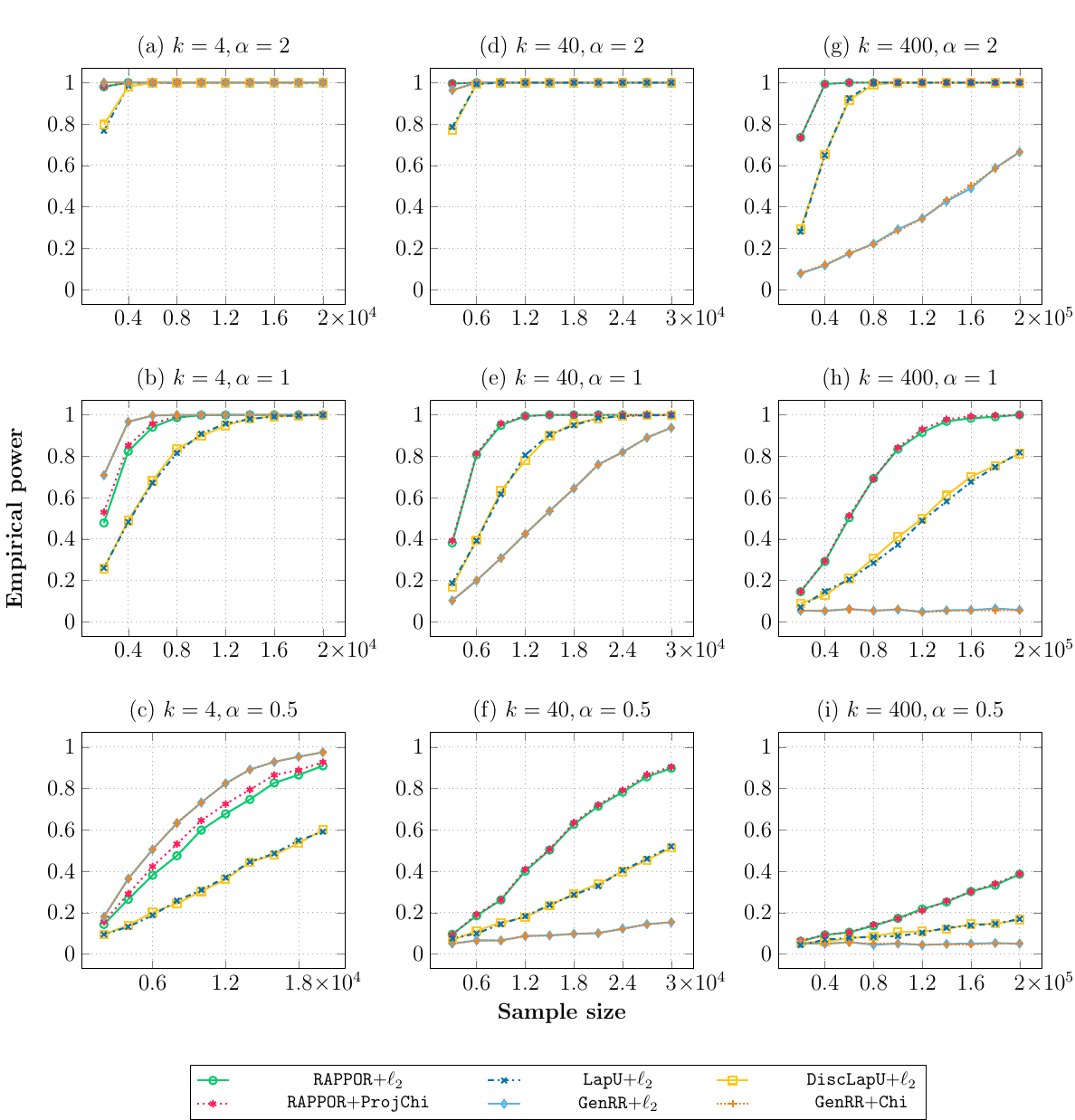}
	\caption{
		Comparison of the testing power between our proposed methods (first row in the legend) and baseline methods (second row in the legend) under the perturbed uniform alternatives~\eqref{simulation_setting:perturbed_uniform}. To ensure a fair comparison, all methods are calibrated using permutation procedures at level $\gamma = 0.05$.
	}
	\label{fig:power_multinomial}
\end{figure}
\vskip 1em

\noindent \textit{Simulation settings for power comparison in multinomial testing.}
We next compare the power of our proposed methods (\texttt{RAPPOR}+$\ell_2$, \texttt{LapU}+$\ell_2$, and \texttt{DiscLapU}+$\ell_2$) with baseline methods (\texttt{GenRR}+$\texttt{Chi}$, \texttt{RAPPOR}+\texttt{ProjChi}, and \texttt{GenRR}+$\ell_2$) for distinguishing between two multinomial distributions. As observed in Figure~\ref{fig:typeone}, both \texttt{GenRR}+\texttt{Chi} and \texttt{RAPPOR}+\texttt{ProjChi} can be miscalibrated when their thresholds are determined by the asymptotic null distributions. In order to ensure a fair power comparison, \texttt{GenRR}+\texttt{Chi} and \texttt{RAPPOR}+\texttt{ProjChi} were calibrated using permutation procedures in this power simulation. We aim to illustrate how the testing power varies with changes in the number of categories  $\alphabetSize$ and the privacy parameter $\privacyParameter$. The analysis is conducted under a perturbed uniform distribution scenario, where for \( m \in [\alphabetSize] \), the \( m \)th elements of \( \vectorize{p}_Y \) and \( \vectorize{p}_Z \) are defined as:
\begin{equation}\label{simulation_setting:perturbed_uniform}
	\probVecElement{\rvTwo}{\vectorIndex} = \frac{1}{\alphabetSize} + (-1)^\vectorIndex \eta~\quad \text{and} \quad~	\probVecElement{\rvThree}{\vectorIndex} = \frac{1}{\alphabetSize} + (-1)^{(\vectorIndex+1)} \eta, \quad \text{for $\vectorIndex \in [\alphabetSize]$.}
\end{equation}
The simulation considers the following combinations of three parameters, namely the number of categories $\alphabetSize$, the perturbation size $\eta$ and the privacy parameter $\privacyParameter$: 
$$(\alphabetSize, \eta, \privacyParameter) \in \{(4,0.04), (40,0.015), (400,0.002)\} \times \{2, 1, 0.5\}.$$ The simulation results for this setting are provided in Figure~\ref{fig:power_multinomial}. 
\begin{figure}[t!]
	\centering
	\includegraphics[width=0.95\linewidth]{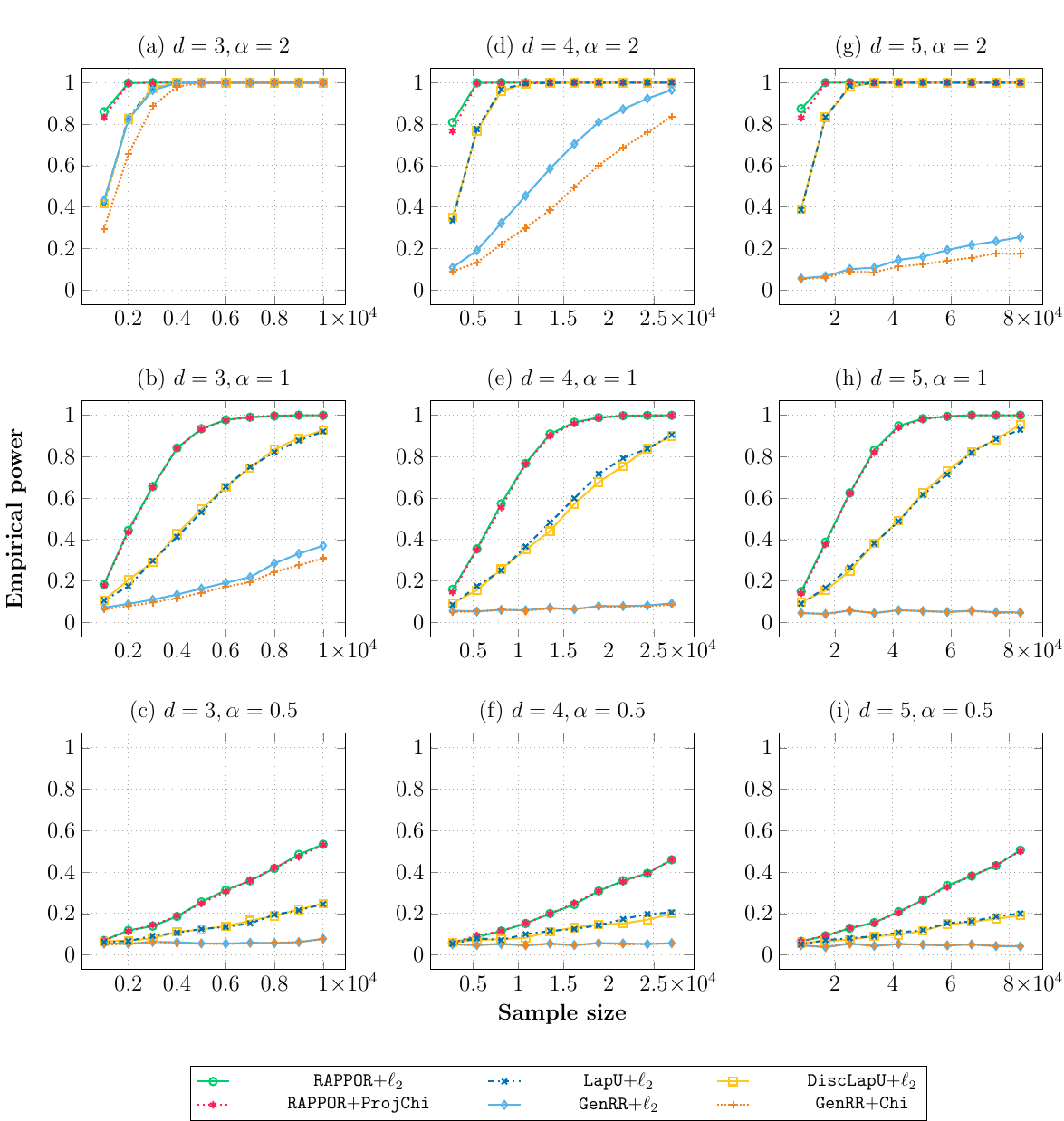}
	\caption{Comparison of the density testing power between our proposed methods (first row in the legend) and baseline methods (second row in the legend) under the location alternatives in~\eqref{alternative:density_location}. To ensure a fair comparison, all methods are calibrated using permutation procedures at level $\gamma = 0.05$.
	}\label{fig:power_density_location}
\end{figure}

\vskip 1em

\noindent \textit{Simulation settings for power comparison in density testing.}
We also evaluate the density testing power of the same methods used in the simulations for multinomial testing. We consider two scenarios where two density functions differ in their location parameters or scale parameters. Since the results for scale difference show trend similar to that of location difference, we present these results in Appendix~\ref{simul_scale}.
For the location difference, we analyze scenarios involving mean differences between two $\dimDensity$-dimensional Gaussian distributions
$P_{\vectorize{\rvTwo}} = \mathcal{N}(
\boldsymbol{\mu}_{\vectorize{\rvTwo}}
,
\Sigma_{\vectorize{\rvTwo}}
)$
and
$P_{\vectorize{\rvThree}} = \mathcal{N}(
\boldsymbol{\mu}_{\vectorize{\rvThree}}
,
\Sigma_{\vectorize{\rvThree}}
)$. 
Let 
$
\vecOne{\dimDensity}:= (1, \ldots , 1)^\top 
\in \numbersetReal^\dimDensity
$,
$
\vecZero{\dimDensity} := (0, \ldots , 0)^\top 
\in \numbersetReal^\dimDensity$, 
$\matOne{\dimDensity}
:=
\vecOne{\dimDensity} {\mathbf{1}_d}^\top
\in
\numbersetReal^{\dimDensity \times \dimDensity}
$, and $\mathbf{I}_d$  denote the identity matrix in
$\mathbb{R}^{d\times d}$.
We set the mean vectors and covariance matrices of the Gaussian distributions as:
\begin{equation}\label{alternative:density_location}
	\boldsymbol{\mu}_{\vectorize{\rvTwo}}=
	0.5 \times \mathbf{1}_{\dimDensity},~
	\boldsymbol{\mu}_{\vectorize{\rvThree}}=
	-\boldsymbol{\mu}_{\vectorize{\rvTwo}},
	\quad \text{and} \quad
	\Sigma_{\vectorize{\rvTwo}}
	=
	\Sigma_{\vectorize{\rvThree}}
	=
	0.5 \times\mathbf{J}_{\dimDensity}
	+
	0.5 \times \mathbf{I}_{\dimDensity}.
\end{equation}
In simulations, the dimension of the original data is chosen as $\dimDensity \in \{3,4,5\}$, and after binning through the map $T$ in \eqref{def:transform_copula} with $\kappa = 4$, the number of categories becomes $k \in \{64,256,1024\}$.

\vskip 1em

\noindent \textit{Simulation results for power comparisons.}
The simulation results in this section are illustrated in Figures~\ref{fig:power_multinomial} and \ref{fig:power_density_location}. We first consistently observe in all of 
the figures that the power tends to decrease as the privacy parameter $\privacyParameter$ decreases meaning a stronger privacy guarantee.
These trade-offs are all predictable from the minimax rates in~\eqref{rate:twosample_disc} and~\eqref{rates:twosample_conti}.
We next highlight the differences in trends related to the number of categories. For multinomial distributions with a small number of categories (\(\alphabetSize=4\)), the generalized randomized response, a natural extension of the classical mechanism \citep{randomizedresponse}, outperforms all other methods. Following the generalized randomized response mechanism is the \texttt{RAPPOR} mechanism, while Laplace-noise-based mechanisms performing the least. However, in scenarios with a larger number of categories, or in density testing scenarios (which also correspond to a large number of categories), the testing power of the generalized randomized response diminishes, and \texttt{RAPPOR} emerges as the method with the highest power. The suboptimal performance of the generalized randomized response in high-dimensional settings, as theoretically explored in Appendix~\ref{genrr_suboptimal_theory} and numerically observed by \citet{Gaboardi2018LDPChisq}, aligns with a simple intuition: since the generalized randomized response modifies data by shifting it from one category to another, the difference between the original sample and its corresponding $\privacyParameter$-LDP view becomes more pronounced as the number of categories increases. These simulation results prove the superiority of \texttt{RAPPOR} over the other mechanisms, especially in multinomial testing with large $k$ and density testing, and we therefore recommend using \texttt{RAPPOR} in practical applications. 
Within the tests built upon \texttt{RAPPOR}, we observe in Figures~\ref{fig:power_multinomial} and \ref{fig:power_density_location}  that \texttt{RAPPOR}+\texttt{ProjChi} and \texttt{RAPPOR}+$\ell_2$ perform comparably to each other, and in some cases, \texttt{RAPPOR}+\texttt{ProjChi} is slightly more powerful. 
It is also possible that the difference between \texttt{RAPPOR}+\texttt{ProjChi} and \texttt{RAPPOR}+$\ell_2$ is more pronounced in some settings. For example, if the signal is large in terms of chi-square divergence but relatively small in terms of the $\ell_2$ distance, we would expect \texttt{RAPPOR}+\texttt{ProjChi} to perform better than \texttt{RAPPOR}+$\ell_2$. Conversely, if the signal is large in the $\ell_2$ distance, the opposite holds true;
we present the numerical results that confirm this in Appendix~\ref{appendix:powerlaw}.
This suggests that when practitioners have insights into the nature of the deviation between two distributions, selecting a statistic that aligns with that specific deviation might be more effective. We leave it as future work to conduct more extensive simulations in diverse settings and using various types of test statistics.
\section{Discussion}
In this work, we studied minimax separation rates for two-sample testing under LDP constraint. Moving beyond the univariate Besov ball with $q=\infty$ considered in \citet{Lam-Weil2021MinimaxConstraint}, our work encompasses a larger Besov class of densities in a multivariate setting without restriction on $q$. We also considered the H\"{o}lder class and extended the non-private results of \citet{kim_minimax_2022} to a locally private setting. By noting the equivalence between the binning approach in~\citet{kim_minimax_2022} and the projection approach in~\citet{Lam-Weil2021MinimaxConstraint}, we proposed an integrated private testing framework that provides an optimal test for a large class of smooth densities. We proved our results using three distinct LDP mechanisms, thereby extending the toolkit available for practitioners. Additionally, an adaptive test is introduced that retains optimality up to a log factor without the knowledge of the smoothness parameter. Echoing prior work~\citep{Aliakbarpour2018DPgof, Aliakbarpour2019PrivatePermutations, Cai2017DPGofPrivit, Lam-Weil2021MinimaxConstraint}, our results reaffirm that there exists an inevitable trade-off between data privacy and statistical efficiency that data analysts should bear in mind.

Our paper leaves several open questions for future investigation. Throughout the paper, we focused on equal-sized binning scheme that returns minimax optimal procedures. However, this framework may be problematic in high-dimensional settings as many bins would be empty. To address this issue, one can develop a data-adaptive binning scheme and improve the high-dimensional performance. In terms of smoothness classes, future work can be dedicated to extending our minimax result to a more general Besov class and other smoothness classes. Another interesting direction of future work is to develop optimal tests of conditional independence under privacy constraints, building on the recent work of \citet{neykov2021minimax} and \citet{kim2022local}. Finally, one can attempt to improve the $\log^2 n_1 \log \log n_1$ cost for adaptivity or find the matching lower bound. We leave all of these interesting but challenging problems for future work.

\bibliographystyle{apalike}
\bibliography{reference}

\begin{thebibliography}{}

\bibitem[Acharya et~al., 2021a]{Acharya2021Lowerbound}
Acharya, J., Canonne, C.~L., Freitag, C., Sun, Z., and Tyagi, H. (2021a).
\newblock Inference under information constraints {III}: Local privacy constraints.
\newblock {\em IEEE Journal on Selected Areas in Information Theory}, 2:253--267.

\bibitem[Acharya et~al., 2019]{acharya_test_2019}
Acharya, J., Canonne, C.~L., Freitag, C., and Tyagi, H. (2019).
\newblock Test without trust: Optimal locally private distribution testing.
\newblock In {\em International Conference on Artificial Intelligence and Statistics (AISTATS)}, volume~89.

\bibitem[Acharya et~al., 2021b]{Acharya2022InteractiveConstraints}
Acharya, J., Canonne, C.~L., Liu, Y., Sun, Z., and Tyagi, H. (2021b).
\newblock Interactive inference under information constraints.
\newblock {\em IEEE International Symposium on Information Theory (ISIT)}.

\bibitem[Acharya et~al., 2020]{Acharya2020Lowerbound}
Acharya, J., Canonne, C.~L., and Tyagi, H. (2020).
\newblock Inference under information constraints {I}: Lower bounds from chi-square contraction.
\newblock {\em IEEE Transactions on Information Theory}, 66:7835--7855.

\bibitem[Acharya et~al., 2015]{Acharya2015Nonprivate}
Acharya, J., Daskalakis, C., and Kamath, G. (2015).
\newblock Optimal testing for properties of distributions.
\newblock In {\em Advances in Neural Information Processing Systems (NeurIPS)}, volume~28.

\bibitem[Acharya et~al., 2021c]{acharya_estimating_2021}
Acharya, J., Kairouz, P., Liu, Y., and Sun, Z. (2021c).
\newblock Estimating sparse discrete distributions under privacy and communication constraints.
\newblock In {\em International Conference on Algorithmic Learning Theory (ALT)}, volume 132.

\bibitem[Acharya et~al., 2018]{Acharya2018dPGofTwosample}
Acharya, J., Sun, Z., and Zhang, H. (2018).
\newblock Differentially private testing of identity and closeness of discrete distributions.
\newblock In {\em Advances in Neural Information Processing Systems (NeurIPS)}.

\bibitem[Acharya et~al., 2021d]{Acharya2021DifferentiallyCam}
Acharya, J., Sun, Z., and Zhang, H. (2021d).
\newblock Differentially private {Assouad}, {Fano}, and {Le Cam}.
\newblock In {\em International Conference on Algorithmic Learning Theory (ALT)}, volume 132.

\bibitem[Aliakbarpour et~al., 2019]{Aliakbarpour2019PrivatePermutations}
Aliakbarpour, M., Diakonikolas, I., Kane, D., and Rubinfeld, R. (2019).
\newblock Private testing of distributions via sample permutations.
\newblock In {\em Advances in Neural Information Processing Systems (NeurIPS)}, volume~32.

\bibitem[Aliakbarpour et~al., 2018]{Aliakbarpour2018DPgof}
Aliakbarpour, M., Diakonikolas, I., and Rubinfeld, R. (2018).
\newblock Differentially private identity and equivalence testing of discrete distributions.
\newblock In {\em International Conference on Machine Learning (ICML)}, volume~80.

\bibitem[{Apple Differential Privacy Team}, 2017]{Apple2017Privacy}
{Apple Differential Privacy Team} (2017).
\newblock Learning with privacy at scale.
\newblock {\em Apple Machine Learning Journal}, 1.

\bibitem[Arias-Castro et~al., 2018]{Arias-Castro2018RememberDimension}
Arias-Castro, E., Pelletier, B., and Saligrama, V. (2018).
\newblock Remember the curse of dimensionality: the case of goodness-of-fit testing in arbitrary dimension.
\newblock {\em Journal of Nonparametric Statistics}, 30(2):448–471.

\bibitem[Autin et~al., 2010]{Autin2010Wavlet}
Autin, F., Le~Pennec, E., and Tribouley, K. (2010).
\newblock Thresholding methods to estimate copula density.
\newblock {\em Journal of Multivariate Analysis}, 101(1):200--222.

\bibitem[Baraud, 2002]{Baraud2002minimax}
Baraud, Y. (2002).
\newblock Non-asymptotic minimax rates of testing in signal detection.
\newblock {\em Bernoulli}, 8(5):577--606.

\bibitem[Batu et~al., 2000]{Batu2000NonprivateTwosample}
Batu, T., Fortnow, L., Rubinfeld, R., Smith, W.~D., and White, P. (2000).
\newblock Testing that distributions are close.
\newblock {\em IEEE Annual Symposium on Foundations of Computer Science (FOCS)}.

\bibitem[Berrett and Butucea, 2020]{Berrett2020Interactive}
Berrett, T.~B. and Butucea, C. (2020).
\newblock Locally private non-asymptotic testing of discrete distributions is faster using interactive mechanisms.
\newblock In {\em Advances in Neural Information Processing Systems (NeurIPS)}, volume~33.

\bibitem[Berrett et~al., 2021]{Berret2021Classification}
Berrett, T.~B., Gy\"{o}rfi, L., and Walk, H. (2021).
\newblock Strongly universally consistent nonparametric regression and classification with privatised data.
\newblock {\em Electronic Journal of Statistics}, 15:2430--2453.

\bibitem[Cai et~al., 2017]{Cai2017DPGofPrivit}
Cai, B., Daskalakis, C., and Kamath, G. (2017).
\newblock Priv'it: Private and sample efficient identity testing.
\newblock In {\em International Conference on Machine Learning (ICML)}, volume~70.

\bibitem[Cai et~al., 2021]{Tony2021DP}
Cai, T.~T., Wang, Y., and Zhang, L. (2021).
\newblock The cost of privacy: Optimal rates of convergence for parameter estimation with differential privacy.
\newblock {\em The Annals of Statistics}, 49(5):2825--2850.

\bibitem[Chan et~al., 2014]{chan2014optimal}
Chan, S.-O., Diakonikolas, I., Valiant, P., and Valiant, G. (2014).
\newblock Optimal algorithms for testing closeness of discrete distributions.
\newblock In {\em ACM-SIAM Symposium on Discrete Algorithms (SODA)}, pages 1193--1203. SIAM.

\bibitem[Couch et~al., 2019]{Couch2019NonparamTwosample}
Couch, S., Kazan, Z., Shi, K., Bray, A., and Groce, A. (2019).
\newblock Differentially private nonparametric hypothesis testing.
\newblock In {\em {ACM} {SIGSAC} Conference on Computer and Communications Security (CCS)}, page 737–751. Association for Computing Machinery.

\bibitem[Diakonikolas et~al., 2018]{Diakonikolas2018Nonprivate}
Diakonikolas, I., Gouleakis, T., Peebles, J., and Price, E. (2018).
\newblock Sample-optimal identity testing with high probability.
\newblock {\em International Colloquium on Automata, Languages, and Programming (ICALP)}, 107:41:1--41:14.

\bibitem[Diakonikolas and Kane, 2016]{Diakonikolas2016Nonprivate}
Diakonikolas, I. and Kane, D.~M. (2016).
\newblock A new approach for testing properties of discrete distributions.
\newblock {\em IEEE Annual Symposium on Foundations of Computer Science (FOCS)}, pages 685--694.

\bibitem[Ding et~al., 2017]{Ding2017CollectingPrivately}
Ding, B., Kulkarni, J., and Yekhanin, S. (2017).
\newblock Collecting telemetry data privately.
\newblock In {\em Advances in Neural Information Processing Systems (NeurIPS)}, pages 3574--3583.

\bibitem[Ding et~al., 2018]{Ding2018TwosampleMean}
Ding, B., Nori, H., Li, P., and Allen, J. (2018).
\newblock Comparing population means under local differential privacy: With significance and power.
\newblock In {\em AAAI Conference on Artificial Intelligence}.

\bibitem[Dubois et~al., 2023]{Dubois2022}
Dubois, A., B.~Berrett, T., and Butucea, C. (2023).
\newblock Goodness-of-fit testing for {H{\"o}lder} continuous densities under local differential privacy.
\newblock In {\em Foundations of Modern Statistics}, pages 53--119, Cham. Springer International Publishing.

\bibitem[Duchi et~al., 2013]{duchi2013local}
Duchi, J.~C., Jordan, M.~I., and Wainwright, M.~J. (2013).
\newblock Local privacy and minimax bounds: sharp rates for probability estimation.
\newblock In {\em Advances in Neural Information Processing Systems (NeurIPS)}, volume~1, page 1529–1537. Curran Associates Inc.

\bibitem[Duchi et~al., 2018]{Duchi2018MinimaxEstimation}
Duchi, J.~C., Jordan, M.~I., and Wainwright, M.~J. (2018).
\newblock Minimax optimal procedures for locally private estimation.
\newblock {\em Journal of the American Statistical Association}, 113(521):182--201.

\bibitem[Dwork et~al., 2006]{CynthiaDwork2006CalibratingAnalysis}
Dwork, C., McSherry, F., Nissim, K., and Adam, S. (2006).
\newblock Calibrating noise to sensitivity in private data analysis.
\newblock In {\em Theory of Cryptography}, pages 265--284. Springer Berlin Heidelberg.

\bibitem[Dwork and Roth, 2013]{Dwork2014Book}
Dwork, C. and Roth, A. (2013).
\newblock The algorithmic foundations of differential privacy.
\newblock {\em Foundations and Trends in Theoretical Computer Science}, 9.

\bibitem[Erlingsson et~al., 2014]{erlingsson_rappor_2014}
Erlingsson, {\'U}., Pihur, V., and Korolova, A. (2014).
\newblock {RAPPOR}: Randomized aggregatable privacy-preserving ordinal response.
\newblock In {\em {ACM} {SIGSAC} Conference on Computer and Communications Security (CCS)}, pages 1054--1067.

\bibitem[Fromont and Laurent, 2006]{Fromont2006AdaptiveModel}
Fromont, M. and Laurent, B. (2006).
\newblock Adaptive goodness-of-fit tests in a density model.
\newblock {\em The Annals of Statistics}, 34(2):680--720.

\bibitem[Gaboardi et~al., 2016]{Gaboardi2016DPChisq}
Gaboardi, M., Lim, H., Rogers, R., and Vadhan, S. (2016).
\newblock Differentially private chi-squared hypothesis testing: Goodness of fit and independence testing.
\newblock In {\em International Conference on Machine Learning (ICML)}, volume~48, pages 2111--2120.

\bibitem[Gaboardi and Rogers, 2018]{Gaboardi2018LDPChisq}
Gaboardi, M. and Rogers, R. (2018).
\newblock Local private hypothesis testing: Chi-square tests.
\newblock In {\em International Conference on Machine Learning (ICML)}, volume~80, pages 1626--1635.

\bibitem[Ghosh et~al., 2009]{ghosh2009}
Ghosh, A., Roughgarden, T., and Sundararajan, M. (2009).
\newblock Universally utility-maximizing privacy mechanisms.
\newblock In {\em Annual ACM Symposium on Theory of Computing (STOC)}, page 351–360. Association for Computing Machinery.

\bibitem[Giné and Nickl, 2015]{Gine2015Besov}
Giné, E. and Nickl, R. (2015).
\newblock {\em Mathematical Foundations of Infinite-Dimensional Statistical Models}.
\newblock Cambridge University Press.

\bibitem[Goldreich and Ron, 2011]{Goldreich2000Nonprivate}
Goldreich, O. and Ron, D. (2011).
\newblock On testing expansion in bounded-degree graphs.
\newblock {\em Lecture Notes in Computer Science (including subseries Lecture Notes in Artificial Intelligence and Lecture Notes in Bioinformatics)}, 6650 LNCS.

\bibitem[Gretton et~al., 2012]{Gretton2012ATest}
Gretton, A., Borgwardt, K.~M., Rasch, M.~J., Sch\"{o}lkopf, B., and Smola, A. (2012).
\newblock A kernel two-sample test.
\newblock {\em Journal of Machine Learning Research}, 13:723--773.

\bibitem[Gretton et~al., 2009]{Gretton2009FastKernel}
Gretton, A., Fukumizu, K., Harchaoui, Z., and Sriperumbudur, B.~K. (2009).
\newblock A fast, consistent kernel two-sample test.
\newblock In {\em Advances in Neural Information Processing Systems (NeurIPS)}, volume~22.

\bibitem[Harchaoui et~al., 2009]{Harchaoui2009ASegmentation}
Harchaoui, Z., Vallet, F., Lung-Yut-Fong, A., and Cappe, O. (2009).
\newblock A regularized kernel-based approach to unsupervised audio segmentation.
\newblock In {\em IEEE International Conference on Acoustics, Speech and Signal Processing}, pages 1665--1668.

\bibitem[Hemerik and Goeman, 2018]{hemerik_exact_2018}
Hemerik, J. and Goeman, J. (2018).
\newblock Exact testing with random permutations.
\newblock {\em TEST}, 27(4):811--825.

\bibitem[Homer et~al., 2008]{homer2008resolving}
Homer, N., Szelinger, S., Redman, M., Duggan, D., Tembe, W., Muehling, J., Pearson, J.~V., Stephan, D.~A., Nelson, S.~F., and Craig, D.~W. (2008).
\newblock Resolving individuals contributing trace amounts of dna to highly complex mixtures using high-density snp genotyping microarrays.
\newblock {\em PLoS Genetics}, 4(8):e1000167.

\bibitem[Ingster, 1993]{Ingster1993AsymptoticallyAlternatives}
Ingster, Y. (1993).
\newblock Asymptotically minimax hypothesis testing for nonparametric alternatives.
\newblock {\em Mathematical Methods of Statistics}, 2(2):85--114.

\bibitem[Ingster, 2000]{Ingster2000AdaptiveTests}
Ingster, Y. (2000).
\newblock Adaptive chi-square tests.
\newblock {\em Journal of Mathematical Sciences}, 99:1110--1119.

\bibitem[Ingster, 1994]{Ingster1994Bin}
Ingster, Y.~I. (1994).
\newblock Minimax detection of a signal in $\ell_p$ metrics.
\newblock {\em Journal of Mathematical Sciences}, 68:503--515.

\bibitem[Inusah and Kozubowski, 2006]{inusah_discrete_2006}
Inusah, S. and Kozubowski, T.~J. (2006).
\newblock A discrete analogue of the {Laplace} distribution.
\newblock {\em Journal of Statistical Planning and Inference}, 136(3):1090--1102.

\bibitem[Johnson and Shmatikov, 2013]{Johnson2013PrivacyStudies}
Johnson, A. and Shmatikov, V. (2013).
\newblock Privacy-preserving data exploration in genome-wide association studies.
\newblock In {\em ACM SIGKDD International Conference on Knowledge Discovery and Data Mining}, page 1079–1087. Association for Computing Machinery.

\bibitem[Joseph et~al., 2019]{Joseph2019Interactivity}
Joseph, M., Mao, J., Neel, S., and Roth, A. (2019).
\newblock The role of interactivity in local differential privacy.
\newblock {\em IEEE Annual Symposium on Foundations of Computer Science (FOCS)}, pages 94--105.

\bibitem[Kasiviswanathan et~al., 2011]{Kasiviswanathan2008WhatPrivately}
Kasiviswanathan, S.~P., Lee, H.~K., Nissim, K., Raskhodnikova, S., and Smith, A. (2011).
\newblock What can we learn privately?
\newblock {\em SIAM Journal on Computing}, 40(3):793--826.

\bibitem[Kim, 2020]{kim_multinomial_2020}
Kim, I. (2020).
\newblock Multinomial goodness-of-fit based on {U}-statistics: High-dimensional asymptotic and minimax optimality.
\newblock {\em Journal of Statistical Planning and Inference}, 205:74--91.

\bibitem[Kim et~al., 2022a]{kim_minimax_2022}
Kim, I., Balakrishnan, S., and Wasserman, L. (2022a).
\newblock Minimax optimality of permutation tests.
\newblock {\em The Annals of Statistics}, 50(1):225--251.

\bibitem[Kim et~al., 2019]{kim2019}
Kim, I., Lee, A.~B., and Lei, J. (2019).
\newblock Global and local two-sample tests via regression.
\newblock {\em Electronic Journal of Statistics}, 13(2):5253--5305.

\bibitem[Kim et~al., 2022b]{kim2022local}
Kim, I., Neykov, M., Balakrishnan, S., and Wasserman, L. (2022b).
\newblock Local permutation tests for conditional independence.
\newblock {\em The Annals of Statistics}, 50(6):3388--3414.

\bibitem[Kim et~al., 2021]{kim2021}
Kim, I., Ramdas, A., Singh, A., and Wasserman, L. (2021).
\newblock Classification accuracy as a proxy for two-sample testing.
\newblock {\em The Annals of Statistics}, 49(1):411--434.

\bibitem[Kim and Schrab, 2023]{kim2023dp}
Kim, I. and Schrab, A. (2023).
\newblock Differentially private permutation tests: Applications to kernel methods.
\newblock {\em arXiv preprint arXiv:2310.19043}.

\bibitem[Kohout and Pevný, 2018]{Kohout2018NetworkTest}
Kohout, J. and Pevný, T. (2018).
\newblock Network traffic fingerprinting based on approximated kernel two-sample test.
\newblock {\em IEEE Transactions on Information Forensics and Security}, 13:788--801.

\bibitem[Lalanne et~al., 2023]{Lalanne2023on}
Lalanne, C., Garivier, A., and Gribonval, R. (2023).
\newblock On the statistical complexity of estimation and testing under privacy constraints.
\newblock {\em Transactions on Machine Learning Research}.

\bibitem[Lam-Weil et~al., 2022]{Lam-Weil2021MinimaxConstraint}
Lam-Weil, J., Laurent, B., and Loubes, J.-M. (2022).
\newblock Minimax optimal goodness-of-fit testing for densities and multinomials under a local differential privacy constraint.
\newblock {\em Bernoulli}, 28(1):579--600.

\bibitem[McSherry and Talwar, 2007]{mcsherry_mechanism_2007}
McSherry, F. and Talwar, K. (2007).
\newblock Mechanism design via differential privacy.
\newblock In {\em IEEE Annual Symposium on Foundations of Computer Science (FOCS)}, pages 94--103.

\bibitem[Near, 2018]{Near2017Uber}
Near, J. (2018).
\newblock Differential privacy at scale: {Uber} and {Berkeley} collaboration.
\newblock In {\em Enigma}. USENIX Association.

\bibitem[Neykov et~al., 2021]{neykov2021minimax}
Neykov, M., Balakrishnan, S., and Wasserman, L. (2021).
\newblock Minimax optimal conditional independence testing.
\newblock {\em The Annals of Statistics}, 49(4):2151--2177.

\bibitem[Rabin et~al., 2019]{Rabin2019ModelingTests}
Rabin, N., Golan, M., Singer, G., and Kleper, D. (2019).
\newblock Modeling and analysis of students’ performance trajectories using diffusion maps and kernel two-sample tests.
\newblock {\em Engineering Applications of Artificial Intelligence}, 85.

\bibitem[Raj et~al., 2020]{Raj2020ATest}
Raj, A., Law, H. C.~L., Sejdinovic, D., and Park, M. (2020).
\newblock A differentially private kernel two-sample test.
\newblock In {\em Machine Learning and Knowledge Discovery in Databases}, pages 697--724, Cham.

\bibitem[Ramdas et~al., 2023]{ramdas2023permutation}
Ramdas, A., Barber, R.~F., Cand{\`e}s, E.~J., and Tibshirani, R.~J. (2023).
\newblock Permutation tests using arbitrary permutation distributions.
\newblock {\em Sankhya A}, 85(2):1156--1177.

\bibitem[Rogers and Kifer, 2017]{Kifer2017DPChisq}
Rogers, R. and Kifer, D. (2017).
\newblock A new class of private chi-square hypothesis tests.
\newblock In {\em International Conference on Artificial Intelligence and Statistics (AISTATS)}, volume~54.

\bibitem[Schrab et~al., 2023]{Schrab2021MMDTest}
Schrab, A., Kim, I., Albert, M., Laurent, B., Guedj, B., and Gretton, A. (2023).
\newblock {MMD Aggregated Two-Sample Test}.
\newblock {\em Journal of Machine Learning Research}, 24(194):1--81.

\bibitem[Sheffet, 2018]{Sheffet2018LDP}
Sheffet, O. (2018).
\newblock Locally private hypothesis testing.
\newblock In {\em International Conference on Machine Learning (ICML)}, volume~80, pages 4605--4614.

\bibitem[Sheffet and Omer, 2024]{sheffet_differentially_2024}
Sheffet, O. and Omer, D. (2024).
\newblock Differentially private equivalence testing for continuous distributions and applications.
\newblock In {\em Advances in Neural Information Processing Systems (NeurIPS)}.

\bibitem[Student, 1908]{Student1908Twosample}
Student (1908).
\newblock {The Probable Error of a Mean}.
\newblock {\em Biometrika}, 6(1):1--25.

\bibitem[Sz\'ekely and Rizzo, 2004]{Szekely2004Energy}
Sz\'ekely, G.~J. and Rizzo, M.~L. (2004).
\newblock Testing for equal distributions in high dimension.
\newblock {\em InterStat}, 5.

\bibitem[Sz\'ekely and Rizzo, 2005]{szekely_new_2005}
Sz\'ekely, G\'abor, G.~J. and Rizzo, M.~L. (2005).
\newblock A new test for multivariate normality.
\newblock {\em Journal of Multivariate Analysis}, 93(1):58--80.

\bibitem[Tang and Yang, 2023]{tang2023Besov}
Tang, R. and Yang, Y. (2023).
\newblock Minimax nonparametric two-sample test under adversarial losses.
\newblock In {\em International Conference on Artificial Intelligence and Statistics (AISTATS)}, volume 206.

\bibitem[Task and Clifton, 2016]{Task2016Wilcoxon}
Task, C. and Clifton, C. (2016).
\newblock Differentially private significance testing on paired-sample data.
\newblock In {\em SIAM International Conference on Data Mining (SDM)}, pages 153--161.

\bibitem[Uhler et~al., 2013]{Uhler2013PrivacyStudies}
Uhler, C., Slavkovic, A.~B., and Fienberg, S.~E. (2013).
\newblock Privacy-preserving data sharing for genome-wide association studies.
\newblock {\em Journal of Privacy and Confidentiality}, 5(1).

\bibitem[Valiant and Valiant, 2014]{Valiant2017Nonprivate}
Valiant, G. and Valiant, P. (2014).
\newblock An automatic inequality prover and instance optimal identity testing.
\newblock In {\em IEEE Annual Symposium on Foundations of Computer Science (FOCS)}, pages 51--60.

\bibitem[van~der Vaart, 1998]{van_der_vaart_asymptotic_1998}
van~der Vaart, A.~W. (1998).
\newblock {\em Asymptotic Statistics}.
\newblock Cambridge Series in Statistical and Probabilistic Mathematics. Cambridge University Press, Cambridge.

\bibitem[Vu and Slavkovic, 2009]{Vu2009DifferentialEvaluations}
Vu, D. and Slavkovic, A. (2009).
\newblock Differential privacy for clinical trial data: Preliminary evaluations.
\newblock In {\em IEEE International Conference on Data Mining Workshops}, pages 138--143.

\bibitem[Wang et~al., 2015]{Wang2015DPChisq}
Wang, Y., Lee, J., and Kifer, D. (2015).
\newblock Revisiting differentially private hypothesis tests for categorical data.
\newblock {\em arXiv preprint arXiv:1511.03376}.

\bibitem[Warner, 1965]{randomizedresponse}
Warner, S.~L. (1965).
\newblock Randomized response: A survey technique for eliminating evasive answer bias.
\newblock {\em Journal of the American Statistical Association}, 60(309):63--69.

\bibitem[Waudby-Smith et~al., 2023]{waudby-smith_nonparametric_2023}
Waudby-Smith, I., Wu, S., and Ramdas, A. (2023).
\newblock Nonparametric extensions of randomized response for private confidence sets.
\newblock In {\em International Conference on Machine Learning (ICML)}, pages 36748--36789. PMLR.

\bibitem[Yu et~al., 2014]{Yu2014Chisq}
Yu, F., Fienberg, S.~E., Slavković, A.~B., and Uhler, C. (2014).
\newblock Scalable privacy-preserving data sharing methodology for genome-wide association studies.
\newblock {\em Journal of Biomedical Informatics}, 50:133--141.

\end{thebibliography}

\begin{appendix}
	\section{Overview of Appendix}\label{appendix}
	This supplementary material provides the technical proofs deferred in the main text, along with some additional results of interest. The content is organized as follows:
	
	\begin{itemize}
		\item Appendix~\ref{section:appendix_prelim} presents the technical lemmas, constructions, and calculations used in the main proofs of our results.
		\item Appendix~\ref{proof:LapUdiscLapULDP} provides the proof of the $\privacyParameter$-LDP guarantees for our proposed privacy mechanisms.
		\item Appendix~\ref{proof:twosample_multinomial_rates} proves the minimax rate result for multinomial testing, as presented in Theorem~\ref{theorem:twosample_multinomial_rates}.
		\item Appendix~\ref{proof:twosample_conti_rates} proves the minimax rate result for density testing, as presented in Theorem~\ref{theorem:twosample_conti_rate}.
		\item Appendix~\ref{proof:twosample_adaptive} establishes the minimax upper bound for adaptive density testing, as shown in Theorem~\ref{theorem:twosample_adaptive_upper}.
		\item Appendix~\ref{appendix:simulation} provides additional information on the numerical studies presented in Section~\ref{section:simulation}.
		\item Appendix~\ref{genrr_suboptimal_theory} discusses the asymptotic suboptimality of the \texttt{GenRR}+$\ell_2$ multinomial test mentioned in Section~\ref{section:simulation}.
	\end{itemize}

	\section{Preliminary Results}\label{section:appendix_prelim}
	This section presents the technical lemmas, constructions, and calculations used in the main proofs.
	\subsection{First Two Moments of Discrete Laplace Noise}\label{proof:discLapUVar}
	We analyze the first two moments of a discrete Laplace noise random variable to establish the upper bound for the separation rate of our private test in Theorem~\ref{theorem:twosample_multinomial_rates}. The next lemma shows that the discrete noise in discrete Laplace mechanism (Definition \ref{def:DiscLapU_formal}) has mean zero and variance at most $8\alphabetSize/\privacyParameter^2$, which matches the variance of the continuous noise in Laplace mechanism (Definition \ref{def:LapU_formal}) with the same privacy guarantee.
	\begin{lemma}\label{lemma:discLapU_Variance}
		If $W$ follows $ \mathrm{DL}(\discLapUParam)$ defined in~\eqref{def:discrete_laplace_distribution} and Definition~\ref{def:DiscLapU_formal}, we have
		$$
		\mathbb{E}(W) = 0\quad \text{and} \quad\mathrm{Var}(W) \leq \frac{8\alphabetSize}{\privacyParameter^2}.
		$$

\vskip 1em

		\begin{proof}
			From Proposition 2.2 of \citet{inusah_discrete_2006}, we have
			$\mathbb{E}(W) = 0$
			and 
			$\mathrm{Var}(W) = 2\discLapUParam / (1-\discLapUParam)^2$.
			Therefore, it suffices to show that
			$$ \frac{2\discLapUParam}{(1-\discLapUParam)^2} \leq \frac{8\alphabetSize}{\privacyParameter^2}.$$
			For notational convenience, let $v :=  \privacyParameter/\sqrt{4\alphabetSize}>0$ for $\privacyParameter>0$, so that we have $
			\discLapUParam
			=
			\exp({
				-
				\privacyParameter/\sqrt{4\alphabetSize}
			})
			=
			\exp({-v}).$
			The proof then reduces to showing that
			$$
			v^2
			\leq
			\frac{(1 - \discLapUParam)^2}{\discLapUParam} = 
			\bigl(
			e^{v/2} - e^{-v/2}
			\bigr)^2.
			$$
			Since $v>0$, both sides of the above inequality is positive. Therefore it suffices to show that
			\begin{align*}
				v \leq 
				e^{v/2} - e^{-v/2}.
				\numberthis \label{eq:cond_for_disc_var}
			\end{align*}
			%Now we exploit the Taylor expansion of the exponential function.
			Note that Taylor expansions of $\exp(v/2)$ and $\exp(-v/2)$ have the same even order terms, and their odd order terms have the same absolute values with opposite signs. Therefore, it holds that
			\begin{align*}
				(e^{v/2} - e^{-v/2} ) = 2 \cdot \sum_{n=0}^{\infty} \frac{\left(v/2\right)^{2n+1}}{(2n+1)!} = 
				2\cdot
				\left(
				(v/2)
				+ 
				\frac{(v/2)^3}{3!}
				+
				\frac{(v/2)^5}{5!}
				+
				\frac{(v/2)^7}{7!}
				+ \cdots 
				\right).
			\end{align*}
			Since $v>0$, all the terms of above are positive. Thus, the condition~\eqref{eq:cond_for_disc_var} is satisfied. This concludes the proof of Lemma~\ref{lemma:discLapU_Variance}. 
		\end{proof}
	\end{lemma}
	\subsection{Summary Statistics  of \texttt{RAPPOR} Mechanism}\label{appendix:rappor_moment}
	This section provides calculations for the summary statistics of the \texttt{RAPPOR} private views, which are used in Appendix~\ref{proof:rappor_optimal}. First, Lemma~\ref{lemma:rappor_var} presents the exact calculation of the expectation of the U-statistic in~\eqref{def:statistic_elltwo}, as well as the variance and covariance of the entries of a \texttt{RAPPOR} private view.
	
	\begin{lemma}[Summary statistics  of \texttt{RAPPOR} private views]
		\label{lemma:rappor_var}
		Let $\{\rvTwo_\sampleIndexOne\}_{\sampleIndexOne \in [\sampleSize_1]}$  be i.i.d.~multinomial sample with $\alphabetSize$ categories. Let $\{\tilde{\rVecY}_i\}_{i \in [n_1]}$ represent the corresponding $\privacyParameter$-LDP views obtained through the \textnormal{\texttt{RAPPOR}} mechanism, where for each $m \in [k]$, the $m$th entry is distributed as 
		\begin{equation*}\tilde{\rvTwo}_{\sampleIndexOne \vectorIndex}
			\sim
			\mathrm{Ber}
			\bigl(
			\privacyParameterrappor \mathds{1}
			(
			\rvTwo_{\sampleIndexOne} = \vectorIndex
			) + \smallNumberrappor
			\bigr),~\text{where}~
			\privacyParameterrappor := \frac{e^{\privacyParameter /2} -1}{e^{\privacyParameter /2} +1} \quad \text{and} \quad
			\smallNumberrappor := \frac{1}{e^{\privacyParameter /2} +1}.
		\end{equation*}
		Then we have $\mE[U_{\sampleSize_1, \sampleSize_2}]
		=
		\privacyParameterrappor^2 \|
		\probVec_\rvTwo - \probVec_\rvThree
		\|_2^2$. Also, we have
		\begin{equation*}
			\mathrm{Var}(\tilde{Y}_{1\vectorIndex})
			=
			\bigl(
			\privacyParameterrappor 
			\probVecElement{\rvTwo}{\vectorIndex}
			+ \smallNumberrappor
			\bigr)
			\bigl(
			1
			-\privacyParameterrappor 
			\probVecElement{\rvTwo}{\vectorIndex}
			- \smallNumberrappor
			\bigr)
			~\text{and}~	\mathrm{Cov}(\tilde{Y}_{1\vectorIndex}, \tilde{Y}_{1\vectorIndex'}) = -
			\privacyParameterrappor^2 \probVecElement{\rvTwo}{\vectorIndex} \probVecElement{\rvTwo}{\vectorIndex'}
			,
		\end{equation*}
			for each $\vectorIndex \in [\alphabetSize]$ and
		for $\vectorIndex,\vectorIndex' \in [\alphabetSize]$ such that $m \neq m'$, respectively.	
	\end{lemma}
\vskip 0.1224mm
	\begin{proof}
		The expectation is verified as follows:
		\allowdisplaybreaks
		\begin{align*}
			\mE[U_{\sampleSize_1, \sampleSize_2}]
			&=
			\mE[
			(
			\tilde{\vectorize{\rvTwo}}_1
			-
			\tilde{\vectorize{\rvThree}}_1
			)^\top
			(
			\tilde{\vectorize{\rvTwo}}_2
			-
			\tilde{\vectorize{\rvThree}}_2
			)
			]
			\\
			&= 
			\mE[
			\tilde{\vectorize{\rvTwo}}_1 - \tilde{\vectorize{\rvThree}}_1]^\top
			\mE[
			\tilde{\vectorize{\rvTwo}}_2
			-
			\tilde{\vectorize{\rvThree}}_2
			]	
			\\&=
			(\privacyParameterrappor \; \probVec_\rvTwo - \privacyParameterrappor\probVec_\rvThree)^\top
			(\privacyParameterrappor \; \probVec_\rvTwo - \privacyParameterrappor\probVec_\rvThree)
			\\&=
			\privacyParameterrappor^2 \|
			\probVec_\rvTwo - \probVec_\rvThree
			\|_2^2.
		\end{align*} 
		The variance is verified as follows:
		\begin{equation}
			\mathrm{Var}(\tilde{Y}_{1\vectorIndex})
			=
			\mE[\tilde{\rvTwo}_{1\vectorIndex}^2] - \mE[\tilde{Y}_{1\vectorIndex}]^2
			\stackrel{(a)}{=}
			\mE[\tilde{Y}_{1\vectorIndex}] - \mE[\tilde{Y}_{1\vectorIndex}]^2
			=
			\bigl(
			\privacyParameterrappor 
			\probVecElement{\rvTwo}{\vectorIndex}
			+ \smallNumberrappor
			\bigr)
			\bigl(
			1
			-\privacyParameterrappor 
			\probVecElement{\rvTwo}{\vectorIndex}
			- \smallNumberrappor
			\bigr),
		\end{equation}
		where step $(a)$ uses the fact that $\tilde{Y}_{1\vectorIndex}^2 = \tilde{Y}_{1\vectorIndex}$ since $\tilde{Y}_{1\vectorIndex}$ is either 0 or 1.	
		Finally, for $\vectorIndex,\vectorIndex' \in [\alphabetSize]$ such that $\vectorIndex \neq \vectorIndex'$, the covariance is calculated as:
		\begin{align*}
			\mathrm{Cov}(\tilde{Y}_{1\vectorIndex}, \tilde{Y}_{1\vectorIndex'})
			= ~&
			\mE[\tilde{Y}_{1\vectorIndex} \tilde{Y}_{1\vectorIndex'}] - \mE[\tilde{Y}_{1\vectorIndex}] \mE[\tilde{Y}_{1\vectorIndex'}]
			\\ \stackrel{(a)}{=} \hskip 1.57mm &
			\mP[\tilde{Y}_{1\vectorIndex}=1, \tilde{Y}_{1\vectorIndex'}=1] - \mE[\tilde{Y}_{1\vectorIndex}] \mE[\tilde{Y}_{1\vectorIndex'}]
			\\ \stackrel{(b)}{=} \hskip 1.62mm &
			\bigl(
			\privacyParameterrappor 
			\probVecElement{\rvTwo}{\vectorIndex}
			+ \smallNumberrappor
			\bigr) \bigl(
			\privacyParameterrappor 
			\probVecElement{\rvTwo}{\vectorIndex'}
			+ \smallNumberrappor	
			\bigr)
			-
			\privacyParameterrappor^2 \probVecElement{\rvTwo}{\vectorIndex} \probVecElement{\rvTwo}{\vectorIndex'}
			-
			\bigl(
			\privacyParameterrappor 
			\probVecElement{\rvTwo}{\vectorIndex}
			+ \smallNumberrappor
			\bigr) \bigl(
			\privacyParameterrappor 
			\probVecElement{\rvTwo}{\vectorIndex'}
			+ \smallNumberrappor	
			\bigr)
			%%%%%%%%%%%%%%%%%%%%%%%%%%%%%%%%%%%%%%%%%%%%%
			\\ = \hskip 1.96mm &% final result
			-
			\privacyParameterrappor^2
			\probVecElement{\rvTwo}{\vectorIndex}
			\probVecElement{\rvTwo}{\vectorIndex'},
		\end{align*}
		where step $(a)$ uses the fact that each entry is 0 or 1, and step $(b)$ is from Fact 1 of~\citet{acharya_test_2019}, which holds only when $m \neq m'.$
	\end{proof}
	We next give upper bounds for the sum of  entrywise variances and covariances of an \texttt{RAPPOR} $\privacyParameter$-LDP view.
	\begin{lemma}\label{rappor:inequalities}
		For any $\vectorIndex, \vectorIndex' \in [\alphabetSize]$ such that $\vectorIndex' \neq \vectorIndex'$, the following inequalities hold:
		\begin{align*}
			& \quad \ \sum_{\vectorIndex=1}^\alphabetSize
			\mathrm{Var}
			( %first term paranthesis opens
			\private{\rvTwo}_{1 \vectorIndex}
			) %first term paranthesis closes
			\bigl( %second term paranthesis opens
			\probVecElement{\rvTwo}{\vectorIndex}
			- % minus
			\probVecElement{\rvThree}{\vectorIndex}
			\bigr)^2
			\leq
			\privacyParameterrappor
			\normSqMultinomMax^{1/2}
			\|\probVec_{\rvTwo} - \probVec_{\rvThree}\|_2^2
			+
			\privacyParameterrappor
			\smallNumberrappor
			\|\probVec_{\rvTwo} - \probVec_{\rvThree}\|_2^2,
			\numberthis
			\label{inequality:var_diff_squared}
			\\& \quad \ 
			\sum_{\vectorIndex = 1}^\alphabetSize
			\mathrm{Var}(\private{Y}_{1 \vectorIndex})^2
			\leq
			2
			\privacyParameterrappor^2
			\normSqMultinomMax
			+
			2
			\smallNumberrappor^2
			\alphabetSize,~\text{and}
			\numberthis \label{inequality:var_squared}
			\\&
			\sum_{1 \leq \vectorIndex \neq \vectorIndex' \leq \alphabetSize}
			\mathrm{Cov}(\private{Y}_{1 \vectorIndex}, \private{Y}_{1 \vectorIndex'})^2
			\leq
			\privacyParameterrappor^2
			\normSqMultinomMax,
			\numberthis
			\label{inequality:cov_squared}
		\end{align*}
		where
		$\normSqMultinomMax
		=
		\max \{
		\|\probVec_{\rvTwo}\|_2^2,
		\|\probVec_{\rvThree}\|_2^2\} 
		$.
		The same type of inequalities also hold for $\private{\rvThree}_{1 \vectorIndex}$ and $\private{\rvThree}_{1 \vectorIndex'}$.
	\end{lemma}

    \vskip 1em
    
	\begin{proof}
		For~\eqref{inequality:var_diff_squared}, we have 
		\begin{align*}
			\sum_{\vectorIndex=1}^\alphabetSize
			\mathrm{Var}
			( %first term paranthesis opens
			\private{\rvTwo}_{1 \vectorIndex}
			) %first term paranthesis closes
			\bigl( %second term paranthesis opens
			\probVecElement{\rvTwo}{\vectorIndex}
			- % minus
			\probVecElement{\rvThree}{\vectorIndex}
			\bigr)^2
			%%%%%%%%%%%%%%%%%%%%%%%%%%%%%%%%%
			&\stackrel{(a)}{=}%5. plug in variance and covariance, use negativeness of covariance
			\sum_{\vectorIndex=1}^\alphabetSize
			\bigl(
			\privacyParameterrappor 
			\probVecElement{\rvTwo}{\vectorIndex}
			+ \smallNumberrappor
			\bigr)
			\bigl\{
			1
			-
			\bigl(
			\privacyParameterrappor 
			\probVecElement{\rvTwo}{\vectorIndex}
			+ \smallNumberrappor
			\bigr)
			\bigr\}
			\bigl( %second term paranthesis opens
			\probVecElement{\rvTwo}{\vectorIndex}
			- % minus
			\probVecElement{\rvThree}{\vectorIndex}
			\bigr)^2 %second term paranthesis closes
			%%%%%%%%%%%%%%%%%%%%%%%%%%%%%%
			\\&\stackrel{(b)}{\leq}%6 plug in variance and covariance
			\sum_{\vectorIndex=1}^\alphabetSize
			\bigl(
			\privacyParameterrappor 
			\probVecElement{\rvTwo}{\vectorIndex}
			+ \smallNumberrappor
			\bigr)
			\bigl( %second term paranthesis opens
			\probVecElement{\rvTwo}{\vectorIndex}
			- % minus
			\probVecElement{\rvThree}{\vectorIndex}
			\bigr)^2 %second term paranthesis closes
			%%%%%%%%%%%%%%%%%%%%%%%%%%%%%%%%%%%
			\\&\stackrel{(c)}{\leq}%7 split the terms and use cauchy-schwarz
			\privacyParameterrappor
			\normSqMultinomMax^{1/2}
			\|\probVec_{\rvTwo} - \probVec_{\rvThree}\|_4^2
			+
			\smallNumberrappor
			\|\probVec_{\rvTwo} - \probVec_{\rvThree}\|_2^2
			%%%%%%%%%%%%%%%%%%%%%%%%%%%%
			\\&\stackrel{(d)}{\leq}%8 split the terms and use cauchy-schwarz
			\privacyParameterrappor
			\normSqMultinomMax^{1/2}
			\|\probVec_{\rvTwo} - \probVec_{\rvThree}\|_2^2
			+
			\smallNumberrappor
			\|\probVec_{\rvTwo} - \probVec_{\rvThree}\|_2^2,
		\end{align*}
		where step $(a)$ uses Lemma~\ref{lemma:rappor_var},
		step $(b)$ uses the fact that
		\( 0 < \privacyParameterrappor \probVecElement{\rvTwo}{\vectorIndex} + \smallNumberrappor < 1 \) for any $\vectorIndex \in [\alphabetSize]$,
		step $(c)$ uses the Cauchy--Schwarz inequality,
		and step $(d)$ uses the monotonicity of the $\ell_p$ norm, specifically $\ell_4 \leq \ell_2$.

		Next, for~\eqref{inequality:var_squared}, we have 
		\begin{align*}
			\sum_{\vectorIndex = 1}^\alphabetSize
			\mathrm{Var}(\private{Y}_{1 \vectorIndex})^2
			%%%%%%%%%%%%%%%%%%%%%%%%%%%%%%%%%%%%%
			\leq ~&
			\sum_{\vectorIndex = 1}^\alphabetSize
			\bigl(
			\privacyParameterrappor 
			\probVecElement{\rvTwo}{\vectorIndex}
			+ \smallNumberrappor
			\bigr)^2
			\bigl\{
			1
			-
			\bigl(
			\privacyParameterrappor 
			\probVecElement{\rvTwo}{\vectorIndex}
			+ \smallNumberrappor
			\bigr)
			\bigr\}^2	
			%%%%%%%%%%%%%%%%%%%
			\\\stackrel{(a)}{\leq} \hskip 1.57mm &
			\sum_{\vectorIndex = 1}^\alphabetSize
			\bigl(
			\privacyParameterrappor 
			\probVecElement{\rvTwo}{\vectorIndex}
			+ \smallNumberrappor
			\bigr)^2
			%%%%%%%%%%%%%%%%%%%%%%%
			\\\stackrel{(b)}{\leq} \hskip 1.57mm &
			\sum_{\vectorIndex = 1}^\alphabetSize
			\bigl(
			2\privacyParameterrappor 
			\probVecElement{\rvTwo}{\vectorIndex}^2
			+ 2\smallNumberrappor^2
			\bigr)
			%%%%%%%%%%%%%%%%%%%%%%
			\\= \hskip 1.8mm &
			2
			\privacyParameterrappor^2
			\|\probVec_\rvTwo\|_2^2
			+
			2
			\smallNumberrappor^2
			\alphabetSize
			%%%%%%%%%%%%%%%%%%%%%%
			\\ \leq \hskip 1.75mm &
			2
			\privacyParameterrappor^2
			\normSqMultinomMax
			+
			2
			\smallNumberrappor^2
			\alphabetSize
			,
		\end{align*}
		where step $(a)$ uses the fact that
		\( 0 < \privacyParameterrappor \probVecElement{\rvTwo}{\vectorIndex} + \smallNumberrappor < 1 \) for any $\vectorIndex \in [\alphabetSize]$,
		and step $(b)$ uses the inequality $(a+b)^2 \leq 2a^2 + 2b^2.$
		Finally, for~\eqref{inequality:cov_squared}, we have
		\begin{align*}
			\sum_{1 \leq \vectorIndex \neq \vectorIndex' \leq \alphabetSize}
			\mathrm{Cov}(\private{Y}_{1 \vectorIndex}, \private{Y}_{1 \vectorIndex'})^2
			%%%%%%%%%%%%%%%%%%%%%%%%%%%%%%
			\stackrel{(a)}{=} \hskip 1.57mm &
			\sum_{1 \leq \vectorIndex \neq \vectorIndex' \leq \alphabetSize}
			\privacyParameterrappor^4
			\;
			\probVecElement{\rvTwo}{\vectorIndex}^2
			\probVecElement{\rvTwo}{\vectorIndex'}^2
			%%%%%%%%%%%%%%%%%%%%%%%%%%%%%%
			\\ \stackrel{(b)}{\leq} \hskip 1.57mm &
			\privacyParameterrappor^2
			\sum_{\vectorIndex=1}^\alphabetSize
			\probVecElement{\rvTwo}{\vectorIndex}^2
			\sum_{\vectorIndex'=1}^\alphabetSize
			\probVecElement{\rvTwo}{\vectorIndex'}^2
			%%%%%%%%%%%%%%%%%%%%%%%%%%%%%%%%
			\\ = \hskip 1.77mm &
			\privacyParameterrappor^2 
			\|\probVec_\rvTwo\|_2^4
			%%%%%%%%%%%%%%%%%%%%%%%%%%%%%%%%
			\\ \stackrel{(c)}{\leq} \hskip 1.36mm &
			\privacyParameterrappor^2
			\|\probVec_\rvTwo\|_2^2
			%%%%%%%%%%%%%%%%%%%%%%%%%%%%%%%%
			\\ \leq \hskip 1.72mm&
			\privacyParameterrappor^2
			\normSqMultinomMax,
		\end{align*}
		where
		step $(a)$ uses Lemma~\ref{lemma:rappor_var},
		step $(b)$ uses the fact that $0 < \privacyParameterrappor <1$ for any $\privacyParameter>0$,
		and
		step $(c)$ uses the fact that $\|\probVec_\rvTwo\|_2^2 \leq 1$.
	\end{proof}
	\subsection{Construction of Multivariate Haar Wavelet Basis}\label{appendix:basis}
	%Last revised: september 2024
	We  outline the construction of multivariate Haar wavelets, following \citet{Gine2015Besov}, Section 4.3.6, and \citet{Autin2010Wavlet}, Section 2.
	For  $f:[0,1] \to \mathbb{R}$ and integers $u, v$, define the re-scaled and shifted version as
	$f_{u,v}(x) := 2^{u/2}f(2^u x - v)$.
	Define the univariate Haar scaling and wavelet functions from $[0,1]$ to $\mathbb{R}$ as
	\begin{align*}
		\wavFatherFunc(x)
		:=
		\mathds{1}(0 \leq x < 1)
		\quad
		\text{and}
		\quad
		\wavMotherFunc(x)
		:=
		\mathds{1}(0 \leq x < 1/2) - \mathds{1}(1/2\leq x < 1),
	\end{align*}
	respectively.
	A multivariate Haar wavelet basis is indexed by its prime resolution level $\primResLev \in \mathbb{N}_0$, with each element characterized by up to three (multi) indices that range over:
	\begin{enumerate}
		\item Re-scaling: $\mathbb{N} \setminus [\primResLev-1]$,
		\item Shifting: $\Lambda(j) := \{0, 1, \ldots, (2^j - 1)\}^\dimDensity$, for an  integer $j \geq \primResLev$,
		\item On-off: $\mathcal{I} := \{0, 1\}^\dimDensity \setminus \mathbf{0}$,
	\end{enumerate}
	respectively.
	For any $\wavFatherIndex = (k_1, \ldots, k_\dimDensity) \in \Lambda(\primResLev)$, let $\wavFatherFunc_{\primResLev, \wavFatherIndex} : [0,1]^\dimDensity \to \mathbb{R}$ be a tensor product of the re-scaled and shifted $\wavFatherFunc$'s, evaluated at $\mathbf{x} = (x_1, \ldots, x_\dimDensity) \in [0,1]^\dimDensity$:
	\begin{equation*}
		\wavFatherFunc_{\primResLev, \wavFatherIndex}(\vectorize{x})
		:=
		\prod_{\dimensionIndex = 1}^{\dimDensity}
		\wavFatherFunc_{\primResLev, \wavFatherUnivIndex_\dimensionIndex}(x_\dimensionIndex).
	\end{equation*}
	For a resolution level $\resLev \geq \primResLev$, and for any $\wavMotherIndex = (\ell_1, \ldots, \ell_\dimDensity) \in \Lambda(\resLev)$,
	let $
	\wavMotherFunc_{\resLev, \wavMotherIndex}^{\wavMotherBooleanIndex}
	:
	[0,1]^\dimDensity \to \mathbb{R}
	$
	be a mixed tensor product of re-scaled and shifted $\wavFatherFunc$'s and $\wavMotherFunc$'s, evaluated at $\mathbf{x} = (x_1, \ldots, x_\dimDensity) \in [0,1]^\dimDensity$:
	\begin{equation*}
		\wavMotherFunc_{\resLev, \wavMotherIndex}^{\wavMotherBooleanIndex}
		(\vectorize{x})
		:=
		\prod_{{\dimensionIndex}=1}^\dimDensity
		\bigl\{
		\wavFatherFunc_{\resLev, \wavMotherUnivIndex_\dimensionIndex}
		(x_\dimensionIndex)
		\bigr\}^{1-\wavMotherBooleanUnivIndex_\dimensionIndex}
		\bigl\{
		\wavMotherFunc_{\resLev, \wavMotherUnivIndex_\dimensionIndex}
		(x_\dimensionIndex)
		\bigr\}^{\wavMotherBooleanUnivIndex_{\dimensionIndex}}.
	\end{equation*}
	The Haar multivariate wavelet basis at prime resolution level $\primResLev$ is defined as 
	\begin{equation}\label{appendix:def:haar_daughter}
		\Phi_{\primResLev}
		\cup
		\bigl(
		\bigcup_{\resLev \geq \primResLev}
		\Psi_{\resLev} 
		\bigr),
		\text{ where }\Phi_{\primResLev}
		:=
		\{\wavFatherFunc_{\primResLev, \wavFatherIndex}\}_{\wavFatherIndex \in \Lambda(\primResLev)}
		\text{ and }
		\Psi_{\resLev}
		:=
		\{
		\wavMotherFunc_{\resLev, \wavMotherIndex}^{\wavMotherBooleanIndex}
		\}_{
			\wavMotherIndex \in \Lambda(\resLev),
			\wavMotherBooleanIndex \in \mathcal{I}
		}.
	\end{equation}

	\subsection{Density Discretization Error Analysis}\label{appendix:disc_error}
	This section analyzes the error arising from comparing discretized multinomial probability vectors instead of the original multivariate densities.
	Given the number of bins $\binNum$, let ${B_1, \dots, B_{\binNum^{\dimDensity}}}$ enumerate  $\dimDensity$-dimensional hypercubes with side length $1/\binNum$. For a density $f_{\vectorize{Y}}$, its step function approximation is defined as:
	\begin{equation}\label{step_function_approx}
		\hat{f}_{\vectorize{Y}}(\vectorize{y})
		:=
		\sum_{ \vectorIndex \in [\binNum^\dimDensity] }
		\mathds{1}(\mathbf{y} \in B_\vectorIndex)
		\binNum^\dimDensity
		\int_{
			B_\vectorIndex
		}
		f_{\vectorize{Y}}(\vectorize{t})\;
		d\vectorize{t}.
	\end{equation}
	Similarly, define $\hat{f}_{\vectorize{Z}}$ from $f_{\vectorize{Z}}$. 
	Then it holds that $\| \probVec_{\vectorize{\rvTwo}} - \probVec_{\vectorize{\rvThree}}\|_2^2 = \binNum^{-\dimDensity}
	\vert\kern-0.25ex
	\vert\kern-0.25ex
	\vert
	\hat{f}_{\vectorize{Y}} - \hat{f}_{\vectorize{Z}}
	\vert\kern-0.25ex
	\vert\kern-0.25ex
	\vert_{\EllTwo}
	^2$.
	We analyze the differenece between 
	$\vert\kern-0.25ex
	\vert\kern-0.25ex
	\vert
	\hat{f}_{\vectorize{Y}} - \hat{f}_{\vectorize{Z}}
	\vert\kern-0.25ex
	\vert\kern-0.25ex
	\vert_{\EllTwo}$
	and
	$\vert\kern-0.25ex
	\vert\kern-0.25ex
	\vert
	f_{\vectorize{Y}} - f_{\vectorize{Z}}
	\vert\kern-0.25ex
	\vert\kern-0.25ex
	\vert_{\EllTwo}$
	when $(f_{\vectorize{Y}} - f_{\vectorize{Z}})$ lies in $\holderBall$ or $\besovBall{}{\infty}$.
	For both cases, the discretization error scales as $\binNum^{-\smoothness}$, where we recall that $\smoothness$ is the smoothness parameter. For the H\"{o}lder density case, the analysis is a simple application of  Lemma 7.2 of~\citet{Arias-Castro2018RememberDimension}, rephrased below:
	\begin{lemma}\label{lemma:arias}
		For a function $h \in \holderBall$, let $\hat{h}$ be its step function approximation as in~\eqref{step_function_approx}, with $\kappa$ bins.
		Then there exist positive constants $C_1$ and $C_2$ depending only on $\dimDensity, \smoothness$ and  $\ballRadius$, but not on $h$, such that
		\begin{align*}
			\vert\kern-0.25ex
			\vert\kern-0.25ex
			\vert 
			\hat{h}_\binNum
			\vert\kern-0.25ex
			\vert\kern-0.25ex
			\vert_{\EllTwo}
			\geq
			C_1 \vertiii{h}_{\EllTwo}-C_2 \binNum^{-s}.
		\end{align*}
	\end{lemma}
	\noindent
	Substituting
	$(f_{\vectorize{Y}} - f_{\vectorize{Z}}) \in \holderBall $ into $h$ in Lemma~\ref{lemma:arias}, the discretization error is characterized as follows:
	\begin{align*}
		\|\probVec_{\vectorize{\rvTwo}}- \probVec_{\vectorize{\rvThree}}\|_2
		\geq
		\binNum^{-\dimDensity/2}
		\bigl(
		C_1
		\vert\kern-0.25ex
		\vert\kern-0.25ex
		\vert
		f_{\vectorize{Y}} - f_{\vectorize{Z}}
		\vert\kern-0.25ex
		\vert\kern-0.25ex
		\vert_{\EllTwo}
		-
		C_2 \binNum^{-\smoothness}
		\bigr).
		\numberthis\label{equation:disc_error_holder_utilize}
	\end{align*}
	\noindent
	For the Besov density case, we derive a similar lemma as follows:
	\begin{lemma}\label{lemma:besov_discretization_error}
		Define $\besovBall{}{\infty}$ using Haar multivariate wavelet basis at prime resolution level $\primResLev$ constructed in Appendix~\ref{appendix:basis}.
		For a function $(f_{\vectorize{Y}} - f_{\vectorize{Z}}) \in \besovBall{}{\infty}$,   the following error bound holds:
		\begin{equation}\label{equation:disc_error_besov_utilize}
			\|\probVec_{\vectorize{\rvTwo}}- \probVec_{\vectorize{\rvThree}}\|_2
			\geq
			\binNum^{-\dimDensity/2}
			\bigl(
			\vert\kern-0.25ex
			\vert\kern-0.25ex
			\vert
			f_{\vectorize{Y}} - f_{\vectorize{Z}}
			\vert\kern-0.25ex
			\vert\kern-0.25ex
			\vert_{\EllTwo}
			-
			\ballRadius \binNum^{-\smoothness}
			\bigr),
		\end{equation}
		where $\kappa = 2^\primResLev$ and $\probVec_{\vectorize{\rvTwo}}$ and $\probVec_{\vectorize{\rvThree}}$ are binned with side length $1/\kappa$.
	\end{lemma}
    
    \vskip 1em
    
	\begin{proof}
		Let $\Phi_{\primResLev}
		\cup
		(
		\bigcup_{\resLev \geq \primResLev}
		\Psi_\resLev
		)$ denote the  Haar multivariate wavelet basis that defines $ \besovBall{2}{\infty}$.
		The analysis proceeds in two steps. First, we show that the sum of squared coefficients from projecting $(f_{\vectorize{Y}} - f_{\vectorize{Z}})$ onto span($\Phi_{\primResLev}$) equals the scaled $\ell_2$ distance between the probability vectors:
		\begin{align*}
			\sum_{\phi \in \Phi_{\primResLev}}
			%squared coefficient
			\coef^2_\wavFatherFunc
			(f_{\vectorize{Y}} - f_{\vectorize{Z}})
			%%%%%%%%%%%%%%%%%%%%%%%%%%%%%%%%%%%%%%%%%%%%%%%%%%%%%%%%
			=~&
			\sum_{(k_1, \ldots k_\dimDensity) \in \Lambda(\primResLev)}
			\left(
			\int_{[0,1]^\dimDensity}
			\bigl(
			f_{\vectorize{Y}}(\vectorize{x}) - f_{\vectorize{Z}}(\vectorize{x})
			\bigr)
			\prod_{\dimensionIndex = 1}^{\dimDensity}
			\wavFatherFunc_{\primResLev, \wavFatherUnivIndex_\dimensionIndex}(x_\dimensionIndex)
			\;
			d\vectorize{x}
			\right)^2
			%%%%%%%%%%%%%%%%%%%%%%%%%%%%
			\\\stackrel{(a)}{=} \hskip 1.57mm &
			\binNum^\dimDensity
			\sum_{(k_1, \ldots k_\dimDensity) \in \Lambda(\primResLev)}
			\left(
			\int_{[0,1]^\dimDensity}
			\bigl(
			f_{\vectorize{Y}}(\vectorize{x}) - f_{\vectorize{Z}}(\vectorize{x})
			\bigr)
			\prod_{p = 1}^{\dimDensity}
			\mathds{1}
			\biggl(
			\frac{k_p}{\binNum}
			\leq
			x_p
			<
			\frac{k_p + 1}{\binNum}
			\biggr)
			\;
			d\vectorize{x}
			\right)^2
			%%%%%%%%%%%%%%%%%%%%%%%%%%%%
			\\=~&
			\binNum^\dimDensity
			\sum_{\vectorIndex \in [\binNum^\dimDensity]}
			\left(
			\int_{B_\vectorIndex}
			\bigl(
			f_{\vectorize{Y}}(\vectorize{x}) - f_{\vectorize{Z}}(\vectorize{x})
			\bigr)
			\;
			d\vectorize{x}
			\right)^2
			%%%%%%%%%%%%%%%%%%%%%%%%%%%%%%%%%%%
			\\=~&
			\binNum^\dimDensity
			\|\probVec_{\rvTwo} - \probVec_{\rvThree}\|_2^2,
			\numberthis \label{eq:equiv_wavCoef_discProb_full}
		\end{align*}
		where step $(a)$ uses the definition $ \wavFatherFunc_{\primResLev, \wavFatherIndex}(\vectorize{x}) = \prod_{\dimensionIndex = 1}^{\dimDensity} \wavFatherFunc_{\primResLev, \wavFatherUnivIndex_\dimensionIndex}(x_\dimensionIndex)$, with $\wavFatherFunc: [0,1] \to \{0,1\}$ defined as $\wavFatherFunc(x) = \mathds{1}(0 \leq x < 1)$,
		and $\wavFatherFunc_{\primResLev, \wavFatherUnivIndex_\dimensionIndex}(x_\dimensionIndex) = 2^{\primResLev/2}f(2^\primResLev x_p - k_p)$ indicates its re-scaled and shifted version. See Appendix~\ref{appendix:basis} for details.
		Since ${\Phi}_{\primResLev} \cup \bigl( \bigcup_{\resLev \geq \primResLev} {\Psi}_\resLev \bigr)$ forms an orthonormal basis of $\LtwoSpace$, the approximation error is given by the sum of squared projection coefficients of $(f_{\vectorize{Y}} - f_{\vectorize{Z}})$ onto $\text{span}\bigl(\bigcup_{\resLev \geq \primResLev} \Psi_\resLev \bigr)$.
		Bounding this term is the second step of our analysis, which begins by noting that from Definition~\ref{def:besov_norm}, for any $j \in \mathbb{N}_0$, the sum of squared wavelet coefficients is bounded as
		\begin{align*}
			\sum_{\wavMotherFunc \in {\Psi}_{\resLev}}
			\wavGenericMotherCoef^2(
			f_{\vectorize{Y}}
			-
			f_{\vectorize{Z}}
			)
			\leq
			2^{ -2 \resLev \smoothness}
			R^2.
			\numberthis\label{equation:besov_bin_error_bound}
		\end{align*}
		Then the approximation error is bounded as follows:
		\begin{align*}
			\vert\kern-0.25ex
			\vert\kern-0.25ex
			\vert 
			f_{\vectorize{Y}} - f_{\vectorize{Z}}
			\vert\kern-0.25ex
			\vert\kern-0.25ex
			\vert_{\EllTwo}^2
			-
			\binNum^\dimDensity
			\|\probVec_{\rvTwo} - \probVec_{\rvThree}\|_2^2
			&=
			\sum_{\resLev = \primResLev}^\infty
			\sum_{\psi \in {\Psi}_{\resLev} }
			\wavGenericMotherCoef^2(
			f_{\vectorize{Y}}
			-
			f_{\vectorize{Z}}
			)
			%%%%%%%%%%%%%%%%%%%%%%%%%%
			\\ &
			\stackrel{(a)}{\leq}
			\sum_{\resLev = \primResLev}^\infty
			2^{-2\resLev \smoothness} R^2
			\\ &
			\stackrel{(b)}{=}
			R^2 
			\frac{2^{-2 \primResLev \smoothness}
			}{
				1 - 2^{-2\smoothness}
			}
			\leq
			R^2 2^{-2 \primResLev \smoothness}
			=
			R^2 \binNum^{-2\smoothness}
			%%%%%%%%%%%%%%%%%%%%%%%%%%
			,
		\end{align*}
		where step~$(a)$ uses~\eqref{equation:besov_bin_error_bound},
		and step~$(b)$ uses an infinite geometric series.
		Applying $\sqrt{x+y} \leq \sqrt{x} + \sqrt{y}$ to the inequality above, we get
		\begin{equation*}
			\vert\kern-0.25ex
			\vert\kern-0.25ex
			\vert 
			f_{\vectorize{Y}} - f_{\vectorize{Z}}
			\vert\kern-0.25ex
			\vert\kern-0.25ex
			\vert_{\EllTwo} 
			= 
			\binNum^{\dimDensity/2}
			\|\probVec_{\rvTwo} - \probVec_{\rvThree}\|_2 + R 2^{- \primResLev \smoothness}
			\binNum^{\dimDensity/2}
			\|\probVec_{\rvTwo} - \probVec_{\rvThree}\|_2 + R \binNum^{- \smoothness}.
		\end{equation*}
	\end{proof}
	\subsection{Gaussian Approximation of One-Sample U-statistic}
	The theorems presented here are used to establish a negative result for our statistic in~\eqref{def:statistic_elltwo} when applied with the generalized randomized response mechanism. First, Theorem~\ref{appendix:genrr_powerless:theorem_null} restates a result from \citet{kim_multinomial_2020}, which demonstrates the asymptotic normality of a one-sample U-statistic, which is a special case of our statistic in~\eqref{def:statistic_elltwo} with either $n_1 = \infty$ or $n_2 = \infty$ under the uniform null hypothesis.
	\begin{theorem}[{\citealp[Corollary 3.3]{kim_multinomial_2020}}]\label{appendix:genrr_powerless:theorem_null}
		Consider a multinomial goodness-of-fit test where the null hypothesis is a discrete uniform distribution with $\alphabetSize$ categories, whose probability vector is denoted as $\boldsymbol{\pi}_0$. Let $\{\mathbf{Y}_i \}_{i=1}^\sampleSize$ be a random sample represented in one-hot vector form drawn from $\boldsymbol{\pi}_0$. Define the U-statistic as:
		\begin{equation}\label{def:U_stat_kim_2020}
			U_I
			:=
			{\binom{n}{2}}^{-1} \sum_{1 \leq i < j \leq n} (\mathbf{Y}_i - \boldsymbol{\pi}_0)^\top (\mathbf{Y}_j - \boldsymbol{\pi}_0). 
		\end{equation}
		If $\sampleSize / \sqrt{\alphabetSize} \to \infty$, then the test statistic $U_I$ has the following asymptotic normality:
		\begin{equation*}
			\sqrt{\binom{n}{2}} \frac{U_I}{
				\sqrt{
					(1 - 1/\alphabetSize)/
					\alphabetSize
			}} \overset{d}{\longrightarrow} \mathcal{N}(0, 1).   
		\end{equation*}
	\end{theorem}

	The next theorem studies the limiting distribution of $U_I$ under the alternative hypothesis. 
	\begin{theorem}[{\citealp[Theorem 3.3]{kim_multinomial_2020}}]\label{appendix:genrr_powerless:theorem:alternative_dist}%, weak signal-to-noise ratio case
		Consider the same testing problem as Theorem~\ref{appendix:genrr_powerless:theorem_null}.
		Let $\Sigma := \mathrm{Cov}(\vectorize{Y}_1)$. Under the alternative
		where the data is generated from a multinomial distribution with a probability vector $\boldsymbol{\pi} \neq \boldsymbol{\pi}_0$,
		suppose that the following conditions hold as $\sampleSize, \alphabetSize \to \infty$:
		\begin{enumerate}
			%%%%%%%%%%%%%%%
			\item[C1.] 
			$
			\dfrac{
				\mathrm{tr}(\Sigma^4)
			}{
				\bigl\{
				\mathrm{tr}(\Sigma^2)
				\bigr\}^2   
			}
			\to
			0.
			$
			%%%%%%%%%%%%%
			\item[C2.]
			$
			\dfrac{
				\mathbb{E}[\bigl\{ (\vectorize{\rvTwo}_1 - \boldsymbol{\pi}_0)^\top (\vectorize{\rvTwo}_2
				-
				\boldsymbol{\pi}_0)
				\bigr\}^4]
				+
				\sampleSize_1
				\mathbb{E}
				[\bigl\{ (\vectorize{\rvTwo}_1 - \probVec_{\rvThree})^\top (\vectorize{\rvTwo}_2
				-
				\probVec_{\rvThree}
				\bigr\}^2
				\bigl\{ (\vectorize{\rvTwo}_1 - \probVec_{\rvThree})^\top (\vectorize{\rvTwo}_3
				-
				\probVec_{\rvThree})
				\bigr\}^2]
			}{
				\sampleSize^2
				\left\{
				\mathrm{tr}(\Sigma^2)
				\right\}^2
			} \to 0.
			$
			\item[C3.]
			$(\boldsymbol{\pi} - \boldsymbol{\pi}_0)^\top  \Sigma  (\boldsymbol{\pi} - \boldsymbol{\pi}_0) < \infty$.
		\end{enumerate}
		Then the test statistic $U_I$ in~\eqref{def:U_stat_kim_2020} has the following asymptotic normality:
		\begin{equation}\label{one_sample_asymptotic_normal_theorem}
			\sqrt{
				\frac{
					\sampleSize(\sampleSize-1)
				}{
					2
				} 
			}
			\frac{
				U_{I}
				-
				\| \boldsymbol{\pi}  - \boldsymbol{\pi}_0\|_2^2}{
				\sqrt{
					\mathrm{tr}(\Sigma^2)
					+
					2(\sampleSize - 1)
					(\boldsymbol{\pi} - \boldsymbol{\pi}_0)^\top 
					\Sigma
					(\boldsymbol{\pi} - \boldsymbol{\pi}_0)	
				}
			}
			\stackrel{d}{\to}
			\mathcal{N}(0,1).
		\end{equation}
	\end{theorem}
	
	\section{Proof of Lemma~\ref{lemma:LapUdiscLapULDP}}
	\label{proof:LapUdiscLapULDP}
	This section provides privacy guarantee proof for the mechanisms proposed in Section~\ref{subsection:twosample_disc_upperbound}. The privacy proof for the \texttt{RAPPOR} mechanism is given in Section 3.2 of \citet{duchi2013local}. The privacy proof for the Laplace mechanism is a minor modification of the proof of Lemma 4.2 from \citet{Lam-Weil2021MinimaxConstraint}, adjusting the domain of the raw samples. Thus, we only present the privacy proof for the discrete Laplace mechanism, defined in Definition~\ref{def:DiscLapU_formal}.
	Its outline  resembles the proof of Lemma 4.2 from \citet{Lam-Weil2021MinimaxConstraint}.

    \vskip 1em
    	
	\begin{proof}
	We note that the conditional distribution of the $\vectorIndex$th entry of its $\privacyParameter$-LDP view,
	denoted as $\private{X}_{im}$,
	is a discrete Laplace distribution with parameter
	$\discLapUParam = \exp({
		-\privacyParameter/(2\sqrt{\alphabetSize}}))$,
	shifted by $\sqrt{\alphabetSize}\mathds{1}(\rvX_i = \vectorIndex)$.
	Note again that for $m \neq m'$, $\private{X}_{im}$ and $\private{X}_{im'}$ are independent.
	With slight abuse of notation, we denote the conditional density function of $\rVecXPriv_{i}$ as $\privacyDensity_\sampleIndexOne(\cdot \,|\, \cdot)$, which can be written as 
	\begin{align*}
		\privacyDensity_\sampleIndexOne(
		\rVecXPrivRealized
		\,|\,
		\rvXRealized
		)
		= 
		\prod_{\vectorIndex=1}^\alphabetSize
		\frac{
			1 - \discLapUParam
		}{
			1 + \discLapUParam
		}
		\discLapUParam^{
			|
			\private{x}_{m}
			-
			\sqrt{\alphabetSize} \mathds{1}(\rvXRealized = \vectorIndex)
			|
		}.
	\end{align*}
	Then for all $\rVecXPrivRealized =(\rVecXPrivRealized_1,\ldots,\rVecXPrivRealized_\alphabetSize)^\top \in \mathbb{R}^\alphabetSize$
	and all $\rvXRealized, \rvXRealized' \in [\alphabetSize]$, we have
	\begin{align*}
		%first term: density ratio
		\frac{
			\privacyDensity_\sampleIndexOne(
			\rVecXPrivRealized
			\,|\,
			\rvXRealized
			)
		}{
			\privacyDensity_\sampleIndexOne(
			\rVecXPrivRealized
			\,|\,
			\rvXRealized'
			)
		}
		% second term: plug in the density definition. 1-zeta / 1+zeta term is canceled by taking ratio.
		~=~ &
		\prod_{\vectorIndex=1}^\alphabetSize
		\discLapUParam^{
			|
			\private{x}_{m}
			-
			\sqrt{\alphabetSize} \mathds{1}(\rvXRealized = \vectorIndex)
			|
			-
			|
			\private{x}_{m}
			-
			\sqrt{\alphabetSize} \mathds{1}(\rvXRealized' = \vectorIndex)
			|
		}
		%  
		%third term
		\\ = ~ &
		\discLapUParam^{
			\sum_{\vectorIndex=1}^\alphabetSize
			|
			\private{x}_{m}
			-
			\sqrt{\alphabetSize} \mathds{1}(\rvXRealized = \vectorIndex)
			|
			-
			|
			\private{x}_{m}
			-
			\sqrt{\alphabetSize} \mathds{1}(\rvXRealized' = \vectorIndex)
			|
		}
		% 
		%fourth term: step (i)
		\\ \stackrel{(a)}{\leq} \hskip 1.57mm &
		\bigl(
		\discLapUParam^{-1}
		\bigr)^{\sum_{\vectorIndex=1}^\alphabetSize
			\bigl|
			|
			\private{x}_{m}
			-
			\sqrt{\alphabetSize} \mathds{1}(\rvXRealized = \vectorIndex)
			|
			-
			|
			\private{x}_{m}
			-
			\sqrt{\alphabetSize} \mathds{1}(\rvXRealized' =\vectorIndex)
			|
			\bigr|
		}
		%fifth term: step (ii)
		\\ \stackrel{(b)}{\leq} \hskip 1.57mm &
		\bigl(
		\discLapUParam^{-1}
		\bigr)^{\sum_{\vectorIndex=1}^\alphabetSize 
			\bigl\{
			\sqrt{\alphabetSize} \mathds{1}(\rvXRealized = \vectorIndex)
			+
			\sqrt{\alphabetSize} \mathds{1}(\rvXRealized' =\vectorIndex)
			\bigr\}
		}
		%sixth term: step (iii)
		\\ \stackrel{(c)}{\leq} \hskip 1.57mm &
		\bigl(
		\discLapUParam^{-1}
		\bigr)^{2\sqrt{\alphabetSize}} = e^{\privacyParameter},
	\end{align*}
	where step~$(a)$ uses the fact that $\discLapUParam^x \leq \discLapUParam^{-|x|}$ for all $x \in \mathbb{R}$ since $\discLapUParam \in (0,1)$, step~$(b)$ uses the reverse triangle inequality,
	and step~$(c)$ holds since
	$\mathds{1}(\rvXRealized = \vectorIndex) \neq 0$ for only a single value of $\vectorIndex$. Using  the inequality above, for any Borel set $A \in \mathcal{B}(\mathbb{R}^d)$, and for any $x, x' \in [k]$, we have:
	\begin{equation*}
		\frac{Q_i( A \,|\,  x)}{Q_i(A \,|\,  x')}
		= 
		\frac{\int_A q_{i}(\private{\vectorize{x}} \,|\, x) \, d\private{\vectorize{x}} 
		}{
			\int_A q_{i}(\private{\vectorize{x}} \,|\, x') \, d\private{\vectorize{x}}}
		\leq 
		\frac{\int_A q_{i}(\private{\vectorize{x}} \,|\, x') e^{\alpha} \, d\private{\vectorize{x}}
		}{
			\int_A q_{i}(\private{\vectorize{x}} \,|\, x) e^{-\alpha} \, d\private{\vectorize{x}}} = 
		\frac{Q_i( A \,|\,  x')}{Q_i(A \,|\,  x)}e^{2\alpha}.
	\end{equation*}
	This completes the proof of the guarantee for the discrete Laplace mechanism. 
\end{proof}
\section{Proof of Theorem~\ref{theorem:twosample_multinomial_rates}}\label{proof:twosample_multinomial_rates}
	In this section, we prove the minimax rate result for multinomial testing presented in Theorem~\ref{theorem:twosample_multinomial_rates}, first focusing on the upper bound result followed by the lower bound result. 
	
	\subsection{Upper Bound} \label{sec: upper bound}
	This section demonstrates that the permutation test with the U-statistic proposed in~\eqref{def:statistic_elltwo}, in conjunction with one of the mechanisms $\{\texttt{LapU}, \texttt{DiscLapU}, \texttt{RAPPOR}\}$, achieves a tight upper bound presented in Theorem~\ref{theorem:twosample_multinomial_rates}.
	Since the permutation procedure controls the type I error, it suffices to prove that the condition in Theorem~\ref{theorem:twosample_multinomial_rates} guarantees type II error control.
	Let 
	\begin{equation}\label{definition:maximum_norm_squared}
		\normSqMultinomMax
		:=
		\max \{
		\|\probVec_{\rvTwo}\|_2^2,
		\|\probVec_{\rvThree}\|_2^2
		\}.
	\end{equation}
	The control of the type II error  is then rephrased in the following lemma.
	\begin{lemma}\label{appendix:lemma:discrete_upper_bound}
		Assume the settings of Theorem~\ref{theorem:twosample_multinomial_rates}.
		%Define a constant:
		%\begin{align*}
		%C_u(\gamma, \beta) := \max \biggl\{	 \sqrt{ \frac{128}{\beta} },~ {\left( \frac{4608}{\gamma \beta} \right)}^{1/4} \biggr\}. 
		%\end{align*}
		For each of the mechanisms $\{\textnormal{\texttt{LapU}}, \textnormal{\texttt{DiscLapU}}, \textnormal{\texttt{RAPPOR}}\}$, there exists a constant $C_u(\gamma, \beta)$ such that the type II error of the permutation test of size $\maxErrorTypeOne$ with U-statistic in~\eqref{def:statistic_elltwo} is uniformly controlled by $\maxErrorTypeTwo$ over $\mathcal{P}_{1,\mathrm{multi}}$ if
		\begin{align*}
			\separation_{\sampleSize_1, \sampleSize_2}
			\geq
			C_u(\gamma, \beta)
			\left(
			\frac
			{\alphabetSize^{1/4}}
			{(\sampleSize_1 \privacyParameter^2)^{1/2}}
			\vee
			\frac
			{{b}^{1/4}}
			{\sampleSize_1^{1/2}}
			\right).
			\numberthis \label{eq:two_moments_two_sample_cond_merge}
		\end{align*}
	\end{lemma} 
	Since ${b} \leq 1$, Lemma~\ref{appendix:lemma:discrete_upper_bound} proves the upper bound result of Theorem~\ref{theorem:twosample_multinomial_rates}. To prove Lemma~\ref{appendix:lemma:discrete_upper_bound}, we leverage the two moments method \citep[Theorem 4.1 of][]{kim_minimax_2022}, rephrased in Lemma~\ref{theorem:twosampleTwoMomentsMethod}.
	This method states that if the test statistic’s expectation is sufficiently larger than a variance proxy, then uniform type II error control is achieved for the permutation test with a two-sample U-statistic of order 2.
	Its  key advantage is that the condition  bypasses the randomness arising from permutations. It lets us avoid the complex analysis typically required for permuted statistics.
	
	To proceed, consider the following kernel:
	\begin{equation}\label{equation:twosample_kernel}
		h_{ts}(\vectorize{y}_1, \vectorize{y}_2; \vectorize{z}_1,\vectorize{z}_2) = \vectorize{y}_1^\top \vectorize{y}_2 + \vectorize{z}_1^\top \vectorize{z}_2 - \vectorize{z}_1^\top \vectorize{y}_2 - \vectorize{z}_2^\top \vectorize{y}_1,
	\end{equation}
	which defines our two-sample U-statistic $U_{\sampleSize_1,\sampleSize_2}$ in~\eqref{def:statistic_elltwo}.
	Its symmetrized version, denoted $h_{ts}$, is defined as: 
	\begin{align*}
		\kernelTwoSampleSym(\vectorize{y}_1,\vectorize{y}_2;\vectorize{z}_1,\vectorize{z}_2
		) := 
		\frac{1}{2!2!}
		\sum_{1 \leq i_1 \neq  i_2 \leq n_1}
		\sum_{1 \leq j_1 \neq  j_2 \leq n_2}
		\kernelTwoSample(\vectorize{y}_{i_1}, \vectorize{y}_{i_2}; \vectorize{z}_{j_1}, \vectorize{z}_{j_2}
		).
	\end{align*}
	For the upper bound analysis, we use a U-statistic represented by $\kernelTwoSampleSym$, which is equivalent to our original statistic in $U_{\sampleSize_1,\sampleSize_2}$~\eqref{def:statistic_elltwo}.
	Recall that under the LDP constraint, the raw samples
	$\sampleSets{\rvY}{\sampleIndexOne}{\sampleSize_1}$
	and
	$\sampleSets{\rvZ}{\sampleIndexTwo}{\sampleSize_2}$
	are generated from $P = (P_\rvY, P_\rvZ)$ and then transformed into LDP-views
	$\sampleSets{\tilde{\vectorize{\rvTwo}}}{\sampleIndexOne}{\sampleSize_1}$
	and
	$\sampleSets{\rVecZPriv}{\sampleIndexTwo}{\sampleSize_2}$
	through an LDP mechanism $Q$.
	Let us denote the associated moments as
	\begin{align*}
		\momentTwosampleVarCondexpY
		&:=
		\mVPQ[\mE\{ \kernelTwoSampleSym (
		\tilde{\vectorize{\rvTwo}}_1,
		\vecRandomPrivTwoSampleYNumber{2}
		;
		\vecRandomPrivTwoSampleZNumber{1},
		\vecRandomPrivTwoSampleZNumber{2}
		)
		\,|\,
		\tilde{\vectorize{\rvTwo}}_1
		\}],
		%%%%%%%%%%%%%%%%%%%%%%%%%%%%%%%%%
		\\[.5em]
		%%%%%%%%%%%%%%%%%%%%%%%%%%%%%%%%%
		%
		%
		\momentTwosampleVarCondexpZ
		&:= 
		\mVPQ[\mE\{ \kernelTwoSampleSym (
		\tilde{\vectorize{\rvTwo}}_1,
		\vecRandomPrivTwoSampleYNumber{2}
		;
		\vecRandomPrivTwoSampleZNumber{1},
		\vecRandomPrivTwoSampleZNumber{2}
		)
		\,|\,
		\vecRandomPrivTwoSampleZNumber{1}
		\}],%%%%%%%%%%%%%%%%%%%%%%%%%%%%%%%%%
		\\[.5em]
		\momentTwosampleExpSquare
		&:=
		\max
		\{
		\mE[
		(
		\private{\vectorize{\rvTwo}}_1^\top 
		\vecRandomPrivTwoSampleYNumber{2}
		)^2
		], \, 
		\mE[
		(
		\private{\vectorize{\rvTwo}}_1^\top
		\private{\vectorize{\rvThree}}_1)^2], \,
		\mE[
		(
		\private{\vectorize{\rvThree}}_1^\top
		\private{\vectorize{\rvThree}}_2)^2
		]
		\}.
		\numberthis
		\label{def:moment_terms}
	\end{align*}
	\noindent
	Using these moments, we rephrase the two moments method under the setting of LDP. 
	\begin{lemma}[Two moments method]\label{theorem:twosampleTwoMomentsMethod}
		Let $U_{\sampleSize_1,\sampleSize_2}$ be a two-sample U-statistic based on the kernel given in~\eqref{equation:twosample_kernel}.
		Assume that the samples are privatized through an $\privacyParameter$-LDP mechanism $\privacyMechanism$.
		Then there exists a sufficiently large constant $C > 0$ 
		such that if
		\begin{align*}
			\mE[U_{\sampleSize_1, \sampleSize_2}]
			\geq
			C
			\sqrt{
				\max
				\left\{
				\frac{ \momentTwosampleVarCondexpY }{\beta \sampleSize_1},
				\frac{ \momentTwosampleVarCondexpZ }{\beta \sampleSize_2},
				\frac{ \momentTwosampleExpSquare }{\gamma \beta}
				\left(
				\frac{1}{\sampleSize_1} + \frac{1}{\sampleSize_2}
				\right)^2
				\right\}
			}
			\numberthis \label{eq:two_moments_original_cond_two_sample}
		\end{align*}
		for all pairs of distributions $P = (P_{Y}, P_{Z}) \in \mathcal{P}_{1,\mathrm{multi}}(\rho_{\sampleSize_1,\sampleSize_2})$,	
		then the type II error of the permutation test over $\mathcal{P}_{1,\mathrm{multi}}(\rho_{\sampleSize_1,\sampleSize_2})$ is uniformly bounded by $\beta$ as in~\eqref{eq:uniform_control}.
	\end{lemma}

	\noindent
	%\textbf{Proof Sketch.}
	%We aim to find a value of
	%$\rho_{\sampleSize_1,\sampleSize_2}$, up to constants,
	%which guarantees that the inequality~\eqref{eq:two_moments_original_cond_two_sample} of Lemma~\ref{theorem:twosampleTwoMomentsMethod} holds.
	Having stated the two moments method, our goal is to verify that inequality~\eqref{eq:two_moments_original_cond_two_sample} holds under the separation conditions described in Section~\ref{subsection:twosample_disc_upperbound}.
	We provide separate proofs for the Laplace-noise based mechanisms (\texttt{LapU} and \texttt{DiscLapU}) and the \texttt{RAPPOR} mechanism. However, both proofs follow the same two steps:
	\begin{enumerate}
		\item Derive upper bounds for the moments $\momentTwosampleVarCondexpY, \momentTwosampleVarCondexpZ$, and $\momentTwosampleExpSquare$ in~\eqref{def:moment_terms}.
		\item  Using the upper bounds established in Step 1,  show that condition~\eqref{eq:two_moments_original_cond_two_sample} in the two moments method is fulfilled as long as inequality~\eqref{eq:two_moments_two_sample_cond_merge} in Lemma~\ref{appendix:lemma:discrete_upper_bound} holds.
	\end{enumerate}
	
	We now verify the previous preliminary steps in order, first assuming the Laplace-noise based mechanisms.
	
	\subsubsection{Proof of Lemma \ref{appendix:lemma:discrete_upper_bound} Through Laplace or Discrete Laplace Mechanism}\label{proof:multinomial_upper_bound_first_moment}
	Since the analysis for Laplace and discrete Laplace mechanism is similar, we present the proof using the former (see Remark~\ref{remark: discrete laplace noise} for details).
    
    \vskip 1em
    
	\begin{proof} We follow the two steps mentioned above.
    
\vskip 1em

	\noindent \textit{Step 1: Bound  the moments from above.}
	We start by bounding the variance of conditional expectation terms.
	Recall the notation
	$
	\momentTwosampleVarCondexpY
	=
	\mVPQ[\mE\{ \kernelTwoSampleSym(
	\tilde{\vectorize{\rvTwo}}_1,
	\vecRandomPrivTwoSampleYNumber{2}
	;
	\vecRandomPrivTwoSampleZNumber{1},
	\vecRandomPrivTwoSampleZNumber{2}
	)
	\,|\,
	\tilde{\vectorize{\rvTwo}}_1
	\}]$.
	To upper bound $\momentTwosampleVarCondexpY$, we first calculate the conditional expectation of the kernel function, namely
	$A:=\mE\{ \kernelTwoSampleSym(
	\tilde{\vectorize{\rvTwo}}_1,
	\vecRandomPrivTwoSampleYNumber{2}
	;
	\vecRandomPrivTwoSampleZNumber{1},
	\vecRandomPrivTwoSampleZNumber{2}
	)
	\,|\,
	\tilde{\vectorize{\rvTwo}}_1
	\}$.
	Then we  bound the variance of $A$.
	\\
	
	For the conditional expectation, recall
	the $\vectorIndex$th component of
	$\probVec_\rvTwo, \probVec_\rvThree,
	\tilde{\vectorize{\rvTwo}}_{\sampleIndexOne},
	\tilde{\vectorize{\rvThree}}_{\sampleIndexTwo}$
	are denoted as
	$
	\probVecElement{\rvTwo}{\vectorIndex},
	\probVecElement{\rvThree}{\vectorIndex},
	\tilde{\rvTwo}_{\sampleIndexOne \vectorIndex}$
	and $\tilde{\rvThree}_{\sampleIndexTwo, \vectorIndex}$,
	respectively.
	Due to i.i.d.~assumptions, regardless of privacy mechanism, we have the following equalities and inequalities.
	First we calculate the conditional expectation:
	\begin{align*}
		A&:=
		\mE
		\bigl[
		\kernelTwoSampleSym(
		\tilde{\vectorize{\rvTwo}}_1,
		\tilde{\vectorize{\rvTwo}}_2
		;
		\tilde{\vectorize{\rvThree}}_1,
		\tilde{\vectorize{\rvThree}}_2)
		\,|\,
		\tilde{\vectorize{\rvTwo}}_1
		\bigr]
		=
		\bigl(
		\tilde{\vectorize{\rvTwo}}_1 - \mE [ \tilde{\vectorize{\rvThree}}_{1}]
		\bigr)^\top
		\bigl(
		\mE [ \tilde{\vectorize{\rvTwo}}_{1}] - \mE [ \tilde{\vectorize{\rvThree}}_{1}]
		\bigr).
	\end{align*}
	Then we calculate the unconditional expectation:
	\begin{align*}
		\mE[A]
		=
		\bigl(
		\mE [ \tilde{\vectorize{\rvTwo}}_{1}]  - \mE [ \tilde{\vectorize{\rvThree}}_{1}]
		\bigr)^\top
		\bigl(
		\mE [ \tilde{\vectorize{\rvTwo}}_{1}] - \mE [ \tilde{\vectorize{\rvThree}}_{1}]
		\bigr)
		=
		\|\mE [ \tilde{\vectorize{\rvTwo}}_{1}]
		-
		\mE [ \tilde{\vectorize{\rvThree}}_{1}]
		\|_2^2.
	\end{align*}
	Then the squared and centered conditional expectation is calculated as follows:
	\begin{align*}
		%%%%% 3. centered conditional expectation
		A - \mE[A]
		&=
		\bigl(
		\tilde{\vectorize{\rvTwo}}_1 - \mE [ \tilde{\vectorize{\rvTwo}}_{1}]
		\bigr)^\top
		\bigl(
		\mE [ \tilde{\vectorize{\rvTwo}}_{1}] - \mE [ \tilde{\vectorize{\rvThree}}_{1}]
		\bigr).
	\end{align*}
	Based on this, the variance is calculated as
	\begin{align*}
		 \mE
		\bigl[
		\bigl(
		A - \mE[A]
		\bigr)^2
		\bigr]
		%%%
		=~&
		\mE
		\bigl[
		\bigl\{
		\bigl(
		\tilde{\vectorize{\rvTwo}}_1 - \mE [ \tilde{\vectorize{\rvTwo}}_{1}]
		\bigr)^\top
		\bigl(
		\mE [ \tilde{\vectorize{\rvTwo}}_{1}] - \mE [ \tilde{\vectorize{\rvThree}}_{1}]
		\bigr)
		\bigr\}^2
		\bigr]
		\numberthis \label{appendix:equation:multinomial:upper:forrappor}
		\\
		= ~ &
		\mE
		\bigl[
		\bigl\{
		\tilde{\vectorize{\rvTwo}}_1^\top
		\bigl(
		\mE [ \tilde{\vectorize{\rvTwo}}_{1}] - \mE [ \tilde{\vectorize{\rvThree}}_{1}]
		\bigr)
		-
		\mE [ \tilde{\vectorize{\rvTwo}}_{1}]^\top
		\bigl(
		\mE [ \tilde{\vectorize{\rvTwo}}_{1}] - \mE [ \tilde{\vectorize{\rvThree}}_{1}]
		\bigr)
		\bigr\}^2
		\bigr]
		%
		% line 2. inside the E operatior: as inner product
		\\  \leq ~ &
		2 \mE 
		\bigl[
		\bigl\{
		\tilde{\vectorize{\rvTwo}}_1^\top 
		\bigl(
		\mE [ \tilde{\vectorize{\rvTwo}}_{1}] - \mE [ \tilde{\vectorize{\rvThree}}_{1}]
		\bigr)
		\bigr\}^2
		\bigr]
		+
		2 
		\bigl[
		\bigl\{
		\mE [ \tilde{\vectorize{\rvTwo}}_{1}]^\top
		\bigl(
		\mE [ \tilde{\vectorize{\rvTwo}}_{1}] - \mE [ \tilde{\vectorize{\rvThree}}_{1}]
		\bigr)
		\bigr\}^2
		\bigr],
		\numberthis \label{appendix:equation:multinomial:upper:cond_var_before_expec}
	\end{align*}
	where the last inequality uses $(x+y)^2 \leq 2x^2 + 2y^2$. Using $\|\probVec_\rvTwo\|_2 \leq 1$, we bound the second term in \eqref{appendix:equation:multinomial:upper:cond_var_before_expec} by
	\begin{align*}
		2\alphabetSize^2 \|\probVec_\rvTwo\|_2
		\|\probVec_\rvTwo - \probVec_\rvThree\|_2^2.
		\numberthis \label{appendix:equation:multinomial:upper:cond_var_term_2}
	\end{align*}
	For the first term, we leverage the structure of our privacy mechanism, which adds independent and centered noises. Applying the Cauchy--Schwarz inequality then provides an upper bound, reducing the order of $\alphabetSize$ compared to a direct application of Cauchy--Schwarz. More formally,
	\begin{align*}
		2 \mE
		\Bigl\{
		\tilde{\vectorize{\rvTwo}}_1^\top&
		\bigl(
		\mE [ \tilde{\vectorize{\rvTwo}}_{1}] - \mE [ \tilde{\vectorize{\rvThree}}_{1}]
		\bigr)
		\Bigr\}^2
		\\ = ~ &
		2\sum_{\vectorIndex=1}^\alphabetSize
		\sum_{\vectorIndex'=1}^\alphabetSize
		\mE[
		\tilde{\rvTwo}_{1 \vectorIndex}
		\tilde{\rvTwo}_{1 \vectorIndex'}
		]
		\bigl(
		\mE [ \tilde{\rvTwo}_{1 \vectorIndex}]
		-
		\mE [ \tilde{\rvThree}_{1 \vectorIndex}]
		\bigr)
		\bigl(
		\mE [ \tilde{\rvTwo}_{1 \vectorIndex'}]
		-
		\mE [ \tilde{\rvThree}_{1 \vectorIndex'}]
		\bigr) 
		%%%%%%%%%%%%%%%%%%%%
		\\ \stackrel{(a)}{=} \hskip 1.57mm &
		2\sum_{\vectorIndex=1}^\alphabetSize
		\mE[ \tilde{\rvTwo}_{1 \vectorIndex}^2 ]
		\bigl(
		\mE [ \tilde{\rvTwo}_{1 \vectorIndex}]
		-
		\mE [ \tilde{\rvThree}_{1 \vectorIndex}]
		\bigr)^2 
		%%%%%%%%%%%%%%%%%%%%
		\\ \stackrel{(b)}{=} \hskip 1.57mm &
		2\sum_{\vectorIndex=1}^\alphabetSize
		\bigl\{ \alphabetSize \probVecElement{\rvTwo}{\vectorIndex} + \LapUParam^2 \bigr\}
		\bigl(
		\mE [ \tilde{\rvTwo}_{1 \vectorIndex}]
		-
		\mE [ \tilde{\rvThree}_{1 \vectorIndex}]
		\bigr)^2
		%%%%%%%%%%%%%%%%%%%%%%%%%
		\\ =~&
		2\sum_{\vectorIndex=1}^\alphabetSize
		\alphabetSize \probVecElement{\rvTwo}{\vectorIndex}
		\bigl(
		\mE [ \tilde{\rvTwo}_{1 \vectorIndex}]
		-
		\mE [ \tilde{\rvThree}_{1 \vectorIndex}]
		\bigr)^2
		+ 
		2 \alphabetSize \LapUParam^2
		\|\probVec_\rvTwo - \probVec_\rvThree\|_2^2
		%%%%%%%%%%%%%%%%%%%%%%%%%
		\\\stackrel{(c)}{\leq}\hskip 1.57mm &
		2
		\alphabetSize^2 \|\probVec_\rvTwo\|_2
		\|\probVec_\rvTwo - \probVec_\rvThree\|_2^2
		+ 
		2 \alphabetSize \LapUParam^2
		\|\probVec_\rvTwo - \probVec_\rvThree\|_2^2,
		\numberthis \label{appendix:equation:multinomial:upper:cond_var_term_1}
	\end{align*}
	where step~$(a)$ uses $\mE[ \tilde{\rvTwo}_{1 \vectorIndex}  \tilde{\rvTwo}_{1 \vectorIndex'}] = 0$ for $m \neq m'$, 
	step~$(b)$ uses 
	$
	\mE[ \tilde{\rvTwo}_{1 \vectorIndex}^2] =\alphabetSize \probVecElement{\rvTwo}{\vectorIndex} +  \LapUParam^2$,
	and step~$(c)$ uses 
	the Cauchy--Schwarz inequality, 
	$\|\probVec_\rvTwo\|_2^2 \leq \|\probVec_\rvTwo\|_2$,
	and monotonicity of $\ell_p$ norm, specifically $\ell_4 \leq \ell_2$.
	Combining the results in~\eqref{appendix:equation:multinomial:upper:cond_var_term_2} and~\eqref{appendix:equation:multinomial:upper:cond_var_term_1}, we achieve an upper bound for $\momentTwosampleVarCondexpY $ given as 
	\begin{align*}
		\momentTwosampleVarCondexpY 
		\leq
		\bigl(
		4 \alphabetSize^2 \|\probVec_\rvTwo\|_2 
		+
		2 \alphabetSize \LapUParam^2
		\bigr)
		\|\probVec_\rvTwo - \probVec_\rvThree\|_2^2.
		\numberthis \label{two_moments_intermediate_1}
	\end{align*}
	By symmetry, we also have
	\begin{align*}
		\momentTwosampleVarCondexpZ
		\leq
		\bigl(
		4 \alphabetSize^2 \|\probVec_\rvThree\|_2 
		+
		2 \alphabetSize \LapUParam^2
		\bigr)
		\|\probVec_\rvTwo - \probVec_\rvThree\|_2^2.
		\numberthis \label{two_moments_intermediate_2}
	\end{align*}
	Combining the upper bound results of~\eqref{two_moments_intermediate_1} and~\eqref{two_moments_intermediate_2}, and keeping in mind that
	${b} = \max\{ \|\probVec_Y\|_2^2, \|\probVec_Z\|_2^2 \}$,
	we obtain the following upper bound:
	\begin{equation}\label{upper_lapu_moment1}
		\max
		\{
		\momentTwosampleVarCondexpY,
		\momentTwosampleVarCondexpZ
		\}
		%%%%%%%%%%%%%%%%%%%%%%%%
		\leq
		%%%%%%%%%%%%%%%%%%%%%%%%
		(4{\alphabetSize}^2 {b}^{1/2} + {2\alphabetSize}
		\LapUParam^2) \|\probVec_\rvTwo - \probVec_\rvThree\|_2^2.
	\end{equation}
	%%%%%%%%%%%%%%%%%%%%%%%%%%%%%%%%%%%%%%%%%%%%%%%%
	We now turn our attention to the expectation of square terms:
	\begin{align*}
		\momentTwosampleExpSquare:=
		\max \bigl\{
		\mE
		\bigl[
		(
		{\tilde{\vectorize{\rvTwo}}^\top }_1
		\tilde{\vectorize{\rvTwo}}_2
		)^2
		\bigr],~
		\mE
		\bigl[
		(
		\tilde{\vectorize{\rvThree}}_1^\top
		\rVecZPriv_2)^2
		\bigr]
		,~
		\mE
		\bigl[
		(
		\tilde{\vectorize{\rvTwo}}_1^\top
		\rVecZPriv_1)^2
		\bigr]
		\bigr\}.
	\end{align*}
	First, we calculate
	$\mE \bigl[(\tilde{\vectorize{\rvTwo}}_1^\top \tilde{\vectorize{\rvTwo}}_2)^2 \bigr]$
	and
	$\mE \bigl[ (
	\tilde{\vectorize{\rvThree}}_1^\top
	\rVecZPriv_2
	)^2\bigr]$.
	Note that
	\begin{align*}
		%line 1: just expand
		\mE
		\bigl[
		(\tilde{\vectorize{\rvTwo}}_1^\top \tilde{\vectorize{\rvTwo}}_2)^2
		\bigr]
		= %inner product to sum
		\mE
		\biggl[
		\biggl( 
		\sum_{\vectorIndex=1}^{\alphabetSize}
		\tilde{\rvTwo}_{1 \vectorIndex}
		\tilde{\rvTwo}_{2 \vectorIndex}
		\biggr)^2
		\biggr]
		\stackrel{(a)}{=}\hskip 1.57mm & %square of sum -> double sum
		\sum_{\vectorIndex=1}^{\alphabetSize}
		\sum_{\vectorIndex'=1}^{\alphabetSize}
		\mE
		\bigl[
		\tilde{\rvTwo}_{1 \vectorIndex}
		\tilde{\rvTwo}_{1 \vectorIndex'}
		\bigr]
		\mE
		\bigl[
		\tilde{\rvTwo}_{2 \vectorIndex}
		\tilde{\rvTwo}_{2 \vectorIndex'}
		\bigr]
		%
		% line 2. step (ii)
		\\ \stackrel{(b)}{=}\hskip 1.57mm &
		\sum_{\vectorIndex=1}^{\alphabetSize}
		\bigl\{
		{\alphabetSize} \probVecElement{\rvTwo}{\vectorIndex} + \LapUParam^2
		\bigr\}^2
		%
		% line 3. just expand the square
		\\ =~&
		\sum_{\vectorIndex=1}^{\alphabetSize}
		\bigl\{
		{\alphabetSize}^2 p^2_Y(\vectorIndex)
		+
		2 {\alphabetSize} \LapUParam^2 \probVecElement{\rvTwo}{\vectorIndex}
		+
		\LapUParam^4
		\bigr\}
		%
		% line 4. sum to norm
		\\ = ~&
		{\alphabetSize}^2
		\|\probVec_\rvTwo \|_2^2
		~+~
		2{\alphabetSize} \LapUParam^2
		~+~
		\alphabetSize \LapUParam^4	
		\numberthis \label{eq:MYZTwoOne},	
	\end{align*}
	\noindent
	where step~$(a)$ uses the independence between observations, and
	step~$(b)$ uses $
	\mE
	\bigl[
	\tilde{\rvTwo}_{1 \vectorIndex}
	\tilde{\rvTwo}_{1 \vectorIndex'}
	\bigr]
	=
	\bigl(
	\alphabetSize
	\probVecElement{\rvTwo}{\vectorIndex} + \LapUParam^2
	\bigr)
	\mathds{1}(\vectorIndex = \vectorIndex')$~(and similar equalitiy for $\tilde{\vectorize{\rvTwo}}_2$).
	%%%%%%%%%%%%%%%%%%%%%%%%%%%%%%%
	By symmetry, we also have:
	\begin{align*}
		\mE
		\bigl[
		(
		\tilde{\vectorize{\rvThree}}_1^\top 
		\tilde{\vectorize{\rvThree}}_2
		)^2
		\bigr]
		=
		{\alphabetSize}^2 \|\probVec_{\rvThree}\|_2^2
		~+~
		2 {\alphabetSize} \LapUParam^2
		~+~
		\alphabetSize  \LapUParam^4.
		\numberthis \label{eq:MYZTwoTwo}
	\end{align*}
	%%%%%%%%%%%%%%%%%%%%%%%%%%%%%%%
	Moving on, we upper bound $\mE \bigl[ ( \tilde{\vectorize{\rvTwo}}_1^\top \tilde{\vectorize{\rvThree}}_1 )^2 \bigr]$ as follows:
	\begin{align*}
		% line 1. transpose to sum
		\mE
		\bigl[
		(
		\tilde{\rvTwo}_{1 \vectorIndex}^\top 
		\tilde{\rvThree}_{1 \vectorIndex}
		)^2
		\bigr]
		= ~&
		\mE
		\biggl[
		\biggl( 
		\sum_{\vectorIndex=1}^{\alphabetSize}
		\tilde{\rvTwo}_{1 \vectorIndex}
		\tilde{\rvThree}_{1 \vectorIndex}
		\biggr)^2
		\biggr]
		%
		% line 2. step (i). expand the square, resulting in double sum
		\\ \stackrel{(a)}{=} \hskip 1.57mm &
		\sum_{\vectorIndex=1}^{\alphabetSize}
		\sum_{\vectorIndex'=1}^{\alphabetSize}
		\mE
		[
		\tilde{\rvTwo}_{1 \vectorIndex}
		\tilde{\rvTwo}_{1 \vectorIndex'}
		]
		\mE
		[
		\tilde{\rvThree}_{1 \vectorIndex}
		\tilde{\rvThree}_{1 \vectorIndex'}
		]
		%
		% line 3. step (ii)
		\\ \stackrel{(b)}{=} \hskip 1.57mm &
		\sum_{\vectorIndex=1}^{\alphabetSize}
		\bigl\{
		{\alphabetSize} \probVecElement{\rvTwo}{\vectorIndex} + \LapUParam^2
		\bigr\}
		\bigl\{
		{\alphabetSize} \probVecElement{\rvThree}{\vectorIndex} + \LapUParam^2
		\bigr\}
		%
		% line 4. just expand
		\\ \stackrel{(c)}{\leq} \hskip 1.57mm &
		\frac{1}{2}
		\sum_{\vectorIndex=1}^{\alphabetSize}
		\bigl[
		\bigl\{
		{\alphabetSize} \probVecElement{\rvTwo}{\vectorIndex} + \LapUParam^2
		\bigr\}^2
		+
		\bigl\{
		{\alphabetSize} \probVecElement{\rvThree}{\vectorIndex} + \LapUParam^2
		\bigr\}^2
		\bigr]
		\\ =~&
		%%%%%%%%%%%%%%%%%%%%%%%%%%%%%%%%%%%%%%%%%
		\frac{{\alphabetSize}^2}{2} 
		\|\probVec_\rvTwo\|^2_2
		~+~
		\frac{{\alphabetSize}^2}{2} 
		\|\probVec_\rvThree\|^2_2
		~+~
		2 {\alphabetSize} \LapUParam^2 
		~+~
		\alphabetSize \LapUParam^4, \numberthis \label{eq:MYZTwoThree}	
	\end{align*}
	%%%%%%%%%%%%%%%%%%%%%%%%%%%%%%%%%%%%%%%%%
	\noindent
	where step~$(a)$ uses the independence between $\tilde{\vectorize{\rvTwo}}_1$ and $\tilde{\vectorize{\rvThree}}_1$,
	step~$(b)$ uses $
	\mE
	\bigl[
	\tilde{\rvTwo}_{1 \vectorIndex}
	\tilde{\rvTwo}_{1 \vectorIndex'}
	\bigr]
	=
	\bigl(
	\alphabetSize
	\probVecElement{\rvTwo}{\vectorIndex} + \LapUParam^2
	\bigr)
	\mathds{1}(\vectorIndex = \vectorIndex')$~(and similar equality for $\tilde{\vectorize{\rvThree}}_1$), and step~$(c)$ applies the inequality $xy \leq x^2/2 + y^2/2$.
	\\
	
	Finally, by combining~\eqref{eq:MYZTwoOne},~\eqref{eq:MYZTwoTwo} and~\eqref{eq:MYZTwoThree}, we obtain the following upper bound:
	\begin{equation}\label{upper_lapu_moment2}
		\momentTwosampleExpSquare
		\leq
		%%%%%%%%%%%%%%%%%%%%%%%%%%%%%%%
		2
		(
		{\alphabetSize}^2
		{b} 
		+
		{\alphabetSize} \LapUParam^2
		+
		{\alphabetSize} \LapUParam^4	
		).
	\end{equation}
	%%%%%%%%%%%%%%%%%%%%%%%%%%%%%
	%discrete laplace
	\begin{remark}[Proving the upper bound using \texttt{DiscLapU}] \label{remark: discrete laplace noise}
		In our proof with continuous Laplace noise from \textnormal{\texttt{LapU}},  we use the independence and the equality  
		$\
		\mE
		\bigl[
		\tilde{\rvTwo}_{1 \vectorIndex}
		\tilde{\rvTwo}_{1 \vectorIndex'}
		\bigr]
		=
		\bigl(
		{\alphabetSize}\probVecElement{\rvTwo}{\vectorIndex} + \LapUParam^2
		\bigr)
		\mathds{1}(\vectorIndex = \vectorIndex')$,
		which holds due to the Laplace noise's moments: mean zero and variance $\sigma_\alpha^2 = 8k/\alpha^2$. The discrete Laplace noise of \textnormal{\texttt{DiscLapU}} also satisfies these independence and moment conditions, with variance upper bounded by $8 \alphabetSize/\privacyParameter^2$ (Lemma~\ref{lemma:discLapU_Variance}).
		Due to these properties, the use of \textnormal{\texttt{DiscLapU}} also leads to
		$\
		\mE
		\bigl[
		\tilde{\rvTwo}_{1 \vectorIndex}
		\tilde{\rvTwo}_{1 \vectorIndex'}
		\bigr]
		\leq
		\bigl(
		{\alphabetSize}\probVecElement{\rvTwo}{\vectorIndex} + \LapUParam^2
		\bigr)
		\mathds{1}(\vectorIndex = \vectorIndex').$
		Given this inequality, the entire proof of this section remains valid for \textnormal{\texttt{DiscLapU}} as well.
	\end{remark}

\vskip 1em

	\noindent \textit{Step 2: Apply the two moments method.}\label{proof:twosample_upper_disc_conclusion}
	Using the bounds \eqref{upper_lapu_moment1} and \eqref{upper_lapu_moment2},  we show that condition \eqref{eq:two_moments_original_cond_two_sample} in the two moments method holds if the separation condition~\eqref{eq:two_moments_two_sample_cond_merge} in Lemma~\ref{appendix:lemma:discrete_upper_bound} is met.
	Since $\mE [U_{\sampleSize_1, \sampleSize_2}] = \alphabetSize \| \probVec_\rvTwo - \probVec_\rvThree \|_2^2$, assuming $\sampleSize_1 \leq \sampleSize_2$, condition~\eqref{eq:two_moments_original_cond_two_sample} of Lemma~\ref{theorem:twosampleTwoMomentsMethod} is satisfied when
	\begin{align*}
		\alphabetSize \| \probVec_\rvTwo - \probVec_\rvThree \|_2^2
		%%%%%%%%%%%%%%%%%%%%%%%%
		&\geq
		%%%%%%%%%%%%%%%%%%%%%%%%
		C_1
		\sqrt{
			\frac{({\alphabetSize}^2 {b}^{1/2} + {\alphabetSize}
				\LapUParam^2) \| \probVec_\rvTwo - \probVec_\rvThree \|_2^2}{ \beta \sampleSize_1}} \quad \text{from}~\eqref{upper_lapu_moment1},~\text{and} 
		\numberthis\label{twoSampleUpperboundConclusionOne}
		\\
		\alphabetSize \| \probVec_\rvTwo - \probVec_\rvThree \|_2^2
		%%%%%%%%%%%%%%%%%%%%%%%%
		&\geq
		%%%%%%%%%%%%%%%%%%%%%%%%
		C_2
		\sqrt{
			\frac{{\alphabetSize}^2
				{b} + \alphabetSize \LapUParam^2 + \alphabetSize \LapUParam^4
			}{ \gamma \beta \sampleSize_1^2}}	
		\numberthis\label{twoSampleUpperboundConclusionTwo} \quad \text{from}~\eqref{upper_lapu_moment2},
	\end{align*}
	for any  $P = (P_{Y}, P_{Z}) \in \mathcal{P}_{1,\mathrm{multi}}(\rho_{\sampleSize_1,\sampleSize_2})$.
	Since $\LapUParam = 2 \sqrt{2 \alphabetSize}  / \privacyParameter$,
	condition~\eqref{twoSampleUpperboundConclusionOne} is satisfied when
	%%%%%%%%%%%%%%%%%%
	$$
	\alphabetSize \| \probVec_\rvTwo - \probVec_\rvThree \|_2^2
	%%%%%%%%%%%%%%%%%%
	\geq
	%%%%%%%%%%%%%%%%%%
	C_3
	\sqrt{
		\frac
		{({\alphabetSize}^2 {b}^{1/2}
			+
			{\alphabetSize}^2 /\privacyParameter^2) \| \probVec_\rvTwo - \probVec_\rvThree \|_2^2}
		{\beta \sampleSize_1}}.
	$$
	Using  $\sqrt{x+y} \leq \sqrt{x} + \sqrt{y} \leq 2 \max\{\sqrt{x}, \sqrt{y}\}$~for $x,y \geq 0$, the condition above is implied by:
	\begin{align*}
		\| \probVec_\rvTwo - \probVec_\rvThree \|_2^2 
		\geq
		C_4(\beta)
		\;
		\frac{
			\max \bigl\{ {b}^{1/4},  1/\privacyParameter \bigr\}
		}{
			\sampleSize_1^{1/2}}.
		\numberthis\label{equation:prelim_max_1}
	\end{align*}
	On the other hand,
	condition~\eqref{twoSampleUpperboundConclusionTwo} is satisfied when 
	$$
	\alphabetSize \| \probVec_\rvTwo - \probVec_\rvThree \|_2^2
	\geq
	C_5 \sqrt{
		\frac{{\alphabetSize}^2
			{b} + 
			\alphabetSize^2 / \privacyParameter^2 +
			\alphabetSize^3 / \privacyParameter^4
		}{\gamma \beta \sampleSize_1^2}}.
	$$ 
	Using $\sqrt{x+y+z} \leq 3 \max\{ \sqrt{x},  \sqrt{y}, \sqrt{z}\}$ for $x, y, z \geq 0$, the condition above is implied by:
	\begin{align*}
		\| \probVec_\rvTwo - \probVec_\rvThree \|_2^2 
		\geq 
		C_6(\gamma, \beta) \;
		\frac{
			\max \{{b}^{1/4}, 1/\sqrt{\privacyParameter}, \alphabetSize^{1/4} / \privacyParameter \}
		}{ \sampleSize_1^{1/2}}.
		\numberthis\label{equation:prelim_max_2}
	\end{align*}
	Then, by combining~\eqref{equation:prelim_max_1} and~\eqref{equation:prelim_max_2}, the condition~\eqref{eq:two_moments_original_cond_two_sample} of Lemma~\ref{theorem:twosampleTwoMomentsMethod} is satisfied when
	\begin{align*}
		\rho_{\sampleSize_1,\sampleSize_2}
		\geq ~&
		C_6(\gamma, \beta)
		\frac{1}{\sampleSize_1^{1/2}}
		\max
		\biggl\{
		\max
		\biggl(
		{b}^{1/4}, \frac{1}{\privacyParameter}
		\biggr), \,
		\max
		\biggl(
		{b}^{1/4},
		\frac{1}{\sqrt{\privacyParameter}},
		\frac{\alphabetSize^{1/4}}{\privacyParameter}
		\biggr)
		\biggr\}
		%
		%%%%%%%%%%%%%%%%%%%%%%%
		\\= ~ &
		%%%%%%%%%%%%%%%%%%%%%%%%
		C_6(\gamma, \beta) 
		\frac{1}{\sampleSize_1^{1/2}}
		\max
		\biggl\{
		{b}^{1/4}, \,
		\max
		\bigg(
		\frac{1}{\privacyParameter},
		\frac{1}{\sqrt{\privacyParameter}},
		\frac{\alphabetSize^{1/4}}{\privacyParameter}
		\biggr)
		\biggr\}
		%%%%%%%%%%%%%%%%%%%%%%%%%%%%%%%%%%
		\\= \hskip 1.91mm &
		%%%%%%%%%%%%%%%%%%%%%%%%%%%%%%%%%%
		\begin{cases}
			\displaystyle C_6(\gamma, \beta) 
			\frac{1}{\sampleSize_1^{1/2}}
			\max
			\biggl(
			{b}^{1/4}, \,
			\frac{\alphabetSize^{1/4}}{\privacyParameter}
			\biggr) & \text{if}~\alpha \leq \alphabetSize^{1/2},
			\\[1em]
			\displaystyle C_6(\gamma, \beta)
			\frac{1}{\sampleSize_1^{1/2}}
			\max
			\biggl(
			{b}^{1/4}, \,
			\frac{1}{\sqrt{\privacyParameter}}
			\biggr) & \text{if}~\alpha \geq \alphabetSize^{1/2}
		\end{cases}
		%%%%%%%%%%%%%%%%%%%%%%%%%%%%%%%%%
		\\ \stackrel{(a)}{=} \hskip 1.37mm &
		%%%%%%%%%%%%%%%%%%%%%%%%%%%%%%%%%%%
		C_6(\gamma, \beta) 
		\frac{1}{\sampleSize_1^{1/2}}
		\max
		\biggl(
		{b}^{1/4}, \,
		\frac{\alphabetSize^{1/4}}{\privacyParameter}
		\biggr)
		\\ = \hskip 1.74mm  &
		C_6(\gamma, \beta) 
		\max\left(
		\frac
		{{b}^{1/4}}
		{\sampleSize_1^{1/2}}, \,
		\frac
		{\alphabetSize^{1/4}}
		{(\sampleSize_1 \privacyParameter^2)^{1/2}}
		\right),
	\end{align*}
	where the step $(a)$ holds because if $\alpha \geq \alphabetSize^{1/2}$, then we have
	$
	{b}^{1/4}
	\geq
	1 / \alphabetSize^{1/4}
	\geq 
	1 / \sqrt{\privacyParameter}
	\geq
	\alphabetSize^{1/4} / \privacyParameter.
	$
		This completes the proof of the upper bound through the (discrete) Laplace  mechanism.
	\end{proof}
	\subsubsection{Proof of Lemma \ref{appendix:lemma:discrete_upper_bound} Through \texttt{RAPPOR} Mechanism}\label{proof:rappor_optimal}
	While we follow the same two main steps as in Appendix~\ref{proof:multinomial_upper_bound_first_moment}, proving a tight upper bound under the \texttt{RAPPOR} mechanism requires a more delicate analysis due to the dependence and bias inherent in private views.  
	To elaborate, recall from Lemma~\ref{lemma:rappor_var} that the entries of an $\privacyParameter$-LDP view under the \texttt{RAPPOR} mechanism are dependent Bernoulli random variables. Specifically, 
	the $\vectorIndex$th entry of $\tilde{\vectorize{\rvTwo}}_{i}$,
	denoted as $\tilde{\rvTwo}_{\sampleIndexOne \vectorIndex}$, 
	and
	the $\vectorIndex'$th entry of $\tilde{\vectorize{\rvThree}}_{\sampleIndexTwo}$,
	denoted as $\tilde{\rvThree}_{\sampleIndexTwo \vectorIndex'}$, 
	follow the following distributions:
	\begin{align*}
		\tilde{\rvTwo}_{\sampleIndexOne \vectorIndex}
		\sim
		\mathrm{Ber}
		\bigl(
		\privacyParameterrappor \mathds{1}
		(
		\rvTwo_{\sampleIndexOne} = \vectorIndex
		) + \smallNumberrappor
		\bigr)
		\quad \text{and} \quad 
		\tilde{\rvThree}_{\sampleIndexTwo \vectorIndex'}
		\sim
		\mathrm{Ber}
		\bigl(
		\privacyParameterrappor \mathds{1}
		(
		\rvThree_{\sampleIndexTwo} = \vectorIndex'
		) + \smallNumberrappor
		\bigr),
	\end{align*}
	where 
	\begin{equation*}
		\privacyParameterrappor = \frac{e^{\privacyParameter /2} -1}{e^{\privacyParameter /2} +1}\quad \text{and} \quad
		%	=	\frac{\alpha}{4} + o(\privacyParameter)
		\smallNumberrappor = \frac{1}{e^{\privacyParameter /2} +1}.
		%	=	\frac{1}{2} + o(1).
	\end{equation*}
	As in Appendix~\ref{proof:multinomial_upper_bound_first_moment}, we employ the two moments method (Lemma~\ref{theorem:twosampleTwoMomentsMethod}), which compares the the expectation of the U-statistic with a variance proxy.
	The challenge is that, unlike \texttt{LapU}, for each $m \in [k]$, the $m$th entry of each $\privacyParameter$-LDP view in \texttt{RAPPOR} is not centered at the scaled multinomial probability:
	\begin{equation}\label{equation:rappor_expecation}
		\mE
		\bigl[
		\tilde{\rvTwo}_{\sampleIndexOne \vectorIndex}
		\bigr]
		=
		\mE
		\bigl[
		\mE %expectation
		\bigl[
		\tilde{\rvTwo}_{\sampleIndexOne \vectorIndex} %private entry
		\,\big|\, %conditional
		\rvTwo_{\sampleIndexOne \vectorIndex}] 
		\bigr]
		=
		\privacyParameterrappor 
		\probVecElement{\rvTwo}{\vectorIndex}
		+ \smallNumberrappor,	
	\end{equation}
	and pairwise negatively correlated for $m, m' \in [k]$ such that $m \neq m'$:
	\begin{equation}\label{cov_neg}
		\mathrm{Cov}(\tilde{Y}_{1\vectorIndex}, \tilde{Y}_{1\vectorIndex'})
		=
		-
		\privacyParameterrappor^2
		\probVecElement{\rvTwo}{\vectorIndex}
		\probVecElement{\rvTwo}{\vectorIndex'},
	\end{equation}
	\noindent
	which is proved in Lemma~\ref{lemma:rappor_var}.
	Keeping these facts in mind, we proceed to the upper bound proof.

    \vskip 1em
    
	\begin{proof}
	We follow the two steps of Appendix~\ref{proof:multinomial_upper_bound_first_moment}.

\vskip 1em

\noindent \textit{Step 1: Bound the moments from above.}
	We start by bounding the variance of conditional expectation terms. We begin with the intermediate form of $\momentTwosampleVarCondexpY$~\eqref{appendix:equation:multinomial:upper:forrappor} found in Appendix~\ref{proof:multinomial_upper_bound_first_moment}:
	%\ilmun{Below and throughout, $\sum_{1 \leq \vectorIndex \neq \vectorIndex' \leq \alphabetSize}, \sum_{1 \leq \vectorIndex \neq \vectorIndex' \leq \alphabetSize} \to \sum_{1 \leq \vectorIndex \neq \vectorIndex' \leq \alphabetSize}$}
	\begin{align*}
		\momentTwosampleVarCondexpY
		= ~ &
		\mE % 1. start from lapu+elltwo intermediate form
		\Bigl[ % full expectation bracket opens
		\bigl\{ %square bracket opens
		\bigl( %first term paranthesis opens
		\tilde{\vectorize{\rvTwo}}_1
		- % minus
		\mE [ \tilde{\vectorize{\rvTwo}}_{1}]
		\bigr)^\top %first term paranthesis closes
		\bigl( %second term paranthesis opens
		\mE [ \tilde{\vectorize{\rvTwo}}_{1}]
		- % minus
		\mE [ \tilde{\vectorize{\rvThree}}_{1}]
		\bigr) %second term paranthesis closes
		\bigr\}^2 %square bracket closes
		\Bigr] % full expectation bracket closes
		%%%%%%%%%%%%%%%%%%%%%%%%%%%%
		\\ \stackrel{(a)}{=} \hskip 1.57mm &
		\mE % 2. plug in full expectation
		\Bigl[ % full expectation bracket opens
		\bigl\{ %square bracket opens
		\bigl( %first term paranthesis opens
		\tilde{\vectorize{\rvTwo}}_1
		- % minus
		\mE [ \tilde{\vectorize{\rvTwo}}_{1}]
		\bigr)^\top %first term paranthesis closes
		\privacyParameterrappor
		\bigl( %second term paranthesis opens
		\probVec_{\rvTwo}
		- % minus
		\probVec_{\rvThree}
		\bigr) %second term paranthesis closes
		\bigr\}^2 %square bracket closes
		\Bigr] % full expectation bracket closes
		%%%%%%%%%%%%%%%%%%%%%%%%%%%%
		\\= ~&
		\privacyParameterrappor^2
		\mE % 3. inner product as summation
		\biggl[ % full expectation bracket opens
		\biggl\{ %square bracket opens
		\sum_{\vectorIndex=1}^\alphabetSize
		\bigl( %first term paranthesis opens
		\private{\rvTwo}_{1 \vectorIndex}
		- % minus
		\mE [ \private{\rvTwo}_{1 \vectorIndex}]
		\bigr) %first term paranthesis closes
		\bigl( %second term paranthesis opens
		\probVecElement{\rvTwo}{\vectorIndex}
		- % minus
		\probVecElement{\rvThree}{\vectorIndex}
		\bigr) %second term paranthesis closes
		\biggr\}^2 %square bracket closes
		\biggr] % full expectation bracket closes
		%%%%%%%%%%%%%%%%%%%%%%%%%%%%
		\\= ~&
		\privacyParameterrappor^2
		\mE % 4. square of sum = sum of square + sum of cross-product
		\biggl[ % full expectation bracket opens
		\sum_{\vectorIndex=1}^\alphabetSize
		\bigl( %first term paranthesis opens
		\private{\rvTwo}_{1 \vectorIndex}
		- % minus
		\mE [ \private{\rvTwo}_{1 \vectorIndex}]
		\bigr)^2 %first term paranthesis closes
		\bigl( %second term paranthesis opens
		\probVecElement{\rvTwo}{\vectorIndex}
		- % minus
		\probVecElement{\rvThree}{\vectorIndex}
		\bigr)^2 %second term paranthesis closes
		\biggr] % full expectation bracket closes
		\\ &+~
		\privacyParameterrappor^2
		\mE 
		\biggl[ % full expectation bracket opens
		\sum_{1 \leq \vectorIndex \neq \vectorIndex' \leq \alphabetSize}
		\bigl( %first term paranthesis opens
		\private{\rvTwo}_{1 \vectorIndex}
		- % minus
		\mE [ \private{\rvTwo}_{1 \vectorIndex}]
		\bigr) %first term paranthesis closes
		\bigl( %first term paranthesis opens
		\private{\rvTwo}_{1 \vectorIndex'}
		- % minus
		\mE [ \private{\rvTwo}_{1 \vectorIndex'}]
		\bigr) %first term paranthesis closes
		\bigl( %second term paranthesis opens
		\probVecElement{\rvTwo}{\vectorIndex}
		- % minus
		\probVecElement{\rvThree}{\vectorIndex}
		\bigr) %second term paranthesis closes
		\bigl( %second term paranthesis opens
		\probVecElement{\rvTwo}{\vectorIndex'}
		- % minus
		\probVecElement{\rvThree}{\vectorIndex'}
		\bigr) %second term paranthesis closes
		\biggr] % full expectation bracket closes
		%%%%%%%%%%%%%%%%%%%%%%%%%%%%
		\\= ~&% 5. variance and covariance
		\privacyParameterrappor^2
		\biggl[ % full expectation bracket opens
		\sum_{\vectorIndex=1}^\alphabetSize
		\mathrm{Var}
		( %first term paranthesis opens
		\private{\rvTwo}_{1 \vectorIndex}
		) %first term paranthesis closes
		\bigl( %second term paranthesis opens
		\probVecElement{\rvTwo}{\vectorIndex}
		- % minus
		\probVecElement{\rvThree}{\vectorIndex}
		\bigr)^2 %second term paranthesis closes
		\biggr] % full expectation bracket closes
		\\ &+~
		\privacyParameterrappor^2
		\biggl[ % full expectation bracket opens
		\sum_{1 \leq \vectorIndex \neq \vectorIndex' \leq \alphabetSize}
		\mathrm{Cov}
		(
		\private{\rvTwo}_{1 \vectorIndex},
		\private{\rvTwo}_{1 \vectorIndex'}
		) %first term paranthesis closes
		\bigl( %second term paranthesis opens
		\probVecElement{\rvTwo}{\vectorIndex}
		- % minus
		\probVecElement{\rvThree}{\vectorIndex}
		\bigr) %second term paranthesis closes
		\bigl( %second term paranthesis opens
		\probVecElement{\rvTwo}{\vectorIndex'}
		- % minus
		\probVecElement{\rvThree}{\vectorIndex'}
		\bigr) %second term paranthesis closes
		\biggr] % full expectation bracket closes
		%%%%%%%%%%%%%%%%%%%%%%%%%%%%
		\\
		\stackrel{(b)}{\leq} \hskip 1.57mm &
		\privacyParameterrappor^3
		\normSqMultinomMax^{1/2}
		\|\probVec_{\rvTwo} - \probVec_{\rvThree}\|_2^2
		+
		\privacyParameterrappor^2
		\smallNumberrappor
		\|\probVec_{\rvTwo} - \probVec_{\rvThree}\|_2^2,
		\numberthis
		\label{term:multinomial:upperbound:rappor_elltwo:var_of_cond_expec}
	\end{align*}
	where step 
	$(a)$ uses \eqref{equation:rappor_expecation}, and step $(b)$ uses \eqref{inequality:var_diff_squared} in Lemma~\ref{rappor:inequalities} for the first term and \eqref{cov_neg} for the second term.
	Therefore we have
	\begin{equation}\label{inequality:rappor_elltwo_var_condexp_y}
		\momentTwosampleVarCondexpY
		%%%%%%%%%%%%%%%%%%%%
		\leq
		%%%%%%%%%%%%%%%%%%%%
		\privacyParameterrappor^3
		\normSqMultinomMax^{1/2}
		\|\probVec_{\rvTwo} - \probVec_{\rvThree}\|_2^2
		+
		\privacyParameterrappor^2
		\smallNumberrappor
		\|\probVec_{\rvTwo} - \probVec_{\rvThree}\|_2^2,
		%%%%%%%%%%%%%%%%%%%%
	\end{equation}
	and by symmetry, we also have
	\begin{equation}\label{inequality:rappor_elltwo_var_condexp_z}
		\momentTwosampleVarCondexpZ
		%%%%%%%%%%%%%%%%%%%%
		\leq
		%%%%%%%%%%%%%%%%%%%%
		\privacyParameterrappor^3
		\normSqMultinomMax^{1/2}
		\|\probVec_{\rvThree} - \probVec_{\rvThree}\|_2^2
		+
		\privacyParameterrappor^2
		\smallNumberrappor
		\|\probVec_{\rvTwo} - \probVec_{\rvThree}\|_2^2.
		%%%%%%%%%%%%%%%%%%%%
	\end{equation}
	Next, we examine the expectation of square terms \(\momentTwosampleExpSquare\). The proof's key technique involves rewriting the kernel 
	$
	(\tilde{\vectorize{\rvTwo}}_{1} - \tilde{\vectorize{\rvThree}}_{1})^\top (\tilde{\vectorize{\rvTwo}}_{2} - \tilde{\vectorize{\rvThree}}_{2})
	$~\eqref{equation:twosample_kernel}
	of the U-statistic~\eqref{def:statistic_elltwo} using
	$\bar{\vectorize{\rvTwo}}_{1}
	:=
	\tilde{\vectorize{\rvTwo}}_{1}
	-
	\mE [\tilde{\vectorize{Y}}_1]$,
	$
	\bar{\vectorize{\rvThree}}_{1}
	:=
	\tilde{\vectorize{\rvThree}}_{1}
	-
	\mE [\tilde{\vectorize{\rvTwo}}_1],
	$
	$\bar{\vectorize{\rvTwo}}_{2}
	:=
	\tilde{\vectorize{\rvTwo}}_{2}
	-
	\mE [\tilde{\vectorize{Y}}_2]$, and
	$
	\bar{\vectorize{\rvThree}}_{2}
	:=
	\tilde{\vectorize{\rvThree}}_{2}
	-
	\mE [\tilde{\vectorize{\rvTwo}}_2].
	$
	Formally, 
	we re-express the kernel~\eqref{equation:twosample_kernel} as the following:
	\begin{align*}
		%1
		(\tilde{\vectorize{\rvTwo}}_{1} - \tilde{\vectorize{\rvThree}}_{1})^\top (\tilde{\vectorize{\rvTwo}}_{2} - \tilde{\vectorize{\rvThree}}_{2})
		&=
		% 2.
		(\tilde{\vectorize{\rvTwo}}_{1}
		-
		\mE [\tilde{\vectorize{Y}}_1]
		+
		\mE [\tilde{\vectorize{Y}}_1]
		-
		\tilde{\vectorize{\rvThree}}_{1})^\top (
		\tilde{\vectorize{\rvTwo}}_{2} 
		-
		\mE [\tilde{\vectorize{Y}}_2]
		+
		\mE [\tilde{\vectorize{Y}}_2]
		-
		\tilde{\vectorize{\rvThree}}_{2})
		%
		% 3.
		\\&=
		\bigl\{
		(
		\tilde{\vectorize{\rvTwo}}_{1}
		-
		\mE [\tilde{\vectorize{Y}}_1]
		)
		-
		(
		\tilde{\vectorize{\rvThree}}_{1}
		-
		\mE [\tilde{\vectorize{Y}}_1]
		)
		\bigr\}^\top \bigl\{
		(
		\tilde{\vectorize{\rvTwo}}_{2}
		-
		\mE [\tilde{\vectorize{Y}}_2]
		)
		-
		(
		\tilde{\vectorize{\rvThree}}_{2}
		-
		\mE [\tilde{\vectorize{Y}}_2]
		)
		\bigr\}
		%%%%%%%%%%%%%%%%%%%%%%%%%%%%%%%%%%%
		\\&=
		(\bar{\vectorize{\rvTwo}}_{1} - \bar{\vectorize{\rvThree}}_{1})^\top (\bar{\vectorize{\rvTwo}}_{2} - \bar{\vectorize{\rvThree}}_{2})
		%%%%%%%%%%%%%%%%%%%%%%%%%%%%%%%%%%%
		\numberthis
		\label{statistic:rappor_elltwo_kernel_debiased}
	\end{align*}
	Let \(\bar{U}_{\sampleSize_1, \sampleSize_2}\) denote the U-statistic defined by the kernel $(\bar{\vectorize{\rvTwo}}_{1} - \bar{\vectorize{\rvThree}}_{1})^\top (\bar{\vectorize{\rvTwo}}_{2} - \bar{\vectorize{\rvThree}}_{2})$ in~\eqref{statistic:rappor_elltwo_kernel_debiased}. Then \(\bar{U}_{\sampleSize_1, \sampleSize_2}\) is equal to our original U-statistic~\eqref{def:statistic_elltwo}. Thus, it suffices to prove Lemma~\ref{appendix:lemma:discrete_upper_bound} with respect to \(\bar{U}_{\sampleSize_1, \sampleSize_2}\). Let \(\momentTwosampleVarCondexpBarY\), \(\momentTwosampleVarCondexpBarZ\), and \(\momentTwosampleExpSquareBar\) be the moments defined as in~\eqref{def:moment_terms} using the kernel $(\bar{\vectorize{\rvTwo}}_{1} - \bar{\vectorize{\rvThree}}_{1})^\top (\bar{\vectorize{\rvTwo}}_{2} - \bar{\vectorize{\rvThree}}_{2})$ . Since \(\momentTwosampleVarCondexpBarY = \momentTwosampleVarCondexpY\) and \(\momentTwosampleVarCondexpBarZ = \momentTwosampleVarCondexpZ\), the bounds \eqref{inequality:rappor_elltwo_var_condexp_y} and \eqref{inequality:rappor_elltwo_var_condexp_z} are also valid for \(\momentTwosampleVarCondexpBarY\) and \(\momentTwosampleVarCondexpBarZ\). Now, we move on to bounding \(\momentTwosampleExpSquareBar\) given as
	\begin{equation*}
		\momentTwosampleExpSquareBar
		=
		\max
		\{
		\mE[
		(
		\bar{\vectorize{\rvTwo}}_1^\top 
		\bar{\vectorize{\rvTwo}}_2
		)^2
		], \, 
		\mE[
		(
		\bar{\vectorize{\rvTwo}}_1^\top
		\bar{\vectorize{\rvThree}}_1)^2], \,
		\mE[
		(
		\bar{\vectorize{\rvThree}}_1^\top
		\bar{\vectorize{\rvThree}}_2)^2
		]
		\}.
	\end{equation*}
	Let us start by upper bounding $\mE[
	(
	\bar{\vectorize{\rvTwo}}_1^\top 
	\bar{\vectorize{\rvTwo}}_2
	)^2
	]$:
	\begin{align*}
%%%%%%%%%%%%%%%1. term definition
		\mE[
		(
		\bar{\vectorize{\rvTwo}}_1^\top 
		\bar{\vectorize{\rvTwo}}_2
		)^2
		]
		&=
		\mE
		\biggl[ %expand the sum
		\biggl\{% brace for square opens
		\sum_{\vectorIndex = 1}^\alphabetSize % sum over vector entry indices
		\bigl( % Y_1 term parenthesis opens
		\private{\rvTwo}_{1 \vectorIndex}
		- %minus
		\mE
		\bigl[
		\private{\rvTwo}_{1 \vectorIndex}
		\bigr]
		\bigr)% Y_1 term parenthesis closes
		\bigl( % Y_2 term parenthesis opens
		\private{\rvTwo}_{2 \vectorIndex}
		-
		\mE
		\bigl[
		\private{\rvTwo}_{2 \vectorIndex}
		\bigr]
		\bigr)% Y_2 term parenthesis closes
		\biggl\}^2 % brace for square closes
		\biggr]
%%%%%%%%%%%%%%%%%%%%%%%%%%%%%%%%%%%%%%%
\\&= %2 matching indices term%%%%%%%%%%
%%%%%%%%%%%%%%%%%%%%%%%%%%%%%%%%%%%%%%%
		\mE
		\biggl[ %expand the sum
		\sum_{\vectorIndex = 1}^\alphabetSize % sum over vector entry indices
		\bigl( % Y_1 term parenthesis opens
		\private{\rvTwo}_{1 \vectorIndex}
		- %minus
		\mE
		\bigl[
		\private{\rvTwo}_{1 \vectorIndex}
		\bigr]
		\bigr)^2% Y_1 term parenthesis closes
		\bigl( % Y_2 term parenthesis opens
		\private{\rvTwo}_{2 \vectorIndex}
		-
		\mE
		\bigl[
		\private{\rvTwo}_{2 \vectorIndex}
		\bigr]
		\bigr)^2% Y_2 term parenthesis closes
		\biggr]
		\\  & \quad+ %change line. non-matching indices term
		\mE
		\biggl[ %expand the sum
		\sum_{1 \leq \vectorIndex \neq \vectorIndex' \leq \alphabetSize} % sum over vector entry indices
		%%%%%%%%%%%%%%%%%%%%%%%%%%%
		\hskip -5mm %removes excessive gap caused by the summing index
		%%%%%%%%%%%%%%%%%%%%%%%%%%%
		\bigl( % Y_1 m term parenthesis opens
		\private{\rvTwo}_{1 \vectorIndex}
		- %minus
		\mE
		\bigl[
		\private{\rvTwo}_{1 \vectorIndex}
		\bigr]
		\bigr)% Y_1 m term parenthesis closes
		\bigl( % Y_1 m' term parenthesis opens
		\private{\rvTwo}_{1 \vectorIndex'}
		- %minus
		\mE
		\bigl[
		\private{\rvTwo}_{1 \vectorIndex'}
		\bigr]
		\bigr)% Y_1 m' term parenthesis closes
		\bigl( % Y_2 m term parenthesis opens
		\private{\rvTwo}_{2 \vectorIndex}
		-
		\mE
		\bigl[
		\private{\rvTwo}_{2 \vectorIndex}
		\bigr]
		\bigr)% Y_2 m term parenthesis closes
		\bigl( % Y_2 m' term parenthesis opens
		\private{\rvTwo}_{2 \vectorIndex'}
		-
		\mE
		\bigl[
		\private{\rvTwo}_{2 \vectorIndex'}
		\bigr]
		\bigr)% Y_2 m' term parenthesis closes
		\biggr]
		%%%%%%%%%%%%%%%%%%%%%%%%%%%
		\\&= %4 var and covariance
		\sum_{\vectorIndex = 1}^\alphabetSize
		\mathrm{Var}(\private{Y}_{1 \vectorIndex})^2
		+
		\hskip -2mm
		\sum_{1 \leq \vectorIndex \neq \vectorIndex' \leq \alphabetSize}
		%%%%%%%%%%%%%%%%%%%%%%%%%%%
\hskip -5mm %removes excessive gap caused by the summing index
%%%%%%%%%%%%%%%%%%%%%%%%%%%
		\mathrm{Cov}(\private{Y}_{1 \vectorIndex}, \private{Y}_{1 \vectorIndex'})^2
		%%%%%%%%%%%%%%%%%%%%%%%%%%%
		\\&\leq
		3
		\privacyParameterrappor^2
		\normSqMultinomMax
		+
		2
		\smallNumberrappor^2
		\alphabetSize,
		\numberthis
		\label{term:multinomial:upper:rappor_elltwo:EY1Y2:bound}
		%%%%%%%%%%%%%%%%%%%%%%%%%%
	\end{align*}
	where the last inequality uses  \eqref{inequality:var_squared} and \eqref{inequality:cov_squared} in Lemma~\ref{rappor:inequalities}. 
	By symmetry, we also have:
	\begin{equation}\label{term:multinomial:upper:rappor_elltwo:EZ1Z2:bound}
		\mE[
		(
		\bar{\vectorize{\rvThree}}_1^\top 
		\bar{\vectorize{\rvThree}}_2
		)^2
		]
		\leq
		3
		\privacyParameterrappor^2
		\normSqMultinomMax
		+
		2
		\smallNumberrappor^2
		\alphabetSize.	
	\end{equation}
	\noindent
	Now we move on to bounding $
	\mE[
	(
	\bar{\vectorize{\rvTwo}}_1^\top
	\bar{\vectorize{\rvThree}}_1)^2]
	$, which is expanded as:
	\begin{align*}
		%1. term definition
		\mE[
		(
		\bar{\vectorize{\rvTwo}}_1^\top 
		\bar{\vectorize{\rvThree}}_1
		)^2
		]
		&=
		\mE
		\biggl[ %expand the sum
		\biggl\{% brace for square opens
		\sum_{\vectorIndex = 1}^\alphabetSize % sum over vector entry indices
		\bigl( % Y_1 term parenthesis opens
		\private{\rvTwo}_{1 \vectorIndex}
		- %minus
		\mE
		\bigl[
		\private{\rvTwo}_{1 \vectorIndex}
		\bigr]
		\bigr)% Y_1 term parenthesis closes
		\bigl( % Y_2 term parenthesis opens
		\private{\rvThree}_{1 \vectorIndex}
		-
		\mE
		\bigl[
		\private{\rvTwo}_{1 \vectorIndex}
		\bigr]
		\bigr)% Y_2 term parenthesis closes
		\biggl\}^2 % brace for square closes
		\biggr]
		%%%%%%%%%%%%%%%%%%%%%%%%%%%
		\\&=
		% matching indices term
		\mE
		\biggl[ %expand the sum
		\sum_{\vectorIndex = 1}^\alphabetSize % sum over vector entry indices
		\bigl( % Y_1 term parenthesis opens
		\private{\rvTwo}_{1 \vectorIndex}
		- %minus
		\mE
		\bigl[
		\private{\rvTwo}_{1 \vectorIndex}
		\bigr]
		\bigr)^2% Y_1 term parenthesis closes
		\bigl( % Y_2 term parenthesis opens
		\private{\rvThree}_{1 \vectorIndex}
		-
		\mE
		\bigl[
		\private{\rvTwo}_{1 \vectorIndex}
		\bigr]
		\bigr)^2% Y_2 term parenthesis closes
		\biggr]
		%
		%non-matching indices term
		\\  & \quad+
		\mE
		\biggl[ %expand the sum
		\sum_{1 \leq \vectorIndex \neq \vectorIndex' \leq \alphabetSize} % sum over vector entry indices
		%%%%%%%%%%%%%%%%%%%%%%%%%%
		\hskip -5mm %removes excessive gap caused by the summing index
		%%%%%%%%%%%%%%%%%%%%%%%%%%%%%%%%%%%%%%%
		\bigl( % Y_1 m term parenthesis opens
		\private{\rvTwo}_{1 \vectorIndex}
		- %minus
		\mE
		\bigl[
		\private{\rvTwo}_{1 \vectorIndex}
		\bigr]
		\bigr)% Y_1 m term parenthesis closes
		\bigl( % Y_1 m' term parenthesis opens
		\private{\rvTwo}_{1 \vectorIndex'}
		- %minus
		\mE
		\bigl[
		\private{\rvTwo}_{1 \vectorIndex'}
		\bigr]
		\bigr)% Y_1 m' term parenthesis closes
		\bigl( % Y_2 m term parenthesis opens
		\private{\rvThree}_{1 \vectorIndex}
		-
		\mE
		\bigl[
		\private{\rvTwo}_{1 \vectorIndex}
		\bigr]
		\bigr)% Y_2 m term parenthesis closes
		\bigl( % Y_2 m' term parenthesis opens
		\private{\rvThree}_{1 \vectorIndex'}
		-
		\mE
		\bigl[
		\private{\rvTwo}_{1 \vectorIndex'}
		\bigr]
		\bigr)% Y_2 m' term parenthesis closes
		\biggr]
		%%%%%%%%%%%%%%%%%%%%%%%%%%%
		\\&=
		\sum_{\vectorIndex = 1}^\alphabetSize
		\mathrm{Var}(\private{Y}_{1 \vectorIndex})
		\mE
		\bigl( % Y_2 term parenthesis opens
		\private{\rvThree}_{1 \vectorIndex}
		-
		\mE
		\bigl[
		\private{\rvTwo}_{1 \vectorIndex}
		\bigr]
		\bigr)^2
		\numberthis
		\label{term:multinomial:upper:rappor_elltwo:EY1Z2:var}
		\\
		&\quad+
		\hskip -3mm
		\sum_{1 \leq \vectorIndex \neq \vectorIndex' \leq \alphabetSize}
		\hskip -5mm
		\mathrm{Cov}(\private{Y}_{1 \vectorIndex}, \private{Y}_{1 \vectorIndex'})
		\mE
		\bigl( % Y_2 m term parenthesis opens
		\private{\rvThree}_{1 \vectorIndex}
		-
		\mE
		\bigl[
		\private{\rvTwo}_{1 \vectorIndex}
		\bigr]
		\bigr)% Y_2 m term parenthesis closes
		\bigl( % Y_2 m' term parenthesis opens
		\private{\rvThree}_{1 \vectorIndex'}
		-
		\mE
		\bigl[
		\private{\rvTwo}_{1 \vectorIndex'}
		\bigr]
		\bigr),% Y_2 m' term parenthesis closes
		\numberthis
		\label{term:multinomial:upper:rappor_elltwo:EY1Z2:cov}
		%%%%%%%%%%%%%%%%%%%%%%%%%%
	\end{align*}
	where the last equality uses Lemma~\ref{lemma:rappor_var}.
	We bound the terms~\eqref{term:multinomial:upper:rappor_elltwo:EY1Z2:var} and~\eqref{term:multinomial:upper:rappor_elltwo:EY1Z2:cov} separately.
	For the term~\eqref{term:multinomial:upper:rappor_elltwo:EY1Z2:var}, we use the following equality that holds for each of $\vectorIndex \in [\alphabetSize]$:
	\begin{align*}
		\mE
		\bigl[ % Y_2 term parenthesis opens
		\private{\rvThree}_{1 \vectorIndex}
		&-
		\mE
		[
		\private{\rvTwo}_{1 \vectorIndex}
		]
		\bigr]^2
		%%%%%%%%%%%%%%%%%%%%%%%%%%%
		\\&=
		\mE
		\bigl[ % Y_2 term parenthesis opens
		\private{\rvThree}_{1 \vectorIndex}
		-
		\mE
		[
		\private{\rvThree}_{1 \vectorIndex}
		]
		+
		\mE
		[
		\private{\rvThree}_{1 \vectorIndex}
		]
		-
		\mE
		[
		\private{\rvTwo}_{1 \vectorIndex}
		]
		\bigr]^2
		%%%%%%%%%%%%%%%%%%%%%%%%%%%
		\\&=
		\mE
		\bigl[ % Y_2 term parenthesis opens
		\private{\rvThree}_{1 \vectorIndex}
		-
		\mE
		[
		\private{\rvThree}_{1 \vectorIndex}
		]
		\bigr]^2
		+
		\bigl[
		\mE
		[
		\private{\rvThree}_{1 \vectorIndex}
		]
		-
		\mE
		[
		\private{\rvTwo}_{1 \vectorIndex}
		]
		\bigr]^2
		+2
		\mE
		\bigl[ % Y_2 term parenthesis opens
		\private{\rvThree}_{1 \vectorIndex}
		-
		\mE
		[
		\private{\rvThree}_{1 \vectorIndex}
		]
		\bigr]
		\bigl(
		\mE
		[
		\private{\rvThree}_{1 \vectorIndex}
		]
		-
		\mE
		[
		\private{\rvTwo}_{1 \vectorIndex}
		]
		\bigr)
		%%%%%%%%%%%%%%%%%%%%%%%%%%%
		\\&=
		\mathrm{Var}(
		\private{\rvThree}_{1 \vectorIndex}
		)
		+
		\bigl[
		\mE
		[
		\private{\rvThree}_{1 \vectorIndex}
		]
		-
		\mE
		[
		\private{\rvTwo}_{1 \vectorIndex}
		]
		\bigr]^2
		%
		%%%%%%%%%%%%%%%%%%%%%%%%%%%%
		\\&=
		\mathrm{Var}(
		\private{\rvThree}_{1 \vectorIndex}
		)
		+
		\privacyParameterrappor ^2
		\bigl(  \probVecElement{\rvTwo}{\vectorIndex} - \probVecElement{\rvThree}{\vectorIndex}\bigr)^2,
		\numberthis
		\label{term:multinomial:upper:rappor_elltwo:EY1Z2:var:trick}
	\end{align*}
	where the last equality uses \eqref{equation:rappor_expecation}.
	Using this equality, the term~\eqref{term:multinomial:upper:rappor_elltwo:EY1Z2:var} is bounded as:
	\begin{align*}
		\sum_{\vectorIndex = 1}^\alphabetSize
		&
		\mathrm{Var}(\private{Y}_{1 \vectorIndex})
		\mE
		\bigl( % Y_2 term parenthesis opens
		\private{\rvThree}_{1 \vectorIndex}
		-
		\mE
		\bigl[
		\private{\rvTwo}_{1 \vectorIndex}
		\bigr]
		\bigr)^2
		%%%%%%%%%%%%%%%%%
		\\= ~ &
		\sum_{\vectorIndex = 1}^\alphabetSize
		\mathrm{Var}(\private{Y}_{1 \vectorIndex})
		\mathrm{Var}(
		\private{\rvThree}_{1 \vectorIndex}
		)
		+
		\sum_{\vectorIndex = 1}^\alphabetSize
		\mathrm{Var}(\private{Y}_{1 \vectorIndex})
		\privacyParameterrappor ^2
		\bigl(  \probVecElement{\rvTwo}{\vectorIndex} - \probVecElement{\rvThree}{\vectorIndex}\bigr)^2
		%%%%%%%%%%%%%%%%%
		\\ \stackrel{(a)}{\leq} \hskip 1.57mm &
		\frac{1}{2}
		\sum_{\vectorIndex = 1}^\alphabetSize
		\mathrm{Var}(\private{Y}_{1 \vectorIndex})^2
		+
		\frac{1}{2}
		\sum_{\vectorIndex = 1}^\alphabetSize
		\mathrm{Var}(
		\private{\rvThree}_{1 \vectorIndex}
		)^2
		+
		\privacyParameterrappor ^2
		\sum_{\vectorIndex = 1}^\alphabetSize
		\mathrm{Var}(\private{Y}_{1 \vectorIndex})
		\bigl(  \probVecElement{\rvTwo}{\vectorIndex} - \probVecElement{\rvThree}{\vectorIndex}\bigr)^2
		%%%%%%%%%%%%%%%%%
		\\ \stackrel{(b)}{\leq} \hskip 1.57mm &
		\privacyParameterrappor^2 \normSqMultinomMax
		+
		\privacyParameterrappor^3
		\normSqMultinomMax^{1/2}
		\|\probVec_{\rvTwo} - \probVec_{\rvThree}\|_2^2
		+
		\privacyParameterrappor^2
		\smallNumberrappor
		\|\probVec_{\rvTwo} - \probVec_{\rvThree}\|_2^2,
	\end{align*}
	where step $(a)$ uses $2ab \leq a^2 + b^2$, and step $(b)$ uses Lemma~\ref{rappor:inequalities}.
	For the term~\eqref{term:multinomial:upper:rappor_elltwo:EY1Z2:cov}, 
	note that:
	\begin{align*}
		\mE&
		\bigl[ % Y_2 term parenthesis opens
		(
		\private{\rvThree}_{1 \vectorIndex}
		-
		\mE
		[
		\private{\rvTwo}_{1 \vectorIndex}
		]
		)
		(
		\private{\rvThree}_{1 \vectorIndex'}
		-
		\mE
		[
		\private{\rvTwo}_{1 \vectorIndex'}
		]
		)
		\bigr]
		%%%%%%%%%%%%%%%%%%%%%%%%%%%
		%1st line
		\\&=
		\mE
		\bigl[ % Y_2 term parenthesis opens
		(
		\private{\rvThree}_{1 \vectorIndex}
		-
		\mE
		[
		\private{\rvThree}_{1 \vectorIndex}
		]
		+
		\mE
		[
		\private{\rvThree}_{1 \vectorIndex}
		]
		-
		\mE
		[
		\private{\rvTwo}_{1 \vectorIndex}
		]
		)
		(
		\private{\rvThree}_{1 \vectorIndex'}
		-
		\mE
		[
		\private{\rvThree}_{1 \vectorIndex'}
		]
		+
		\mE
		[
		\private{\rvThree}_{1 \vectorIndex'}
		]
		-
		\mE
		[
		\private{\rvTwo}_{1 \vectorIndex'}
		]
		)
		\bigr]
		%%%%%%%%%%%%%%%%%%%%%%%%%%%
		\\&=
		\mE
		\bigl[ % Y_2 term parenthesis opens
		(
		\private{\rvThree}_{1 \vectorIndex}
		-
		\mE
		[
		\private{\rvThree}_{1 \vectorIndex}
		])
		(
		\private{\rvThree}_{1 \vectorIndex'}
		-
		\mE
		[
		\private{\rvThree}_{1 \vectorIndex'}
		])
		\bigr]
		+
		(
		\mE
		[
		\private{\rvThree}_{1 \vectorIndex}
		]
		-
		\mE
		[
		\private{\rvTwo}_{1 \vectorIndex}
		]
		)
		(
		\mE
		[
		\private{\rvThree}_{1 \vectorIndex'}
		]
		-
		\mE
		[
		\private{\rvTwo}_{1 \vectorIndex'}
		]
		)
		%%%%%%%%%%%%%%%%%%%%%%%%%%%
		\\&=
		\mathrm{Cov}
		(
		\private{\rvThree}_{1 \vectorIndex}
		,
		\private{\rvThree}_{1 \vectorIndex'}
		)
		+
		\privacyParameterrappor^2
		\bigl(  \probVecElement{\rvTwo}{\vectorIndex} - \probVecElement{\rvThree}{\vectorIndex}\bigr)
		\bigl(  \probVecElement{\rvTwo}{\vectorIndex'} - \probVecElement{\rvThree}{\vectorIndex'}\bigr),
	\end{align*}
	for each $\vectorIndex \in [\alphabetSize]$.
	Using this equality, the term~\eqref{term:multinomial:upper:rappor_elltwo:EY1Z2:cov} is bounded as:
	\begin{align*}
		\sum_{1 \leq \vectorIndex \neq \vectorIndex' \leq \alphabetSize}
		\hskip -5mm
		&%%%%%%%%
		\mathrm{Cov}(\private{Y}_{1 \vectorIndex}, \private{Y}_{1 \vectorIndex'})
		\mE
		\bigl( % Y_2 m term parenthesis opens
		\private{\rvThree}_{1 \vectorIndex}
		-
		\mE
		\bigl[
		\private{\rvTwo}_{1 \vectorIndex}
		\bigr]
		\bigr)% Y_2 m term parenthesis closes
		\bigl( % Y_2 m' term parenthesis opens
		\private{\rvThree}_{1 \vectorIndex'}
		-
		\mE
		\bigl[
		\private{\rvTwo}_{1 \vectorIndex'}
		\bigr]
		\bigr)
		%%%%%%%%%%%%%%%%%%%%%%%%%%%%%%	
		\\ = ~ &
		\hskip -3mm
		\sum_{1 \leq \vectorIndex \neq \vectorIndex' \leq \alphabetSize}
		\hskip -5mm
		\mathrm{Cov}(\private{Y}_{1 \vectorIndex}, \private{Y}_{1 \vectorIndex'})
		\mathrm{Cov}
		(
		\private{\rvThree}_{1 \vectorIndex}
		,
		\private{\rvThree}_{1 \vectorIndex'}
		)
		\\
		&~+ \hskip -3mm
		\sum_{1 \leq \vectorIndex \neq \vectorIndex' \leq \alphabetSize}
		\hskip -5mm
		\mathrm{Cov}(\private{Y}_{1 \vectorIndex}, \private{Y}_{1 \vectorIndex'})
		\privacyParameterrappor^2
		\bigl(  \probVecElement{\rvTwo}{\vectorIndex} - \probVecElement{\rvThree}{\vectorIndex}\bigr)
		\bigl(  \probVecElement{\rvTwo}{\vectorIndex'} - \probVecElement{\rvThree}{\vectorIndex'}\bigr)
		%%%%%%%%%%%%%%%%%%%%%%%%%%%%%%%%%%%%%
		\\ \stackrel{(a)}{\leq} \hskip 1.57mm &
		\frac{1}{2}\sum_{1 \leq \vectorIndex \neq \vectorIndex' \leq \alphabetSize}
		\hskip -5mm
		\mathrm{Cov}^2(\private{Y}_{1 \vectorIndex}, \private{Y}_{1 \vectorIndex'})
		+
		\frac{1}{2}\sum_{1 \leq \vectorIndex \neq \vectorIndex' \leq \alphabetSize}
		\hskip -5mm
		\mathrm{Cov}^2
		(
		\private{\rvThree}_{1 \vectorIndex}
		,
		\private{\rvThree}_{1 \vectorIndex'}
		)
				\\ &
		~+
		\frac{1}{2}\sum_{1 \leq \vectorIndex \neq \vectorIndex' \leq \alphabetSize}
		\hskip -5mm
		\mathrm{Cov}^2(\private{Y}_{1 \vectorIndex}, \private{Y}_{1 \vectorIndex'})
~ +
		\frac{1}{2}
		\sum_{1 \leq \vectorIndex \neq \vectorIndex' \leq \alphabetSize}
		\hskip -5mm
		\privacyParameterrappor^4
		\bigl(  \probVecElement{\rvTwo}{\vectorIndex} - \probVecElement{\rvThree}{\vectorIndex}\bigr)^2
		\bigl(  \probVecElement{\rvTwo}{\vectorIndex'} - \probVecElement{\rvThree}{\vectorIndex'}\bigr)^2
		%%%%%%%%%%%%%%%%%%%%%%%%%%%%
		\\ \stackrel{(b)}{\leq}\hskip 1.57mm &
		\frac{3}{2} \privacyParameterrappor^2
		\normSqMultinomMax
		+
		\frac{\privacyParameterrappor^4
		}{2}
		\|\probVec_{\rvTwo} - \probVec_{\rvThree}\|_2^2
		\\ \stackrel{(c)}{\leq} \hskip 1.57mm &
		\frac{5}{2} \privacyParameterrappor^2
		\normSqMultinomMax,
	\end{align*}
	where step $(a)$ uses $ab \leq a^2/2 + b^2/2$, step $(b)$ uses Lemma~\ref{rappor:inequalities}, and step $(c)$ uses the fact that
	$0 < \privacyParameterrappor <1$ for any $\privacyParameter>0$
	and
	$\probVec_{\rvTwo}^\top \probVec_{\rvThree} > 0$.
	Collecting the bounds for~\eqref{term:multinomial:upper:rappor_elltwo:EY1Z2:var} and~\eqref{term:multinomial:upper:rappor_elltwo:EY1Z2:cov}, we finally bound $\mE[
	(
	\bar{\vectorize{\rvTwo}}_1^\top
	\bar{\vectorize{\rvThree}}_1)^2]
	$ as 
	\begin{equation}\label{term:multinomial:upper:rappor_elltwo:EY1Z2:bound}
		\mE[
		(
		\bar{\vectorize{\rvTwo}}_1^\top
		\bar{\vectorize{\rvThree}}_1)^2]
		\leq
		\frac{7}{2}
		\privacyParameterrappor^2 \normSqMultinomMax
		+
		\privacyParameterrappor^3
		\normSqMultinomMax^{1/2}
		\|\probVec_{\rvTwo} - \probVec_{\rvThree}\|_2^2
		+
		\privacyParameterrappor^2
		\smallNumberrappor
		\|\probVec_{\rvTwo} - \probVec_{\rvThree}\|_2^2.
	\end{equation}
	Collecting the bounds
	\eqref{term:multinomial:upper:rappor_elltwo:EY1Y2:bound},
	\eqref{term:multinomial:upper:rappor_elltwo:EZ1Z2:bound},
	and
	\eqref{term:multinomial:upper:rappor_elltwo:EY1Z2:bound}, we finally bound $\momentTwosampleExpSquareBar$ as
	\begin{equation}\label{inequality:rappor_elltwo_exp_square}
		\momentTwosampleExpSquareBar
		\leq
		\frac{7}{2}
		\privacyParameterrappor^2 \normSqMultinomMax
		+
		\privacyParameterrappor^3
		\normSqMultinomMax^{1/2}
		\|\probVec_{\rvTwo} - \probVec_{\rvThree}\|_2^2
		+
		\privacyParameterrappor^2
		\smallNumberrappor
		\|\probVec_{\rvTwo} - \probVec_{\rvThree}\|_2^2
		+
		2
		\smallNumberrappor^2 \alphabetSize.
	\end{equation}

    \vskip 1em
	
	\noindent \textit{Step 2: Apply the two moments method.}
	Using the bounds in
	\eqref{inequality:rappor_elltwo_var_condexp_y}, 
	\eqref{inequality:rappor_elltwo_var_condexp_z} 
	and
	\eqref{inequality:rappor_elltwo_exp_square},
	we show that condition~\eqref{eq:two_moments_original_cond_two_sample} in the two moments method holds if the separation condition~\eqref{eq:two_moments_two_sample_cond_merge} in Lemma~\ref{appendix:lemma:discrete_upper_bound} is met.
	Since $\mE[U_{n_1, n_2}]=
	\privacyParameterrappor^2 \|
	\probVec_\rvTwo - \probVec_\rvThree
	\|_2^2$, assuming $\sampleSize_1 \leq \sampleSize_2$ and $\maxErrorTypeOne = \maxErrorTypeTwo$ for simplicity, the condition~\eqref{eq:two_moments_original_cond_two_sample} of Lemma~\ref{theorem:twosampleTwoMomentsMethod} is satisfied if, 	for all pairs of distributions $P = (P_{Y}, P_{Z}) \in \mathcal{P}_{1,\mathrm{multi}}(\rho_{\sampleSize_1,\sampleSize_2})$, the following conditions hold:
	\begin{align*}
		\privacyParameterrappor^2 \| \probVec_\rvTwo - \probVec_\rvThree \|_2^2
		%%%%%%%%%%%%%%%%%%%%%%%%
		&\geq
		%%%%%%%%%%%%%%%%%%%%%%%%
		\sqrt{ %start square root
			\frac{ %start fraction
				% start numerator
				\privacyParameterrappor^3
				\normSqMultinomMax^{1/2}
				\| \probVec_\rvTwo - \probVec_\rvThree \|_2^2
				+
				\privacyParameterrappor^2
				\smallNumberrappor		
				\| \probVec_\rvTwo - \probVec_\rvThree \|_2^2
				%start numerator
			}{
				%start denumerator
				\beta \sampleSize_1}
		}~\bigl(\text{from}~\eqref{inequality:rappor_elltwo_var_condexp_y}, 
		\eqref{inequality:rappor_elltwo_var_condexp_z}\bigr),~\text{and} 
		\\
		\privacyParameterrappor^2 \| \probVec_\rvTwo - \probVec_\rvThree \|_2^2
		%%%%%%%%%%%%%%%%%%%%%%%%
		&\geq
		%%%%%%%%%%%%%%%%%%%%%%%%
		\sqrt{
			\frac{
				\frac{7}{2}
				\privacyParameterrappor^2 \normSqMultinomMax
				+
				\privacyParameterrappor^3
				\normSqMultinomMax^{1/2}
				\|\probVec_{\rvTwo} - \probVec_{\rvThree}\|_2^2
				+
				\privacyParameterrappor^2
				\smallNumberrappor
				\|\probVec_{\rvTwo} - \probVec_{\rvThree}\|_2^2
				+
				\smallNumberrappor^2 \alphabetSize
			}{ \gamma \beta \sampleSize_1^2}}	~\bigl(\text{from}~\eqref{inequality:rappor_elltwo_exp_square}\bigr).
	\end{align*}
	Since $\sqrt{a} + \sqrt{b} \geq \sqrt{a+b}$ for any nonnegative $a$ and $b$, the conditions above are satisfied when%%%%%%%%%%%%%%%%%%
	\begin{align*}
		\privacyParameterrappor^2 \| \probVec_\rvTwo - \probVec_\rvThree \|_2^2
		%%%%%%%%%%%%%%%%%%%%%%%%
		&\geq
		%%%%%%%%%%%%%%%%%%%%%%%%
		\sqrt{ %start square root
			\frac{ %start fraction
				% start numerator
				\privacyParameterrappor^3
				\normSqMultinomMax^{1/2}
				\| \probVec_\rvTwo - \probVec_\rvThree \|_2^2
				+
				\privacyParameterrappor^2
				\smallNumberrappor		
				\| \probVec_\rvTwo - \probVec_\rvThree \|_2^2
				%start numerator
			}{
				%start denumerator
				\beta \sampleSize_1}
		} ,~\text{and} \\
		\privacyParameterrappor^2 \| \probVec_\rvTwo - \probVec_\rvThree \|_2^2
		%%%%%%%%%%%%%%%%%%%%%%%%
		&\geq
		%%%%%%%%%%%%%%%%%%%%%%%%
		\sqrt{
			\frac{
				\frac{7}{2}
				\privacyParameterrappor^2 \normSqMultinomMax
				+
				\smallNumberrappor^2 \alphabetSize
			}{ \gamma \beta \sampleSize_1^2}}.
	\end{align*}
	Since $\alphabetSize \geq 2$, the above inequalities are further satisfied when
	\begin{align*}
		\| \probVec_\rvTwo - \probVec_\rvThree \|_2^2 
		%%%%%%%%%%%%%%%%%%%%%%%
		\geq
		%%%%%%%%%%%%%%%%%%%%%%%
		C_1(\beta)
		\max \;
		\biggl\{
		\frac{
			\normSqMultinomMax^{1/4}
		}{
			\privacyParameterrappor^{1/2}
			\sampleSize_1^{1/2}}
		,
		\frac{
			\alphabetSize^{1/4}
			\smallNumberrappor^{1/2} 
		}{
			\privacyParameterrappor
			\sampleSize_1^{1/2}}
		\biggr\}.
	\end{align*}
	Therefore, to obtain our desire result, it suffices to show that 
	\begin{equation}\label{rappor_final_goal}
		\max \;
		\biggl\{
		\frac{
			\normSqMultinomMax^{1/4}
		}{
			\privacyParameterrappor^{1/2}
			\sampleSize_1^{1/2}}
		,
		\frac{
			\alphabetSize^{1/4}
			\smallNumberrappor^{1/2} 
		}{
			\privacyParameterrappor
			\sampleSize_1^{1/2}}
		\biggr\}
		\leq 
		6
		\max
		\biggl\{
		%rate 1
		\frac{
			\normSqMultinomMax^{1/4}
		}{
			{\sampleSize_1}^{1/2}
		},
		%rate 2
		\frac{
			\alphabetSize^{1/4}
		}{
			\sampleSize_1^{1/2}\privacyParameter
		}
		\biggr\},
	\end{equation}
	and set $C_u(\gamma, \beta) = 6 C_1(\beta)$.
	Using an indicator function, we separately consider the cases of high privacy \((\privacyParameter \leq 1)\) and low privacy \((\privacyParameter > 1)\).
	Then we derive \eqref{rappor_final_goal} for both cases.
	First, assuming $0 < \privacyParameter \leq 1$, we have
	\begin{align*}
		\max \;
		\biggl\{
		\frac{
			\normSqMultinomMax^{1/4}
		}{
			\privacyParameterrappor^{1/2}
			\sampleSize_1^{1/2}}
		,
		\frac{
			\alphabetSize^{1/4}
			\smallNumberrappor^{1/2} 
		}{
			\privacyParameterrappor
			\sampleSize_1^{1/2}}
		\biggr\}
		%%%%%%%%%%%%%%%%%%%%%%%%%%
		&\leq
		\max \;
		\biggl\{
		\frac{
			\normSqMultinomMax^{1/4}
		}{
			\sampleSize_1^{1/2}}
		,
		\frac{
			\alphabetSize^{1/4}
			\smallNumberrappor^{1/2} 
		}{
			\privacyParameterrappor
			\sampleSize_1^{1/2}}
		\biggr\}
		%%%%%%%%%%%%%%%%%%%%%%%%%%
		\leq
		\max \;
		\biggl\{
		\frac{
			\normSqMultinomMax^{1/4}
		}{
			\sampleSize_1^{1/2}}
		,
		\frac{
			6\alphabetSize^{1/4}
		}{
			\privacyParameter
			\sampleSize_1^{1/2}}
		\biggr\},
		\numberthis
		\label{rappor_final_high_privacy}
	\end{align*}
	where the last inequality uses the fact that for $0 \leq \privacyParameter \leq 1$, we have
	\begin{equation*}
		\frac{\smallNumberrappor^{1/2}}{
			\privacyParameterrappor}
		= 
		\frac{\sqrt{e^{\alpha/2}+1}}{
			(e^{\alpha/2}-1)}
		<
		\frac{6}{\privacyParameter}.
	\end{equation*}
	Next, assuming $ \privacyParameter > 1$, we have
	\begin{align*}
		\max \;
		\biggl\{
		\frac{
			\normSqMultinomMax^{1/4}
		}{
			\privacyParameterrappor^{1/2}
			\sampleSize_1^{1/2}}
		,
		\frac{
			\alphabetSize^{1/4}
			\smallNumberrappor^{1/2} 
		}{
			\privacyParameterrappor
			\sampleSize_1^{1/2}}
		\biggr\}
		%%%%%%%%%%%%%%%%%%%%%%%%%%
		= ~&
		\max \;
		\biggl\{
		\frac{b^{1/4}}{\sampleSize_1^{1/2}}\sqrt{\frac{e^{\alpha/2}+1}{e^{\alpha/2}-1}} + \frac{k^{1/4}}{\sampleSize_1^{1/2}}
		\frac{1}{
			\sqrt{
				e^{\alpha/2 }+1
		} }
		\frac{e^{\alpha/2}+1}{e^{\alpha/2}-1}
		\biggr\}
		%%%%%%%%%%%%%%%%%%%%%%%%%%
		\\
		= ~&
		\sqrt{\frac{e^{\alpha/2}+1}{e^{\alpha/2}-1}} 
		\max \;
		\biggl\{
		\frac{b^{1/4}}{\sampleSize_1^{1/2}},
		\frac{k^{1/4}}{
			n^{1/2}
			\sqrt{
				e^{\alpha/2 }-1
			}
		}
		\biggr\}
		%%%%%%%%%%%%%%%%%%%%%%%%%%
		\\
		\stackrel{(a)}{\leq} \hskip 1.57mm &
		3
		\max \;
		\biggl\{
		\frac{b^{1/4}}{\sampleSize_1^{1/2}},
		\frac{k^{1/4}}{
			n^{1/2}
			\sqrt{
				e^{\alpha/2 }-1
			}
		}
		\biggr\}
		%%%%%%%%%%%%%%%%%%%%%%%%%%
		\\
		\stackrel{(b)}{\leq} \hskip 1.57mm &
		3
		\max \;
		\biggl\{
		\frac{b^{1/4}}{\sampleSize_1^{1/2}},
		\frac{2k^{1/4}}{
			n^{1/2}
			\privacyParameter
		}
		\biggr\},
		%%%%%%%%%%%%%%%%%%%%%%%%%%
		\numberthis
		\label{rappor_final_low_privacy}
	\end{align*}
	where step $(a)$ and step $(b)$ use the following inequalities:
	\begin{equation*}
		\sqrt{\frac{e^{\alpha/2}+1}{e^{\alpha/2}-1}}
		\leq 3,~\text{and}~
		\frac{1}{\sqrt{e^{\privacyParameter/2}-1}} < \frac{2}{\privacyParameter}
		,
	\end{equation*}
	respectively, both of which holds for $\privacyParameter>1$.
	Combining \eqref{rappor_final_high_privacy} and \eqref{rappor_final_low_privacy}, we obtain the following inequality:
	\begin{align*}
		\max \;&
		\biggl\{
		\frac{
			\normSqMultinomMax^{1/4}
		}{
			\privacyParameterrappor^{1/2}
			\sampleSize_1^{1/2}}
		,
		\frac{
			\alphabetSize^{1/4}
			\smallNumberrappor^{1/2} 
		}{
			\privacyParameterrappor
			\sampleSize_1^{1/2}}
		\biggr\}
		\\&=
		\mathds{1}(\alpha \leq 1) \cdot
		\max \;
		\biggl\{
		\frac{
			\normSqMultinomMax^{1/4}
		}{
			\privacyParameterrappor^{1/2}
			\sampleSize_1^{1/2}}
		,
		\frac{
			\alphabetSize^{1/4}
			\smallNumberrappor^{1/2} 
		}{
			\privacyParameterrappor
			\sampleSize_1^{1/2}}
		\biggr\}
		+
		\mathds{1}(\alpha > 1) \cdot
		\max \;
		\biggl\{
		\frac{
			\normSqMultinomMax^{1/4}
		}{
			\privacyParameterrappor^{1/2}
			\sampleSize_1^{1/2}}
		,
		\frac{
			\alphabetSize^{1/4}
			\smallNumberrappor^{1/2} 
		}{
			\privacyParameterrappor
			\sampleSize_1^{1/2}}
		\biggr\}
		\\&\leq
		\mathds{1}(\alpha \leq 1) \cdot
		6
		\max \;
		\biggl\{
		\frac{b^{1/4}}{\sampleSize_1^{1/2}},
		\frac{k^{1/4}}{
			n^{1/2}
			\privacyParameter
		}
		\biggr\}
		+
		\mathds{1}(\alpha > 1) \cdot
		6
		\max \;
		\biggl\{
		\frac{b^{1/4}}{\sampleSize_1^{1/2}},
		\frac{k^{1/4}}{
			n^{1/2}
			\privacyParameter
		}
		\biggr\}
		\\&=
		6
		\max \;
		\biggl\{
		\frac{b^{1/4}}{\sampleSize_1^{1/2}},
		\frac{k^{1/4}}{
			n^{1/2}
			\privacyParameter
		}
		\biggr\}.
	\end{align*}
	This completes the proof of the upper bound through \texttt{RAPPOR} mechanism.
	\end{proof}
	\subsection{Lower Bound}
	The lower bound result follows by combining the lower bound results for the one-sample problem under LDP \citep[Theorem 3.2 in][]{Lam-Weil2021MinimaxConstraint} with results for the two-sample problem without privacy constraints  \citep{chan2014optimal,kim_minimax_2022}. Hence we omit the details. 
	
	\section{Proof of Theorem~\ref{theorem:twosample_conti_rate}}\label{proof:twosample_conti_rates}
	Here we prove the upper bound and lower bound results for density two-sample testing.
	We start with the upper bound~(Appendix~\ref{proof_theorem:twosample_conti_rate}) and move onto the lower bound~(Appendix~\ref{proof:twosample_conti_lower_bound}).
	\subsection{Upper Bound}\label{proof_theorem:twosample_conti_rate}
	The proof of the upper bound  proceeds in two steps, where step 1 leverage the result from multinomial testing (Lemma~\ref{appendix:lemma:discrete_upper_bound}), and step 2 uses the discretization error analysis (Appendix~\ref{appendix:disc_error}).
    
    \vskip 1em
    
		\begin{proof} We verify the two steps mentioned above in order.

\vskip 1em

	\noindent \textit{Step 1: Separation condition for the probability vectors.}
	We derive a separation condition for the probability vectors $\probVec_{\vectorize{\rvTwo}}$ and $\probVec_{\vectorize{\rvThree}}$ of $\binNum^\dimDensity$ 
	categories, obtained by binning the densities $f_{\vectorize{\rvTwo}}$ and $f_{\vectorize{\rvThree}}$, that ensures our multinomial test distinguishes between them using the samples. By substituting $(\probVec_{{\rvTwo}}, \probVec_{{\rvThree}})$ with $(\probVec_{\vectorize{\rvTwo}}, \probVec_{\vectorize{\rvThree}})$ 
	and $\alphabetSize$ with $\binNum^\dimDensity$ in the separation condition~\eqref{eq:two_moments_two_sample_cond_merge} of Lemma~\ref{appendix:lemma:discrete_upper_bound}, we obtain the  following separation condition:
	\begin{align*}
		\| \probVec_{\vectorize{\rvTwo}} - \probVec_{\vectorize{\rvThree}} \|_2
		\geq
		%upper bound
		C_u( \maxErrorTypeOne, \maxErrorTypeTwo)
		\biggl(
		\frac
		{
			\binNum^{\dimDensity/4}
		}
		{(\sampleSize_1 \privacyParameter^2)^{1/2}} 
		\vee
		\frac
		{ \max\{
			\| \probVec_{\vectorize{\rvTwo}} \|^{1/2}_2, 
			\| \probVec_{\vectorize{\rvThree}} \|^{1/2}_2
			\}}
		{{\sampleSize_1}^{1/2}}
		\biggr).
		%upper bound
		\numberthis \label{twosample_upper_conti_start}
	\end{align*}
	Since we assume that 
	$\vertiii{f_{\vectorize{Y}}}_{\mathbb{L}_\infty} < \ballRadius$ and
	$\vertiii{f_{\vectorize{Z}}}_{\mathbb{L}_\infty} < \ballRadius$ in Definition~\ref{def:smooth_distribution_class},
	$\| \probVec_{\vectorize{\rvTwo}} \|_2^2$ is upper bounded as
	\begin{equation*}
		\|
		\probVec_{\vectorize{\rvTwo}}
		\|_2^2
		=
		\sum_{\vectorIndex \in [\binNum^\dimDensity]}
		\left(
		\int_{B_\vectorIndex}
		f_{\vectorize{Y}}(\vectorize{t})\;
		d\vectorize{t}
		\right)^2
		\leq
		\ballRadius
		\sum_{\wavFatherIndex \in [\binNum]^\dimDensity}
		\left(
		\int_{B_\vectorIndex}
		f_{\vectorize{Y}}(\vectorize{t})\;
		d\vectorize{t}
		\right)^2 =
		\ballRadius
		\binNum^\dimDensity
		(\binNum^{-\dimDensity})^2
		=
		\ballRadius
		\binNum^{-\dimDensity},
	\end{equation*}
	and in a similar manner, we also have $\| \probVec_{\vectorize{\rvThree}} \|_2^2 \leq \ballRadius
	\binNum^{-\dimDensity}$.
	Thus, the condition~\eqref{twosample_upper_conti_start} is implied by:
	\begin{equation}\label{twosample_upper_conti_2}
		\| \probVec_{\vectorize{\rvTwo}} - \probVec_{\vectorize{\rvThree}} \|_2
		\geq
		C_1( \ballRadius, \maxErrorTypeOne, \maxErrorTypeTwo)
		\biggl(
		\frac
		{\binNum^{\dimDensity/4}}
		{(\sampleSize_1 \privacyParameter^2)^{1/2}} 
		\vee
		\frac
		{\binNum^{-\dimDensity/4}}
		{{\sampleSize_1}^{1/2}}
		\biggr).
	\end{equation}

    \vskip 1em

	\noindent \textit{Step 2: Separation condition for the densities.}
	Now we find a density separation condition  that ensures the probability vector separation condition \eqref{twosample_upper_conti_2}. 
	From the discretization error analysis~\eqref{equation:disc_error_holder_utilize} and~\eqref{equation:disc_error_besov_utilize} in Appendix~\ref{appendix:disc_error}, both for the H\"{o}lder and Besov case, we have the following inequality:
	\begin{equation*}
		\|\probVec_{\vectorize{\rvTwo}}- \probVec_{\vectorize{\rvThree}}\|_2
		\geq
		C_2(\smoothness, \ballRadius, \dimDensity, \maxErrorTypeOne, \maxErrorTypeTwo)
		\;
		\binNum^{-\dimDensity/2}
		\bigl(
		\vert\kern-0.25ex
		\vert\kern-0.25ex
		\vert
		f_{\vectorize{Y}} - f_{\vectorize{Z}}
		\vert\kern-0.25ex
		\vert\kern-0.25ex
		\vert_{\EllTwo}
		-
		\binNum^{-\smoothness}
		\bigr).
		%\numberthis\label{equation:disc_error_both}
	\end{equation*}
	Using the above inequality,  we derive the following sufficient condition for~\eqref{twosample_upper_conti_2}:
	\begin{align*}
		\vert\kern-0.25ex
		\vert\kern-0.25ex
		\vert
		f_{\vectorize{Y}} - f_{\vectorize{Z}}
		\vert\kern-0.25ex
		\vert\kern-0.25ex
		\vert_{\EllTwo}
		\geq
		C_3 (\smoothness, \ballRadius, \dimDensity, \maxErrorTypeOne, \maxErrorTypeTwo)
		\max
		\biggl\{
		%1 term 1
		\frac
		{\binNum^{3\dimDensity/4}}
		{(\sampleSize_1 \privacyParameter^2)^{1/2}} + \binNum^{-\smoothness}
		,
		% term 2
		\frac
		{\binNum^{\dimDensity/4}}
		{{\sampleSize_1}^{1/2}} + \binNum^{-\smoothness}
		\biggr\}.
	\end{align*}
	Recalling from~\eqref{eq:kappaValueTwosample} that $\binNum \leq \bigl( n^{2/(4\smoothness+\dimDensity)}\wedge (n \privacyParameter^2)^{2/(4\smoothness+3\dimDensity)} \bigr)$, 
	and utilizing
	\begin{equation*}
		\frac{2}{4\smoothness+3\dimDensity}
		\frac{3\dimDensity}{4}
		- 
		\frac{1}{2}
		= 
		\frac{-2\smoothness}{4\smoothness+3\dimDensity}
		~\text{and}~
		\frac{2}{4\smoothness+\dimDensity}
		\frac{d}{4} 
		-
		\frac{1}{2}
		= 
		\frac{-2\smoothness}{4\smoothness+\dimDensity},
	\end{equation*}
	the condition above is implied by:
	\begin{equation*}
		\vert\kern-0.25ex
		\vert\kern-0.25ex
		\vert
		f_{\vectorize{Y}} - f_{\vectorize{Z}}
		\vert\kern-0.25ex
		\vert\kern-0.25ex
		\vert_{\EllTwo}
		\geq
		C_3 (\smoothness, \ballRadius, \dimDensity, \maxErrorTypeOne, \maxErrorTypeTwo)
		\max
		\biggl\{
		(\sampleSize_1 \privacyParameter^2)^{\frac{-2\smoothness}{4\smoothness+3\dimDensity}}
		+
		\binNum^{-\smoothness}
		,
		{\sampleSize_1}^{
			\frac
			{-2\smoothness}
			{4\smoothness+\dimDensity}
		}+
		\binNum^{-\smoothness}
		\biggr\}.
	\end{equation*}
	Since $\bigl( n^{2/(4\smoothness+\dimDensity)}\wedge (n \privacyParameter^2)^{2/(4\smoothness+3\dimDensity)} \bigr) = \binNum + \delta$, where $0 \leq \delta < 1$ and $\binNum \geq 1$, we have $2 \binNum \geq \bigl( n^{2/(4\smoothness+\dimDensity)}\wedge (n \privacyParameter^2)^{2/(4\smoothness+3\dimDensity)} \bigr)$.
	Therefore we have $\binNum^{-\smoothness} \leq 2^\smoothness \bigl( n^{2/(4\smoothness+\dimDensity)}\wedge (n \privacyParameter^2)^{2/(4\smoothness+3\dimDensity)} \bigr)$. Thus the condition above is implied by:
	\begin{equation*}
		\vert\kern-0.25ex
		\vert\kern-0.25ex
		\vert
		f_{\vectorize{Y}} - f_{\vectorize{Z}}
		\vert\kern-0.25ex
		\vert\kern-0.25ex
		\vert_{\EllTwo}
		\geq
		C_4 (\smoothness, \ballRadius, \dimDensity, \maxErrorTypeOne, \maxErrorTypeTwo)
		\max
		\biggl\{
		(\sampleSize_1 \privacyParameter^2)^{\frac{-2\smoothness}{4\smoothness+3\dimDensity}}
		+
		{\sampleSize_1}^{
			\frac
			{-2\smoothness}
			{4\smoothness+\dimDensity}
		}
		\biggr\}.
	\end{equation*}
This concludes the proof of the upper bound result of Theorem~\ref{theorem:twosample_conti_rate}.
	\end{proof}
	\subsection{Lower Bound}\label{proof:twosample_conti_lower_bound}
	By the argument in \citet{Arias-Castro2018RememberDimension}, the lower bound for the minimax testing rate in the one-sample problem, denoted as $\separation^\ast_{n_1, \privacyParameter}$, also provides a lower bound for the minimax rate in two-sample case, denoted as $\separation^\ast_{n_1, n_2, \privacyParameter}$. 
	Thus, we focus on bounding $\separation^\ast_{n_1, \privacyParameter}$ in the following proof.
    
    \vskip 1em
    
	\begin{proof}
	In the one-sample setting,
	let $f_0$ denote the  known density $f_{\vectorize{\rvThree}}$, and assume that the data-generating distribution $P_{\vectorize{Y}}$ lies within the following class:
	\begin{definition}[Smooth distribution classes]\label{def:smooth_distribution_class_gof}
		Let $\pHolderGof$ denote the set of distributions $P_{\vectorize{Y}}$ whose density function $f_{\vectorize{Y}}$ satisfies $(f_{\vectorize{\rvTwo}} - f_0) \in \holderBall$ and
		$
		\vert\kern-0.25ex \vert\kern-0.25ex \vert f_{\vectorize{\rvTwo}} - f_0 \vert\kern-0.25ex \vert\kern-0.25ex \vert_{\EllTwo} \leq \ballRadius.
		$
		Similarly, define $\pBesovGof$ by replacing $\holderBall$ with $\besovBall{2}{\besovParamMicroscope}$.
	\end{definition}
	\noindent
	The one-sample density testing problem is defined as follows: 
	Given  $\sampleSets{\tilde{\vectorize{\rvTwo}}}{\sampleIndexOne}{\sampleSize_1}$ generated from $P_{\vectorize{Y}}$ and privatized via $\privacyParameter$-LDP mechanism $\privacyMechanism$, decide whether $P_{\vectorize{Y}}$ came from
	\begin{equation}\label{one_sample_hypothesis}
		\distClassGeneric_0 = \{ P_{\vectorize{Y}} \in \distClassGeneric : f_{\vectorize{Y}} = f_0 \}
		\quad \text{or} \quad
		\distClassGeneric_1(\rho_{n_1}) = \{ P_{\vectorize{Y}} \in \distClassGeneric : \normEllp{f_{\vectorize{Y}} - f_0}{2}{} \geq \rho_{n_1} \}.  
	\end{equation}
	We consider two problems: $(i)~\distClassGeneric = \pHolderGof$, and $(ii)~\distClassGeneric = \pBesovGof$, with significant overlap in most of the proof steps.
	
	We derive the lower bound by testing the uniform null hypothesis against a carefully constructed mixture alternative. This alternative balances two objectives regarding the distances between the distributions: achieving sufficient separation in \(\EllTwo\) distance without the \(\privacyParameter\)-LDP constraint, and ensuring indistinguishability in the total variation distance under the \(\privacyParameter\)-LDP constraint.
	The mixture construction follows the approach in Appendix A.2.2 of \citet{Lam-Weil2021MinimaxConstraint}, where uniform densities are perturbed by eigenfunctions of the integral operator created from the $\privacyParameter$-LDP mechanism. We generalize this construction to multivariate Besov and H\"{o}lder smoothness classes.
	
	As defined in Definition~\ref{def:LDP},
	let $\privacyMechanism \in \mathcal{\privacyMechanism}_\privacyParameter$ be a non-interactive $\privacyParameter$-LDP mechanism, and
	$
	\{\privacyMechanism_\sampleIndexOne\}_{
		\sampleIndexOne \in [\sampleSizeOne]
	}$ be its marginals.
	% probability measure
	By Lemma B.1 of \citet{Lam-Weil2021MinimaxConstraint}, for each $i \in [n_1]$, there exists a probability measure $\mu_i$ such that  $\privacyMechanism_i(\cdot \,|\, \vectorize{y}_i)$ is absolutely continuous with respect to $\mu_i$ for all $\vectorize{y}_i \in [0,1]^d$. Denote its density with respect to $\mu_i$ by $\privacyDensity_i(\tilde{\vectorize{y}}_i \,|\, \vectorize{y}_i)$.
	Using this, we define the private counterpart of a density function on $[0,1]^d$:
	\begin{definition}[Private counterpart of density]\label{def:private_density}
		Let \( f \) be a probability density on \( [0,1]^d \) and $Q$ be an $\privacyParameter$-LDP mechanism with marginal conditional  densities \( q_i(\tilde{\vectorize{y}}_i \mid \vectorize{y}) \) with respect to  \( \mu_i \). The private counterpart of \( f \) is defined as
		\[
		\tilde{f}_i(\tilde{\vectorize{y}}_i) := \int_{[0,1]^d} q_i(\tilde{\vectorize{y}}_i \mid \vectorize{y}) f(\vectorize{y}) \, d\vectorize{y}.
		\]
	\end{definition}
	Let \( f_0 \) denote the uniform density supported on \( \domainTs \). Using its private counterpart, we introduce the following integral operator whose kernel resembles the likelihood ratio of the privatized densities:
	\begin{definition}[Privacy mechanism intergral operator]\label{integral_operator}
		For each  $\sampleIndexOne \in [\sampleSize_1],$
		define an operator $\kernelOperator_\sampleIndexOne:
		[0,1]^\dimDensity
		\times
		\rvXPrivCodomain
		\to \mathbb{R}$ as:
		\begin{equation*}
			\kernelOperator_\sampleIndexOne(
			\vectorize{s},
			\vectorize{t}
			)
			%%%%%
			:= 
			%%%%%
			\frac{
				\privacyDensity_\sampleIndexOne(
				\vectorize{t} |\vectorize{s})
			}{
				\sqrt{
					\private{f}_{0,\sampleIndexOne}(\vectorize{t})
				}
			}.
		\end{equation*}
		Using $L_i$ as a kernel function, we define an integral operator $K_i:
		\EllTwo([0,1]^\dimDensity)
		\to
		\EllTwo(\rvXPrivCodomain, d\mu_i)
		$ 
		as:
		\begin{equation*}
			(K_\sampleIndexOne f)(\cdot) :=
			\int_{[0,1]^\dimDensity} 
			\kernelOperator_\sampleIndexOne
			( \vectorize{y}, \cdot)
			f(\vectorize{y})
			d\vectorize{y}.
		\end{equation*}
		Let $K_\sampleIndexOne^\ast$ denote the adjoint of $K_\sampleIndexOne$. Finally, by aggregating all the operators for  $\sampleIndexOne \in [\sampleSize_1],$
		we define a symmetric and positive semidefinite integral operator:
		\begin{equation*}
			K:=\sum_{\sampleIndexOne=1}^{\sampleSize_1}
			\frac{1}{\sampleSize_1}
			K_\sampleIndexOne^\ast K_\sampleIndexOne.
		\end{equation*}
	\end{definition}
	For $f \in \EllTwo \bigl( [0,1]^d \bigr) $, 
	the operator $K$ yields the expected  squared privatized likelihood ratio:
	\begin{align*} 
		\langle K f, f \rangle
		%%%%%%
		& = % d mu i notation
		\frac{1}{\sampleSize_1}
		\sum_{\sampleIndexOne = 1}^{\sampleSize}
		\int_{
			\private{
				\rvCodomain{\rvOne}
			}_{\sampleIndexOne}}
		\frac{
			\bigl(
			\int_{[0,1]^\dimDensity}
			\privacyDensity_\sampleIndexOne(
			\private{\vectorize{\rvTwoObs}}_\sampleIndexOne
			|
			\vectorize{\rvTwoObs}
			) f(\vectorize{\rvTwoObs})
			d{\vectorize{\rvTwoObs}}
			\bigr)^2
		}{
			\private{f}_{0,\sampleIndexOne}(\private{\vectorize{y}}_\sampleIndexOne)
		}
		d \mu_i (\private{\vectorize{y}}_\sampleIndexOne)
		%%%%%%%%%%%%%%
		\\
		& =% d Q_f0 notation
		\frac{1}{\sampleSize_1}
		\sum_{\sampleIndexOne = 1}^{\sampleSize}
		\int_{
			\private{
				\rvCodomain{\rvOne}
			}_{\sampleIndexOne}}
		\frac{
			\private{f}_\sampleIndexOne(\private{\vectorize{y}}_\sampleIndexOne)^2
		}{
			\private{f}_{0,\sampleIndexOne}(\private{\vectorize{y}}_\sampleIndexOne)
		}
		\frac{dQ_{f_0}}{
			\private{f}_{0,\sampleIndexOne}(\private{\vectorize{y}}_\sampleIndexOne)
		}
		%%%%%%%%%%%%
		=% likelihood notation
		\frac{1}{\sampleSize_1}
		\sum_{\sampleIndexOne = 1}^{\sampleSize}
		\mE_{Q_{f_0}}
		\left[
		\frac{
			\private{f}^2_\sampleIndexOne(\private{\vectorize{y}}_\sampleIndexOne)
		}{
			\private{f}^2_{0,\sampleIndexOne}(\private{\vectorize{y}}_\sampleIndexOne)
		}
		\right].
		\numberthis
		\label{operator_meaning}
	\end{align*}
	We construct each mixture element by adding carefully chosen eigenfunctions of \( K \) as bumps to \( f_0 \), ensuring the resulting densities belong to the H\"older or Besov classes.
	Specifically, we define two orthonormal sets of smoothness-inducing functions—one for Besov smoothness and one for H\"older smoothness.
	For a positive integer $\primResLev$ and $\omega$>0,
	we set \( \binNum = 2^\primResLev \)
	and determine the number of functions as $\binNum^\dimDensity$.
	The parameter  $\omega$ controls the perturbation height.
	The value of $\primResLev$ and $\omega$ will be optimized later to achieve a tight lower bound.
	Both orthonormal sets have cardinality $\binNum^{\dimDensity}$,
	the supports of their elements partition $[0,1]^d$ without overlap, each element has the $\Ell_\infty$ norm bounded by $\binNum^{\dimDensity/2}$,
	and is orthogonal to $f_0$. We start by defining the orthonormal set for the Besov ball, consisting of bounded step functions from the multivariate Haar basis.
	\begin{definition}[Orthonormal set for the Besov ball]\label{proof:def:ON_set_Besov}
		Consider the Haar wavelet basis on $[0,1]^\dimDensity$ with a prime resolution level $\primResLev \in \mathbb{N}$, as defined in Appendix~\ref{appendix:basis}.
		Set \( \binNum = 2^\primResLev \) and define
		$ \Lambda(\primResLev): = \{0, 1, \dots, (\binNum - 1)\}^\dimDensity $.
		Fix a $\dimDensity$-dimensional on-off multi-index \(\vectorize{\epsilon}^\ast := \{1,0,0,\ldots, 0\}\). Let \(\ONset_{B}\) represent the elements with this multi-index, drawn from the mixed tensor product set at the lowest resolution level \(\primResLev\). Specifically, we define:
		\begin{equation*}
			\ONset_{B}:=\{
			\wavMotherFunc_{\primResLev, \wavMotherIndex}^{\wavMotherBooleanIndex^\ast}
			\}_{
				\wavMotherIndex \in \Lambda(\primResLev)
			}.
		\end{equation*}
	\end{definition}
	\noindent
	Since all results ahead hold for any non-zero \(\vectorize{\epsilon}^\ast\), the specific choice in Definition~\ref{proof:def:ON_set_Besov} is arbitrary.
	Each element of \( \ONset_{B} \) is orthogonal to \( f_0 \) because its multi-index has at least one non-zero entry. In that dimension of non-zero entry, the scaled and shifted Haar wavelet function \( \wavMotherFunc(x) = \mathds{1}_{[0, 1/2)}(x) - \mathds{1}_{[1/2, 1)}(x) \), which integrates to zero, contributes to the integral. By Fubini's theorem, this ensures that the entire integral equals zero.
	Next we define the orthonormal set for H\"{o}lder ball:
	\begin{definition}[Orthonormal set for the H\"{o}lder ball]\label{proof:def:ON_set_Holder}
		For a smoothness parameter \( s > 0 \), let \( \bar{\varphi}: \mathbb{R}^d \to \mathbb{R} \) be an infinitely differentiable function supported on $[0,1]^d$ satisfying
		$\normEllp{\bar{\varphi}}{2}{}=1$,
		$\normEllp{\bar{\varphi}}{\infty}{}<\infty$,
		$\int_{\domainTs} \bar{\varphi} =0$,
		and
		$\vert\kern-0.25ex\vert\kern-0.25ex\vert \bar{\varphi}^{(s')} 
		\vert\kern-0.25ex \vert\kern-0.25ex \vert_\infty < \infty
		$
		for $s' \in  \{ 1, \ldots,  (\lfloor s \rfloor+1)\}$.
		For \( \primResLev \in \mathbb{N} \), set \( \binNum = 2^\primResLev \) and define
		$ \Lambda(\primResLev): = \{0, 1, \dots, (\binNum - 1)\}^\dimDensity $,
		$\bumpHolder_{\primResLev, \wavMotherIndex}(\vectorize{y})
		:=
		\binNum^{\dimDensity/2}
		\bar{\varphi}(
		\binNum \vectorize{y} - \wavMotherIndex
		)
		$
		for \( \wavMotherIndex \in \Lambda(J) \),
		and
		$
		\ONset_{H} := \{ \bumpHolder_{\primResLev, \wavMotherIndex} \}_{\wavMotherIndex \in \Lambda(J)}$, meaning  scaled and shifted copies of $\bar{\varphi}$.
	\end{definition}
	Now we are ready to define the mixtures, one for Besov and one for H\"{o}lder ball.
	\begin{definition}[Mixture alternatives]
		For the H\"older ball, let \( V = \operatorname{Span}(\{f_0\} \cup \ONset_{H}) \). Construct an orthonormal basis of \( V \) as \( \{f_0\} \cup \{u_t\}_{t=1}^{\binNum^\dimDensity}\), where each \( u_t \) is an eigenfunction of \( K \) with eigenvalue \( \lambda_t \). For each $t \in [\binNum^\dimDensity]$, since \( u_t \) is orthogonal to \( f_0 \), it integrates to zero.
		Recall that \( z_\privacyParameter = e^{2\privacyParameter} - e^{-2\privacyParameter} \) for \( \privacyParameter > 0 \),
		and set \( \tilde{\lambda}_t =  (\lambda_t / z_\privacyParameter^2) \vee \binNum^{-d} \).
		Fix \( \omega > 0 \). For any \( \boldsymbol{\eta} = (\eta_1, \ldots, \eta_{\binNum^{d}}) \in \{-1, 1\}^{\binNum^{d}} \), define the perturbed density as:
		\begin{equation}\label{eq:eigen_expansion}
			f^{\boldsymbol{\eta}}_{\omega, \binNum} := f_0 + \omega \sum_{t=1}^{\binNum^{d}} \eta_t \tilde{\lambda}_t^{-1/2} u_t.  
		\end{equation}
		Since  \( u_t \in V \) for each $t \in [\binNum^\dimDensity]$, we can also write:
		\begin{equation}\label{eq:eigenfunction_expansion_by_basis}
			f^{\boldsymbol{\eta}}_{\omega, \binNum} = f_0 + \omega \sum_{t=1}^{\binNum^{d}} \eta_t \tilde{\lambda}_t^{-1/2} \sum_{\psi \in \ONset_{H}} \coef_{\psi}(u_t) \, \psi,  
		\end{equation}
		where \( \coef_{\psi}(u_t) \) is the Fourier coefficient of \( u_t \) with respect to \( \psi \).
		The mixture for H\"{o}lder ball, denoted as $\nu_\rho^H$, is defined as  the uniform probability measure over
		$\{f^{{\boldsymbol{\eta}}}_{\omega, \binNum}: {\boldsymbol{\eta}} \in \{-1, 1\}^{\binNum^{\dimDensity}} \}$.
		The mixture for the Besov ball, denoted as $\nu_\rho^B$, is defined through the same procedure as above, with $\ONset_{H}$ replaced by $\ONset_{B}$.
	\end{definition}
	We now outline the properties of these mixtures, beginning with  their $\EllTwo$ separation from the null:
	\begin{lemma}[$\EllTwo$ separation]\label{lemma:mixture_separation} For any density $f^{{\boldsymbol{\eta}}}_{\omega, \binNum}$ drawn from  $\nu_\rho$,  the $\EllTwo$ distance from $f_0$ is bounded as:
		\begin{align*}\numberthis \label{eq:two_sample_lower_bound_translation}
			\vertiii{ f^{{\boldsymbol{\eta}}}_{\omega, \binNum}- f_0}_{\mathbb{L}_2}
			\geq
			\binNum^{\dimDensity} \omega \sqrt{
				3/4}.
		\end{align*}   
	\end{lemma}
	We next establish a series of conditions on \( \binNum \) and \( \omega \) to ensure that the mixture elements belong to the smooth density class with high probability.
	\begin{lemma}[Nonnegativity]\label{lemma:nonneg}
		For any \( f^{\boldsymbol{\eta}}_{\omega, \binNum} \) drawn from either the distribution \( \nu_\rho^H \) or \( \nu_\rho^B \), we always have $\int_{[0,1]^\dimDensity} f^{{\boldsymbol{\eta}}}_{\omega, \binNum} (\vectorize{y}) d\vectorize{y} = 1$. Also, \( f^{\boldsymbol{\eta}}_{\omega, \binNum} \geq 0 \) with a probability greater than \( 1-\maxErrorTypeOne \) if
		\begin{equation}\label{eq:two_sample_lower_bound_main_suff_2}
			\omega \leq \frac{\binNum^{-\dimDensity}}{\sqrt{2 \log(2 {\binNum^{\dimDensity}}/\maxErrorTypeOne)}}.
		\end{equation} 
	\end{lemma}
	\begin{lemma}[H\"{o}lder class]\label{lemma:holder_ball}
		With probability at least \(1 - \gamma\), a function drawn from \(\nu_\rho^H\) lies in $\holderBall$ if
		\begin{equation*}
			\omega \leq \frac{R\kappa^{-(\dimDensity+s)}
			}{
				C(\smoothness, \bar{\varphi}) 
				\sqrt{2 \log({2\binNum^{\dimDensity}}/\gamma)}},  
		\end{equation*}
		where
		$C(\smoothness, \bar{\varphi})  := 
		\max_{s' \in \{1, \ldots, (\floor{s}+1) \}}
		4\vert\kern-0.25ex
		\vert\kern-0.25ex
		\vert
		\bar{\varphi}^{(s')}
		\vert\kern-0.25ex
		\vert\kern-0.25ex
		\vert_{\mathbb{L}_\infty}
		$.
	\end{lemma}
	\begin{lemma}[Besov class]\label{lemma:besovball}
		With probability at least $1-\maxErrorTypeOne$,
		a function drawn from  $\nu_\rho^{\mathrm{B}}$ lies in $\besovBall{}{\besovParamMicroscope}$ if 
		\begin{equation}
			\omega \leq
			\frac{\ballRadius \binNum^{-(\smoothness + \dimDensity)}}{
				\sqrt{
					2\log(\binNum^\dimDensity/\maxErrorTypeOne)
				}
			}.
		\end{equation}
	\end{lemma}
	Finally, we characterize the condition for $\binNum$ and $\omega$ that ensures any valid private test fails.
	\begin{lemma}[Indistinguishability]\label{lemma:indistinguishability}
		Under the conditions of Lemmas~\ref{lemma:nonneg} and \ref{lemma:holder_ball}, if \begin{equation} \label{eq:two_sample_lower_bound_main_suff_1}
			\omega
			\leq
			(n z_\privacyParameter^2)^{-1/2}
			\left(
			\frac{\log \left[ 1 + 4(1- 2\gamma - \beta)^2\right]}{\binNum^{\dimDensity}}
			\right)^{1/4},
		\end{equation} 
		then no valid $\privacyParameter$-LDP test of level $\gamma$ can distinguish between the uniform density $f_0$ and the alternatives in~\eqref{one_sample_hypothesis} (with $\distClassGeneric = \pHolderGof$) with type II error less than $\beta$. The same conclusion holds when $\pHolderGof$ and the condition in \ref{lemma:holder_ball} is replaced by $\pBesovGof$ and the condition in \ref{lemma:besovball}, respectively. \end{lemma}
	We now gather the results to form a conclusion.
	Recall from Lemma~\ref{lemma:mixture_separation} that $\vertiii{ f^{{\boldsymbol{\eta}}}_{\omega, \binNum}- f_0}_{\mathbb{L}_2}
	\geq
	\binNum^{\dimDensity} \omega \sqrt{
		3/4}$ .
	To maximize this separation while satisfying the conditions of Lemmas \ref{lemma:nonneg}, \ref{lemma:holder_ball}, \ref{lemma:besovball}, and \ref{lemma:indistinguishability}, we choose $\omega$ as the largest possible value:
	\begin{equation*}
		\omega 
		= 
		(n z_\privacyParameter^2)^{-1/2}
		\left(
		\frac{\log \left[ 1 + 4(1- 2\gamma - \beta)^2\right]}{\binNum^{\dimDensity}}
		\right)^{1/4}
		\wedge
		\frac{(R \wedge 1)\kappa^{-(\dimDensity + \smoothness)}
		}{
			C(\smoothness, \bar{\varphi}) 
			\sqrt{2 \log({2 \binNum^{\dimDensity}}/\gamma)}}.
	\end{equation*}
	Substituting this $\omega$ into the $\EllTwo$ separation expression, and collecting the constant terms, 
	we can show that any $\privacyParameter$-LDP test fails if the $\EllTwo$ separation between the hypotheses is larger than:
	\begin{align*}
		C_1(\maxErrorTypeOne, \maxErrorTypeTwo, R, \smoothness)
		\left(
		\frac{\binNum^{3\dimDensity/4}
		}{(n z_\privacyParameter^2)^{1/2}}
		\wedge  
		\frac{\kappa^{-s}
		}{
			\sqrt{ \log({ \binNum^{\dimDensity}})}}
		\right).
	\end{align*}
	Taking $J$ as the largest integer such that
	$2^J \leq C_1(\gamma, \beta, R, \smoothness)(\sampleSize_1 z_\privacyParameter^2)^{2/(4\smoothness+3\dimDensity)}$ 
	and taking $\binNum = 2^J$ leads to the following lower bound for the minimax testing rate:
	\begin{equation*}
		\rho_{\sampleSize_1, \privacyParameter}^\ast
		\geq
		C_2(\gamma, \beta, R, \smoothness)
		\left[
		\frac{
			(n z_\privacyParameter^2)^{
				-2s/(4\smoothness+3\dimDensity)        }
		}{
			\sqrt{
				\log(n z_\privacyParameter^2)
			}
		}\right].
		%lower bound
	\end{equation*}
	As justified in \citet{Lam-Weil2021MinimaxConstraint}, the $\privacyParameter$-LDP 
	minimax testing rate is lower bounded by its non-private counterpart.
	Therefore, combining the result above with the non-private minimax testing rate $\sampleSize_1^{
		-2s/(4s + \dimDensity)
	}$ of \citet{Arias-Castro2018RememberDimension},
	we obtain the final result:
	\begin{align*}
		\rho_{\sampleSize_1, \privacyParameter}^\ast
		\geq
		C_3(\gamma, \beta, R, \smoothness)
		\left[
		\frac{
			(\sampleSize_1 z_\privacyParameter^2)^{
				-2s/(4\smoothness+3\dimDensity)        }
		}{
			\sqrt{
				\log(\sampleSize_1 z_\privacyParameter^2)
			}
		}
		\vee
		n^{
			-2s/(4s + \dimDensity)
		}
		\right].
		%lower bound
	\end{align*}
	This completes the proof for the lower bound result in Theorem~\ref{theorem:twosample_conti_rate}.
	\end{proof}
	%%%%%%%%%%%%%%%%%%%%%%%%%%%%%%%%%%%%%%%%%%%%%%%%%%%
	\subsection{Proofs of the Lemmas in Appendix~\ref{proof:twosample_conti_lower_bound}}
	This section provides proofs for the lemmas used in the lower bound proof given in Appendix~\ref{proof:twosample_conti_lower_bound}.
	\subsubsection{Proof of Lemma~\ref{lemma:mixture_separation}}\label{proof:mixture_separation}
	The proof follows the same reasoning as inequality (31), Lemma B.4, and inequality (32) in the Appendix of \citet{Lam-Weil2021MinimaxConstraint}, with the only difference being the number of orthonormal basis functions. Therefore, we only outline the key steps.

    \vskip 1em
    
	\begin{proof}
	Since \( f^{{\boldsymbol{\eta}}}_{\omega, \binNum} - f_0 \) is a weighted sum of orthonormal singular vectors of \(\privacyMechanism\), its \(\EllTwo\) norm \( \vertiii{f_{{\boldsymbol{\eta}}} - f_0}_{\EllTwo} \) is bounded below by a quantity involving the sum of the corresponding singular values of \(\privacyMechanism\). Bounding this sum  by \( z_\privacyParameter^2 = e^{2\alpha} - e^{-2\alpha} \), using the definition of the \(\privacyParameter\)-LDP constraint, completes the proof.
	\end{proof}
	\subsubsection{Proof of Lemma~\ref{lemma:indistinguishability}}\label{proof:indistinguishability}
	As the proof follows Appendix B.4 and Lemma 3.1 in \citet{Lam-Weil2021MinimaxConstraint}, we outline only the key steps.
    
    \vskip 1em
    
	\begin{proof}
	The argument applies to both \((\pHolderGof, \nu_\rho^H)\) and \((\pBesovGof, \nu_\rho^B)\). For simplicity, denote the distribution class and mixture as \(\mathcal{P}\) and \(\nu_\rho\), respectively.
	Let
	$L_{Q_{\nu_\rho}^n}(\tilde{\mathbf{Y}}_1, \dots, \tilde{\mathbf{Y}}_{n_1})$ be the likelihood ratio between  $Q_{\nu_\rho}^n$ and $Q_{f_0}^n$.
	Define the total variation distance as
	$\| \mathbb{P}_{Q_{\nu_\rho}^n} - \mathbb{P}_{Q_{f_0}^n} \|_{\text{TV}} := \frac{1}{2} \int | L Q_{\nu_\rho}^n - 1 |\; d\mathbb{P}_{Q_{f_0}^n}$,
	and chi-square divergence as
	$
	{\chi^2} 
	(
	Q^\sampleSize_{\nu_\rho}
	\|
	Q^\sampleSize_{f_0}
	)
	:=\frac{1}{2} 
	\bigl( \mathbb{E}_{Q_{f_0}^n} 
	\bigl[
	L_{Q_{\nu_\rho}^n}(\tilde{\mathbf{Y}}_1, \dots, \tilde{\mathbf{Y}}_{n_1})^2 - 1 \bigr]
	\bigr)^{1/2}$.
	The minimax type II error is lower bounded as:
	\begin{align*}
		\inf_{\Delta_{\gamma,Q}}
		\sup_{f \in \mathcal{P}}
		\mathbb{P}_{Q^n_f}
		\left( \Delta_{\gamma,Q}
		(
		\tilde{\vectorize{Y}}_1, \dots, \tilde{\vectorize{Y}}_{\sampleSize_1}
		) = 0 \right)
		&\geq
		\inf_{\Delta_{\gamma,Q}} \mathbb{P}_{Q^n_{\nu_\rho}} \left( \Delta_{\gamma,Q}(
		\tilde{\vectorize{Y}}_1, \dots, \tilde{\vectorize{Y}}_{\sampleSize_1}
		) = 0 \right) - \gamma
		\\
		&\geq 
		1-2\maxErrorTypeOne- \|Q^\sampleSize_{\nu_\rho^H}
		-
		Q^\sampleSize_{f_0}\|_{\mathrm{TV}}
		\\&\geq 
		1-2\maxErrorTypeOne- 
		\frac{1}{2}
		\sqrt{
			{\chi^2} 
			(
			Q^\sampleSize_{\nu_\rho}
			\|
			Q^\sampleSize_{f_0}
			) 
		},
		\numberthis
		\label{inf_sup_chisq}
	\end{align*}
	where the last inequality uses the inequality between chi-square divergence and total variation distance $ {\chi^2} 
	(
	Q^\sampleSize_{\nu_\rho}
	\|
	Q^\sampleSize_{f_0}
	) \geq 4 \|Q^\sampleSize_{\nu_\rho}
	-
	Q^\sampleSize_{f_0}\|^2_{\mathrm{TV}}$ \citep[see, for example, Lemma B.2 of][]{Lam-Weil2021MinimaxConstraint}.
	Thus it suffices to verify that the condition in~\eqref{eq:two_sample_lower_bound_main_suff_1} implies that the last right-hand side term in \eqref{inf_sup_chisq} is lower bounded by $\maxErrorTypeTwo$.
	Thus the proof is completed by bounding the chi-square divergence as
	$
	{\chi^2} 
	(
	Q^\sampleSize_{\nu_\rho}
	\|
	Q^\sampleSize_{f_0}
	) \leq  
	1+
	\exp(\sampleSizeOne^2 \omega^4 z_\privacyParameter^4 \binNum^\dimDensity),
	$
	by utilizing the equation~\eqref{operator_meaning} and  the definition in~\eqref{eq:eigen_expansion} involving the orthonormal singular vectors of $\privacyMechanism$.
	\end{proof}
	\subsubsection{Proof of Lemma~\ref{lemma:nonneg}}
	The proof proceeds in two steps: we first verify that \( f^{\boldsymbol{\eta}}_{\omega, \binNum} \) integrates to 1, and prove  the nonnegativeness property,  using Hoeffding's inequality.
    
    \vskip 1em
    
	\begin{proof} We verify the two steps mentioned above in order.

\vskip 1em

			\noindent \textit{Verification of integration to 1.}
	 Using the expansion of \( f^{\boldsymbol{\eta}}_{\omega, \binNum} \) in terms of the singular vectors from~\eqref{eq:eigen_expansion}, which are orthogonal to \( f_0 \) for any \( \boldsymbol{\eta} \), the integral is computed as follows:
	\begin{align*}
		% integral form
		\int_{[0,1]^\dimDensity}
		f^{\boldsymbol{\eta}}_{\omega, \binNum}(\vectorize{y})
		\; d \vectorize{y}
		%%%%%%%%%%%%%%%%%%%%
		=% inner product form
		\langle
		f^{\boldsymbol{\eta}}_{\omega, \binNum},
		f_0
		\rangle_{\EllTwo}
		%%%%%%%%%%%%%%%%%%%%
		=
		\langle
		f_0, f_0
		\rangle_{\EllTwo}
		+
		\sum_{t = 1}^{\binNum^\dimDensity}
		\omega \eta_t \tilde{\lambda_t}^{-1/2}
		\langle
		f_0, u_t
		\rangle
		= 1,
	\end{align*}
	where the final equality follows from the orthogonality between \( f_0 \) and the \( u_t \)'s. Note that this equation holds for both \( \nu_\rho^H \) and \( \nu_\rho^B \).

    \vskip 1em
    
	\noindent \textit{Verification of nonnegativeness with high probability.}
	The proof follows the approach  of Lemma B.5 in \citet{Lam-Weil2021MinimaxConstraint}, with an adjustment for the change of basis. 
	Since the proof is the same for \( \nu_\rho^H \) and \( \nu_\rho^B \), we state if only for \( \nu_\rho^B \).
	For each \( \psi \in \ONset_{B} \), define an event \( A_{\psi} \) as:
	\[
	A_{\psi}:= \left[ \left| \omega \sum_{t=1}^{\binNum^{\dimDensity}} \eta_t \tilde{\lambda_t}^{-1/2} \coef_{\psi}(u_t) \right| \geq \binNum^{-\dimDensity/2} \right],
	\]
	which belongs to \( \sigma \)-algebra generated by \( \boldsymbol{\eta} \).
	Since the sum of independent scaled Rademacher variables inside $A_{\psi}$ is  sub-Gaussian  with variance proxy
	$
	\omega^2 \tilde{\lambda}_t^{-1} \sum_{t=1}^{\binNum^\dimDensity} \coef_\psi^2(u_t),
	$
	by the union bound, we have:
	\begin{align*}
		\mP_{\nu_\rho^B}
		\left( \bigcup_{\psi \in \ONset_{B}} A_{\psi} \right)
		%%%%%%%%%%%%
		\leq  ~ &
		\sum_{\psi \in \ONset_{B}}
		2 \exp
		\left(
		\frac{
			- \binNum^{-\dimDensity}}{
			2\omega^2 \tilde{\lambda}_t^{-1}
			\sum_{t=1}^{\binNum^\dimDensity}
			\coef_\psi^2(u_t)}
		\right)
		%%%%%%%%%%%%%%%%%%%%%%%%%%%%%%%%%%%%%%
		\\ \stackrel{(a)}{\leq} \hskip 1.57mm &
		%%%%%%%%%%%%%%%%%%%%%%%%%%%%%%%%%%%%%%%%%%%%%%%%
		2 \sum_{\psi \in \ONset_{B}} \exp \left( \frac{-\kappa^{-2\dimDensity}}{2 \omega^2 \sum_{t=1}^{\binNum^{\dimDensity}} \coef^2_{\psi}(u_t) } \right) 
		%%%%%%%%%%%%%%%%%%%%%%%%%%%%%%%%%%%%%%%%%
		\\ \stackrel{(b)}{=} \hskip 1.57mm &
		%%%%%%%%%%%%%%%%%%%%%%%%%%%%%%%%%%%%%%%
		2 \binNum^{\dimDensity} \exp \left( \frac{-\kappa^{-2\dimDensity}}{2 \omega^2 } \right),
		\numberthis
		\label{holder_union_bound}
	\end{align*}
	where step $(a)$ uses \( \tilde{\lambda_t} \leq \binNum^{-\dimDensity} \) and step $(b)$ uses the following equality that stems from the fact that $u_t$'s are orthonormal eigenbasis of $V$ and the orthonormality of $\ONset_{B}$:
	\begin{align*}
		\sum_{t=1}^{\binNum^{\dimDensity}} \coef^2_{\psi}(u_t) 
		&=
		\langle
		\psi,
		\sum_{t=1}^{\binNum^\dimDensity} \langle \psi, u_t \rangle_{\EllTwo} \; 
		u_t
		\rangle_{\EllTwo}
		=
		\langle \psi, \psi \rangle_{\EllTwo} = 1.
	\end{align*}
	Utilizing the  expansion in~\eqref{eq:eigenfunction_expansion_by_basis}, the expression \( f^{\boldsymbol{\eta}}_{\omega, \binNum} - f_0 \) can be represented as:
	\[
	f^{\boldsymbol{\eta}}_{\omega, \binNum} - f_0 = \sum_{\psi \in \ONset_{B}} \left( \omega \sum_{t=1}^{\binNum^{\dimDensity}} \eta_t \tilde{\lambda_t}^{-1/2} \coef_{\psi}(u_t) \right) \psi.
	\]
	Each \( \psi \in \ONset_{B} \) has mutually disjoint support, and the maximum magnitude of \( \psi \) across its domain is \( \binNum^{\dimDensity/2} \). Therefore, \( f^{\boldsymbol{\eta}}_{\omega, \binNum} \) is nonnegative if the event \( \bigcap_{\psi \in \ONset_{B}} A_{\psi}^c \) holds, which occurs with a probability exceeding \( 1-2 \binNum^\dimDensity \exp(-\kappa^{-2\dimDensity}/2\omega^2) \).
	Finally, solving \( \gamma \leq 2 \binNum^\dimDensity \exp(-\binNum^{-2\dimDensity}/2\omega^2) \) for \( \omega \) confirms the condition  \eqref{eq:two_sample_lower_bound_main_suff_2}.
	This completes the proof of Lemma~\ref{lemma:nonneg}.
	\end{proof}
	\subsubsection{Proof of Lemma~\ref{lemma:holder_ball}}
	For a $f^{{\boldsymbol{\eta}}}_{\omega, \binNum}$ 
	drawn from $\nu_\rho^H$,
	we use the basis expansion form in~\eqref{eq:eigenfunction_expansion_by_basis} using $\ONset_{H}$~(defined in~\eqref{proof:def:ON_set_Holder}):
	\begin{align*}
		f^{{\boldsymbol{\eta}}}_{\omega, \binNum}
		=
		f_0
		+
		\omega \;
		\sum_{t=1}^{\binNum^{\dimDensity}}
		\; \eta_t \; \tilde{\lambda_t}^{-1/2}
		\sum_{
			\varphi \in \ONset_{H}
		}
		\coef_{\varphi}(u_t) \;\varphi.
	\end{align*}
	The proof proceeds in two steps: we first identify a sufficient condition for the Lipschitz property and then proceed to establish a sufficient condition for bounded derivatives.

    \vskip 1em
    
	\begin{proof} We verify the two steps mentioned above in order.

\vskip 1em

	\noindent \textit{H\"{o}lder continuity of derivatives with high probability.}
	Fix arbitrary $\vectorize{y}_1, \vectorize{y}_2 \in \domainTs$.
	Since each $\varphi \in \ONset_{H}$ has disjoint support,
	we can specify the element of $\ONset_{H}$ whose $\floor{s}$-derivative has support containing  $\vectorize{y}_1$ and  $\vectorize{y}_2$,
	say  $\varphi_{\vectorize{y}_1}$ and  $\varphi_{\vectorize{y}_2}$.
	Then 
	using the fact that
	$\varphi_{\vectorize{y}_1}^{(\floor{s})}(\vectorize{y}_2) 
	-
	\varphi_{\vectorize{y}_2}^{(\floor{s})}(\vectorize{y}_1) = 0$,
	we can write the difference of the $\floor{s}$-derivative of $f^{{\boldsymbol{\eta}}}_{\omega, \binNum}$ as
	\begin{align*}
		|
		{(f^{{\boldsymbol{\eta}}}_{\omega, \binNum})}^{ (\floor{s})}(\vectorize{y}_1)
		&-
		{(f^{{\boldsymbol{\eta}}}_{\omega, \binNum})}^{ (\floor{s})}(\vectorize{y}_2)
		|
		%%%%%%%%%%%%%%%%%%%%%%%%
		\\=~& % disjoint support
		%%%%%%%%%%%%%%%%%%%%%%%%
		\omega
		\bigg|
		\sum_{t=1}^{\binNum^{\dimDensity}}
		\eta_t
		\tilde{\lambda_t}^{-1/2}
		\bigl\{
		\coef_{\varphi_{\vectorize{y}_1}}(u_t)
		\varphi_{\vectorize{y}_1}^{(\floor{s})}(\vectorize{y}_1)
		-
		\coef_{\varphi_{\vectorize{y}_2}}(u_t)
		\varphi_{\vectorize{y}_2}^{(\floor{s})}(\vectorize{y}_2)
		\bigr\}
		\bigg|
		\\
		%%%%%%%%%%%%%%%%%%%%%%%%
		~\leq~& % basis expansion
		%%%%%%%%%%%%%%%%%%%%%%%%
		\omega
		\bigg|
		\sum_{t=1}^{\binNum^{\dimDensity}}
		\coef_{\varphi_{\vectorize{y}_1}}(u_t)
		\eta_t
		\tilde{\lambda_t}^{-1/2}
		\left(
		\varphi^{(\floor{s})}_{\vectorize{y}_1}(\vectorize{y}_1)
		-
		\varphi^{(\floor{s})}_{\vectorize{y}_1}(\vectorize{y}_2)
		\right)
		\bigg|
		\numberthis \label{eq:two_sample_lower_bound_holder_start}
		\\
		&~+~
		\omega
		\bigg|
		\sum_{t=1}^{\binNum^{\dimDensity}}
		\coef_{\varphi_{\vectorize{y}_2}}(u_t)
		\eta_t
		\tilde{\lambda_t}^{-1/2}
		\left(
		\varphi^{(\floor{s})}_{\vectorize{y}_2}(\vectorize{y}_1)
		-
		\varphi^{(\floor{s})}_{\vectorize{y}_2}(\vectorize{y}_2)
		\right)
		\bigg|
		\numberthis \label{eq:two_sample_lower_bound_holder_start_second},
	\end{align*}
	where the last inequality uses the triangle inequality.
	We separately bound the terms \eqref{eq:two_sample_lower_bound_holder_start} and \eqref{eq:two_sample_lower_bound_holder_start_second}, starting with
	\eqref{eq:two_sample_lower_bound_holder_start}.
	For $\vectorize{y} \in \domainTs$,
	recall from~\eqref{proof:def:ON_set_Holder} that  $\varphi_{\vectorize{y}_1}$ has the form
	$
	\binNum^{\dimDensity/2}
	\varphi(
	\binNum \vectorize{y} - \wavMotherIndex
	)$. 
	By the chain rule, its $\floor{s}$-mixed partial derivatives are written as
	%%%%%%%%%%%%%%
	\begin{equation}\label{holder_chainrule}
		\varphi_{\vectorize{y}_1}^{(\floor{s})}(\vectorize{y})
		=
		\kappa^{(\dimDensity/2) + \floor{s}}
		\bar{\varphi}
		^{(\floor{s})}
		(
		\binNum
		\vectorize{y} - \wavMotherIndex
		).
	\end{equation}
	%%%%%%%%%%%%%%
	Using the above equality, the mean value theorem and the boundedness of all partial mixed derivatives specified in Definition~\ref{proof:def:ON_set_Holder}, the term~\eqref{eq:two_sample_lower_bound_holder_start} is upper bounded by
	\begin{equation*}\label{eq:two_sample_lower_bound_holder_MVT}
		\omega
		\kappa^{(\dimDensity/2) + \floor{s}}
		\bigg|
		\sum_{t=1}^{\binNum^{\dimDensity}}
		\coef_{\varphi_{\vectorize{y}_1}}(u_t)
		\eta_t
		\tilde{\lambda_t}^{-1/2}
		\bigg|
		%%%%%%%%%%%%%%%%%
		C_1(\smoothness, \bar{\varphi})
		\;
		\binNum
		\|\vectorize{y}_1 - \vectorize{y}_2\|,
	\end{equation*}
	where $C_1(\smoothness, \bar{\varphi}) := 
	\bigl(
	2 
	\vert\kern-0.25ex
	\vert\kern-0.25ex
	\vert
	\bar{\varphi}^{(\floor{s})}
	\vert\kern-0.25ex
	\vert\kern-0.25ex
	\vert_{L_\infty}
	%%%%%%%
	\wedge
	%%%%%%%
	\vert\kern-0.25ex
	\vert\kern-0.25ex
	\vert
	\bar{\varphi}^{(\floor{s+1})}
	\vert\kern-0.25ex
	\vert\kern-0.25ex
	\vert_{L_\infty}
	\bigr).
	$
	Bounding the term~\eqref{eq:two_sample_lower_bound_holder_start_second} using the same argument, we obtain the following bound:
	\begin{align*}
		%%%%%% START %%%%%%
		&|
		{(f^{{\boldsymbol{\eta}}}_{\omega, \binNum})}^{ (\floor{s})}(\vectorize{y}_1)
		-
		{(f^{{\boldsymbol{\eta}}}_{\omega, \binNum})}^{ (\floor{s})}(\vectorize{y}_2)
		|
		%%%%%%%%%%%%%%%%%%%
		\\&~\leq
		%%%%% 2nd line %%%%
		\omega
		\kappa^{(\dimDensity/2) + \floor{s}}
		\left(
		\bigg|
		\sum_{t=1}^{\binNum^{\dimDensity}}
		\coef_{\varphi_{\vectorize{y}_1}}(u_t)
		\eta_t
		\tilde{\lambda_t}^{-1/2}
		\bigg|
		+
		\bigg|
		\sum_{t=1}^{\binNum^{\dimDensity}}
		\coef_{\varphi_{\vectorize{y}_2}}(u_t)
		\eta_t
		\tilde{\lambda_t}^{-1/2}
		\bigg|
		\right)
		%%%%
		C_1(\smoothness, \bar{\varphi})\;
		\kappa
		\|\vectorize{y}_1-\vectorize{y}_2\|.
		\numberthis
		\label{holder_lipschitz_final_bound_start}
	\end{align*}
	%%%%%
	For each \( \varphi \in \ONset_{H} \), let \( A_{\varphi} \) be an event within the \( \sigma \)-algebra generated by \( \boldsymbol{\eta} \), defined as follows:
	\begin{equation}\label{holder_event}
		A_{\varphi}:= \left[ \big|  \sum_{t=1}^{\binNum^{\dimDensity}} \eta_t \tilde{\lambda_t}^{-1/2} \coef_{\varphi}(u_t) \bigr|
		\geq
		\sqrt{2 \binNum^\dimDensity \log (2 \binNum^\dimDensity / \maxErrorTypeOne )} \right].  
	\end{equation}
	
	%%%%%
	Since the sum of independent scaled Rademacher variables inside $A_{\varphi}$ is sub-Gaussian with variance proxy
	$
	\tilde{\lambda}_t^{-1} \sum_{t=1}^{\binNum^\dimDensity} \coef_\varphi^2(u_t),
	$
	by the same technique as~\eqref{holder_union_bound}, we have $\mP_{\nu_\rho^H}
	\bigl(
	\bigcap_{\varphi \in \ONset_H} A_\varphi^c
	\bigr) \geq 1-\gamma/2$.
	Under the event $\bigcap_{\varphi \in \ONset_H} A_\varphi^c$,
	the bound~\eqref{holder_lipschitz_final_bound_start} continues to
	\begin{align*}
		%%%%%% START %%%%%%
		|
		{(f^{{\boldsymbol{\eta}}}_{\omega, \binNum})}^{ (\floor{s})}(\vectorize{y}_1)
		-
		{(f^{{\boldsymbol{\eta}}}_{\omega, \binNum})}^{ (\floor{s})}(\vectorize{y}_2)
		|
		%%%%%%%%%%%%%%%%%%%
		\stackrel{(a)}{\leq} \hskip 1.57mm &
		%%%%%%%%%%%%%%%%%%%%%%
		\omega
		\kappa^{\frac{\dimDensity}{2} + \floor{s}}
		2 \sqrt{2\binNum^{\dimDensity} \log(2{\binNum^{\dimDensity}}/\gamma)}
		%%%%%
		\; C_1(\smoothness, \bar{\varphi})\;
		\kappa
		\|\vectorize{y}_1-\vectorize{y}_2\|
		%%%%%%%%%%%%%%%%%%%%%%
		\\ \stackrel{(b)}{\leq} \hskip 1.57mm &
		%%%%%%%%%%%%%%%%%%%%%%
		\omega
		\kappa^{\frac{\dimDensity}{2} + \floor{s}}
		\sqrt{2\binNum^{\dimDensity} \log(2{\binNum^{\dimDensity}}/\gamma)}
		\; C_2(\smoothness, \bar{\varphi}) \;
		%%%%
		\left(
		1
		\wedge
		\kappa
		\|\vectorize{y}_1-\vectorize{y}_2\|
		\right)
		%%%%%%%%%%%%%%%%%%%%%%
		\\ \stackrel{(c)}{\leq} \hskip 1.57mm &
		%%%%%%%%%%%%%%%%%%%%%%
		\omega
		\kappa^{\frac{\dimDensity}{2} + \floor{s}}
		\sqrt{2\binNum^{\dimDensity} \log(2{\binNum^{\dimDensity}}/\gamma)}
		\; C_2(s, \bar{\varphi}) \;
		\binNum^{(s-\floor{s})} \|\vectorize{y}_1 - \vectorize{y}_2\|^{s-\floor{s}}
		%%%%%%%%%%%%%%%%%%%%
		\\ = \hskip 1.73mm&
		%%%%%%%%%%%%%%%%%%%%%%
		\omega
		\kappa^{\dimDensity + s}
		\sqrt{2 \log(2{\binNum^{\dimDensity}}/\gamma)}
		\; C_2(s, \bar{\varphi}) \;
		\|\vectorize{y}_1 - \vectorize{y}_2\|^{s-\floor{s}},
		\numberthis \label{eq:two_sample_lower_bound_Holder_basic_ineqs}
	\end{align*}
	where $C_2(\smoothness, \bar{\varphi}) := 
	\bigl(
	4
	\vert\kern-0.25ex
	\vert\kern-0.25ex
	\vert
	\bar{\varphi}^{(\floor{s})}
	\vert\kern-0.25ex
	\vert\kern-0.25ex
	\vert_{L_\infty}
	%%%%%%%
	\vee
	%%%%%%%
	2\vert\kern-0.25ex
	\vert\kern-0.25ex
	\vert
	\bar{\varphi}^{(\floor{s+1})}
	\vert\kern-0.25ex
	\vert\kern-0.25ex
	\vert_{L_\infty}
	\bigr)$,
	step~$(a)$ uses the event $\bigcap_{\varphi \in \ONset_H} A_\varphi^c$, 
	step~$(b)$ uses the basic inequality that
	\begin{align*}
		(a \times 1) \wedge (b \times c) \leq (a \vee b) \times (1 \wedge c)
		\quad \text{for} \; a,b,c \geq 0,
	\end{align*}
	and
	step~$(c)$ holds by the following fact:
	When $(1 \wedge c) \leq 1$
	and
	$0 \leq s- \floor{s} < 1$,
	it holds that
	\begin{align*}
		(1 \wedge c) \leq (1 \wedge c)^{s - \floor{s}} \leq c^{s - \floor{s}}.
	\end{align*}
	Starting from inequality~\eqref{eq:two_sample_lower_bound_Holder_basic_ineqs}, we solve for \(\omega\) in
	$
	\omega
	\kappa^{\dimDensity + s}
	\sqrt{2 \log(2{\binNum^{\dimDensity}}/\gamma)}
	\; C_2(s, \bar{\varphi}) \leq \ballRadius
	$.
	Solving for \(\omega\), we obtain
	\begin{equation}\label{condition_holder_lipschitz}
		\omega
		\leq
		\frac{\ballRadius \binNum^{-\dimDensity + \smoothness}}{
			C_2(s, \bar{\varphi})  \sqrt{2 \log(2{\binNum^{\dimDensity}}/\gamma)}.
		}
	\end{equation}
	Therefore, if \(\omega\) satisfies this inequality, then with probability at least \(1 - \gamma/2\), the Lipschitz property from Definition~\ref{def:holder_ball} holds for a function \(f^{{\boldsymbol{\eta}}}_{\omega, \binNum}\) drawn from \(\nu_\rho^H\).

    \vskip 1em

    \noindent \textit{Bounded derivatives with high probability.}
	Fix an integer $s' \in \bigl[ \lfloor \smoothness \rfloor \bigr]$ and an arbitrary $\vectorize{y} \in \domainTs$.
	Since each $\varphi \in \ONset_{H}$ has disjoint support,
	we can specify the element of $\ONset_{H}$ whose $s'$-derivative has support containing  $\vectorize{y}$,
	say  $\varphi_{\vectorize{y}}$.
	Utilizing~\eqref{holder_chainrule} to the basis expansion of $f^{{\boldsymbol{\eta}}}_{\omega, \binNum}$ with respect to $\ONset_{H}$,
	we compute the $s'$-derivative of $f^{{\boldsymbol{\eta}}}_{\omega, \binNum}$ as:
	\begin{equation*}
		({f^{{\boldsymbol{\eta}}}_{\omega, \binNum})}^{(s')}(\vectorize{y})
		=~
		%%%%%%%%%%%
		\omega
		\kappa^{\frac{\dimDensity}{2} + s'}
		\biggl(
		\sum_{t=1}^{\binNum^{\dimDensity}}
		\eta_t
		\tilde{\lambda_t}^{-1/2}
		\coef_{\varphi_{\vectorize{y}}}(u_t)
		\biggr)
		\bar{\varphi}
		^{(s')}
		(
		\kappa \vectorize{y} - \wavMotherIndex
		).
	\end{equation*}
	%%%%%%%%%%%%%
	\noindent
	Under the event $\bigcap_{\varphi \in \ONset_H} A_\varphi^c$ as defined in~\eqref{holder_event},
	the derivative above is bounded by $$
	({f^{{\boldsymbol{\eta}}}_{\omega, \binNum})}^{(s')}(\vectorize{y}) \leq \kappa^{\dimDensity + s'}
	\vert\kern-0.25ex
	\vert\kern-0.25ex
	\vert
	\varphi^{(s')}
	\vert\kern-0.25ex
	\vert\kern-0.25ex
	\vert_{\mathbb{L}_\infty}
	\omega \sqrt{2\log(2{\binNum^{\dimDensity}}/\gamma)}.$$
	We solve for \(\omega\) in
	$\kappa^{\dimDensity + s'}
	\vert\kern-0.25ex
	\vert\kern-0.25ex
	\vert
	\varphi^{(s')}
	\vert\kern-0.25ex
	\vert\kern-0.25ex
	\vert_{\mathbb{L}_\infty}
	\omega \sqrt{2\log(2{\binNum^{\dimDensity}}/\gamma)} \leq \ballRadius$. Solving for $\omega$, we obtain
	\begin{equation}
		\omega
		\leq
		\frac{R\kappa^{-(\dimDensity+s')}}
		{
			\vert\kern-0.25ex
			\vert\kern-0.25ex
			\vert
			\varphi^{(s')}
			\vert\kern-0.25ex
			\vert\kern-0.25ex
			\vert_{\mathbb{L}_\infty}
			\sqrt{2\log(2{\binNum^{\dimDensity}}/\gamma)}
		}.
	\end{equation}
	%%%%%%%%%%%
	\noindent
	Therefore, if \(\omega\) satisfies this inequality, then with probability at least \(1 - \gamma/2\), the $s'$-derivative of  \(f^{{\boldsymbol{\eta}}}_{\omega, \binNum}\) drawn from \(\nu_\rho^H\) is bounded by $\ballRadius$.
	Combining the above condition for  $s' \in \bigl[ \lfloor \smoothness \rfloor \bigr]$, 
	we conclude that if
	\begin{equation}\label{condition_holder_bounded_deriv}
		\omega \leq \frac{R\kappa^{-(\dimDensity+s)}
		}{
			C_3(\smoothness, \bar{\varphi}) 
			\sqrt{2 \log({2\binNum^{\dimDensity}}/\gamma)}},	
	\end{equation}
	where $C_3(\smoothness, \bar{\varphi}) := \max_{s' \in \{1, \ldots, \floor{s} \}}
	\vert\kern-0.25ex
	\vert\kern-0.25ex
	\vert
	\bar{\varphi}^{(s')}
	\vert\kern-0.25ex
	\vert\kern-0.25ex
	\vert_{\mathbb{L}_\infty}
	$,
	then with probability at least \(1 - \gamma/2\), the bounded derivatives property of Definition~\ref{def:holder_ball} holds for a function \(f^{{\boldsymbol{\eta}}}_{\omega, \binNum}\) drawn from \(\nu_\rho^H\).

\vskip 1em

	\noindent\textit{Conclusion.} Combining the conditions~\eqref{condition_holder_lipschitz} and~\eqref{condition_holder_bounded_deriv}, we conclude that if
	\begin{equation}
		\omega \leq \frac{R\kappa^{-(\dimDensity+s)}
		}{
			C_4(\smoothness, \bar{\varphi}) 
			\sqrt{2 \log({2\binNum^{\dimDensity}}/\gamma)}},  
	\end{equation}
	where
	$C_4(\smoothness, \bar{\varphi})  := 
	\max_{s' \in \{1, \ldots, (\floor{s}+1) \}}
	4\vert\kern-0.25ex
	\vert\kern-0.25ex
	\vert
	\bar{\varphi}^{(s')}
	\vert\kern-0.25ex
	\vert\kern-0.25ex
	\vert_{\mathbb{L}_\infty}
	$,
	then with probability at least \(1 - \gamma\),
	for a function \(f^{{\boldsymbol{\eta}}}_{\omega, \binNum}\) drawn from \(\nu_\rho^H\) , we have
	$f^{{\boldsymbol{\eta}}}_{\omega, \binNum} \in \holderBall$.
	%%%%%%%%%%%%%%%%%%%%%%%%%%%%%%%%%%%%%%%%%%%
This completes the proof  Lemma~\ref{lemma:holder_ball}.
\end{proof}
	\subsubsection{Proof of Lemma~\ref{lemma:besovball}}
	Recall from Definition~\ref{proof:def:ON_set_Besov} that \(\ONset_{B}\) is a subset of the multivariate Haar wavelet basis defined in Appendix~\ref{appendix:basis}, specifically the mixed tensor products at resolution level \(\primResLev\) with on-off multi-index \(\vectorize{\epsilon}^\ast = \{1,0,0,\ldots, 0\}\). Using the orthonormality of the Haar wavelet basis and the basis expansion of \(f^{{\boldsymbol{\eta}}}_{\omega, \binNum}\) in~\eqref{eq:eigenfunction_expansion_by_basis}, the Fourier coefficient of \(f^{{\boldsymbol{\eta}}}_{\omega, \binNum} - f_0\) with respect to a Haar wavelet \(\psi\) is calculated as
	\begin{align*}
		%%%%% 1. start
		\coef_{\psi}
		(
		f^{{\boldsymbol{\eta}}}_{\omega, \binNum}
		-
		f_0
		)
		%%%%%%%%%%%%%%%%%%
		=
		\begin{cases}
			%%%%%% case 1
			\omega \sum_{t=1}^{\binNum^{\dimDensity}} \eta_t \tilde{\lambda_t}^{-1/2} \coef_{\psi}(u_t),
			~&\text{if}~\psi \in \ONset_B,
			\\
			%%%%%case 2
			0,&\text{otherwise}.
		\end{cases}
	\end{align*}
	%%%%%%%%%%%%%%%%%%%%%%%%
	\noindent
	due to the orthonormality of Haar wavelets~(see Appendix~\ref{appendix:basis} for details). The following proof uses the calculation above.
    
    \vskip 1em
    
	\begin{proof}
	Recall that we set $\binNum = 2^\primResLev$.
	We proceed the analysis for the case of $1 \leq \besovParamMicroscope < \infty$ and $\besovParamMicroscope = \infty$.

\vskip 1em

	\noindent \textit{The case of $1 \leq \besovParamMicroscope < \infty$.}
	The $q$th power of the Besov seminorm~(Definition~\ref{def:besov_norm}) of $
	f^{{\boldsymbol{\eta}}}_{\omega, \binNum}
	-
	f_0
	$ is:
	\begin{align*}
		%norm notation
		\|f^{{\boldsymbol{\eta}}}_{\omega, \binNum}
		-
		f_0\|^\besovParamMicroscope_{
			\smoothness,
			2,
			\besovParamMicroscope
		}
		& =
		% wavelet coefficient form
		\binNum^{\smoothness \besovParamMicroscope}
		\biggl(
		\sum_{
			\psi
			\in 
			\ONset_{B}
		}
		\coef^2_{\psi}
		(
		f^{{\boldsymbol{\eta}}}_{\omega, \binNum}
		-
		f_0
		)
		\biggr)^{\besovParamMicroscope/2}
		\\&=
		% 3. plug in wavelet coefficient
		\binNum^{\smoothness \besovParamMicroscope}
		\biggl\{
		\sum_{
			\psi
			\in 
			\ONset_{B}
		}
		\biggl(
		\omega
		\sum_{t=1}^{\binNum^{\dimDensity}}
		\; \eta_t \;\tilde{\lambda_t}^{-1/2}
		\coef_{\psi}(u_t)
		\biggr)^2
		\biggr\}^{\besovParamMicroscope/2}.
		\numberthis
		\label{bound_besov_start}
	\end{align*}
	%%%%%%%%%%%%%%%%%%%%%%%%%%%%%
	For each \( \psi \in \ONset_{B} \), let \( E_{\psi} \) denote an event within the \( \sigma \)-algebra generated by \( \boldsymbol{\eta} \), defined as follows:
	\begin{equation*}
		E_{\psi}:= \left[ \big|  \sum_{t=1}^{\binNum^{\dimDensity}} \eta_t \tilde{\lambda_t}^{-1/2} \coef_{\psi}(u_t) \bigr|
		\geq
		\sqrt{2 \binNum^\dimDensity \log ( \binNum^\dimDensity / \maxErrorTypeOne )} \right].  
	\end{equation*}
	%%%%%%%%%%%%%%%%%%%%%%%%%%%%%
	Since the sum of independent scaled Rademacher variables is sub-Gaussian with variance proxy
	$
	\tilde{\lambda}_t^{-1} \sum_{t=1}^{\binNum^\dimDensity} \coef_\psi^2(u_t),
	$
	by the same technique as~\eqref{holder_union_bound}, we have $\mP_{\nu_\rho^B}
	\bigl(
	\bigcap_{\psi \in \ONset_B} E_\psi^c
	\bigr) \geq 1-\gamma$.
	%%%%%%%%%%%%%%%%%%%%%%
	Under the event $\bigcap_{\psi \in \ONset_B} E_{\psi}^c$, and noting that the cardinality of $\ONset_B$ is $\binNum^\dimDensity$,
	we bound the quantity in~\eqref{bound_besov_start}  as:
	\begin{equation*}
		\|f^{{\boldsymbol{\eta}}}_{\omega, \binNum}
		-
		f_0\|^\besovParamMicroscope_{
			\smoothness,
			2,
			\besovParamMicroscope
		}
		\;
		\leq
		\;
		\binNum^{\smoothness \besovParamMicroscope}
		\biggl(
		\sum_{
			\phi
			\in 
			\ONset_{B}
		}
		\omega^2
		2 \binNum^\dimDensity
		\log(\binNum^\dimDensity/\maxErrorTypeOne)
		\biggr)^{\besovParamMicroscope/2}
		\\
		\;
		=
		\;
		\binNum^{(\smoothness + \dimDensity) \besovParamMicroscope}
		\omega^\besovParamMicroscope
		\bigl(
		2\log(\binNum^\dimDensity/\maxErrorTypeOne)
		\bigr)^{\besovParamMicroscope/2},
	\end{equation*}
	which is equivalent to
	\begin{equation*}
		\|f^{{\boldsymbol{\eta}}}_{\omega, \binNum}
		-
		f_0\|_{
			\smoothness,
			2,
			\besovParamMicroscope
		}
		\;
		\leq
		\;
		\binNum^{(\smoothness + \dimDensity)}
		\omega
		\sqrt{
			2\log(\binNum^\dimDensity/\maxErrorTypeOne)
		}.
	\end{equation*}

    \vskip 1em
    
	\noindent \textit{The case of  $\besovParamMicroscope = \infty$.}
	Under the event $\bigcap_{\psi \in \ONset_B} E_{\psi}^c$, and noting that the cardinality of $\ONset_B$ is $\binNum^\dimDensity$,
	the Besov seminorm~(Definition~\ref{def:besov_norm}) of $(
	f^{{\boldsymbol{\eta}}}_{\omega, \binNum}
	-
	f_0
	)$ is bounded as: 
	\begin{align*}
		%norm notation
		\|f^{{\boldsymbol{\eta}}}_{\omega, \binNum}
		-
		f_0\|_{
			\smoothness,
			2,
			\besovParamMicroscope
		}
		& =
		% wavelet coefficient form
		\binNum^{\smoothness}
		\biggl(
		\sum_{
			\psi
			\in 
			\ONset_{B}
		}
		\coef^2_{\psi}
		(
		f^{{\boldsymbol{\eta}}}_{\omega, \binNum}
		-
		f_0
		)
		\biggr)^{1/2}
		\\&=
		% 3. plug in wavelet coefficient
		\binNum^{\smoothness }
		\biggl\{
		\sum_{
			\psi
			\in 
			\ONset_{B}
		}
		\biggl(
		\omega
		\sum_{t=1}^{\binNum^{\dimDensity}}
		\; \eta_t \;\tilde{\lambda_t}^{-1/2}
		\coef_{\psi}(u_t)
		\biggr)^2
		\biggr\}^{1/2}
		%%%%%%%%%%%%%%%%%%%%
		\\&\leq
		% 3. plug in wavelet coefficient
		\;
		\binNum^{(\smoothness + \dimDensity)}
		\omega
		\sqrt{
			2\log(\binNum^\dimDensity/\maxErrorTypeOne)
		}.
		\numberthis
		\label{bound_besov_start_infq}
	\end{align*}

    \vskip 1em
    
	\noindent \textit{Conclusion.} 
	From~\eqref{bound_besov_start} and~\eqref{bound_besov_start_infq},
	we have
	$
	\|f^{{\boldsymbol{\eta}}}_{\omega, \binNum}
	-
	f_0\|_{
		\smoothness,
		2,
		\besovParamMicroscope
	}
	\;
	\leq
	\;
	\binNum^{(\smoothness + \dimDensity)}
	\omega
	\sqrt{
		2\log(\binNum^\dimDensity/\maxErrorTypeOne)
	}
	$
	for all possible values of $\besovParamMicroscope$.
	We solve for \(\omega\) in
	$
	\binNum^{(\smoothness + \dimDensity)}
	\omega
	\sqrt{
		2\log(\binNum^\dimDensity/\maxErrorTypeOne)
	}
	\;
	\leq
	\;
	\ballRadius.
	$
	Solving for $\omega$, we obtain
	\begin{equation}
		\omega \leq
		\frac{\ballRadius \binNum^{-(\smoothness + \dimDensity)}}{
			\sqrt{
				2\log(\binNum^\dimDensity/\maxErrorTypeOne)
			}
		}.
	\end{equation}
	%%%%%%%%%%%%%
	Therefore, if \(\omega\) satisfies this inequality, then with probability at least \(1 - \gamma\),  a function \(f^{{\boldsymbol{\eta}}}_{\omega, \binNum}\) drawn from \(\nu_\rho^B\) lies in the Besov ball $\besovBall{}{\besovParamMicroscope}$.
	This concludes the proof of Lemma~\ref{lemma:besovball}.
	\end{proof}
	\section{Proof of Theorem~\ref{theorem:twosample_adaptive_upper}}\label{proof:twosample_adaptive}
	This section proves the upper bound for the adaptive  density test  presented in Theorem~\ref{theorem:twosample_adaptive_upper}.
    
    \vskip 1em
    
	\begin{proof}
	For any $(P_{\vectorize{\rvTwo}}, P_{\vectorize{\rvThree}}) \in \hypoNull$, the type I error is controlled through the union bound:
	\begin{equation*}
		\mathbb{E}
		[	\Delta^{\text{adapt}}_{\maxErrorTypeOne}]
		\leq
		\sum_{{\adaptiveBinNumIndex} \in [\nTest]}
		\mathbb{E}
		[
		\adaptiveSingleTest{\adaptiveBinNumIndex}_{\maxErrorTypeOne/\nTest}
		]
		\leq
		\nTest \cdot \frac{\gamma}{\nTest} = \gamma.
	\end{equation*}
	Now we move onto the type II error guarantee, following the proof strategy of Proposition 7.1 of~\citet{kim_minimax_2022} for non-private adaptive two-sample test.
	Recall from~\eqref{def:n_test_adaptive} that the number of tests is defined as
	\begin{align*}
		\mathcal{N}
		=
		\left\lceil
		\dfrac{2}{\dimDensity}
		\log_2{
			\left(
			\dfrac
			{\sampleSize_1}
			{\log \log \sampleSize_1}
			\right)
		}
		\wedge
		\dfrac{2}{3\dimDensity}
		\log_2{
			\left(
			\dfrac
			{\sampleSize_1 \privacyParameter^2}
			{(\log \sampleSize_1)^2\log \log \sampleSize_1}
			\right)
		}	
		\right\rceil.
	\end{align*}
	Since
	$2/\dimDensity > 2/(d+4\smoothness)$
	and
	$2/3\dimDensity > 2/(3d+4\smoothness)$ for $\smoothness>0$, there exists an integer
	$t^\ast \in [\nTest]$
	such that
	\begin{equation}\label{tau_range}
		\binNum^\ast := 
		2^{t^\ast }
		\leq
		2
		\left[
		\left(
		\dfrac
		{\sampleSize_1}
		{\log \log \sampleSize_1}
		\right)^{
			2/(4\smoothness+\dimDensity)
		}
		\wedge
		\left(
		\dfrac
		{ \sampleSize_1 \privacyParameter^2 }
		{ (\log{\sampleSize_1})^2\log \log \sampleSize_1 }
		\right)^{
			2/(4\smoothness+3\dimDensity)
		}
		\right]
		\leq	
		2^{t^\ast+1}.
	\end{equation}
	Since the type II error of $\Delta^{\text{adapt}}_{\maxErrorTypeOne}$ is upper bounded by
	that of a single inner test $\adaptiveSingleTest{\adaptiveBinNumIndex}_{\maxErrorTypeOne/\nTest}$,
	it suffices to  control the type II error of $\test^{t^\ast}_{\maxErrorTypeOne/\nTest}$.
	The dependence of the two-moments method on the significance level as $1/(\maxErrorTypeOne/\nTest)$ significantly affects the resulting rate. Therefore, we use the following improved two-moments method (Lemma~\ref{lemma:two_moments_method_improve}), which has a logarithmic dependence on $1/(\maxErrorTypeOne/\nTest)$, at the cost of assuming $\sampleSize_1 \asymp \sampleSize_2$:
	\begin{lemma}[{\citealp[Lemma C.1]{kim_minimax_2022}}]\label{lemma:two_moments_method_improve}
		For $0 < \gamma < 1/e$, suppose that there is a sufficiently large constant $C>0$ such that	
		\begin{align*}
			\mE[U_{\sampleSize_1,\sampleSize_2}]
			\geq
			C
			\max\bigg\{
			\sqrt{
				\frac
				{\momentTwosampleVarCondexpY}
				{\beta \sampleSize_1}
			}
			, 
			\sqrt{
				\frac
				{\momentTwosampleVarCondexpZ}
				{\beta \sampleSize_2}
			}
			, 
			\sqrt{
				\frac
				{\momentTwosampleExpSquare}
				{\beta}
			}
			\log\left(\frac{1}{\gamma}\right)
			\left( \frac{1}{\sampleSize_1} + \frac{1}{\sampleSize_2} \right)
			\bigg\}
		\end{align*}
		for all pairs of distributions $P = (P_{\vectorize{Y}}, P_{\vectorize{Z}}) \in  \pAlterTwosample$.
		Then under the assumptions that $\sampleSize_1 \asymp \sampleSize_2$, the type II error of the permutation test over $\pAlterTwosample$ is uniformly bounded by $\beta$.
	\end{lemma}
	Applying Lemma~\ref{lemma:two_moments_method_improve} on the upper bounds of moments in~\eqref{upper_lapu_moment1} and~\eqref{upper_lapu_moment2} (for \texttt{LapU} or \texttt{DiscLapU}) or the bounds in
	\eqref{inequality:rappor_elltwo_var_condexp_y}, 
	\eqref{inequality:rappor_elltwo_var_condexp_z} 
	and
	\eqref{inequality:rappor_elltwo_exp_square} (for \texttt{RAPPOR}), one can verify that the type II error of the two-sample multinomial test with $\alphabetSize$ categories is at most $\maxErrorTypeTwo$ if
	\begin{equation*}\label{appendix:proof:adaptive:old_two_moment}
		\| \probVec_{\rvTwo} - \probVec_{\rvThree} \|_2
		\geq
		%upper bound
		C_1(\maxErrorTypeTwo)
		\sqrt{ \log(\nTest/\maxErrorTypeOne)}
		\biggl(
		\frac
		{
			\alphabetSize^{1/4}
		}
		{(\sampleSize_1 \privacyParameter^2)^{1/2}} 
		\vee
		\frac
		{ \max\{
			\| \probVec_{\rvTwo} \|^{1/2}_2, 
			\| \probVec_{\rvThree}\|^{1/2}_2
			\}}
		{{\sampleSize_1}^{1/2}}
		\biggr).
	\end{equation*}
	Note that since $\adaptiveSingleTest{\adaptiveBinNumIndex^\ast}_{\maxErrorTypeOne/\nTest}$ is $(\privacyParameter/\mathcal{N})$-LDP, the scaling factor for Laplace noise is multiplied by $\nTest$.
	As in Appendix~\ref{proof_theorem:twosample_conti_rate}, 
	we substitute $(\probVec_{{\rvTwo}}, \probVec_{{\rvThree}})$ with $(\probVec_{\vectorize{\rvTwo}}, \probVec_{\vectorize{\rvThree}})$ 
	and $\alphabetSize$ with $(\binNum^\ast)^\dimDensity$ and
	use the discretization error analysis result:
	\begin{equation*}
		\|\probVec_{\vectorize{\rvTwo}}- \probVec_{\vectorize{\rvThree}}\|_2
		\geq
		C_2(\smoothness, \ballRadius, \dimDensity, \maxErrorTypeOne, \maxErrorTypeTwo)
		\;
		(\binNum^\ast)^{-\dimDensity/2}
		\bigl(
		\vert\kern-0.25ex
		\vert\kern-0.25ex
		\vert
		f_{\vectorize{Y}} - f_{\vectorize{Z}}
		\vert\kern-0.25ex
		\vert\kern-0.25ex
		\vert_{\EllTwo}
		-
		(\binNum^\ast)^{-\smoothness}
		\bigr).
	\end{equation*}
	As a result, type II error  of $\adaptiveSingleTest{\adaptiveBinNumIndex^\ast}_{\maxErrorTypeOne/\nTest}$ is at most $\maxErrorTypeTwo$ if:
	\begin{align*}
		\vert\kern-0.25ex
		\vert\kern-0.25ex
		\vert
		f_{\vectorize{Y}} - f_{\vectorize{Z}}
		\vert\kern-0.25ex
		\vert\kern-0.25ex
		\vert_{\EllTwo}
		&\geq
		C_3 (\smoothness, \ballRadius, \dimDensity, \maxErrorTypeOne,\maxErrorTypeTwo)
		\max
		\biggl\{
		%1 term 1
		\nTest 	\sqrt{\log{\nTest}}
		\frac{
			(\binNum^\ast)^{3\dimDensity/4}
		}{
			(\sampleSize_1 \privacyParameter^2)^{1/2}
		} 
		+(\binNum^\ast)^{-\smoothness}
		,
		% term 2
		\sqrt{\log{\nTest}}
		\frac{
			(\binNum^\ast)^{\dimDensity/4}
		}{
			{\sampleSize_1}^{1/2}
		}	+(\binNum^\ast)^{-\smoothness}
		\biggr\}.
	\end{align*}
	Since we have $\mathcal{N} \leq C(\dimDensity) \log \sampleSize_1$
	for any $\dimDensity \geq 1$
	and
	$\sampleSize_1 \geq e^e$, the condition above is satisfied if $	\vert\kern-0.25ex
	\vert\kern-0.25ex
	\vert
	f_{\vectorize{Y}} - f_{\vectorize{Z}}
	\vert\kern-0.25ex
	\vert\kern-0.25ex
	\vert_{\EllTwo}$ is larger than:
	\begin{equation*}
		C_4 (\smoothness, \ballRadius, \dimDensity, \maxErrorTypeOne, \maxErrorTypeTwo) \max
		\biggl\{
		%1 term 1
		(\binNum^\ast)^{3\dimDensity/4}
		\biggl(
		\frac{
			(\log{\sampleSize_1})^2
			\log{\log \sampleSize_1}
		}{
			\sampleSize_1 \privacyParameter^2
		}
		\biggr)^{1/2}
		+	(\binNum^\ast)^{-\smoothness}
		,
		% term 2
		(\binNum^\ast)^{\dimDensity/4}
		\biggl(
		\frac{
			\log{\log \sampleSize_1}
		}{
			{\sampleSize_1}
		}
		\biggr)^{1/2}+	(\binNum^\ast)^{-\smoothness}
		\biggr\}	.
	\end{equation*}
	Applying the left part of inequality~\eqref{tau_range}, the condition above is implied by:
	\begin{equation*}%\label{eq:twosample_adaptive_rate_proof_start_4}
		\vert\kern-0.25ex
		\vert\kern-0.25ex
		\vert
		f_{\vectorize{Y}} - f_{\vectorize{Z}}
		\vert\kern-0.25ex
		\vert\kern-0.25ex
		\vert_{\EllTwo}
		\geq
		C_5 (\smoothness, \ballRadius, \dimDensity, \maxErrorTypeOne, \maxErrorTypeTwo)
		\biggl\{
		(\binNum^\ast)^{-\smoothness}
		+
		%1 term 1
		\biggl(
		\frac{ \sampleSize_1 \privacyParameter^2
		}{
			(\log{\sampleSize_1})^2
			\log{\log \sampleSize_1}
		}
		\biggr)^{\frac{-2\smoothness}{4\smoothness+3\dimDensity}}
		+
		% term 2
		\biggl(
		\frac{
			\sampleSize_1
		}{
			\log{\log \sampleSize_1}
		}
		\biggr)^{
			\frac{-2\smoothness}{4\smoothness+\dimDensity}
		}
		\biggr\}.
	\end{equation*}
	Applying the right part of inequality~\eqref{tau_range}, the condition above is implied by:
	\begin{equation*}%\label{eq:twosample_adaptive_rate_proof_start_5}
		\vert\kern-0.25ex
		\vert\kern-0.25ex
		\vert
		f_{\vectorize{Y}} - f_{\vectorize{Z}}
		\vert\kern-0.25ex
		\vert\kern-0.25ex
		\vert_{\EllTwo}
		\geq
		C_6 (\smoothness, \ballRadius, \dimDensity, \maxErrorTypeOne, \maxErrorTypeTwo)
		\biggl\{
		\biggl(
		\frac{ \sampleSize_1 \privacyParameter^2
		}{
			(\log{\sampleSize_1})^2
			\log{\log \sampleSize_1}
		}
		\biggr)^{\frac{-2\smoothness}{4\smoothness+3\dimDensity}}
		\vee
		% term 2
		\biggl(
		\frac{
			\sampleSize_1
		}{
			\log{\log \sampleSize_1}
		}
		\biggr)^{
			\frac{-2\smoothness}{4\smoothness+\dimDensity}
		}
		\biggr\}.
	\end{equation*}
	This completes the proof of~Theorem~\ref{theorem:twosample_adaptive_upper}.
	\end{proof}
	%%%%%%%%%%%%%%%%%%%%%%%%%%%%%%%%%%%%%%%%%%%%%%%%%%%%%%%%%%%%
	\section{Supplementary Information for the Numerical Results}\label{appendix:simulation}
	This section describes the baseline methods for evaluating our proposed methods in Section \ref{section:simulation} and presents additional numerical results.
	\subsection{Baseline Methods}\label{appendix:baseline}
	To the best of our knowledge, no methods exist for nonparametric two-sample testing under LDP. For a reliable evaluation, we adapt two one-sample testing methods from \citet{Gaboardi2018LDPChisq} to the two-sample setting for our simulations. The first combines the generalized randomized response (\texttt{GenRR}) with a two-sample chi-square statistic (\texttt{Chi}), while the second integrates the \texttt{RAPPOR} mechanism with a projected two-sample chi-square statistic (\texttt{ProjChi}), both calibrated using asymptotic chi-square null distributions. Lastly, we adopt the combination of the generalized randomized response mechanism and $\ell_2$-type U-statistic~\eqref{def:statistic_elltwo} as the third baseline method.

\vskip 1em

	\noindent \textit{Baseline method 1}: \texttt{GenRR}+\texttt{Chi}. We begin by introducing the generalized randomized response mechanism.
	\begin{definition}[Generalized randomized response   for multinomial data: \texttt{GenRR}] \label{def:GenRR_formal}
		Consider a pooled raw multinomial sample $\{X_i\}_{i \in [n]}$ with $\alphabetSize$ categories. Fix the privacy level $\privacyParameter>0$.  Each data owner perturbs their data point, resulting in the privatized sample $\{\tilde{X}_{i}\}_{i \in [n]} = \{\tilde{Y}_i\}_{i\in [n_1]} \cup \{\tilde{Z}_i\}_{i \in [n_2]}$; a  multinomial random sample where the $i$th sample $\tilde{\rvOne}_{i}$ represents a modified version of the original $i$th sample $\rvOne_{i}$, with category change
		from $m \in [\alphabetSize]$ to $\tilde{m} \in [\alphabetSize]$
		occurring according to the following conditional probabilities:
		\begin{equation}\label{def:genrr}
			\mP
			\bigl(
			\tilde{\rvX}_{i} = \tilde{m}
			\;|\;
			\rvX_i = m
			\bigr)
			=
			\dfrac
			{\exp \bigl( \alpha \indicator{	\tilde{m} = m} \bigr) }
			{\exp(\alpha) + \alphabetSize - 1 }.
		\end{equation}
	\end{definition}
	\noindent
	\citet{Gaboardi2018LDPChisq} invoke the exponential mechanism arguement to prove the privacy guarantee. We here present an alternative proof which directly uses the definition of LDP~(Definition~\ref{def:LDP}).
	\begin{lemma}
		The random variables $\{\rvXPriv_{i}\}_{i \in [n]}$ generated by \textnormal{\texttt{GenRR}}  are $\privacyParameter$-LDP views of $\{X_i\}_{i \in [n]}$.
	\end{lemma}
	\begin{proof}
		\begin{align*}
			\sup_{\tilde{m} \in [k], (m, m')\in [k]^2}
			\frac{\mP
				\bigl(
				\tilde{\rvX}_{i} = \tilde{m}
				\;|\;
				\rvX_i = m
				\bigr)
			}{
				\mP
				\bigl(
				\tilde{\rvX}_{i} = \tilde{m}
				\;|\;
				\rvX_i = m'
				\bigr)
			}
			&=
			\sup_{\tilde{m} \in [k], (m, m')\in [k]^2}
			\frac{
				\exp \bigl( \alpha \indicator{	\tilde{m} = m} \bigr)
				/
				\bigl(  \exp(\alpha) + \alphabetSize - 1 \bigr) 
			}{
				\exp \bigl( \alpha \indicator{	\tilde{m} = m'} \bigr)
				/
				\bigl(  \exp(\alpha) + \alphabetSize - 1 \bigr)
			}
			\\&=
			\sup_{\tilde{m} \in [k], (m, m')\in [k]^2}
			\exp \Bigl( \alpha 
			\bigl(
			\indicator{	\tilde{m} = m}
			-
			\indicator{	\tilde{m} = m'}
			\bigr)
			\Bigr)
			\\&= 
			\exp(\privacyParameter).
		\end{align*}
	\end{proof}
	Now we define a chi-square statistic, referred to as \texttt{Chi}, 
	which takes the \texttt{GenRR} views  as inputs:
	\begin{equation}\label{def:chi_statistic}
		T_{\sampleSize_1, \sampleSize_2}
		:=
		\biggl( \dfrac{1}{\sampleSize_1}+\dfrac{1}{\sampleSize_2} \biggr)^{-1}
		(
		\hat{\boldsymbol{\mu}}_{\tilde{\rvY}}
		-
		\hat{\boldsymbol{\mu}}_{\tilde{\rvY}}
		)^\top
		\mathrm{diag}(
		\hat{\probVec}
		)^{-1}
		(
		\hat{\boldsymbol{\mu}}_{\tilde{\rvY}}
		-
		\hat{\boldsymbol{\mu}}_{\tilde{\rvY}}
		),
	\end{equation}
	where for $\vectorIndex \in  [\alphabetSize]$,
	the $\vectorIndex$th elements of $\hat{\boldsymbol{\mu}}_{\tilde{\rvY}}, \hat{\boldsymbol{\mu}}_{\tilde{\rvY}}$ and $\hat{\probVec}$ are defined as:
	\begin{align*}
		% mean y
		\hat{\mu}_{\tilde{\rvTwo} \vectorIndex}
		&:=
		\frac{1}{\sampleSize_1}
		\sum_{i=1}^{n_1}
		\mathds{1}
		({\tilde{\rvY}_i=\vectorIndex}),
		%
		%
		% mean z
		\quad
		\hat{\mu}_{\tilde{\rvThree} \vectorIndex}
		:=
		\frac{1}{\sampleSize_2}
		\sum_{j=1}^{n_2}
		\mathds{1}
		({\tilde{\rvZ}_j=\vectorIndex}),~\text{and}
		\\
		\hat{p}_\vectorIndex
		&:=
		\frac{
			\sum_{i=1}^{n_1}
			\mathds{1}
			({\tilde{\rvY}_i=\vectorIndex})
			+
			\sum_{j=1}^{n_2}
			\mathds{1}
			({\tilde{\rvZ}_j=\vectorIndex})
		}{\sampleSize_1 + \sampleSize_2},
	\end{align*}
	respectively.
	The test statistic $T_{\sampleSize_1, \sampleSize_2}$~\eqref{def:chi_statistic} can  be calibrated using their asymptotic null distributions:
	\begin{lemma}[Asymptotic null distribution of \texttt{GenRR+Chi}] \label{asymp_genrr_chi}
		Fix the number of categories $\alphabetSize$ and the multinomial probability vectors $\mathbf{p}_Y$ and $\mathbf{p}_Z$, each contained in the interior of the set $\Omega := \{\mathbf{p} \in \mathbb{R}^k : p_\vectorIndex > 0, \sum_{m=1}^{k-1} p_m \leq 1, p_{k} = 1 - \sum_{m=1}^{k-1} p_m\}$. Fix the privacy level $\privacyParameter > 0$. 
		Under the null hypothesis $\mathbf{p}_Y = \mathbf{p}_Z$, for each pair of sample sizes $(n_1, n_2)$, compute the statistic $T_{n_1, n_2}$ based on the  \textnormal{\texttt{GenRR}} $\privacyParameter$-LDP views $\{\tilde{Y}_i\}_{i \in [n_1]}$ and $\{\tilde{Z}_i\}_{i \in [n_2]}$, generated from $\mathbf{p}_Y$ and $\mathbf{p}_Z$, respectively. Then, as $n_1, n_2 \to \infty$, we have $T_{n_1, n_2} \stackrel{d}{\to} \chi^2_{\alphabetSize-1}$.
	\end{lemma}
	The proof given in Appendix~\ref{proof:asymp_genrr_chi} follows directly from Lindberg's central limit theorem (CLT), as outlined in Proposition 2.27 of \citet{van_der_vaart_asymptotic_1998}, and the fact that a chi-square random variable arises from an asymptotic normal distribution with a projection matrix as its covariance. 

    \vskip 1em
\noindent \textit{Baseline method 2}: \texttt{RAPPOR}+\texttt{ProjChi}.
	We propose a projected Hotelling's-type statistic, referred to as \texttt{ProjChi}, which takes the outputs of \texttt{RAPPOR} as input, defined as follows: 
	\begin{equation}\label{def:statistic:projchi}
		T^{\mathrm{proj}}_{\sampleSize_1, \sampleSize_2}
		:=\biggl( \dfrac{1}{\sampleSize_1}+\dfrac{1}{\sampleSize_2} \biggr)^{-1}
		(
		\check{\rVecY} 
		-
		\check{\rVecZ} 
		)^\top
		\Pi
		\hat{\Sigma}^{-1}
		\Pi
		(
		\check{\rVecY} 
		-
		\check{\rVecZ}
		),
	\end{equation}
	where
	$
	% mean y
	\check{\rVecY} 
	:=
	\sum_{i=1}^{n_1}
	\tilde{\rVecY}_i
	/\sampleSize_1$
	and
	$
	\check{\rVecZ} 
	:=
	\sum_{j=1}^{n_2}
	\tilde{\rVecZ}_j
	/\sampleSize_2
	$
	are the sample means of the \texttt{RAPPOR} $\privacyParameter$-LDP views,
	$\hat{\Sigma}$ is the pooled empirical covariance matrix:
	\begin{equation}
		\hat{\Sigma} := \frac{(n_1 - 1) \hat{\Sigma}_1 + (n_2 - 1) \hat{\Sigma}_2 }{n_1 + n_2 - 2},
	\end{equation}
	where
	\begin{align*}
		\hat{\Sigma}_1 &:= \frac{1}{n_1 - 1} \sum_{i=1}^{n_1} (\tilde{\mathbf{Y}}_i - \check{\mathbf{Y}})(\tilde{\mathbf{Y}}_i - \check{\mathbf{Y}})^\top,~\text{and}
		\\
		\hat{\Sigma}_2 &:= \frac{1}{n_2 - 1} \sum_{i=1}^{n_2} (\tilde{\mathbf{Z}}_i - \check{\mathbf{Z}})(\tilde{\mathbf{Z}}_i - \check{\mathbf{Z}})^\top,
	\end{align*}
	and $\Pi = I_\alphabetSize - \mathbf{1} \mathbf{1}^\top$ is the projection matrix to the subspace spanned by the one-vector $\mathbf{1}$.
	The test statistic  $T^{\mathrm{proj}}_{\sampleSize_1, \sampleSize_2}$~\eqref{def:statistic:projchi} can  be calibrated using its asymptotic null distribution:
	
	\begin{lemma}[Asymptotic null distribution of \texttt{RAPPOR+ProjChi}] \label{asymp_rappor_projchi} 
		Assume the same multinomial setting as in Lemma~\ref{asymp_genrr_chi}.
		Fix the privacy level $\privacyParameter>0$.
		Under the null hypothesis $\mathbf{p}_Y = \mathbf{p}_Z$, for each pair of sample sizes $(n_1, n_2)$, compute the statistic $T_{n_1, n_2}$ based on the  \textnormal{\texttt{RAPPOR}} $\privacyParameter$-LDP views $\{\tilde{\mathbf{Y}}_i\}_{i \in [n_1]}$ and $\{\tilde{\mathbf{Z}}_i\}_{i \in [n_2]}$, generated from $\mathbf{p}_Y$ and $\mathbf{p}_Z$ , respectively. Then, as $n_1, n_2 \to \infty$, we have $T^{\mathrm{proj}}_{\sampleSize_1, \sampleSize_2}
		\stackrel{d}{\to} \chi^2_{\alphabetSize-1}$.
	\end{lemma}
	The proof, provided in Appendix~\ref{proof:asymp_rappor_projchi}, follows a similar approach to the proof of Lemma~\ref{asymp_genrr_chi}. The key difference is that the \texttt{RAPPOR} $\privacyParameter$-LDP views are no longer multinomial samples.

\vskip 1em
\noindent \textit{Baseline method 3}: \texttt{GenRR}+$\ell_2$. This method applies the $\ell_2$-type U-statistic in~\eqref{def:statistic_elltwo} to the \texttt{GenRR} $\privacyParameter$-LDP views, with calibration via  permutation procedures.
	\subsection{Power Comparison of \texttt{RAPPOR}-Based Methods under Additional Scenarios}\label{appendix:powerlaw}
	Within the tests based on the \texttt{RAPPOR} mechanism, as shown in Figures~\ref{fig:power_multinomial}, \ref{fig:power_density_location}, and~\ref{fig:power_density_scale}, we observe that \texttt{RAPPOR}+\texttt{ProjChi} and \texttt{RAPPOR}+$\ell_2$ exhibit comparable performance. In certain cases, \texttt{RAPPOR}+\texttt{ProjChi} demonstrates slightly greater power. The dominance of one test over the other depends on the specific scenario, particularly on the privacy level $\privacyParameter$ and the relative signal strength of the chi-square divergence versus the $\ell_2$ distance. 
	
	In the perturbed uniform distribution scenario considered for Figure~\ref{fig:power_multinomial}, the chi-square divergence is relatively strong, while the $\ell_2$ distance is relatively weak compared to other potential scenarios. To contrast this, we now present a multinomial testing scenario where the chi-square divergence is relatively weak. In this scenario, the probability vectors follow a power law as in~\eqref{nulldist}, but with different powers. Specifically, we consider the following setup: for $\vectorIndex = 1, \ldots, \alphabetSize=40$, the \( m \)th elements of \( \vectorize{p}_Y \) and \( \vectorize{p}_Z \) are defined as:
	\begin{equation}\label{alternative:powerlaw}
		\probVecElement{\rvTwo}{\vectorIndex}
		\propto 
		\vectorIndex^{2.45},
		~\text{and}~
		\probVecElement{\rvThree}{\vectorIndex}
		\propto 
		\vectorIndex^{2.3}.
	\end{equation}
	We also examine the effect of varying the privacy parameter $\privacyParameter \in \{4, 2, 1\}$ to assess how the privacy level influences the power difference between the two methods. 
	\begin{figure}
		\centering
		\includegraphics[width=0.95\linewidth]{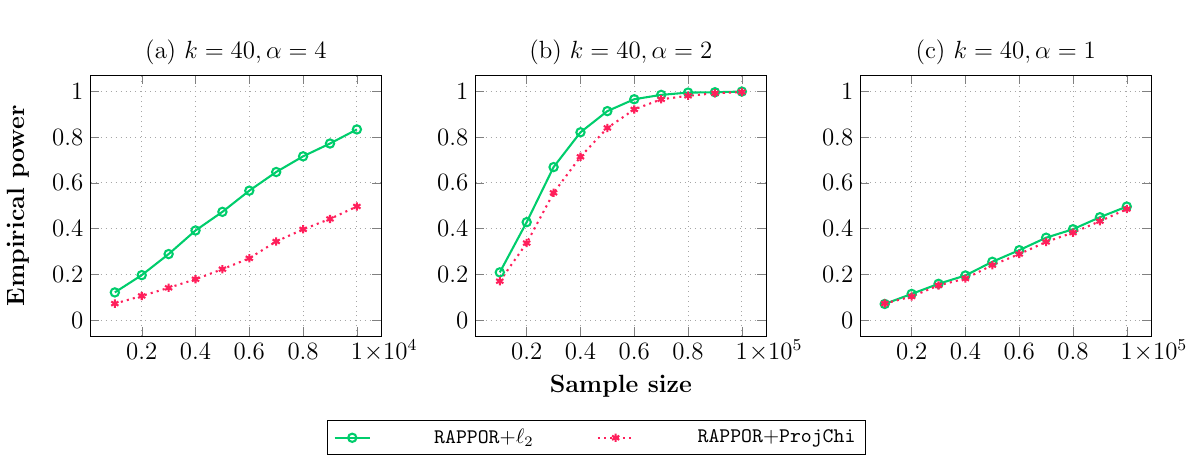}
		\caption{Comparison of the testing power between 
			\texttt{RAPPOR}+$\ell_2$ and \texttt{RAPPOR}+\texttt{ProjChi} under the power law alternatives in~\eqref{alternative:powerlaw}. To ensure a fair comparison, both tests are calibrated using permutation procedures at level $\gamma = 0.05$.}
		\label{fig:power_powerlaw_24}
	\end{figure}
	The results presented in Figure~\ref{fig:power_powerlaw_24} demonstrate that in this power law scenario, where the signal is strong in $\ell_2$ distance and weak in chi-square divergence, \texttt{RAPPOR}+$\ell_2$ outperforms \texttt{RAPPOR}+\texttt{ProjChi}. This difference is more pronounced in the low-privacy setting (meaning  larger $\privacyParameter$ values) and becomes more subtle as the privacy level increases (meaning  smaller $\privacyParameter$ values).
	
	An intuition behind this phenomenon is as follows. For each $\vectorIndex \in [\alphabetSize]$, the $\vectorIndex$th entry of the \texttt{RAPPOR} output follows  
	$
	\mathrm{Ber} \left( \privacyParameterrappor \mathds{1}(\rvTwo_{\sampleIndexOne} = \vectorIndex) + \smallNumberrappor \right),
	$
	where $\privacyParameterrappor = \privacyParameter/4 + o(\privacyParameter)$ and $\smallNumberrappor = 1/2 + o(1)$, with \( o(1) \) representing a term that vanishes as $\privacyParameter \to 0$. Threrfore, as $\privacyParameter$ decreases, the distributions approach uniform multinomial distributions, diminishing the difference between chi-square divergence and $\ell_2$ distance in the signal.
	
	\subsection{Numerical Result for Density Testing for Scale Alternatives}\label{simul_scale}
	Similar to the location alternative scenario~\eqref{alternative:density_location}, we analyze scenarios involving covariance differences between two $\dimDensity$-dimensional Gaussian distributions
	$P_{\vectorize{\rvTwo}} = \mathcal{N}(
	\boldsymbol{\mu}_{\vectorize{\rvTwo}}
	,
	\Sigma_{\vectorize{\rvTwo}}
	)$
	and
	$P_{\vectorize{\rvThree}} = \mathcal{N}(
	\boldsymbol{\mu}_{\vectorize{\rvThree}}
	,
	\Sigma_{\vectorize{\rvThree}}
	)$. 
	We set the mean vectors and covariance matrices of the Gaussian distributions as
	\begin{equation}\label{alternative:density_scale}
		\boldsymbol{\mu}_{\vectorize{\rvTwo}}
		=
		\boldsymbol{\mu}_{\vectorize{\rvThree}}
		= \mathbf{0}_{\dimDensity},
		\Sigma_{\vectorize{\rvTwo}} = 
		0.5\times \mathbf{J}_{\dimDensity}
		+
		0.5 \times \mathbf{I}_{\dimDensity},
		\quad \text{and} \quad
		\Sigma_{\vectorize{\rvThree}}
		=
		5\times \Sigma_{\vectorize{\rvTwo}}.
	\end{equation}
    	The power results against the scale  alternatives, provided in Figure~\ref{fig:power_density_scale}, shows similar patterns to Figure~\ref{fig:power_density_location}.
	\begin{figure}
		\centering
		\includegraphics[width=0.95\linewidth]{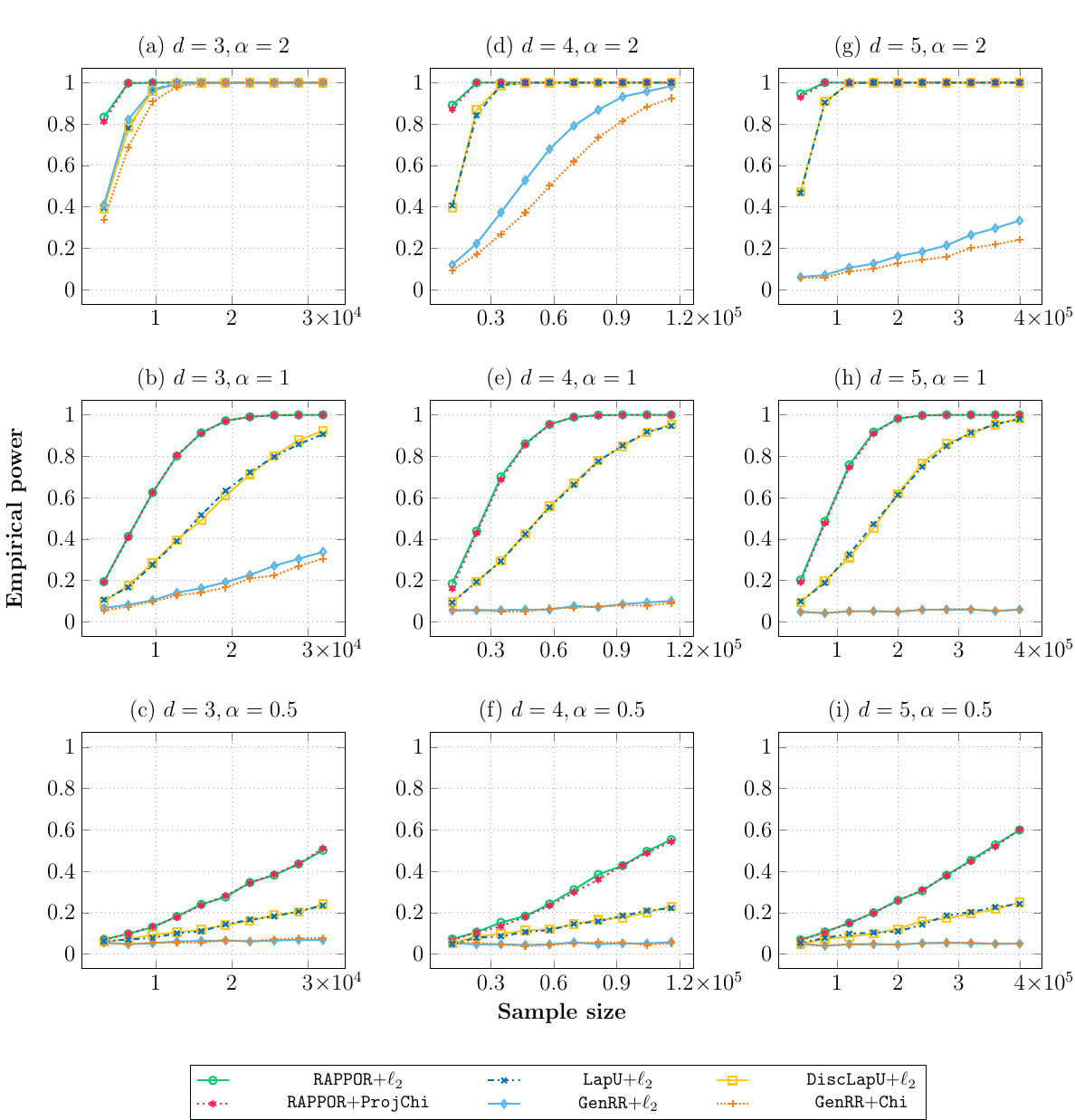}

		\caption{Comparison of the density testing power between our proposed methods (first row in the legend) and baseline methods (second row in the legend) under the scale alternatives in~\eqref{alternative:density_scale}. To ensure a fair comparison, all methods are calibrated using permutation procedures at level $\gamma = 0.05$.
		}
		\label{fig:power_density_scale}
	\end{figure}

	\subsection{Proof of Asymptotic Null Distributions}
	This section presents the derivation of chi-square asymptotic null distributions of \texttt{GenRR+Chi} and \texttt{RAPPOR+Chi} $\privacyParameter$-LDP tests introduced in Appendix~\ref{appendix:baseline}, starting with $\texttt{GenRR+Chi}$.
	\subsubsection{Proof of Lemma~\ref{asymp_genrr_chi}}\label{proof:asymp_genrr_chi}
	The proof is a straighforward application of Lindberg's CLT.
    
    \vskip 1em
    
	\begin{proof}
	Throughout the proof, assume that the null hypothesis $\probVec_{\rvTwo} = \probVec_{\rvZ}$ holds. Thus, for each $\vectorIndex \in [\alphabetSize]$, we have $\probVecElement{\tilde{\rvTwo}}{\vectorIndex} = \probVecElement{\tilde{\rvThree}}{\vectorIndex}$. Without loss of generality, we  denote all $\probVecElement{\tilde{\rvThree}}{\vectorIndex}$ as $\probVecElement{\tilde{\rvTwo}}{\vectorIndex}$.
	For each
	$\sampleIndexOne \in [\sampleSize_1]$
	and
	$\sampleIndexTwo \in [\sampleSize_2]$,
	we can  view $\tilde{\rvY}_{\sampleIndexOne}$ and $\tilde{\rvZ}_{\sampleIndexTwo}$ as  drawn independently from a multinomial distribution with probability vector defined as:
	\begin{align*}
		\probVec_{\tilde{\rvTwo}}
		:=
		\frac
		{\exp(\alpha)}
		{\exp(\alpha) + \alphabetSize - 1 }\probVec_{\rvTwo}
		+
		\frac
		{1}
		{\exp(\alpha) + \alphabetSize - 1 }(\mathbf{1}_\alphabetSize-\probVec_{\rvTwo})	.
		\numberthis \label{genrr_transformed_p}
	\end{align*}
	For each $m \in [\alphabetSize]$, the $\vectorIndex$th term of $
	\hat{\boldsymbol{\mu}}_{\tilde{\rvY}}
	-
	\hat{\boldsymbol{\mu}}_{\tilde{\rvY}}
	$ can be written as
	\begin{equation*}
		\hat{{\mu}}_{ \tilde{\rvY} \vectorIndex}
		-
		\hat{{\mu}}_{ \tilde{\rvZ} \vectorIndex}
		=
		\frac{1}{\sampleSize_1}
		\sum_{\sampleIndexOne=1}^{\sampleSize_1}
		\bigl(
		\mathds{1}(\tilde{\rvTwo}_\sampleIndexOne = \vectorIndex)
		-
		\probVecElement{\tilde{\rvTwo}}{\vectorIndex}
		\bigr)
		-
		\frac{1}{\sampleSize_2}
		\sum_{\sampleIndexTwo=1}^{\sampleSize_2}
		\bigl(
		\mathds{1}(\tilde{\rvThree}_\sampleIndexTwo = \vectorIndex)
		-
		\probVecElement{\tilde{\rvTwo}}{\vectorIndex}
		\bigr).
	\end{equation*}
	To apply Lindeberg's CLT, we construct a triangular array sequence. Each row of the array corresponds to the pooled sample size $N \coloneqq \sampleSize_1 + \sampleSize_2$. Within each row, the columns are indexed by $\ell \in [N]$. For each $N$ and $\ell$, we define a $\alphabetSize$-dimensional random vector $\tilde{\rVecX}_{N[:,\ell]}$, whose $\vectorIndex$th element, denoted as $\tilde{\rVecX}_{N[\vectorIndex,\ell]}$, is given by: 
	\begin{align*}
		\tilde{\rVecX}_{N[\vectorIndex,\ell]} :=
		\begin{cases}
			\sqrt{\biggl( \dfrac{1}{\sampleSize_1}+\dfrac{1}{\sampleSize_2} \biggr)^{-1}}
			\dfrac{1}{\sampleSize_1}
			\bigl(
			\mathds{1}(\tilde{\rvTwo}_\ell = \vectorIndex)
			-
			\probVecElement{\tilde{\rvTwo}}{\vectorIndex}
			\bigr)
			~&\text{if}~
			1 \leq \ell \leq \sampleSize_1,
			\\
			\sqrt{\biggl( \dfrac{1}{\sampleSize_1}+\dfrac{1}{\sampleSize_2} \biggr)^{-1}}
			\dfrac{1}{\sampleSize_2}
			\bigl(
			\probVecElement{\tilde{\rvTwo}}{\vectorIndex}
			-
			\mathds{1}
			(\tilde{\rvThree}_{\ell-\sampleSize_1} = \vectorIndex)
			\bigr)
			~&\text{if}~
			(\sampleSize_1 + 1) \leq \ell \leq ( \sampleSize_1 + \sampleSize_2),
		\end{cases}
	\end{align*}
	for $\vectorIndex \in [\alphabetSize]$. 
	Then for each $\ell \in [N]$, we have $\mE[\tilde{\rVecX}_{N[:,\ell]}] = 0$  and
	$$\sum_{\ell=1}^N \tilde{\rVecX}_{N[:,\ell]} = 
	\sqrt{\biggl( \dfrac{1}{\sampleSize_1}+\dfrac{1}{\sampleSize_2} \biggr)^{-1}}(
	\hat{\boldsymbol{\mu}}_{\tilde{\rvY}}
	-
	\hat{\boldsymbol{\mu}}_{\tilde{\rvZ}}).
	$$
	We also have:
	\begin{align*}
		\mathrm{Cov}(\tilde{\rVecX}_{N[:,\ell]})
		=
		\begin{cases}
			\biggl( \dfrac{1}{\sampleSize_1}+\dfrac{1}{\sampleSize_2} \biggr)^{-1}
			\dfrac{1}{\sampleSize_1^2}
			\bigl(
			\text{diag}(\probVec_{\tilde{\rvTwo}}) - \probVec_{\tilde{\rvTwo}} \probVec_{\tilde{\rvTwo}}^\top
			\bigr)
			~&\text{if}~
			1 \leq \ell \leq \sampleSize_1,
			\\
			\biggl( \dfrac{1}{\sampleSize_1}+\dfrac{1}{\sampleSize_2} \biggr)^{-1}
			\dfrac{1}{\sampleSize_2^2}
			\bigl(
			\text{diag}(\probVec_{\tilde{\rvTwo}}) - \probVec_{\tilde{\rvTwo}} \probVec_{\tilde{\rvTwo}}^\top
			\bigr)
			~&\text{if}~
			(\sampleSize_1 + 1) \leq \ell \leq ( \sampleSize_1 + \sampleSize_2).
		\end{cases}
	\end{align*}
    Now we verify two conditions of  Lindberg's CLT in order.

    \vskip 1em 

    \noindent \textit{Condition 1: Convergence of the sum of covariances.}
	We aim to verify the following condition:
	\begin{equation}\label{lindberg_cond_1}
		\sum_{\ell=1}^N\mathrm{Cov}(\tilde{\rVecX}_{N[:,\ell]}) \to \Sigma~\text{as}~N \to \infty,
	\end{equation}
	for some fixed covariance matrix $\Sigma$. The condition~\eqref{lindberg_cond_1} is verified as follows. For each $N$,
	\begin{align*}
		\sum_{\ell=1}^N \mathrm{Cov}(\tilde{\rVecX}_{N[:,\ell]})
		&=
		\biggl( \dfrac{1}{\sampleSize_1}+\dfrac{1}{\sampleSize_2} \biggr)^{-1}
		\left(	\sum_{i=1}^{\sampleSize_1}
		\frac{1}{\sampleSize_1^2}
		\bigl(
		\text{diag}(\probVec_{\tilde{\rvTwo}}) - \probVec_{\tilde{\rvTwo}} \probVec_{\tilde{\rvTwo}}^\top
		\bigr)
		+
		\sum_{j=1}^{\sampleSize_2}
		\frac{1}{\sampleSize_2^2}
		\bigl(
		\text{diag}(\probVec_{\tilde{\rvTwo}}) - \probVec_{\tilde{\rvTwo}} \probVec_{\tilde{\rvTwo}}^\top
		\bigr)
		\right)
		\\&=
		\text{diag}(\probVec_{\tilde{\rvTwo}}) - \probVec_{\tilde{\rvTwo}} \probVec_{\tilde{\rvTwo}}^\top.
	\end{align*}
	Since the covariance computed above depends on neither $N$ nor $\ell$, the condition~\eqref{lindberg_cond_1} holds for $\Sigma = \text{diag}(\probVec_{\tilde{\rvTwo}}) - \probVec_{\tilde{\rvTwo}} \probVec_{\tilde{\rvTwo}}^\top.$
	Now we check the second condition of Lindberg's CLT:

\vskip 1em

\noindent \textit{Condition 2: Lindberg's condition.} We aim to verify the following condition:
	\begin{equation}\label{lindberg_cond_2}
		\sum_{\ell=1}^N
		\mE
		\bigl[
		\| \tilde{\rVecX}_{N[:,\ell]} \|_2^2
		\indicator{
			\|\tilde{\rVecX}_{N[:,\ell]}\|_2 > \epsilon
		}
		\bigr]
		\to 0~\text{as}~N \to \infty,~\text{for any}~\epsilon > 0.
	\end{equation}
	The condition~\eqref{lindberg_cond_2} is verified as follows. For any fixed $\ell$, we have:
	\begin{align*}
		\|\tilde{\rVecX}_{N[:,\ell]}\|_2
		=
		\begin{cases}
			%1
			\sqrt{\biggl( \dfrac{1}{\sampleSize_1}+\dfrac{1}{\sampleSize_2} \biggr)^{-1}}
			\dfrac{1}{\sampleSize_1}
			\bigl\{
			\sum_{\vectorIndex = 1}^\alphabetSize
			\bigl(
			\mathds{1}(\tilde{\rvTwo}_{\ell} = \vectorIndex)
			-  \probVecElement{\tilde{\rvTwo}}{\vectorIndex}
			\bigr)^2
			\bigr\}^{1/2}		
			~&\text{if}~
			\ell=1, \ldots, \sampleSize_1,
			%
			%
			%2
			\\
			\sqrt{\biggl( \dfrac{1}{\sampleSize_1}+\dfrac{1}{\sampleSize_2} \biggr)^{-1}}
			\dfrac{1}{\sampleSize_2}
			\bigl\{
			\sum_{\vectorIndex = 1}^\alphabetSize
			\bigl(
			\mathds{1}(\tilde{\rvThree}_{\ell-\sampleSize_1} = \vectorIndex)
			-
			\probVecElement{\tilde{\rvTwo}}{\vectorIndex}
			\bigr)^2
			\bigr\}^{1/2}
			~&\text{if}~
			(\sampleSize_1 + 1) \leq \ell \leq ( \sampleSize_1 + \sampleSize_2).		
		\end{cases}
		\numberthis \label{appendix:proof:genRR:norm}
	\end{align*}
	Note that
	$ | \mathds{1}(\tilde{\rvTwo}_{\ell} = \vectorIndex)  - \probVecElement{ \tilde{\rvY}}{\vectorIndex} |\leq 1$
	and
	$| \mathds{1}(\tilde{\rvThree}_{\ell-\sampleSize_1} = \vectorIndex)   - \probVecElement{ \tilde{\rvY}}{\vectorIndex}
	| \leq 1$ for $\vectorIndex \in [\alphabetSize]$.
	Therefore, for any $\ell \in [N]$, we have the following upper bound:
	\begin{align*} 
		\|\tilde{\rVecX}_{N[:,\ell]}\|_2
		\leq
		\sqrt{\biggl( \dfrac{1}{\sampleSize_1}+\dfrac{1}{\sampleSize_2} \biggr)^{-1}}
		\frac{\sqrt{\alphabetSize}}{\min(\sampleSize_1, \sampleSize_2)},
	\end{align*}
	where the right-hand side is not random and does not depend on $\ell$. Therefore, when verifying the condition~\eqref{lindberg_cond_2}, we can pull the indicator out of the expectation and the sum:
	\begin{align*}
		\sum_{\ell=1}^N
		\mE
		\bigl[
		\| \tilde{\rVecX}_{N[:,\ell]} \|_2^2
		\indicator{
			\|\tilde{\rVecX}_{N[:,\ell]}\|_2 > \epsilon
		}
		\bigr]
		&\leq
		\sum_{\ell=1}^N
		\mE
		\Biggl[
		\| \tilde{\rVecX}_{N[:,\ell]} \|_2^2
		\mathds{1}\Biggl(\sqrt{\biggl( \dfrac{1}{\sampleSize_1}+\dfrac{1}{\sampleSize_2} \biggr)^{-1}}
		\frac{\sqrt{\alphabetSize}}{\min(\sampleSize_1, \sampleSize_2)}
		> \epsilon
		\Biggr)
		\Biggr]
		\\&=
		\indicator{
			\sqrt{\biggl( \dfrac{1}{\sampleSize_1}+\dfrac{1}{\sampleSize_2} \biggr)^{-1}}
			\frac{\sqrt{\alphabetSize}}{\min(\sampleSize_1, \sampleSize_2)}
			> \epsilon
		}
		\sum_{\ell=1}^N
		\mE
		\bigl[
		\| \tilde{\rVecX}_{N[:,\ell]} \|_2^2
		\bigr].
	\end{align*}
	From~\eqref{appendix:proof:genRR:norm}, it can be shown that 	$\sum_{\ell=1}^N
	\mE
	\bigl[
	\| \tilde{\rVecX}_{N[:,\ell]} \|_2^2
	\bigr] = 2\mathrm{tr}\bigl(
	\text{diag}(\probVec_{\tilde{\rvTwo}}) - \probVec_{\tilde{\rvTwo}} \probVec_{\tilde{\rvTwo}}^\top
	\bigr),$ where the right-hand side is finite and do not depend on $N$.
	For any $\sampleSize_1$ and $\sampleSize_2$ large enough, we have
	\begin{align*}
		\sqrt{\biggl( \dfrac{1}{\sampleSize_1}+\dfrac{1}{\sampleSize_2} \biggr)^{-1}}
		\frac{\sqrt{\alphabetSize}}{\min(\sampleSize_1, \sampleSize_2)}
		=
		\sqrt{\frac{\sampleSize_1  \sampleSize_2}{\sampleSize_1 + \sampleSize_2}}
		\frac{\sqrt{\alphabetSize}}{\min(\sampleSize_1, \sampleSize_2)}
		< \epsilon.
	\end{align*}
	Therefore, the condition~\eqref{lindberg_cond_2} is satisfied when $\sampleSize_1, \sampleSize_2 \to \infty$.
	Therefore by Lindberg's CLT, we have
	\begin{align*}
		\sum_{\ell=1}^N \tilde{\rVecX}_{N[:,\ell]} = 
		\sqrt{\biggl( \dfrac{1}{\sampleSize_1}+\dfrac{1}{\sampleSize_2} \biggr)^{-1}}
		(\hat{\mathbf{\mu}}_{ \tilde{\rvY} }
		-
		\hat{\mathbf{\mu}}_{ \tilde{\rvZ} }
		)
		\stackrel{d}{\to}
		\mathcal{N}
		\bigl(0, \text{diag}(\probVec_{\tilde{\rvTwo}}
		\bigr) - \probVec_{\tilde{\rvTwo}} \probVec_{\tilde{\rvTwo}}^\top	).
	\end{align*}
	By the weak law of large numbers and Slutsky's theorem, it can be shown that
	\begin{equation}\label{clt_final}
		\sqrt{\biggl( \dfrac{1}{\sampleSize_1}+\dfrac{1}{\sampleSize_2} \biggr)^{-1}}
		\mathrm{diag}(\hat{\probVec})^{-1/2}
		(\hat{\mathbf{\mu}}_{ \tilde{\rvY} }
		-
		\hat{\mathbf{\mu}}_{ \tilde{\rvZ} }
		)
		\stackrel{d}{\to}
		\mathcal{N}
		\bigl(
		0,
		I_\alphabetSize - 
		\mathrm{diag}(\probVec_{\tilde{\rvTwo}})^{-1/2}
		\probVec_{\tilde{\rvTwo}} \probVec_{\tilde{\rvTwo}}^\top
		\mathrm{diag}(\probVec_{\tilde{\rvTwo}})^{-1/2}
		\bigr).
	\end{equation}
	Since the covariance matrix in~\eqref{clt_final} is an identity matrix minus a rank-one matrix formed by an orthonormal vector, its eigenvalues are 0 (with multiplicity 1) and 1 (with multiplicity $\alphabetSize-1$). Thus, the covariance matrix is a projection matrix of rank $k-1$. By a standard result \citep[for example, Lemma 17.1 of][]{van_der_vaart_asymptotic_1998}, the test statistic in~\eqref{def:chi_statistic}, equivalent to the squared $\ell_2$ norm of the left-hand side of~\eqref{clt_final}, converges to a chi-square distribution with $\alphabetSize - 1$ degrees of freedom.
	This concludes the proof of Lemma~\ref{asymp_genrr_chi}.
	\end{proof}
	\subsubsection{Proof of Lemma~\ref{asymp_rappor_projchi}}\label{proof:asymp_rappor_projchi}
	The proof is a straightforward application of Lindberg's CLT.
    
    \vskip 1em
    
	\begin{proof} 
	Throughout the proof, assume that the null hypothesis $\probVec_{\rvTwo} = \probVec_{\rvZ}$ holds. Thus, for each $\vectorIndex \in [\alphabetSize]$, we have $\mE(\tilde{\rvTwo}_{1\vectorIndex}) = \mE(\tilde{\rvThree}_{1\vectorIndex})$. Without loss of generality, we  denote all $\mE(\tilde{\rvThree}_{1\vectorIndex})$ as $\mE(\tilde{\rvTwo}_{1\vectorIndex})$.
	The proof leverages the content and structure of the argument used in the proof for $T_{\sampleSize_1, \sampleSize_2}$ in Appendex~\ref{proof:asymp_genrr_chi}.
	To apply Lindeberg's CLT, we construct a triangular array sequence. Each row of the array corresponds to the pooled sample size $N \coloneqq \sampleSize_1 + \sampleSize_2$. Within each row, the columns are indexed by $\ell \in [N]$. For each $N$ and $\ell$, we define a $\alphabetSize$-dimensional random vector $\tilde{\rVecX}_{N[:,\ell]}$,whose $\vectorIndex$th element, denoted as $\tilde{\rVecX}_{N[\vectorIndex,\ell]}$, is given by: 
	\begin{align*}
		\tilde{\rVecX}_{N[\vectorIndex,\ell]} :=
		\begin{cases}
			\sqrt{\biggl( \dfrac{1}{\sampleSize_1}+\dfrac{1}{\sampleSize_2} \biggr)^{-1}}
			\dfrac{1}{\sampleSize_1}
			\bigl(
			\tilde{\rvTwo}_{\ell \vectorIndex}
			-
			\mE(\tilde{\rvTwo}_{1\vectorIndex})
			\bigr)
			~&\text{if}~
			1 \leq \ell \leq \sampleSize_1,
			\\
			\sqrt{\biggl( \dfrac{1}{\sampleSize_1}+\dfrac{1}{\sampleSize_2} \biggr)^{-1}}
			\dfrac{1}{\sampleSize_2}
			\bigl(
			\mE(\tilde{\rvTwo}_{1\vectorIndex})
			-
			\tilde{\rvThree}_{(\ell-\sampleSize_1)\vectorIndex}
			\bigr)
			~&\text{if}~
			(\sampleSize_1 + 1) \leq \ell \leq ( \sampleSize_1 + \sampleSize_2),
		\end{cases}
	\end{align*}
	for $\vectorIndex \in [\alphabetSize]$,
	where we recall from~\eqref{equation:rappor_expecation} that
	\begin{equation*}
		\mE(\tilde{\rvTwo}_{1\vectorIndex})
		=
		\frac{
			(e^{\privacyParameter/2}-1)
			\probVecElement{\rvTwo}{\vectorIndex} + 1
		}{
			e^{\privacyParameter/2}+1
		}.
	\end{equation*}
    Now we verify two conditions of  Lindberg's CLT in order.

    \vskip 1em

    \noindent \textit{Condition 1: Convergence of the sum of covariances.}
	To verify  the  condition~\eqref{lindberg_cond_1}, first note that:
	\begin{align*}
		\sum_{\ell=1}^N \mathrm{Cov}(\tilde{\rVecX}_{N[:,\ell]})
		&=
		\biggl( \dfrac{1}{\sampleSize_1}+\dfrac{1}{\sampleSize_2} \biggr)^{-1}
		\left(	\sum_{i=1}^{\sampleSize_1}
		\frac{1}{\sampleSize_1^2}
		\mathrm{Cov}(\tilde{\rVecY_i})
		+
		\sum_{j=1}^{\sampleSize_2}
		\frac{1}{\sampleSize_2^2}
		\mathrm{Cov}(\tilde{\rVecZ_j})
		\right)
		\\&=
		\mathrm{Cov}(\tilde{\rVecY}_1)
		\\&=
		\left(\frac{e^{\privacyParameter/2}-1}{e^{\privacyParameter/2}+1}
		\right)^2
		\bigl(
		\mathrm{diag}(\probVec_{\rvTwo}) - \probVec_{\rvTwo} {\probVec^\top_{\rvY}}
		\bigr)
		+
		\frac{e^{\privacyParameter/2}}{(e^{\privacyParameter/2}+1)^2}
		I_d,
	\end{align*}
	where the last calculation is from Lemma~\ref{lemma:rappor_var}.
	Since the covariance computed above depends on neither \( N \) nor \( \ell \), the condition~\eqref{lindberg_cond_1} is satisfied with $\Sigma = \mathrm{Cov}(\tilde{\rVecY}_1)$.

    \vskip 1em 
    \noindent \textit{Condition 2: Lindberg's condition.}
	To verify  the  condition~\eqref{lindberg_cond_2} , first note that
	$0 \leq \tilde{\rvTwo}_{\ell \vectorIndex}
	-
	\mE(\tilde{\rvTwo}_{1\vectorIndex}) \leq 1$ and
	$0 \leq \mE(\tilde{\rvTwo}_{1\vectorIndex})
	-
	\tilde{\rvThree}_{(\ell-\sampleSize_1)\vectorIndex} \leq 1$.
	Therefore by the same analysis as the proof of $T_{\sampleSize_1, \sampleSize_2}$, the  condition~\eqref{lindberg_cond_2} is also satisfied.
	Therefore, by Linbderg's CLT, we have:
	\begin{align*}
		\sqrt{\biggl( \dfrac{1}{\sampleSize_1}+\dfrac{1}{\sampleSize_2} \biggr)^{-1}}
		(
		\check{\rVecY} 
		-
		\check{\rVecZ} 
		)
		\stackrel{d}{\to}
		\mathcal{N}(0, \mathrm{Cov}(\tilde{\rVecY}_1) ).
		\numberthis \label{appendix:proof:rappor_mean_diff}
	\end{align*}
	Note that $\mathrm{Cov}(\tilde{\rVecY}_1)$ has an eigenvector $\mathbf{1}$, which corresponds to an eigenvalue which is neither $0$ nor $1$.
	To turn the covariance matrix  into a projection matrix, 
	we delete the eigenvector $\mathbf{1}$ by pre-multiplying the random vector in the left-hand side of~\eqref{appendix:proof:rappor_mean_diff} by $\mathrm{Cov}(\tilde{\rVecY}_1)^{-1/2} \Pi$.
	Then we have:
	\begin{align*}
		\mathrm{Cov}(\tilde{\rVecY_1})^{-1/2} \Pi
		\sqrt{\biggl( \dfrac{1}{\sampleSize_1}+\dfrac{1}{\sampleSize_2} \biggr)^{-1}}
		(
		\check{\rVecY} 
		-
		\check{\rVecZ} 
		)
		\stackrel{d}{\to}
		\mathcal{N}
		\bigl(
		0,	\mathrm{Cov}(\tilde{\rVecY}_1)^{-1/2} \Pi \mathrm{Cov}(\tilde{\rVecY}_1) \Pi \mathrm{Cov}(\tilde{\rVecY}_1)^{-1/2}	\bigr).
	\end{align*}
	For a map $h : A \to A^{-1/2}$,  define $\mathrm{Disc}(h) := \{A \in \mathbb{R}^{p \times p}: h~\text{is not continuous at}~A\}$.
	By Lemma 5.6 of \citet{Gaboardi2018LDPChisq}, $\mathrm{Cov}(\tilde{\rVecY}_1)$ is invertible for any $\privacyParameter > 0$ and any $\mathbf{p}_Y \in \mathrm{int}(\Omega)$. Therefore  we have $ \mathrm{Cov}(\tilde{\rVecY}_1 \bigr) \notin \mathrm{Disc}(h)$ almost surely.
	Thus by the weak law of large numbers and the continuous mapping theorem, we have
	$	\hat{\Sigma}^{-1/2} \stackrel{p}{\to} \mathrm{Cov}(\tilde{\rVecY}_1)^{-1/2}$. 
	Then by Slutsky's theorem, 
	we have	\begin{align*}
		\hat{\Sigma}^{-1/2} \Pi
		\sqrt{\biggl( \dfrac{1}{\sampleSize_1}+\dfrac{1}{\sampleSize_2} \biggr)^{-1}}
		(
		\check{\rVecY} 
		-
		\check{\rVecZ} 
		)
		\stackrel{d}{
			\to}
		\mathcal{N}(0,
		\mathrm{Cov}(\tilde{\rVecY}_1)^{-1/2} \Pi \mathrm{Cov}(\tilde{\rVecY}_1) \Pi \mathrm{Cov}(\tilde{\rVecY}_1)^{-1/2}).
		\numberthis \label{appendix:proof:rappor_project_statistic}
	\end{align*}
	By the analysis of~\citet{Gaboardi2018LDPChisq},	
	$\mathrm{Cov}(\tilde{\rVecY}_1)^{-1/2} \Pi \mathrm{Cov}(\tilde{\rVecY_1}) \Pi \mathrm{Cov}(\tilde{\rVecY}_1)^{-1/2}$ is an identity matrix except one of the entries on the diagonal being zero, thus a projection matrix with rank $\alphabetSize - 1$.
	Therefore, by a standard result \citep[for example, Lemma 17.1][]{van_der_vaart_asymptotic_1998}, the test statistic~\eqref{def:statistic:projchi}, which is equivalent to squared $\ell_2$ norm of the random vector on the left-hand side of~\eqref{appendix:proof:rappor_project_statistic}, converges to a chi-square distribution with $\alphabetSize - 1$ degrees of freedom. This concludes the proof of Lemma \ref{asymp_rappor_projchi}.
	\end{proof}
	\section{Suboptimality of Generalized Randomized Response}\label{genrr_suboptimal_theory}
	This section illustrates that  the generalized randomized response mechanism in~\eqref{def:genrr} can lead to suboptimal power. Specifically, we show that the \texttt{GenRR}+$\ell_2$ test performs suboptimally in certain privacy regimes where the number of categories $\alphabetSize$ increases with the sample size $\sampleSize_1$, but the privacy level $\privacyParameter$ is fixed. This negative result is further supported by our numerical studies in Section~\ref{section:simulation}. 
	
	\begin{theorem}[Asymptotic powerlessness of \texttt{GenRR}+$\ell_2$ test]\label{genrr_suboptimal}
		Let $(\probVec_\rvTwo, \probVec_\rvThree)$ be a pair of  data-generating multinomial distributions with $\alphabetSize$ categories, where  $\probVec_\rvThree$ is a uniform distribution and $\probVec_\rvTwo$ is a perturbed uniform distribution parametrized by $\epsilon > 0$. Formally, for each $\vectorIndex \in [ \alphabetSize]$, the $m$th entry of $\probVec_\rvTwo$ and $\probVec_\rvThree$ are defined as:
		\begin{equation}\label{appendix:genrr_powerless:alternative}
			\probVecElement{\rvTwo}{\vectorIndex} := \frac{1}{\alphabetSize} + \frac{(-1)^\vectorIndex \epsilon}{\sqrt{\alphabetSize}}, \quad
			\probVecElement{\rvThree}{\vectorIndex} := \frac{1}{\alphabetSize},
		\end{equation}
		ensuring that $\|\probVec_\rvTwo - \probVec_\rvThree\|_2 = \epsilon$.
		Assume the privacy parameter $ \privacyParameter > 0 $ is sufficiently small so that the minimax testing rate, given in equation~\eqref{rate:twosample_disc}, satisfies
		\begin{equation}
			\rho^\ast_{\sampleSize_1, \sampleSize_2, \privacyParameter} \asymp \frac{\alphabetSize^{1/4}}{(\sampleSize_1 \privacyParameter^2)^{1/2}}.
		\end{equation}
		\noindent
		For a sufficiently large, but fixed constant $C >0$, set
		\begin{equation}\label{epsilon_C}
			\epsilon = \frac{C \cdot \alphabetSize^{1/4}}{(\sampleSize_1 \privacyParameter^2)^{1/2}},
		\end{equation}
		and consider the regime where
		\begin{equation}\label{appendix:genrr_powerless:asymptotic_regime}
			\sqrt{\alphabetSize} \epsilon = \frac{C \cdot \alphabetSize^{3/4}}{(\sampleSize_1 \privacyParameter^2)^{1/2}} \to 0,
		\end{equation}
		which guarantees that $0 < \probVecElement{\rvTwo}{\vectorIndex} < 1$  in~\eqref{appendix:genrr_powerless:alternative} for each $\vectorIndex \in [\alphabetSize]$ and for sufficiently large $\sampleSize_1$. Under this regime, the \textnormal{\texttt{GenRR}}+$\ell_2$ test is asymptotically powerless, meaning  its asymptotic power becomes at most the size. 
	\end{theorem}
	Given the setting in~\eqref{epsilon_C} from Theorem~\ref{genrr_suboptimal}, an optimal test would achieve non-trivial power greater than the significance level by selecting a sufficiently large $C$. However, Theorem~\ref{genrr_suboptimal} proves that the \textnormal{\texttt{GenRR}}+$\ell_2$ test becomes asymptotically powerless for any fixed $C > 0$, thereby underscoring its suboptimality.
    
    \vskip 1em
    
	\begin{proof} For the proof,
		we analyze the power function of the test statistic. To achieve this, we use high-dimensional asymptotics for U-statistics, specifically Corollary 3.3 and Theorem 3.3 from \citet{kim_multinomial_2020}, which are restated as Theorem~\ref{appendix:genrr_powerless:theorem_null} and Theorem~\ref{appendix:genrr_powerless:theorem:alternative_dist}, respectively.
		
		Let $w :=(e^\privacyParameter-1)/(e^\privacyParameter + \alphabetSize-1)$.
		Recall from~\eqref{genrr_transformed_p} that
		the $\privacyParameter$-LDP view  $\{\tilde{\rvTwo}_i\}_{i\in [\sampleSize_1]}$ via \texttt{GenRR} mechanism is equivalent to a random sample drawn from a mixture of $\probVec_{\rvTwo}$ and $\probVec_{\rvThree}$ defined as follows:
		\begin{align*}
			\probVec_{\tilde{\rvTwo}}
			&:=
			\frac
			{e^{\alpha}}
			{e^{\alpha} + \alphabetSize - 1 }
			\probVec_{\rvTwo}
			+
			\frac
			{1}
			{e^{\alpha} + \alphabetSize - 1 }
			(\mathbf{1}_\alphabetSize-\probVec_{\rvTwo})	
			%%%%%
			\\&=
			\frac
			{e^{\alpha}-1}
			{e^{\alpha} + \alphabetSize - 1 }
			\probVec_{\rvTwo}
			+
			\frac
			{\alphabetSize}
			{e^{\alpha} + \alphabetSize - 1 }
			\frac{1}{\alphabetSize}
			\mathbf{1}_\alphabetSize
			%%%%%%%%%
			\\&=
			w\probVec_{\rvTwo} + (1-w)\probVec_{\rvThree},
			\numberthis
			\label{appendix:genrr_powerless:w}
		\end{align*}
		Similarly, $\{\tilde{\rvThree}_i\}_{i \in [n_2]}$
		can be viewed as a random sample drawn from
		$\probVec_{\tilde{\rvThree}} := 	w\probVec_\rvThree +  (1-w) \probVec_{\rvThree} = \probVec_{\rvThree}$.
		Note that the $\ell_2$ distance between the probability vectors shrinks to $\|\probVec_{\tilde{\rvTwo}} - \probVec_{\tilde{\rvThree}} \|_2 = w \epsilon$.
		
		For convenience, let $\{\tilde{\vectorize{\rvTwo}}_i\}_{i\in [\sampleSize_1]}$ be the one-hot vectorized version of $\{\tilde{\rvTwo}_i\}_{i\in [\sampleSize_1]}$.
		We define a one-sample U-statistic for uniformity testing as follows:
		\begin{equation}\label{appendix:genrr_powerless:test_statistic}
			U_{\sampleSize_1}
			:=	
			\frac{2}{\sampleSize_1(\sampleSize_1-1)}
			\sum_{1 \leq \sampleIndexOne < \sampleIndexTwo \leq \sampleSize_1}
			h(
			\tilde{\vectorize{\rvTwo}}_\sampleIndexOne, \tilde{\vectorize{\rvTwo}}_\sampleIndexTwo
			),
		\end{equation}
		where
		$
		h(
		\tilde{\vectorize{\rvTwo}}_\sampleIndexOne, \tilde{\vectorize{\rvTwo}}_\sampleIndexTwo
		):=
		\bigl(
		\tilde{\vectorize{\rvTwo}}_\sampleIndexOne
		- 
		\probVec_{\rvThree}
		\bigr)^\top
		\bigl(
		\tilde{\vectorize{\rvTwo}}_\sampleIndexTwo
		- 
		\probVec_{\rvThree}
		\bigr)$,
		which is a special case of the two-sample statistic $U_{\sampleSize_1, \sampleSize_2}$ in~\eqref{def:statistic_elltwo} with $\sampleSize_2 = \infty$. Now we analyze the power function of the test based on $U_{\sampleSize_1}$ by specifying the critical value,  the distribution of the test statistic under the alternative, and the power function in order.

\vskip 1em

\noindent \textit{Critical value.}
		We  specify the critical value by establishing the asymptotic normality of the test statistic under the null by utilizing Theorem~\ref{appendix:genrr_powerless:theorem_null}.
		Since the statistic~\eqref{appendix:genrr_powerless:test_statistic} is equivalent to the statistic $U_I$ in~\eqref{def:U_stat_kim_2020} of Theorem~\ref{appendix:genrr_powerless:theorem_null}, it suffices to show that $\sampleSize_1 / \sqrt{\alphabetSize} \to \infty$ to establish the asymptotic normality.
		Note that the assumption	$\sqrt{\alphabetSize} \epsilon = C
		\alphabetSize^{3/4}
		/
		(\sampleSize_1 \privacyParameter^2)^{1/2} \to 0$~in \eqref{appendix:genrr_powerless:asymptotic_regime} implies $\sampleSize_1 / \sqrt{\alphabetSize} \to \infty$ for fixed $\privacyParameter$. 
		Therefore, under the null of $\probVec_\rvTwo =\probVec_\rvThree$, which implies  $
		\probVec_{\tilde{\rvTwo}} =
		w \probVec_\rvThree + (1-w)\probVec_\rvThree = \probVec_\rvThree$, we have
		\begin{equation*}
			\sqrt{\binom{\sampleSize_1}{2}} \frac{U_{\sampleSize_1}}{\sqrt{
					\alphabetSize^{-1}  (1 - \alphabetSize^{-1})}} \overset{d}{\longrightarrow} \mathcal{N}(0, 1).   
		\end{equation*}
		\noindent
		Based on this asymptotic normality, an asymptotically exact critical value is specified as follows:
		\begin{align*}
			\text{Reject}~H_0~\text{if}~	U_{\sampleSize_1} \geq
			z_{\gamma}
			\sqrt{
				\frac{
					2(1 -  1/ \alphabetSize)}{
					\alphabetSize n_1(n_1-1)}
			}.
			\numberthis \label{appendix:genrr_powerless:critical_value}
		\end{align*}
		Here $z_{\gamma}$ denotes the upper $\gamma$ quantile of $W \sim \mathcal{N}(0,1)$ satisfying $\mP(W > z_\gamma) = \gamma$. 

\vskip 1em

\noindent \textit{Alternative distribution.}
		Next we derive that the asymptotic distribution of the test statistic under the alternative specified in~\eqref{appendix:genrr_powerless:alternative}.
		By substituting $
		\|
		\probVec_{\tilde{\rvTwo}}
		-
		\probVec_{\rvThree}
		\|_2^2 = w^2 \epsilon^2
		$
		into~\eqref{one_sample_asymptotic_normal_theorem},
		we can establish the following asymptotic normality:
		\begin{equation}\label{genrr_ustat_asympototic_dist}
			\sqrt{
				\frac{
					\sampleSize_1(\sampleSize_1-1)
				}{
					2
				} 
			}
			\frac{
				U_{\sampleSize_1}
				-
				w^2 \epsilon^2}{
				\sqrt{
					\mathrm{tr}(\Sigma_{\tilde{\rvTwo}}^2)
					+
					2(\sampleSize_1 - 1)
					(\probVec_{\tilde{\rvTwo}} - \probVec_\rvThree)^\top 
					\Sigma_{\tilde{\rvTwo}}
					(\probVec_{\tilde{\rvTwo}} - \probVec_\rvThree)	
				}
			}
			\stackrel{d}{\to}
			\mathcal{N}(0,1),
		\end{equation}
		provided that the conditions C1, C2, and C3 of  Theorem~\ref{appendix:genrr_powerless:theorem:alternative_dist} hold.
		We verify these conditions in order.

\vskip 1em 

\noindent \textit{Verification of the condition C1.} The condition C1 is restated as follows:
			\begin{equation*}%\label{C1_restated}
				\frac{
					\mathrm{tr}(\Sigma_{\tilde{\rvTwo}}^4)
				}{
					\bigl\{
					\mathrm{tr}(\Sigma_{\tilde{\rvTwo}}^2)
					\bigr\}^2   
				}
				\to
				0.
			\end{equation*}
			We first analyze the denominator  $\bigl\{
			\mathrm{tr}(\Sigma_{\tilde{\rvTwo}}^2)
			\bigr\}^2 $. Since $\Sigma_{\tilde{\rvTwo}}  = \mathrm{diag}(\probVec_{\tilde{\rvTwo}}) - \probVec_{\tilde{\rvTwo}} \probVec_{\tilde{\rvTwo}}^\top$,
			we have 
			\begin{equation*}%\label{tr_sigma_2_calculation}
				\Sigma_{\tilde{\rvTwo}}^2
				= 
				\{ \mathrm{diag}(\probVec_{\tilde{\rvTwo}})
				\}^2
				+
				( \probVec_{\tilde{\rvTwo}} \probVec_{\tilde{\rvTwo}}^\top )^2
				-
				\mathrm{diag}(\probVec_{\tilde{\rvTwo}})\probVec_{\tilde{\rvTwo}} \probVec_{\tilde{\rvTwo}}^\top
				-
				\probVec_{\tilde{\rvTwo}} \probVec_{\tilde{\rvTwo}}^\top
				\mathrm{diag}(\probVec_{\tilde{\rvTwo}})
				. 
			\end{equation*}
			\sloppy Utilizing the equalities
			%equality 1
			$\bigl(\probVec_{\tilde{\rvTwo}} \probVec_{\tilde{\rvTwo}}^\top\bigr)^2 = \|\probVec_{\tilde{\rvTwo}}\|_2^2 \probVec_{\tilde{\rvTwo}} \probVec_{\tilde{\rvTwo}}^\top$,
			%
			%equality 2
			$
			\mathrm{tr}
			\bigl(
			\probVec_{\tilde{\rvTwo}} \probVec_{\tilde{\rvTwo}}^\top
			\bigr) = \|
			\probVec_{\tilde{\rvTwo}}\|_2^2$,
			and
			%equality 3
			$\mathrm{tr}\bigl(\text{diag}(\probVec_{\tilde{\rvTwo}})\probVec_{\tilde{\rvTwo}} \probVec_{\tilde{\rvTwo}}^\top\bigr)
			=
			\probVec_{\tilde{\rvTwo}}^\top \text{diag}(\probVec_{\tilde{\rvTwo}}) \probVec_{\tilde{\rvTwo}}
			$,
			we can calculate $\mathrm{tr}(\Sigma_{\tilde{\rvTwo}}^2)$ as
			\begin{equation*}
				\mathrm{tr}(\Sigma_{\tilde{\rvTwo}}^2)
				=
				\|\probVec_{\tilde{\rvTwo}}\|_2^2
				+
				\|
				\probVec_{\tilde{\rvTwo}}
				\|_2^4
				-
				2
				\|\probVec_{\tilde{\rvTwo}}\|_3^3.
			\end{equation*}
			In a similar manner, we can expand the numerator term as follows:
			\begin{equation*}
				\mathrm{tr}(\Sigma_{\tilde{Y}}^4)
				=
				\|\mathbf{p}_{\tilde{Y}}\|_4^4 + \|\mathbf{p}_{\tilde{Y}}\|_2^8 + 2\|\mathbf{p}_{\tilde{Y}}\|_3^6
				+
				4\|\mathbf{p}_{\tilde{Y}}\|_2^2 \|\mathbf{p}_{\tilde{Y}}\|_4^4
				- 4\|\mathbf{p}_{\tilde{Y}}\|_2^4\|\mathbf{p}_{\tilde{Y}}\|_3^3 - 4\|\mathbf{p}_{\tilde{Y}}\|_5^5.
			\end{equation*}
			Therefore, it suffices to study the asymptotic behavior of
			\begin{align*}
					\mathrm{tr}&(\Sigma_{\tilde{\rvTwo}}^4)
				/
					\bigl\{
					\mathrm{tr}(\Sigma_{\tilde{\rvTwo}}^2)
					\bigr\}^2   
				\\&=
				\frac{
					%numerator
					%%%%%%%%%% common terms start
					\bigl(
					\|\mathbf{p}_{\tilde{{Y}}}\|_2^8
					- 4\|\mathbf{p}_{\tilde{{Y}}}\|_2^4\|\mathbf{p}_{\tilde{{Y}}}\|_3^3
					+ 2\|\mathbf{p}_{\tilde{{Y}}}\|_3^6
					\bigr)
					%%%%%%%%%% common terms end
					+
					\|\mathbf{p}_{\tilde{{Y}}}\|_4^4
					+
					4\|\mathbf{p}_{\tilde{{Y}}}\|_2^2 \|\mathbf{p}_{\tilde{{Y}}}\|_4^4
					- 4\|\mathbf{p}_{\tilde{{Y}}}\|_5^5
				}{
					% denominator
					%%%%%%%%%% common terms start
					\bigl(
					\|
					\probVec_{\tilde{\rvTwo}}
					\|_2^8
					-
					4
					\|
					\probVec_{\tilde{\rvTwo}}
					\|_2^4
					\|\probVec_{\tilde{\rvTwo}}\|_3^3
					+
					4
					\|\probVec_{\tilde{\rvTwo}}\|_3^6
					\bigr)
					%%%%%%%%%%%%%%%%%%
					+
					\|\probVec_{\tilde{\rvTwo}}\|_2^4
					+
					2
					\|\probVec_{\tilde{\rvTwo}}\|_2^2
					\|
					\probVec_{\tilde{\rvTwo}}
					\|_2^4
					-
					4
					\|\probVec_{\tilde{\rvTwo}}\|_2^2
					\|\probVec_{\tilde{\rvTwo}}\|_3^3
				}
				\numberthis \label{trace_ratio_norm_expression}
				,
			\end{align*}
			where the parentheses group together terms of the same order.
			\noindent
			Since 
			$\probVecElement{\rvTwo}{\vectorIndex}
			=
			1/\alphabetSize
			+
			(-1)^\vectorIndex
			\epsilon
			/
			\sqrt{\alphabetSize}$, we have
			\begin{equation}\label{appendix:genrr_powerless:probvec_o1}
				\probVecElement{\tilde{\rvTwo}}{\vectorIndex}
				=
				\frac{w}{\alphabetSize}
				+
				\frac{
					w
					(-1)^\vectorIndex
					\sqrt{k}\epsilon }{\alphabetSize}
				+
				\frac{1}{\alphabetSize}
				-
				\frac{w}{\alphabetSize}
				=
				\frac{1}{\alphabetSize}
				+
				\frac{(-1)^m w \epsilon}{
					\sqrt{\alphabetSize}
				},
			\end{equation}
			for $\vectorIndex \in [\alphabetSize]$.
			Utilizing~\eqref{appendix:genrr_powerless:probvec_o1} and the fact that $\alphabetSize$ is an even integer, the powers of norms of $\probVec_{\tilde{\rvTwo}}$ are calculated
			as follows:
			\begin{equation} \label{power_of_norms_calculation}
				\begin{aligned}
					& \|\probVec_{\tilde{\rvTwo}}\|_2^2
					=
					\frac{1}{\alphabetSize} + w^2 \epsilon^2
					, \ 
					%%%%%%%%%%%%%%%% 3-3 
					\|\probVec_{\tilde{\rvTwo}}\|_3^3
					=
					\frac{1}{\alphabetSize^2}
					+
					\frac{3w^2 \epsilon^2}{\alphabetSize}
					, \\
					%%%%%%%%%%%%%%%% 4-4
					& \|\probVec_{\tilde{\rvTwo}}\|_4^4 = \frac{1}{\alphabetSize^3}
					+
					6
					\frac{w^2 \epsilon^2}{\alphabetSize^2}
					+
					\frac{w^4 \epsilon^4}{\alphabetSize}, \
					%%%%%%%%%%%%%%%% 5-5
					\|\probVec_{\tilde{\rvTwo}}\|_5^5 = \frac{1}{\alphabetSize^4}
					+
					10
					\frac{w^2 \epsilon^2}{\alphabetSize^3}
					+
					5
					\frac{w^4 \epsilon^4}{\alphabetSize^2}.
				\end{aligned}
			\end{equation}
			By substituting~\eqref{power_of_norms_calculation} into~\eqref{trace_ratio_norm_expression}, 
			and recalling that $w\epsilon$ is of order $\alphabetSize^{-3/4}\sampleSize_1^{-1/2}$,
			we find that among the terms  in~\eqref{trace_ratio_norm_expression}, the term that converges to 0 the most slowly is $-4/k^3$,
			which comes from $-
			4
			\|\probVec_{\tilde{\rvTwo}}\|_2^2
			\|\probVec_{\tilde{\rvTwo}}\|_3^3$
			in the denominator. Consequently, the quantity in~\eqref{power_of_norms_calculation} converges to 0, confirming condition C1 of Theorem~\ref{appendix:genrr_powerless:theorem:alternative_dist}.

\vskip 1em

			\noindent \textit{Verification of the condition C2.}
			The condition C2 is recalled below as
			\begin{equation}\label{condition_C2_restated}
				\dfrac{
					\mathbb{E}[W_1]
					+
					\sampleSize_1
					\mathbb{E}
					[W_2]
				}{
					\sampleSize_1^2
					\big\{
					\mathrm{tr}(\Sigma_{\tilde{\rvTwo}}^2)
					\big\}^2
				} \to 0,
			\end{equation}
			where
			\begin{align*}
				W_1&:=
				\bigl\{ (\tilde{\vectorize{\rvTwo}}_1 - \probVec_{\rvThree})^\top (\tilde{\vectorize{\rvTwo}}_2
				-
				\probVec_{\rvThree})
				\bigr\}^4,
				\\
				W_2&:=
				\bigl\{ (\tilde{\vectorize{\rvTwo}}_1 - \probVec_{\rvThree})^\top (\tilde{\vectorize{\rvTwo}}_2
				-
				\probVec_{\rvThree})
				\bigr\}^2
				\bigl\{ (\tilde{\vectorize{\rvTwo}}_1 - \probVec_{\rvThree})^\top (\tilde{\vectorize{\rvTwo}}_3
				-
				\probVec_{\rvThree})
				\bigr\}^2.
			\end{align*}
			We separately analyze the terms in~\eqref{condition_C2_restated}.
			First  for $
			\mathbb{E}[W_1]
			$,
			the random variable $W_1$ takes two distinct values:
			\begin{align*}
				W_1
				=
				\begin{cases}
					\biggl(1-\dfrac{1}{\alphabetSize}\biggr)^4, & 
					\tilde{\vectorize{\rvTwo}}_1 = \tilde{\vectorize{\rvTwo}}_2,
					\\[1em]
					~\dfrac{1}{\alphabetSize^4},
					& \text{otherwise.}
				\end{cases}	
			\end{align*}
			Observe that, according to~\eqref{appendix:genrr_powerless:probvec_o1}, we have
			\begin{equation}\label{prob_m_m'}
				\mathbb{P}(\tilde{\vectorize{\rvTwo}}_1 = \tilde{\vectorize{\rvTwo}}_2)
				=
				\sum_{\vectorIndex=1}^\alphabetSize
				\probVecElement{\tilde{\rvTwo}}{\vectorIndex}^2
				= 
				\sum_{\vectorIndex=1}^\alphabetSize
				\biggl(
				\frac{1}{\alphabetSize}
				+
				\frac{
					(-1)^m w \epsilon
				}{
					\sqrt{\alphabetSize}}
				\biggr)^2
				=
				\frac{1}{\alphabetSize}
				+w^2\epsilon^2.
			\end{equation}
			Using the established probability~\eqref{prob_m_m'}, we calculate the target expected value as
			\begin{align*}
				\mathbb{E}[W_1]
				= 
				-\frac{1}{k^5}
				+
				\frac{k}{k^5}
				-
				\frac{k \epsilon^2 w^2 }{k^5}
				+
				\frac{(-1 + k)^4}{k^5}
				+
				\frac{(-1 + k)^4 (1 + k\epsilon^2  w^2)}{k^5}.
			\end{align*}
			Among the terms in the expectation computed above,
			the term that converges to 0 most slowly is of order
			$1/\alphabetSize$, which arises from the fourth term of the numerator. 
			Similarly, for
			$
			\mathbb{E}[W_2],
			$
			the random variable $W_2$ takes on the values as:
			\begin{align*}
				W_2
				=
				\begin{cases}
					\biggl(1-\dfrac{1}{\alphabetSize}\biggr)^4, & 
					\tilde{\vectorize{\rvTwo}}_1
					=
					\tilde{\vectorize{\rvTwo}}_2
					=
					\tilde{\vectorize{\rvTwo}}_3,
					\\[1em]
					~\dfrac{1}{\alphabetSize^2}
					\biggl(1-\dfrac{1}{\alphabetSize}\biggr)^2, & 
					[\tilde{\vectorize{\rvTwo}}_1 = \tilde{\vectorize{\rvTwo}}_2,
					\tilde{\vectorize{\rvTwo}}_1 \neq
					\tilde{\vectorize{\rvTwo}}_3]
					\cup
					[\tilde{\vectorize{\rvTwo}}_1 = \tilde{\vectorize{\rvTwo}}_3,
					\tilde{\vectorize{\rvTwo}}_1 \neq
					\tilde{\vectorize{\rvTwo}}_2]
					\numberthis
					\label{kernel_three_genrr}
					\\[1em]
					\dfrac{1}{\alphabetSize^4}, & \text{otherwise.}
				\end{cases}	 
			\end{align*}
			The corresponding probabilities are calculated as
			\begin{align*}
				& \mP(\tilde{\vectorize{\rvTwo}}_1
				=
				\tilde{\vectorize{\rvTwo}}_2
				=
				\tilde{\vectorize{\rvTwo}}_3)
				=
				\sum_{\vectorIndex = 1}^\alphabetSize
				\probVecElement{\tilde{\rvTwo}}{\vectorIndex}^3
				=
				\|\probVec_{\tilde{\rvTwo}}\|_3^3
				\stackrel{(a)}{=}
				\frac{1}{\alphabetSize^2}
				+
				\frac{3 w^2 \epsilon^2}{\alphabetSize},~\text{and}
				%%%%%%%%%%%%%%%%%%%
				% 2nd
				\\
				& \mP([\tilde{\vectorize{\rvTwo}}_1 = \tilde{\vectorize{\rvTwo}}_2,
				\tilde{\vectorize{\rvTwo}}_1 \neq
				\tilde{\vectorize{\rvTwo}}_3]
				\cup
				[\tilde{\vectorize{\rvTwo}}_1 = \tilde{\vectorize{\rvTwo}}_3,
				\tilde{\vectorize{\rvTwo}}_1 \neq
				\tilde{\vectorize{\rvTwo}}_2])
				\stackrel{(b)}{=}
				2
				\mP([\tilde{\vectorize{\rvTwo}}_1 = \tilde{\vectorize{\rvTwo}}_2,
				\tilde{\vectorize{\rvTwo}}_1 \neq
				\tilde{\vectorize{\rvTwo}}_3])
				%%%%%%%%%%%%%%%%%%%%%%%%%%%%%%%%
				\\& \hskip 20.6em =
				2\sum_{\vectorIndex=1}^\alphabetSize
				\probVecElement{\tilde{\rvTwo}}{\vectorIndex}^2
				(1-\probVecElement{\tilde{\rvTwo}}{\vectorIndex})
				%%%%%%%%%%%%%%%%%
				\\& \hskip 20.6em =
				2\|\probVec_{\tilde{\rvTwo}}\|_2^2-
				2\|\probVec_{\tilde{\rvTwo}}\|_3^3
				\\& \hskip 20.35em \stackrel{(c)}{=}
				\frac{2}{\alphabetSize}
				+
				2w^2 \epsilon^2
				-
				\frac{2}{\alphabetSize^2}
				-
				\frac{6 w^2 \epsilon^2}{\alphabetSize}
				,
			\end{align*}
			where steps $(a)$ and $(c)$ use the equations in~\eqref{power_of_norms_calculation}
			and
			step $(b)$ uses the independence between samples, as well as the exclusiveness and symmetry of the events.
			Now $\mE[W_2]$ is calculated as
			\begin{align*}
				\mE[W_2]&=
				\biggl(
				1-\frac{1}{\alphabetSize}
				\biggr)^4
				\biggl(
				\frac{1}{\alphabetSize^2}
				+
				\frac{3w^2 \epsilon^2}{\alphabetSize}
				\biggr)
				+
				\biggl\{
				\frac{1}{\alphabetSize^2}
				\biggl(
				1- \frac{1}{\alphabetSize}
				\biggr)^2
				\biggr\}
				\biggl(
				\frac{2}{\alphabetSize}
				+
				2w^2 \epsilon^2
				-
				\frac{2}{\alphabetSize^2}
				-
				\frac{6w^2 \epsilon^2}{\alphabetSize}
				\biggr)
				\\&\quad +
				\frac{1}{\alphabetSize^4}
				\biggl(
				1
				-
				\frac{1}{\alphabetSize^2}
				-
				\frac{3w^2 \epsilon^2}{\alphabetSize}
				-
				\frac{2}{\alphabetSize}
				-
				2w^2 \epsilon^2
				+
				\frac{2}{\alphabetSize^2}
				+
				\frac{6w^2 \epsilon^2}{\alphabetSize}
				\biggr),
				\numberthis
				\label{E_W2}
			\end{align*}
			where the term that converges to $0$ the slowest is of order $1/\alphabetSize^2$, originating from the first term on the right-hand side of~\eqref{E_W2}.
			Finally, the trace term in the denominator is calculated as
			\begin{align*}
				\mathrm{tr}(\Sigma_{\tilde{\rvTwo}}^2)
				&=
				\|\probVec_{\tilde{\rvTwo}}\|_2^2
				+
				\|
				\probVec_{\tilde{\rvTwo}}
				\|_2^4
				-
				2
				\|\probVec_{\tilde{\rvTwo}}\|_3^3
				%%%%%%%%%%%%%
				\\&=
				\frac{1}{\alphabetSize} + w^2 \epsilon^2+
				\frac{1}{\alphabetSize^2} + w^4 \epsilon^4
				+
				\frac{2w^2 \epsilon^2}{\alphabetSize}
				-
				\frac{2}{\alphabetSize^2}
				-
				\frac{6w^2 \epsilon^2}{\alphabetSize}
				%%%%%%%%%%%%%
				\\&=
				\frac{1}{\alphabetSize}
				-
				\frac{1}{\alphabetSize^2}
				+
				w^2 \epsilon^2
				+
				w^4 \epsilon^4
				-
				\frac{4w^2 \epsilon^2}{\alphabetSize},
				\numberthis
				\label{trace_calculation_plugged_in}
			\end{align*}
			where the term that converges to 0 the slowest is $1/k$.
			Now to verify the condition C2,
			we separately analyze the asymptotic behaviors of $\mathbb{E}[W_1]/\sampleSize_1^2
			\big\{
			\mathrm{tr}(\Sigma_{\tilde{\rvTwo}}^2)
			\big\}^2$
			and
			$\sampleSize_1
			\mathbb{E}[W_2]/\sampleSize_1^2
			\big\{
			\mathrm{tr}(\Sigma_{\tilde{\rvTwo}}^2)
			\big\}^2$,
			focusing on the terms that converges the slowest to $0$ in the numerator and denominator, respectively.
			For the first term, it suffices to investigate
			$(1/\alphabetSize)/(\sampleSize_1^2/\alphabetSize^2)$, which converges to 0
			as long as $\alphabetSize/\sampleSize_1^2 \to 0$.
			Note that $\alphabetSize/\sampleSize_1^2 \to 0$ is implied by our assumption~\eqref{appendix:genrr_powerless:asymptotic_regime}.
			For the second term, it suffices to investigate $(\sampleSize_1/\alphabetSize^2)/(\sampleSize_1^2/\alphabetSize^2)$, which converges to 0.
			Therefore the convergence in~\eqref{condition_C2_restated} is verified, 
			confirming the condition C2 of Theorem~\ref{appendix:genrr_powerless:theorem:alternative_dist}. 

\vskip 1em

			\noindent \textit{Verification of the condition C3.}
			The condition C3 is restated as
			\begin{equation}\label{cond_c3_negative}
				(\probVec_{\tilde{\rvTwo}} - \probVec_\rvThree)^\top 
				\Sigma_{
					\tilde{
						{\rvTwo}
					}
				}
				(\probVec_{\tilde{\rvTwo}} - \probVec_\rvThree)
				=
				o\left(
				\sampleSize_1^{-1} \mathrm{tr}
				(\Sigma_{
					\tilde{{\rvTwo}}
				}^2)
				\right).
			\end{equation}
			To calculate the quadratic form therein,  note that since
			$
			\probVec_{\tilde{\rvTwo}}
			-
			\probVec_\rvThree
			=
			w(\probVec_\rvTwo- \probVec_\rvThree)
			$,
			we have
			\begin{equation*}%\label{C3_first_breakdown}
				(\probVec_{\tilde{\rvTwo}} - \probVec_\rvThree)^\top 
				\Sigma_{
					\tilde{{\rvTwo}}
				}
				(\probVec_{\tilde{\rvTwo}} - \probVec_\rvThree)
				%%%%%%%%%%%%%%%%%%%%%%%%%%
				= %2. plug in
				%%%%%%%%%%%%%%%%%%%%%%%%%%
				w^2
				(\probVec_\rvTwo - \probVec_\rvThree)^\top 
				\mathrm{diag}(\probVec_{\tilde{\rvTwo}})
				(\probVec_\rvTwo - \probVec_\rvThree)
				-
				w^2
				(\probVec_\rvTwo - \probVec_\rvThree)^\top 
				\probVec_{\tilde{\rvTwo}}
				\probVec_{\tilde{\rvTwo}}^\top
				(\probVec_\rvTwo - \probVec_\rvThree).  
			\end{equation*}
			We separately analyze the two terms on the right-hand side. For the first term, we have
			\begin{align*}
				w^2
				(\probVec_\rvTwo - \probVec_\rvThree)^\top 
				\mathrm{diag}(\probVec_{\tilde{\rvTwo}})
				(\probVec_\rvTwo - \probVec_\rvThree)
				&=
				w^2
				\sum_{\vectorIndex=1}^\alphabetSize
				\frac{
					(-1)^\vectorIndex \epsilon
				}{
					\sqrt{\alphabetSize}
				}
				\biggl(
				\frac{1}{\alphabetSize}
				+
				\frac{
					(-1)^\vectorIndex 
					w\epsilon
				}{
					\sqrt{\alphabetSize}
				}
				\biggr)
				\frac{
					(-1)^\vectorIndex \epsilon
				}{
					\sqrt{\alphabetSize}
				}
				%%%%%%%%%%%%%%%%%%%%%
				\\&=%line 2 : merge
				\frac{w^2\epsilon^2}{
					\alphabetSize
				}
				\sum_{\vectorIndex=1}^\alphabetSize
				\biggl(
				\frac{1}{\alphabetSize}
				+
				\frac{
					(-1)^\vectorIndex w\epsilon
				}{
					\sqrt{\alphabetSize}
				}
				\biggr)
				%%%%%%%%%%%%%%%%%%%%%
				\\&=%line 3 : \pm cancels each other
				\frac{w^2\epsilon^2}{
					\alphabetSize
				},
			\end{align*}
			where the last equality  holds since $\alphabetSize$ is even. The second term is computed as
			\begin{align*}
				w^2
				(\probVec_\rvTwo - \probVec_\rvThree)^\top 
				\probVec_{\tilde{\rvTwo}}
				\probVec_{\tilde{\rvTwo}}^\top
				(\probVec_\rvTwo - \probVec_\rvThree)
				&=
				w^2
				\bigl\{
				(\probVec_\rvTwo - \probVec_\rvThree)^\top 
				\probVec_{\tilde{\rvTwo}}
				\bigr\}^2
				%%%%%%%%%%%%
				\\&=
				w^2
				\biggl\{% square bracket starts
				\sum_{\vectorIndex=1}^\alphabetSize
				%py-pz term
				\frac{
					(-1)^\vectorIndex \epsilon
				}{
					\sqrt{\alphabetSize}
				}
				%p_y_tilde term
				\biggl(
				\frac{1}{\alphabetSize}
				+
				\frac{
					(-1)^\vectorIndex w\epsilon
				}{
					\sqrt{\alphabetSize}
				}
				\biggr)
				\biggr\}^2% square bracket ends
				\\&=w^4 \epsilon^4,
			\end{align*}
			where  the last equality again uses the fact that $\alphabetSize$ is an even integer.
			Therefore we have
			\begin{equation}\label{noise_calculation_done}
				(\probVec_{\tilde{\rvTwo}} - \probVec_\rvThree)^\top 
				\Sigma_{
					\tilde{
						{\rvTwo}
					}
				}
				(\probVec_{\tilde{\rvTwo}} - \probVec_\rvThree)
				=
				\frac{w^2 \epsilon^2}{\alphabetSize}
				-
				w^4 \epsilon^4,
			\end{equation}
			where $w^2 \epsilon^2 / \alphabetSize$ is the term that converges to 0 most slowly. 
			Utilizing the trace calculation~\eqref{trace_calculation_plugged_in}, we can verify the condition C3~\eqref{cond_c3_negative}
			by studying the asymptotic behavior of $(w^2 \epsilon^2 / \alphabetSize)/(1/\alphabetSize \sampleSize_1)$. Since this quantity converges to 0, the condition C3  is verified.
		
	As all conditions of Theorem~\ref{appendix:genrr_powerless:theorem:alternative_dist} are satisfied, the asymptotic distribution in~\eqref{genrr_ustat_asympototic_dist} is confirmed.

\vskip 1em

		\noindent \textit{Power function.}
		Using the critical value in~\eqref{appendix:genrr_powerless:critical_value}
		and the asymptotic alternative distribution established in~\eqref{genrr_ustat_asympototic_dist}, 
		under the asymptotic regime in~\eqref{appendix:genrr_powerless:asymptotic_regime},
		the power function is written as:
		\begin{align*}
			\gamma&_{\sampleSize_1, \alphabetSize, \privacyParameter}
			(\probVec_\rvTwo,
			\probVec_\rvThree)
			\\&=
			\mathbb{P}_{H_1} \!
			\left(
			U_{n_1}
			\geq
			z_{\gamma}
			\sqrt{
				\frac{
					2(1 -  1/ \alphabetSize)}{
					\alphabetSize n_1(n_1-1)}
			}
			\right)
			%%%%%%%%%%%%%%%%%%%%%%%%
			\\&=
			\mathbb{P}_{H_1} \!
			\left(
			W
			\geq
			\sqrt{
				\frac{
					\sampleSize_1(\sampleSize_1-1)
				}{
					2
				} 
			}
			\frac{z_{\gamma}
				\sqrt{
					2\{1 - ( 1/ \alphabetSize) \}/ \{ \alphabetSize n_1(n_1-1)  \}
				}
				- w^2 \epsilon^2
			}{
				\sqrt{
					\mathrm{tr}(\Sigma_{\tilde{\rvTwo}}^2)
					+
					2(\sampleSize_1 - 1)
					(\probVec_{\tilde{\rvTwo}} - \probVec_\rvThree)^\top 
					\Sigma_{\tilde{\rvTwo}}
					(\probVec_{\tilde{\rvTwo}} - \probVec_\rvThree)	
				}
			}
			\right) + o(1)
			%%%%%%%%%%%%%%%%%%%%%%%%
			\\&=
			\Phi \!
			\left(
			\sqrt{
				\frac{
					\sampleSize_1(\sampleSize_1-1)
				}{
					2
				} 
			}
			\frac{
				w^2 \epsilon^2
				-
				z_{\gamma}
				\sqrt{
					2\{1 - ( 1/ \alphabetSize) \}/ \{ \alphabetSize n_1(n_1-1)  \}
				}
			}{
				\sqrt{
					\mathrm{tr}(\Sigma_{\tilde{\rvTwo}}^2)
					+
					2(\sampleSize_1 - 1)
					(\probVec_{\tilde{\rvTwo}} - \probVec_\rvThree)^\top 
					\Sigma_{\tilde{\rvTwo}}
					(\probVec_{\tilde{\rvTwo}} - \probVec_\rvThree)	
				}
			}
			\right)
			%%%%%%%%%%%%%%%%%%%%%%%%
			\\&=
			\Phi \!
			\left(
			\frac{
				\sqrt{
					\sampleSize_1(\sampleSize_1-1)
					/
					2
				}
				w^2 \epsilon^2
				-
				z_{\gamma}
				\sqrt{
					(1/\alphabetSize)
					-
					(1/\alphabetSize^2)
				}
			}{
				\sqrt{
					\left\{
					(1/\alphabetSize)
					-
					(1/\alphabetSize^2)
					+
					w^2 \epsilon^2
					+
					w^4 \epsilon^4
					-
					(4w^2 \epsilon^2/\alphabetSize)
					\right\}
					+
					2(\sampleSize_1 - 1)
					\bigl\{
					(w^2 \epsilon^2/\alphabetSize)
					-
					w^4 \epsilon^4
					\bigr\}	
				}
			}
			\right)
			%%%%%%%%%%%%%%%
			\\ &\to
			\Phi \!
			\left( -z_\gamma \right) = \gamma,
			\numberthis
			\label{appendix:genrr_powerless:power_function}
		\end{align*}
		where $W \sim \mathcal{N}(0,1)$,   $\Phi$ denotes the cumulative distribution function of $W$, and $o(1)$ refers to a quantity that goes to 0 as $\sampleSize_1$ and $\alphabetSize$ go to $\infty$, in the regime of~\eqref{appendix:genrr_powerless:asymptotic_regime}.
		The last line in~\eqref{appendix:genrr_powerless:power_function} uses the fact that given that $w\epsilon$ is of the order $\alphabetSize^{-3/4}\sampleSize_1^{-1/2}$, the term in~\eqref{appendix:genrr_powerless:power_function} that converges most slowly to 0 is $\sqrt{1/\alphabetSize}$, which appears both on the second term in the numerator and the first term in the denominator.  It is important to note that the convergence result \eqref{appendix:genrr_powerless:asymptotic_regime} holds for any sufficiently small $\privacyParameter$, indicating that the \texttt{GenRR}+$\ell_2$ test becomes asymptotically powerless in the regime where the minimax testing rate is $\alphabetSize^{1/4} / (\sampleSize_1 \privacyParameter^2)^{1/2}$.
This completes the proof of Theorem \ref{genrr_suboptimal}.
	\end{proof}
\end{appendix}

\end{document}